\PassOptionsToPackage{dvipsnames,table}{xcolor}
\PassOptionsToPackage{ruled}{algorithm2e}
\PassOptionsToClass{cleveref}{jmlr}
\PassOptionsToPackage{noabbrev,nameinlink,capitalize}{cleveref}

\documentclass{arxiv}

\usepackage[utf8]{inputenc}
\usepackage[T1]{fontenc}
\usepackage{times}

\usepackage[T1]{fontenc}
\usepackage[utf8]{inputenc}

\usepackage{dsfont}
\usepackage{bm}

\usepackage{microtype}
\usepackage[autostyle]{csquotes}

\usepackage{enumitem}
\usepackage{xspace}

\usepackage[normalem]{ulem}

\renewcommand{\labelenumi}{\color{black!67}{\arabic{enumi}.}}

\usepackage{wrapfig}
\usepackage{adjustbox}

\usepackage{booktabs}
\usepackage{array}
\usepackage{multirow}

\usepackage[noabbrev,nameinlink,capitalize]{cleveref}
\creflabelformat{equation}{#2\textup{#1}#3}

\usepackage{tikz}
\usetikzlibrary{
  backgrounds,
  calc,
  quotes,
  arrows.meta,
  shapes.misc,
  positioning,
  graphs
}

\makeatletter
\usetikzlibrary{shapes,fit,chains,arrows}

\tikzstyle{latent} = [circle,fill=white,draw=black,inner sep=1pt,minimum size=20pt,font=\fontsize{10}{10}\selectfont,node distance=1]
\tikzstyle{obs}    = [latent,fill=gray!25]              %
\tikzstyle{const}  = [rectangle,inner sep=0pt,node distance=1]  %
\tikzstyle{factor} = [rectangle,fill=black,minimum size=5pt,inner sep=0pt,node distance=0.4]  %
\tikzstyle{det}    = [latent,diamond]                   %

\tikzstyle{plate}  = [draw,rectangle,rounded corners,fit=#1]  %
\tikzstyle{wrap}   = [inner sep=0pt,fit=#1]             %
\tikzstyle{gate}   = [draw,rectangle,dashed,fit=#1]     %

\tikzstyle{caption}        = [font=\footnotesize,node distance=0]
\tikzstyle{plate caption}  = [caption,node distance=0,inner sep=0pt,below left=5pt and 0pt of #1.south east]
\tikzstyle{factor caption} = [caption]
\tikzstyle{every label}   += [caption]

\makeatother

\pgfdeclarelayer{edgelayer}
\pgfdeclarelayer{nodelayer}
\pgfsetlayers{edgelayer,nodelayer,main}

\tikzset{>=latex}

\tikzstyle{none}   = [inner sep=0pt]
\tikzstyle{line}   = [-,thick,shorten <=1pt,shorten >=1pt]
\tikzstyle{arrow}  = [->,thick,shorten <=1pt,shorten >=1pt]
\tikzstyle{ardash} = [dashed,->,thick,shorten <=1pt,shorten >=1pt]

\tikzstyle{box}    = [rectangle,minimum width=1.5cm,minimum height=1.5cm,text centered,draw=black,inner sep=7pt]
\tikzstyle{neuron} = [circle,minimum width=4mm,very thick,draw=blue!80!black]
\tikzstyle{empty}  = [circle,opacity=0.0,text opacity=1.0,inner sep=0pt]

\tikzstyle{filled}       = [circle,thick,fill=hexcolor0xbfbfbf,draw=black]
\tikzstyle{hollow}       = [circle,thick,fill=white,draw=black]
\tikzstyle{param}        = [rectangle,fill=black,draw=black,inner sep=0pt,minimum width=4pt,minimum height=4pt]
\tikzstyle{paramhollow}  = [rectangle,thick,fill=white,draw=black,inner sep=0pt,minimum width=4pt,minimum height=4pt]

\usepackage{pgfplots}
\pgfplotsset{compat=newest}
\pgfplotsset{plot coordinates/math parser=false}
\usepgfplotslibrary{statistics}

\newlength\figureheight
\newlength\figurewidth
\newlength\figureheightsmall
\newlength\figurewidthsmall
\makeatletter
\tikzset{
  prefix after node/.style={
    prefix after command={\pgfextra{#1}}
  },
  /semifill/ang/.store in=\semi@ang,
  /semifill/ang=0,
  semifill/.style={
    circle,draw,
    prefix after node={
      \let\nodename\tikz@last@fig@name
      \fill[/semifill/.cd,/semifill/.search also={/tikz},#1]
        let \p1=(\nodename.north),\p2=(\nodename.center) in
        let \n1={\y1-\y2} in
        (\nodename.\semi@ang) arc[radius=\n1,start angle=\semi@ang,delta angle=180];
    },
  }
}
\makeatother

\usepackage{algorithm}
\csundef{If}\csundef{ElseIf}\csundef{Else}
\csundef{For}\csundef{ForPar}\csundef{ForEach}\csundef{ForAll}
\csundef{While}\csundef{Repeat}\csundef{Until}\csundef{Loop}
\csundef{Return}\csundef{Input}\csundef{Output}
\usepackage{algpseudocode}
\usepackage{tcolorbox}
\usepackage{wrapfig}
\usepackage{longtable}

\setlist[itemize]{leftmargin=0.5cm,topsep=2pt,itemsep=1pt,parsep=0pt}

\makeatletter
\let\c@lemma\relax
\let\c@proposition\relax
\let\c@corollary\relax
\let\c@definition\relax
\let\c@remark\relax
\let\c@conjecture\relax

\newtheorem{lemma}{Lemma}
\newtheorem{proposition}{Proposition}
\newtheorem{corollary}{Corollary}
\newtheorem{definition}{Definition}
\newtheorem{remark}{Remark}
\newtheorem{assumption}{Assumption}
\crefname{lemma}{Lemma}{Lemmas}
\crefname{proposition}{Proposition}{Propositions}
\crefname{corollary}{Corollary}{Corollaries}
\crefname{definition}{Definition}{Definitions}
\crefname{remark}{Remark}{Remarks}
\crefname{assumption}{Assumption}{Assumptions}
\makeatother

\newcommand{\defterm}[1]{\uline{\emph{#1}}}


\renewenvironment{proof}[1][\proofname]{\par\noindent{\bfseries\upshape #1.\ }}{\jmlrQED}

\usepackage{hyperref}
\hypersetup{
  colorlinks=true,
  citecolor=MidnightBlue,
  linkcolor=MidnightBlue,
  urlcolor=MidnightBlue
}

\title[What Information Buys in Active Inference]{What Information Buys in Active Inference}

\altauthor{\arxivnames{\Name{Melih Kandemir} \and \Name{Mustafa Mert \c{C}elikok}}}
\arxivaffil{\itshape Department of Mathematics and Computer Science, University of Southern Denmark}

\begin{document}

\maketitle

\begin{abstract}
Information-theoretic exploration is of primary interest in methodological studies of reinforcement learning (RL) and computational cognitive neuroscience, while the established RL theory builds regret guarantees on confidence sets. This situation leaves open the conditions under which the information-driven exploration found in nature is more advantageous than the exploration of the established RL theory. We put the state-of-the-art methods of these two paradigms, Active Inference (AIF) and Bayes-UCBVI, on an equal footing: both act through the same state abstraction of a tabular MDP with $S$ states, and their greedy forms differ only in the exploration term. We find that AIF's information gain decays too fast to serve as a confidence width, locking greedy agents that explore by it out of the optimal action. Calibrating the information gain by a square root yields calibrated AIF, which attains regret of order $\sqrt T$ for an exact abstraction, but no better guarantee than Bayes-UCBVI for a fixed reward. When an adversary chooses rewards from the agent's data, Bayes-UCBVI's value-specific widths lock out, whereas calibrated AIF pays a factor $\sqrt S$, which no goal-memoryless agent avoids on some MDP whose start states the adversary also chooses. Calibrated information gain thus acts as a reward-agnostic $\ell_1$ confidence set, preferable to value-specific widths when rewards are chosen from the agent's data, not when they are fixed.

\end{abstract}

\section{Introduction}
\label{sec:introduction}

Humans think in concepts, and each concept is an abstraction that treats many raw sensory states as equivalent. As the token spaces of frontier AI systems inherit such abstractions from human language, the remarkable navigation power of these systems is limited to the human concept space, and nothing in their training lets them form abstractions outside it. To form a new abstraction, an agent must decide which states to treat as equivalent while it acts, leaving the agent with the dual problems of collecting data to estimate the quality of the existing abstraction and acting in a way to find a better abstraction. We study this coupling of exploration and state abstraction with the tools of learning theory. Our agents act in tabular episodic Markov decision processes (MDPs) and observe the state through an abstraction $\phi$, which in general makes the observed process non-Markov \citep{starre2023analysis}.

Active Inference (AIF) is a normative account of this coupling. It formalises the enactivist thesis that an agent's representation cannot be decoupled from its action \citep{varela1991embodied,hutto2013radicalizing} by letting a single objective drive both its actions and the learning of its world model \citep{friston2017active}. Its single objective, the \emph{expected free energy}, scores an action by its expected reward and by the information it is expected to yield. The same inference learns the agent's world model by predicting the agent's observations, and this model includes the abstraction, that is, which states the agent treats as equivalent. AIF thus makes information gain the common currency of explorative policy search and of the abstraction of the state space. In its standard formulation, rewards enter the picture as \emph{preferences}, a fixed prior distribution over desired outcomes on which inference conditions and which it does not update.

Reinforcement learning (RL) theory instead derives exploration terms from the estimation error of a value, such as an upper-confidence bonus \citep{auer2002finite,azar2017minimax} or a posterior quantile \citep{kaufmann2012bayesian,tiapkin2022dirichlet}. It selects abstractions by their consequences for value \citep{maillard2013optimal,ortner2014selecting,jiang2015abstraction,dong2019provably}. Control as Inference (CaI) \citep{toussaint2009robot,levine2018reinforcement} is known to fail to explore, because it ignores epistemic uncertainty \citep{odonoghue2020making}, which the information-gain term of AIF captures. K-learning \citep{odonoghue2021variational} and VAPOR \citep{tarbouriech2023probabilistic} add a posterior-uncertainty term to CaI and attain sublinear Bayesian regret, that is, regret averaged over a prior on MDPs. EUBRL \citep{ma2026eubrl} adds an epistemic-uncertainty variable to CaI and attains high-probability regret on every MDP. No prior study reports such guarantees for AIF, or for CaI without an added exploration term. Motivated by this gap, we investigate the following question.
\begin{tcolorbox}[colback=gray!6,colframe=gray!55,boxrule=0.5pt,arc=1pt,left=5pt,right=5pt,top=3pt,bottom=3pt]
\emph{When is information gain advantageous, and when is it disadvantageous, in terms of regret, compared with confidence sets, for an agent that explores and learns through a state abstraction?}
\end{tcolorbox}

We put the state-of-the-art methods of these two paradigms, AIF and Bayes-UCBVI \citep{tiapkin2022dirichlet}, on an equal footing. Bayes-UCBVI is minimax-optimal for stage-dependent kernels, and its posterior quantiles act as value-specific confidence widths. Both agents act through the same abstraction and, under a fixed reward function, hold the same Dirichlet posterior over its transitions, and Bayes-UCBVI differs from the greedy variants of AIF only in its exploration term (Section~\ref{sec:aif-model}). The comparison therefore isolates information from value as the currency of exploration. We also compare AIF with CaI and with UCBVI \citep{azar2017minimax}. Two features of AIF make the question delicate. First, AIF adds information, measured in nats, to reward at an exchange rate, the precision. Which precision, if any, calibrates the two is unspecified by the core theory. Second, when the abstraction is learned, the data that select it are generated by the agent's exploration through it. Thus, a poor abstraction can steer the agent away from the data that would reveal its error. The three paragraphs below answer the question for a reward function fixed across episodes, episode-specific reward functions that an adaptive adversary chooses from the agent's data, and an abstraction that the agent selects from a candidate set, and the three parts of Table~\ref{tab:regret} list the results. We provide all the proofs in the appendix.

\begin{table}[t]
\centering
\caption{Regret over $K$ episodes of horizon $H$, $T=HK$ steps, with $S=|\mathcal S|$, $X=|\mathcal X|\ll S$, $A=|\mathcal A|$, the representation scale $\epsilon_{\text{rep}}$ (Definition~\ref{def:spread}) and its bound $\bar\varepsilon$ (Corollary~\ref{cor:minimax}), the branching number $\Gamma_\phi\le X$ (Theorem~\ref{thm:calibrated}), and the candidate set $\Phi$ of abstractions. Each row is a separate result for its agent. Lower bounds hold for every agent of the named kind on some MDP. $\widetilde{\mathcal O}$ hides $\mathrm{poly}\log(HSXAT/\delta)$. \emph{Status} marks a result as known, adapted, an instance of a general theorem, or new, and Appendix~\ref{app:notation} states the conditions of each row. (a)~PAC form in \citet[Thm.~2]{kolter2009near}. (b)~Frequentist counterpart of \citet{lu2019information}. (c)~Of \citet[Thm.~1]{xie2023role}. (d)~Boltzmann exploration in bandits \citep{cesabianchi2017boltzmann}.}
\label{tab:regret}
\small
\setlength{\tabcolsep}{2pt}
\begin{tabular}{@{}clllc@{}}
\toprule
\# & Agent & Bound & Result & Status \\
\midrule
\multicolumn{5}{@{}l}{\emph{Fixed reward function, with the accuracy of $\phi$ (rows 1--3) and exploration by information (rows 4--8)}}\\
1 & any agent acting through $\phi$ & $\Omega(\bar\varepsilon K+\sqrt{HXT})$ & Cor.~\ref{cor:minimax} & new \\
2 & UCBVI-CH on $\phi$ & $\mathcal O(\epsilon_{\text{rep}}K)+\widetilde{\mathcal O}(H\sqrt{XAT})$ & Thm.~\ref{thm:ucbvi-fixed}(a) & new \\
3 & Bayes-UCBVI on $\phi$ & $\mathcal O(\epsilon_{\text{rep}}K)+\widetilde{\mathcal O}(H^{3/2}\sqrt{XAT})$ & Thm.~\ref{thm:ucbvi-fixed}(b) & new \\
4 & AIF and CaI, softmax & $\Omega(K)$ & Prop.~\ref{prop:aif-cai-regret} & adapted$^{\text{d}}$ \\
5 & greedy, $\lambda\,\mathrm{Novelty}^q$, $q\in(\frac12,1]$ & $\Omega(K)$ & Thm.~\ref{thm:exponent}(a) & new$^{\text{a}}$ \\
6 & greedy, $\lambda\,\mathrm{Novelty}^q$, $q<\frac12$ & $\Omega(\min\{\lambda K^{1-q},K\})$ & Thm.~\ref{thm:exponent}(b) & new \\
7 & AIF's state term & $\Omega(K)$ & Prop.~\ref{prop:noisy-tv} & new \\
8 & calibrated AIF & $\mathcal O(\epsilon_{\text{rep}}K)+\widetilde{\mathcal O}(H^{3/2}\sqrt{\Gamma_\phi XAT})$ & Thm.~\ref{thm:calibrated} & new$^{\text{b}}$ \\
\midrule
\multicolumn{5}{@{}l}{\emph{Adaptive goals ($\phi$ exact)}}\\
9 & any goal-memoryless agent & $\Omega(X\sqrt{K})$ & Thm.~\ref{thm:price}(b) & new \\
10 & best agent, using past goals as data & $\widetilde\Theta(X^{3/4}\sqrt{K})$ & Thm.~\ref{thm:price}(a,\,c) & new \\
11 & calibrated AIF, uniform prior & $\widetilde{\mathcal O}(H^{3/2}X\sqrt{AT})$ & Thm.~\ref{thm:price}(c) & new \\
12 & Bayes-UCBVI & $\Omega(K/\ln K)$ & Prop.~\ref{prop:lockout} & new \\
13 & UCBVI-CH, goals fixed in advance & $\widetilde{\mathcal O}(H\sqrt{XAT})$ & Prop.~\ref{prop:oblivious} & adapted \\
\midrule
\multicolumn{5}{@{}l}{\emph{Learned abstractions ($\phi$ selected from $\Phi$)}}\\
14 & any agent, $\Phi$ all maps & $\Omega(\sqrt{HST})$ & Prop.~\ref{prop:unstructured} & adapted \\
15 & AIF, curious about its emission & $\Omega\big(\min\{K,S\}\big)$ & Prop.~\ref{prop:lure} & new \\
16 & latent-predictive, decoder-free & $\Omega(HK)$ & Prop.~\ref{prop:self-confirming} & new \\
17 & GOLF, fixed or adaptive latent goals & $\widetilde{\mathcal O}\big(H^2\sqrt{XAT(\ln|\Phi|+XA)}\big)$ & Thm.~\ref{thm:enactive} & instance$^{\text{c}}$ \\
\bottomrule
\end{tabular}
\end{table}

\emph{Information gain costs regret under a fixed reward function.} When the reward function is fixed across episodes, optimistic agents explore through $\phi$ at a cost that scales with $X$ instead of $S$. We call an exploration term \emph{valid} if, with high probability, it is at least the amount by which the estimated next-step value falls short of the true one, so that the agent's values stay optimistic. This shortfall decays with the visit count $N$ as $N^{-1/2}$, whereas the information gain about the transition kernel, which AIF calls novelty, decays as $N^{-1}$. With a weight that does not grow with the number of episodes, an agent with a bonus of this order stops playing the optimal action with constant probability (Theorem~\ref{thm:exponent}), and AIF and CaI incur linear regret (Proposition~\ref{prop:aif-cai-regret}). A transportation inequality shows that the square root of novelty, multiplied by the range of the value function and a logarithmic factor, is a valid width, a frequentist counterpart of the widths of \citet{lu2019information}. The resulting greedy agent, \emph{calibrated AIF}, has regret $\mathcal O(\epsilon_{\text{rep}}K)+\widetilde{\mathcal O}(H^{3/2}\sqrt{\Gamma_\phi XAT})$ (Theorem~\ref{thm:calibrated}). This bound exceeds that of Bayes-UCBVI's value-specific quantile by a factor $\sqrt{\Gamma_\phi}$, and calibrated information gain therefore does not improve on the bound attainable with visit counts. Under a fixed reward function, information seeking is thus disadvantageous before calibration and brings no advantage after it.

\emph{Under adaptive reward functions, calibrated information gain is advantageous over value-specific widths.} When an adaptive adversary chooses the reward functions, which we call goals, from the agent's data, Bayes-UCBVI's value-specific quantile is accurate for every fixed goal but not for a goal chosen after the data, which causes Bayes-UCBVI to incur regret $\Omega(K/\ln K)$ (Proposition~\ref{prop:lockout}). In contrast, novelty does not depend on the reward function. Calibrated AIF with a uniform prior is therefore optimistic for every reward function simultaneously, with a width within constants of the radius of an $\ell_1$ confidence set \citep{jaksch2010near}. Its advantage is thus over value-specific widths, not over $\ell_1$ confidence sets, which \citet{rosenberg2019online} use against adaptive rewards. Theorem~\ref{thm:price} determines the price in regret of this simultaneous optimism. On some MDP with $S$ states whose start states the adversary also chooses, every agent incurs a factor $S^{1/4}$ more regret than under a fixed goal. Every \emph{goal-memoryless} agent (Definition~\ref{def:goal-memoryless}), which conditions only on the observed transitions and the current goal and breaks ties at the last step by a local rule, incurs a factor $\sqrt S$ on such an MDP, and AIF is such an agent under the last-step conventions of Lemma~\ref{lem:membership}. Both factors hold up to logarithmic factors. Calibrated AIF with a uniform prior attains the factor $\sqrt S$, and goals that an oblivious adversary fixes before the first episode cost no extra regret. An agent that uses past goals as data about the transitions, which AIF's generative model cannot do, halves the exponent of the price.

\emph{Learned abstractions need a common target.} When the agent learns its abstraction, value is again the better target. AIF learns $\phi$ by predicting observations through an emission model, a probabilistic map from abstract states to ground states. Its information gain about the emission makes it prefer a worse action that yields many distinct observations, with regret linear in $K$ whenever $S\ge2K+2$ (Proposition~\ref{prop:lure}). Without the decoder, as in joint-embedding predictive architectures (JEPA) \citep{lecun2022path,assran2023self}, the agent scores each candidate abstraction by how well it predicts its own labels. This target is self-referential. A coarse abstraction, one that merges states whose best actions differ, predicts its own labels well and is never refuted by the data that it generates. It thus incurs linear regret (Proposition~\ref{prop:self-confirming}). GOLF \citep{jin2021bellman} avoids this problem, because it uses the value function as a target that all candidate abstractions share. It attains regret $\widetilde{\mathcal O}(\sqrt T)$ in realizable block MDPs with latent rewards. The rate also holds when an adaptive adversary chooses latent goals (Theorem~\ref{thm:enactive}).

Consequently, calibrated information gain acts as a reward-agnostic confidence set of the $\ell_1$ type. It is advantageous over value-specific confidence widths when the reward functions are chosen from the agent's own data, and it is not when the reward function is fixed. The first case also covers an agent that sets its own goals from its data, such as a homeostatic agent whose set points adapt to its trajectory \citep{sterling2012allostasis,pezzulo2015active} or an autotelic agent \citep{colas2022autotelic}. Calibrated AIF with a uniform prior stays optimistic whatever rule sets the goals, whereas Bayes-UCBVI locks out under a rule that uses only the agent's own data. When the agent learns its abstraction, information about the emission of observations misleads its exploration, and value is the better target.

\section{Episodic Control Under an Abstraction}
\label{sec:setting}
\label{sec:ucbvi}

We work in the finite-horizon episodic tabular setting of \citet{azar2017minimax}. An episodic MDP is a tuple $\mathcal M=(\mathcal S,\mathcal A,P^*,R^*,H)$ with $\mathcal S,\mathcal A$ finite, a known episode length $H$, a time-homogeneous kernel $P^*$, and a deterministic reward $R^*:\mathcal S\times\mathcal A\to[0,1]$. For a policy $\pi=(\pi_1,\dots,\pi_H)$, we set $V^\pi_{H+1}\equiv0$, that is, $V^\pi_{H+1}(s)=0$ for all $s$, and $Q^\pi_h(s,a) = R^*(s,a) + \sum_{s'} P^*(s'\mid s,a)\, V^\pi_{h+1}(s')$ and $V^\pi_h(s) = Q^\pi_h(s,\pi_h(s))$.
The optimal values satisfy $V^*_h(s)=\max_aQ^*_h(s,a)$ and $0\le V^*_h\le H-h+1$. The agent does not know $P^*$. Before episode $k$ it commits to a policy $\pi^k$ that may depend on all past observations. Its \emph{regret} over $T=HK$ steps is its shortfall against an agent that knows $P^*$ and follows an optimal policy, $\mathfrak R_K:=\sum_{k\le K}\big(V^*_1(s^k_1)-V^{\pi^k}_1(s^k_1)\big)$.

\emph{Value-based baselines.}\label{sec:bayesucbvi} We compare the information-driven agents of Section~\ref{sec:aif} with two value-based agents, UCBVI \citep{azar2017minimax} and Bayes-UCBVI \citep{tiapkin2022dirichlet}. On the ground MDP, both know $R^*$, as in their original analyses.\footnote{Every agent of this paper that acts through an abstraction, value-based or information-driven, uses the same empirical reward estimate (Section~\ref{sec:abstraction}), which makes the comparisons fair. On the ground MDP, knowing $R^*$ loses no generality for the upper bounds, since estimating a deterministic reward in $[0,1]$ adds a standard lower-order term \citep{jaksch2010near}.} Both compute optimistic values $\overline V^k_h$ by backward induction. At each step they add a nonnegative exploration term, which we call a \emph{confidence width}, to the Bellman backup $R^*+\widehat P\,\overline V^k_{h+1}$ under a kernel estimate $\widehat P$. Let $e_{k,h}:=(\widehat P-P^*)V^*_{h+1}$ be the error of the estimated next-step value. A width is \emph{valid} if, with high probability, it is at least $-e_{k,h}(s,a)$ for all $k$, $h$, $s$ and $a$, that is, if it compensates every underestimation of the next-step value. By backward induction, valid widths keep the computed values optimistic, $\overline V^k_h\ge V^*_h$.

UCBVI uses its empirical kernel $\widehat P_k$ and a Chernoff--Hoeffding (UCBVI-CH) or Bernstein--Freedman (UCBVI-BF) bonus as its width. Bayes-UCBVI places a Dirichlet posterior $\rho^k_h(s,a)$ on each row of the kernel, one for each step $h$. Its prior puts $n_0$ pseudo-counts on an absorbing pseudo-state $s_{\mathrm E}$ with reward $r_0>1$, an instrumental device that makes unvisited rows optimistic (Appendix~\ref{app:baselines}). Value functions are extended to the pseudo-state by its value, $V^*_{h+1}(s_{\mathrm E}):=r_0(H-h)$. The width of Bayes-UCBVI is its \emph{quantile gap} $\Delta^k_h(s,a)$, the difference between the quantile of order $\kappa^k_h(s,a)$ (Definition~\ref{def:bayesq-ground}) of the next-step value $p^\top\overline V^k_{h+1}$ under $p\sim\rho^k_h(s,a)$ and the next-step value under the posterior mean $\widehat P^k_h(\cdot\mid s,a)$. Let $\bar n=\bar n^k_h(s,a)$ be the visit count plus $n_0$, and let $\ell=\ell^k_h(s,a):=\ln\big(1/(1-\kappa^k_h(s,a))\big)=\mathcal O(\ln(SAHT/\delta))$ be the logarithmic factor of the quantile order $\kappa^k_h(s,a)$ (Definition~\ref{def:bayesq-ground}). On every sample path, the quantile gap is at most of order $H\sqrt{\ell/\bar n}$ (Lemma~\ref{lem:tiapkin-quantile} in Appendix~\ref{app:baselines}). It thus decays as $\bar n^{-1/2}$, the rate at which an empirical mean concentrates. At a fixed count, no valid width decays faster (Lemma~\ref{lem:novelty-order}(ii)).

\subsection{State Abstraction and Its Cost}
\label{sec:abstraction}
\label{sec:fixed-abstraction}
\label{sec:given}

An \emph{abstraction} is a map $\phi:\mathcal S\to\mathcal X$ with $|\mathcal X|=X\ll S$. Its \emph{blocks} $\phi^{-1}(x)$ are the sets of ground states it merges. A policy is \emph{$\phi$-measurable} if it depends on the ground state only through $\phi$. With a weighting $W(s\mid x)$ of each block, $\phi$ induces the abstract MDP $\overline{\mathcal M}_\phi=(\mathcal X,\mathcal A,\overline P_\phi,\overline R_\phi,H)$ with $\overline R_\phi(x,a)=\sum_{s\in\phi^{-1}(x)} W(s\mid x)\, R^*(s,a)$ and
\begin{equation}
\overline P_\phi(x'\mid x,a) = \textstyle\sum_{s\in\phi^{-1}(x)} W(s\mid x) \sum_{s'\in\phi^{-1}(x')} P^*(s'\mid s,a). \label{eq:abstract-mdp}
\end{equation}
We write $\overline V^\pi_{\phi,h},\overline Q^\pi_{\phi,h}$ for the value functions of $\overline{\mathcal M}_\phi$ and $\overline V^*_{\phi,h},\overline Q^*_{\phi,h}$ for its optimal value functions, defined as in Section~\ref{sec:setting} with $(\overline P_\phi,\overline R_\phi)$ in place of $(P^*,R^*)$. The agent cannot compute $\overline{\mathcal M}_\phi$ because $W$ is unknown. In particular, $\overline R_\phi$ is unknown even when $R^*$ is known.

\begin{definition}[Representation scale, cf.~\citealp{abel2016near}]
\label{def:spread}
With $\phi_\#$ the pushforward through $\phi$, let $\eta_R(\phi) := \max_{x,a}\max_{s,s'\in\phi^{-1}(x)} |R^*(s,a)-R^*(s',a)|$ and $\eta_P(\phi) := \max_{x,a}\max_{s,s'\in\phi^{-1}(x)}\allowbreak \|\phi_\# P^*(\cdot\mid s,a)-\phi_\# P^*(\cdot\mid s',a)\|_1$.
The \defterm{representation scale} of $\phi$ is $\epsilon_{\text{rep}}(\phi) := H\eta_R(\phi) + \tfrac12H^2\eta_P(\phi)$. We call $\phi$ \defterm{exact} if $\epsilon_{\text{rep}}(\phi)=0$ \citep{li2006towards}.
\end{definition}
In an exact abstraction, merged states have the same rewards and the same pushed-forward transition laws (Appendix~\ref{app:related}). The decoder of a block MDP with latent rewards (Assumption~\ref{assum:block} in Section~\ref{sec:learning-phi}) is an example. The representation scale bounds the value that an agent can lose in one episode by treating each block as a single state, and this bound is tight up to a constant factor (Proposition~\ref{prop:lb-abstraction}). A standard simulation argument \citep{abel2016near} shows that it bounds two gaps, the gap between the optimal values of $\mathcal M$ and $\overline{\mathcal M}_\phi$ and the gap between the values of any $\phi$-measurable policy in the two models (Lemma~\ref{lem:bias-metric} in Appendix~\ref{app:lem1}). Every bound below depends on $\phi$ only through $\epsilon_{\text{rep}}(\phi)$.

Under \emph{abstract optimism}, we run an optimistic algorithm unchanged on $\mathcal X$, with $x_\tau=\phi(s_\tau)$ as the observed state. It estimates the unknown $\overline R_\phi(x,a)$ by the empirical mean $\widehat R^k_\phi(x,a)$ of the rewards observed at the $N_k(x,a)$ visits to $(x,a)$ before episode $k$, and by $\widehat R^k_\phi(x,a):=1$ if $N_k(x,a)=0$. Optimism survives up to $\epsilon_{\text{rep}}(\phi)$ per episode (Appendix~\ref{app:thm-fixed}).

\begin{theorem}[Optimism is robust to abstraction]
\label{thm:ucbvi-fixed}\label{thm:bayesucbvi-fixed}
Run UCBVI and Bayes-UCBVI on the abstract data of $\phi$ with $\widehat R^k_\phi$, and let $L:=\max\{\iota,\bar\ell,n_0,1+\ln K\}$, where $\iota:=\ln(16X^2AHT^2/\delta)$, $\bar\ell:=\max_{k,h,x,a}\ell^k_h(x,a)$ is the largest logarithmic factor of Bayes-UCBVI's quantile orders (Definition~\ref{def:bayesq-ground} and Lemma~\ref{lem:tiapkin-quantile}), and $n_0$ is its pseudo-count. Then $L=\mathcal O(\ln(XAHT/\delta))$. With probability at least $1-\delta$, the following hold.
\begin{itemize}
\item[(a)] UCBVI-CH has $\mathfrak R_K\le\mathcal O(\epsilon_{\text{rep}}(\phi)K)+\mathcal O\big(HL\sqrt{XAT}\big)+\mathcal O\big(H^2X^2AL(H+\ln T)\big)$, in particular $\mathfrak R_K=\mathcal O(\epsilon_{\text{rep}}(\phi)K)+\widetilde{\mathcal O}(H\sqrt{XAT})$ for $T\ge H^4X^3A$.
\item[(b)] Bayes-UCBVI has $\mathfrak R_K\le\mathcal O(\epsilon_{\text{rep}}(\phi)K)+\mathcal O\big(H^{3/2}\sqrt{XATL}\big)+\mathcal O\big(H^3X^2AL^2\big)$, in particular $\mathfrak R_K=\mathcal O(\epsilon_{\text{rep}}(\phi)K)+\widetilde{\mathcal O}(H^{3/2}\sqrt{XAT})$ for $T\ge H^3X^3A$.
\item[(c)] If $\phi$ is exact and the agent is given $\overline R_\phi$, the abstract data form an episodic MDP on $\mathcal X$, and Theorems~\ref{thm:ucbvi-ground} and~\ref{thm:bayesucbvi-ground} hold with $S$ replaced by $X$. Their leading terms are $\widetilde{\mathcal O}(\sqrt{HXAT})$ for UCBVI-BF if $XA\ge H$, and $\widetilde{\mathcal O}(H\sqrt{XAT})$ for Bayes-UCBVI.
\end{itemize}
\end{theorem}
Optimism thus turns an abstraction error into an additive cost, which the agent incurs without knowing $\epsilon_{\text{rep}}(\phi)$, as \citet{dong2019provably} showed for optimistic Q-learning with state aggregation. In the worst case, no agent avoids this cost, and UCBVI-CH matches the minimax rate up to $\sqrt H$ and logarithmic factors (Table~\ref{tab:regret}, row~1, and Corollary~\ref{cor:minimax} in Appendix~\ref{app:unstructured}).

\subsection{Goals}
\label{sec:goals}

Section~\ref{sec:fixed-goals} uses a reward function $R^*$ that is fixed across all episodes. Section~\ref{sec:goal-uniform} instead considers episode-specific reward functions $R_k:\mathcal X\times\mathcal A\to[0,1]$, which we call \emph{goals}. Before episode $k$, the agent receives the goal $R_k$ together with a start state $x^k_1$, chosen by an \emph{adversary}. The goals are said to be \emph{adaptive} if the adversary chooses them from the interaction history of the agent, and \emph{oblivious} if it fixes all goals and start states before the first episode. We call the adversary adaptive or oblivious accordingly. The regret against these goals is $\mathfrak R^{\mathrm{goal}}_K:=\sum_k\big(V^*_1(x^k_1;R_k)-V^{\pi^k}_1(x^k_1;R_k)\big)$.
The \emph{records} of a triple $(h,x,a)$ are the successors observed after playing $a$ at $x$ at step $h$.
\begin{definition}[Goal-memoryless agents]
\label{def:goal-memoryless}
An agent is \defterm{goal-memoryless} if (i)~its policy in each episode is a function of the records of all triples, of the current goal and of randomness drawn afresh for the episode, and (ii)~at the last step, at every state $x$ at which all actions have the same reward under the current goal, it draws its action by a fixed rule from the records of the triples $(H,x,\cdot)$ and from $U_{x,j}$. Here $j$ counts its earlier last-step visits to $x$, and the $U_{x,j}$ are independent uniform variables, independent of everything else. Fixed parameters, such as $K$ and $\delta$, may enter both rules.
\end{definition}
Condition~(ii) excludes the goal, the records of other states and the episode index from these choices. Otherwise the agent's actions at the last step could record the order of its earlier, goal-dependent choices (Appendix~\ref{app:adaptive-b}). A goal-memoryless agent conditions on the current goal but does not use past goals as data about the transitions. Bayes-UCBVI, CaI, AIF and calibrated AIF (Sections~\ref{sec:calibrated} and~\ref{sec:goal-uniform}) are goal-memoryless if, at the last step, the greedy agents break ties by a fixed order of the actions and AIF and calibrated AIF with Bayes-UCBVI's prior act on their index before the cap at $r_0$ (Lemma~\ref{lem:membership} in Appendix~\ref{app:adaptive-b}). UCBVI is not goal-memoryless, since it clamps its values at those of the previous episode, which were computed under earlier goals. Condition~(ii) constrains only the last step at states where all actions earn the same reward, and the agents above meet it with a fixed tie order or with randomness drawn by local visit. Neither our upper bound nor our lock-out argument covers UCBVI under adaptive goals.

\section{What Information Gain Costs Under a Fixed Reward Function}
\label{sec:fixed-goals}
\label{sec:learning}
\label{sec:aif}
\label{sec:certified}
\label{sec:aif-regret}
\label{sec:cai}

\subsection{AIF and CaI on an Abstraction}
\label{sec:aif-model}
\label{sec:aif-bonus}
We run AIF and CaI on the abstract MDP $\overline{\mathcal M}_\phi$, as Section~\ref{sec:fixed-abstraction} ran UCBVI and Bayes-UCBVI. Following \citet{friston2017active} and \citet{dacosta2020synthesis}, AIF posits hidden states $x_\tau$ that emit observations $s_\tau$ through an emission matrix $\mathbf A_\phi$ and evolve under actions through a transition matrix $B_\phi$. We identify the hidden states with the abstract states, $x_\tau=\phi(s_\tau)\in\mathcal X$. The rows $B_\phi(\cdot\mid x,a)$, one for each step as in Bayes-UCBVI, then model $\overline P_\phi(\cdot\mid x,a)$. Until Section~\ref{sec:learning-phi}, the agent applies a fixed abstraction $\phi$ to its observations, and our results hold for every such $\phi$. The observation $s_\tau$ then determines the hidden state $x_\tau$, which AIF calls \emph{exact perception}. The emission factor does not depend on $B_\phi$ and thus plays no role in the posterior over $B_\phi$ or in planning. \citeauthor{friston2017active}'s Dirichlet learning of $B_\phi$ (their Eq.~2.7), with Bayes-UCBVI's instrumental prior and one row per step, yields exactly $\rho^k_h(x,a)$. AIF thus holds the same posterior as Bayes-UCBVI. The representation cost adds $\mathcal O(\epsilon_{\text{rep}}(\phi)K)$ to each upper bound below, as in Theorem~\ref{thm:ucbvi-fixed}. Our lower bounds use exact abstractions, so that this section isolates the cost of exploration.

The AIF agent scores a policy $\pi$ by the expected free energy of the discrete-state formulation of AIF \citep[Eq.~2.5]{friston2017active},
\begin{equation}
\mathcal G(\pi) = \textstyle\sum_{\tau>t}\Big(-\mathds{E}_{\mathcal Q(r_\tau\mid\pi)}\big[\ln \widetilde P(r_\tau)\big] - I_{\mathcal Q}\big(x_\tau;s_\tau\mid\pi\big)\Big). \label{eq:efe-canonical}
\end{equation}
Here $\mathcal Q$ denotes the agent's approximate posterior and the predictive distributions that it induces (Appendix~\ref{app:models-aif}), and the preferences $\ln\widetilde P(r)=\beta r+\text{const}$ at \emph{precision} $\beta>0$ concern the reward outcome. Under the agent's model, this outcome is its estimate $r_\tau=\widehat R^k_\phi(x_\tau,a_\tau)$ (Appendix~\ref{app:models-aif}). An estimate is needed even when $R^*$ is known, because the abstract reward $\overline R_\phi$ depends on the unknown weighting $W$. The extrinsic value is thus $\beta$ times the expected estimated reward, up to a constant that does not depend on $\pi$. The epistemic value, which \citet{friston2017active} also call intrinsic value and we call the \emph{state term}, is the information that the observation carries about the hidden state. Under exact perception, it equals the entropy $\mathds{H}[\mathcal Q(x_\tau\mid\pi)]$ of the predicted abstract state \citep[Eq.~5]{millidge2021relationship}. Learning $B_\phi$ adds a third term. \citet[Eq.~16]{dacosta2020synthesis} and \citet[Eq.~7c]{nuijten2026whattype} call the information that an outcome carries about the emission parameters \emph{novelty}. We apply the same definition to $B_\phi$, with the next abstract state as the outcome. We compute the information under the predictive joint in which the row is drawn from the posterior and the next state from the drawn row, $B\sim\rho^k_h(x,a)$ and $x_{\tau+1}\mid B\sim\mathrm{Cat}(B)$, yielding
\begin{equation}
\mathrm{Novelty}(x,a) := I\big(B;x_{\tau+1}\mid x_\tau{=}x,a_\tau{=}a\big) = \mathds{H}\big[\widehat P^k_h(\cdot\mid x,a)\big] - \mathds{E}_{B\sim\rho^k_h(x,a)}\big[\mathds{H}[B]\big]. \label{eq:novelty-def}
\end{equation}
Under this joint, the marginal law of $x_{\tau+1}$ is the posterior mean $\widehat P^k_h(\cdot\mid x,a)$ that the planners use. The state term one step ahead, therefore, equals $\mathds{E}_B\mathds{H}[B]+\mathrm{Novelty}$. It contains novelty once, and for an exact abstraction it converges to the entropy of the true transition law, which does not vanish unless the transitions are deterministic. CaI uses the same model and posterior. It places reward in the likelihood of an observed optimality variable instead of a preference over outcomes \citep{levine2018reinforcement}. With $\phi$ given, $B_\phi$ enters CaI's planning only through its posterior mean (Appendix~\ref{app:models-cai}).

Dividing AIF's objective by $\beta$ expresses it in units of reward. With $\phi$ given, every Bayesian agent of this paper holds the posterior $\rho^k_h(x,a)$, with mean $\widehat P^k_h$, and values actions by
\begin{equation}
Q_{k,h}(x,a)=\widehat R^k_\phi(x,a)+\widehat P^k_h(\cdot\mid x,a)^\top V_{k,h+1}+\Psi_{k,h}(x,a),\qquad V_{k,h}(x)=\max_aQ_{k,h}(x,a),
\label{eq:psi-slot}
\end{equation}
with $V_{k,H+1}\equiv0$ and the pseudo-state value $V_{k,h}(x_{\mathrm E})=r_0(H-h+1)$. The agents differ in the epistemic term $\Psi$ and in the action rule. Bayes-UCBVI uses the quantile gap $\Delta^k_h$ and acts greedily. The AIF agent we analyse instantiates the discrete-state AIF of \citet{friston2017active} with Dirichlet learning of $B_\phi$, evaluated recursively as in sophisticated inference \citep{friston2021sophisticated}. Its term $\Psi$ is $\beta^{-1}$ times the sum of the state term and novelty. It thus counts novelty twice, once inside the state term and once as the parameter information gain of standard AIF (Appendix~\ref{app:models-aif}). It clamps its $Q$-values at $r_0(H-h+1)$ and samples its action at each step from a softmax at precision $\beta$. CaI uses $\Psi=0$ and samples its action from the same softmax. In the backup of Eq.~\eqref{eq:psi-slot}, it computes $V_{k,h}(x)$ by the soft maximum $\beta^{-1}\ln\big(A^{-1}\sum_ae^{\beta Q_{k,h}(x,a)}\big)$ in place of the maximum. Calibrated AIF uses the term of Section~\ref{sec:calibrated} and acts greedily. Appendix~\ref{app:models} derives these choices from the generative models, and Appendix~\ref{app:algs} specifies every agent.

\subsection{Raw Information Gain Has the Wrong Order}
\label{sec:validity}
\label{sec:exponent}

The next lemma compares novelty with the estimation error that every valid width must exceed.

\begin{lemma}[Novelty versus valid widths]
\label{lem:novelty-order}
Fix $(x,a,h)$ and let $\sigma_V^2:=\mathds{V}_{\overline P_\phi(\cdot\mid x,a)}\big[\overline V^*_{\phi,h+1}\big]$.
\begin{itemize}
\item[(i)] At visit count $n$, with $\bar n=n+n_0$, $\mathrm{Novelty}(x,a)\le X/(\bar n+1)$.
\item[(ii)] Let $\phi$ be exact, and let $e_n:=\big((\widetilde P_n-\overline P_\phi)\overline V^*_{\phi,h+1}\big)(x,a)$ be the estimation error of the empirical kernel $\widetilde P_n$ of the row's first $n$ samples. Let $\sigma_V^2\ge\nu H^2$ with $\nu\in(0,1]$ and $n\nu\ge1$. Then at every count $n$ fixed in advance, whatever the algorithm, the empirical kernel underestimates the value with constant probability, $\mathds{P}\big(e_n\le-H\sqrt{\nu}/(8\sqrt n)\big)\ge1/64$. Hence every $b_n$ with $\mathds{P}(b_n\ge-e_n)\ge1-\delta'>63/64$ satisfies $\mathds{P}\big(b_n\ge H\sqrt{\nu}/(8\sqrt n)\big)\ge1/64-\delta'$.
\end{itemize}
\end{lemma}
Part~(ii) concerns the lower tail of the estimation error, which validity constrains, at every fixed count. The consequences of this lower bound on regret follow from Theorem~\ref{thm:exponent} (Appendix~\ref{app:lem-novelty-order}). By~(i), novelty falls below the threshold of~(ii) once $\bar n\ge64X^2/(\nu H^2)$. It scales as the square of a valid width, as a variance does against a standard deviation. No constant factor, the precision included, changes its order.

This mismatch costs regret, as we now show for a family of greedy agents. We call an agent \emph{information-driven with exponent $q>0$ and weight $\lambda\ge0$} if it acts greedily on Eq.~\eqref{eq:psi-slot} with $\Psi_{k,h}=\lambda\,\mathrm{Novelty}^q$, the known reward $R^*$ in place of $\widehat R^k_\phi$, and optionally a clamp at $r_0(H-h+1)$. The greedy novelty-only variant of AIF with the known reward has $q=1$ and $\lambda=1/\beta$. Greedy CaI with the known reward and a maximum in place of its soft maximum has $\lambda=0$. The first term of the calibrated width of Section~\ref{sec:calibrated} has $q=\frac12$.

\begin{theorem}[The exponent decides how the weight must scale]
\label{thm:exponent}
Let an information-driven agent with $n_0,r_0\ge1$, $q$, and $\lambda$ act on an exact abstraction.
\begin{itemize}
\item[(a)] \uline{Exponent above one half.} Let $q\in(\frac12,1]$. There is an MDP with $S=4$ and $A=H=2$, which depends only on $n_0$, $r_0$, $q$, and $\lambda$, on which, with probability at least $\frac1{24}$, the agent plays the optimal action at most $N^\star=\mathcal O(\lambda^{2/(2q-1)}+n_0^2r_0^2)$ times in total, with $N^\star$ given explicitly in Proposition~\ref{prop:exponent-explicit}. Hence $\mathds{E}\mathfrak R_K=\Omega\big((K-N^\star)/\sqrt{N^\star}\big)$.
\item[(b)] \uline{Exponent below one half.} Let $q\in(0,\frac12)$ and $0<\lambda\le r_0/2$. For every $K\ge16(\max\{n_0,4^{1/q}\}+1)$, there is an MDP with $S=A=H=2$ on which $\mathfrak R_K=\Omega(\min\{\lambda K^{1-q},K\})$.
\end{itemize}
\end{theorem}
With $n_0$, $r_0$, and $\lambda$ independent of $K$, part~(a) gives $\mathds{E}\mathfrak R_K=\Omega(K)$. With Bayes-UCBVI's $n_0=\Theta(\ln K)$, it gives $\Omega(K/\ln K)$, on an MDP that depends on $K$ (Appendix~\ref{app:exponent}, which also covers every $H\ge2$). Unless $n_0r_0\gtrsim\sqrt K$, regret $\mathcal O(\sqrt K)$ therefore requires $\lambda\gtrsim K^{q-1/2}$, that is, $\beta\lesssim K^{-1/2}$ for novelty. At $q=\frac12$, a weight that grows only as the square root of a logarithmic confidence factor suffices (Section~\ref{sec:calibrated}).

\begin{proposition}[Regret of AIF and CaI]
\label{prop:aif-cai-regret}
For every precision $\beta>0$ and every prior with $r_0\ge1$, AIF and CaI sample every action with probability at least $e^{-\beta r_0H}/A$. Hence $\mathds{E}\mathfrak R_K\ge K\Delta e^{-\beta r_0H}/A$ on every MDP on which every start state has an action with gap at least $\Delta>0$.
\end{proposition}
With $\beta$ and the prior independent of $K$, the greedy variants of AIF and CaI with the known reward also incur linear regret. With constant probability, they \emph{lock out} of the optimal action, that is, from some episode on they never play it again (Theorem~\ref{thm:exponent}(a)). AIF's state term incurs linear regret even when AIF is given $P^*$ and $R^*$ (Proposition~\ref{prop:noisy-tv}).

\subsection{Calibrated AIF}
\label{sec:calibrated}

AIF adds an epistemic term measured in nats to a value measured in reward. Pinsker's inequality relates the two units. Let $\mathrm{span}(f):=\max f-\min f$, taken over all categories, the pseudo-state included, and let a belief $B$ differ from the posterior mean $\bar p$ by the information $D_{\mathrm{KL}}(B\|\bar p)$. The expectations of a bounded function $f$ under $B$ and under $\bar p$ then differ by at most $\mathrm{span}(f)\sqrt{D_{\mathrm{KL}}(B\|\bar p)/2}$. Averaged over the posterior, this information is novelty. A transportation lemma makes this precise (Lemma~\ref{lem:transport} in Appendix~\ref{app:calibrated}). For every Dirichlet posterior and every $f$, the posterior variance of $B^\top f$ is at most $\frac12\mathrm{span}(f)^2\,\mathrm{Novelty}$. Consequently, Bayes-UCBVI's quantile gap at $f$ is at most $2\,\mathrm{span}(f)\sqrt{\ell\,\mathrm{Novelty}}+2\sqrt2\,r_0H\ell/\bar n$.

The first-order term of a valid width is thus the square root of novelty, multiplied by the range of the value function, as in the Bayesian widths of \citet[Lemma~3]{lu2019information}. \emph{Calibrated AIF} drops the state term and sets $\Psi_{k,h}=\Psi^+_{k,h}$ in Eq.~\eqref{eq:psi-slot}, where $\Psi^+_{k,h}$ is built from the novelty of Eq.~\eqref{eq:novelty-def},
\begin{equation}
\Psi^{+}_{k,h}(x,a) \;:=\; 2\,\mathrm{span}\big(V^+_{k,h+1}\big)\sqrt{\ell^k_h(x,a)\,\mathrm{Novelty}(x,a)} \;+\; 2\sqrt2\,r_0H\,\ell^k_h(x,a)/\bar n^k_h(x,a),
\label{eq:calibrated}
\end{equation}
with $\ell^k_h$ and $\bar n^k_h$ as in Section~\ref{sec:setting}. Calibrated AIF acts greedily on the recursive expected free energy \citep{friston2021sophisticated} with values clamped at $r_0(H-h+1)$ (Algorithm~\ref{alg:aif}). Its first term multiplies novelty by the exchange rate $2\,\mathrm{span}(V^+_{k,h+1})\sqrt{\ell/\mathrm{Novelty}}$. This rate is at least $2\,\mathrm{span}(V^+_{k,h+1})\sqrt{\ell\bar n/X}$ and thus grows as $\bar n^{1/2}$. This resolves the calibration problem of the introduction: no constant precision converts nats into reward, whereas this exchange rate does.

\begin{theorem}[Calibrated AIF]
\label{thm:calibrated}
Let Bayes-UCBVI's $n_0$, $r_0$, and quantile schedule (Definitions~\ref{def:dirichlet-ground}--\ref{def:bayesq-ground}) be used on $\mathcal X$. Let the \defterm{branching number} of the abstract data, the abstract counterpart of the maximal support size of \citet{fruit2018efficient}, be $\Gamma_\phi:=\max_{x,a}|\bigcup_{s\in\phi^{-1}(x)}\mathrm{supp}\,\phi_\#P^*(\cdot\mid s,a)|\le X$.
\begin{itemize}
\item[(i)] $\Psi^+_{k,h}$ dominates Bayes-UCBVI's quantile gap at every count, evaluated at calibrated AIF's own next-step values $V^+_{k,h+1}$. Hence, with probability at least $1-\delta/5$, simultaneously for all $k$ and $h$, calibrated AIF is optimistic for an exact abstraction, $V^+_{k,h}\ge\overline V^*_{\phi,h}$, and approximately optimistic in general, $V^+_{k,1}\ge\overline V^*_{\phi,1}-\epsilon_{\text{rep}}(\phi)$.
\item[(ii)] $\Psi^+_{k,h}(x,a)=\mathcal O\big(H\sqrt{\Gamma_\phi\ell/\bar n}+H\ell/\bar n\big)$, of the order $\bar n^{-1/2}$ that Lemma~\ref{lem:novelty-order}(ii) requires.
\item[(iii)] Let $L=\mathcal O(\ln(XAHT/\delta))$ be the logarithmic factor defined in Theorem~\ref{thm:ucbvi-fixed}. Let the \defterm{novelty budget} $\Gamma_T:=\sum_{k,h}\mathrm{Novelty}^k_h(x_{k,h},a_{k,h})$ be the novelty accrued along the played trajectory. It satisfies $\Gamma_T\le\Gamma_\phi XAH(1+\ln K)$. With probability at least $1-\delta$, $\mathfrak R_K\le\mathcal O(\epsilon_{\text{rep}}(\phi)K)+\mathcal O(H\sqrt{L\,T\,\Gamma_T}+H^{3/2}\sqrt{XATL})+\mathcal O(H^3X^2AL^2)$. In particular, $\mathfrak R_K=\mathcal O\big(\epsilon_{\text{rep}}(\phi)K\big)+\widetilde{\mathcal O}\big(H^{3/2}\sqrt{\Gamma_\phi XAT}\big)$ for $T\ge H^3X^3A$.
\end{itemize}
\end{theorem}
Calibration thus turns AIF's epistemic term into a valid width, but the bound of Theorem~\ref{thm:calibrated} exceeds Bayes-UCBVI's by $\sqrt{\Gamma_\phi}$, and its variance-aware bound for exact abstractions by $\sqrt{H\Gamma_\phi}$. The next section shows that calibrated information gain is advantageous over value-specific widths when the goals are chosen adaptively.

\section{Adaptive Goals}
\label{sec:goal-uniform}
\label{sec:adaptive-goals}

The bound of calibrated AIF exceeds that of Bayes-UCBVI because its width controls the deviation of the whole row $B$, whereas the value needs only the projection $B^\top V$. Controlling the whole row is useful when the value function changes with an adaptively chosen goal. Novelty does not depend on the goal, and its square root bounds the posterior deviation of every value function at once (Eq.~\eqref{eq:goal-uniform-posterior} in Appendix~\ref{app:adaptive-upper}). Under AIF's own uniform Dirichlet prior on $\mathcal X$, every category carries mass, and novelty is of the order $X/\bar n$, the squared radius of an $\ell_1$ confidence set \citep{jaksch2010near}. \emph{Calibrated AIF with a uniform prior} uses this prior, the width $\Psi^{\mathrm u}_{k,h}:=2\,\mathrm{span}(V_{k,h+1})\sqrt{\ell'\,\mathrm{Novelty}}$ with $\ell'=9+3\ln(2XAHK/\delta)$, and the clamp at $H-h+1$. It is goal-memoryless (Lemma~\ref{lem:membership}).

\emph{The Gaussian family.} The hard MDPs have horizon $2$, $A=2$, and $M$ start states, each with one decision. Action $b$ leads to a single absorbing state, and action $a$ leads to one of $Y=2m$ absorbing successors shared by all start states. At start state $i$, it reaches $y_{2j-1}$ with probability $(1+\varepsilon\tilde\theta^{(i)}_j)/(2m)$ and $y_{2j}$ with probability $(1-\varepsilon\tilde\theta^{(i)}_j)/(2m)$, with $\varepsilon=\sqrt{mM/(64K)}$, independent standard Gaussian vectors $\theta^{(1)},\dots,\theta^{(M)}$ in $\mathds R^m$ as the prior, and $\tilde\theta$ clipped to $[-1/(2\varepsilon),1/(2\varepsilon)]^m$. Hence $S=M+Y+1$ (Appendix~\ref{app:goal-uniform}). A goal pays $g\in[0,1]^Y$ at the successors and a value $v$ at the other absorbing state, and the agent must compare the expected payoff $P_ig$ of action $a$ at start state $i$ with $v$. In the following theorem, an adversary chooses goals and start states from the transition law and the history, and expectations include the prior.

\begin{theorem}[The price of adaptive goals]
\label{thm:price}
There are absolute constants $c,C>0$ such that the following holds for $M\ge1$, $Y=2m$ with $m\ge C$, and $K\ge CMm^2\ln^3(mK)$ with $M$ dividing $K$, on the Gaussian family with $S=M+Y+1$ states.
\begin{itemize}
\item[(a)] \uline{A lower bound for every agent.} Some adversary forces every agent, including one that uses past goals as data and knows the adversary's strategy, to $\mathds{E}\mathfrak R^{\mathrm{goal}}_K\ge cY^{1/4}\sqrt{MK/\ln Y}$.
\item[(b)] \uline{A larger lower bound for goal-memoryless agents.} If $Y\ge600\ln(3K)$, every goal-memoryless agent has $\mathds{E}\mathfrak R^{\mathrm{goal}}_K\ge\sqrt{MYK}/16000$ against some adversary.
\item[(c)] \uline{Both rates are attained.} On every MDP with an exact abstraction, with probability at least $1-\delta/2$ and simultaneously for all sequences of goals and start states, calibrated AIF with a uniform prior is optimistic for every goal, $V_{k,h}\ge V^*_h(\cdot\,;R_k)$. Against every adversary, with probability at least $1-\delta$, it has $\mathfrak R^{\mathrm{goal}}_K\le8H^{3/2}X\sqrt{A\ell'T}+2H\sqrt{2T\ln(2/\delta)}$, which is $\mathcal O(S\sqrt{K\ln(SK/\delta)})$ on the Gaussian family. Against every adversary, some agent that uses past goals as data and knows the adversary's strategy has expected regret at most $2Y^{1/4}\sqrt{MK\ln(1+16K)}$.
\item[(d)] \uline{Goals fixed in advance cost no extra regret.} Against goals and start states that an oblivious adversary fixes before the first episode, a single fixed goal included, UCBVI-CH without its clamp at the previous values has regret $\widetilde{\mathcal O}(\sqrt{SK})$ on every member of the family for $K\ge S^3$ (Proposition~\ref{prop:oblivious}).
\end{itemize}
\end{theorem}
For $M=Y$ and $K\ge S^3$, a fixed goal thus costs $\widetilde{\mathcal O}(\sqrt{SK})$, every agent incurs $\widetilde\Theta(S^{3/4}\sqrt K)$, and goal-memoryless agents incur $\widetilde\Theta(S\sqrt K)$.

Value-specific widths are not even optimistic for all goals at once, and Bayes-UCBVI's quantiles lock out under adaptive goals (Proposition~\ref{prop:lockout} in Appendix~\ref{app:adaptive-lockout}). A value-specific quantile is accurate for every fixed goal but not for a data-dependent goal, whereas the confidence set behind $\Psi^{\mathrm u}$ certifies whatever value the agent computes. In Proposition~\ref{prop:lockout}, Bayes-UCBVI keeps its pseudo-state prior, whereas part~(c) uses the uniform prior, so this contrast changes the prior as well as the width.

\emph{Goals as data.} AIF treats goals as preferences, on which it conditions but which it does not update. It therefore does not use past goals as data, and part~(b) applies to it and to its calibrated variant (Lemma~\ref{lem:membership}). By part~(c), calibrated AIF with a uniform prior attains the rate of part~(b) in $S$ and $K$. Calibration is thus what makes information gain useful under adaptive goals, and no goal-memoryless agent does better. The remaining gap between parts~(a) and~(b) is the value of using past goals as data. A goal chosen by an adversary that knows the transition law is an observation of that law, and AIF's generative model has no variable for it.

\section{Learned Abstractions}
\label{sec:learning-phi}
\label{sec:enactive}
\label{sec:selection}
\label{sec:learned}
\label{sec:structural}

We now let the agent learn $\phi$ from a candidate set $\Phi$ of abstractions, with a uniform prior, in the realizable case. The agent scores each candidate abstraction by how well it predicts a target.

\begin{assumption}[Realizable block MDP]
\label{assum:block}
For some $\phi^\dagger\in\Phi$, latent kernel $B^*$, and emission $\mathbf A^*$ with $\mathbf A^*(\cdot\mid x)$ supported on $\phi^{\dagger-1}(x)$, $P^*(s'\mid s,a)=B^*(\phi^\dagger(s')\mid\phi^\dagger(s),a)\,\mathbf A^*(s'\mid\phi^\dagger(s'))$.
The block MDP has \defterm{latent rewards} if moreover $R^*(s,a)=r^*(\phi^\dagger(s),a)$ for some $r^*:\mathcal X\times\mathcal A\to[0,1]$.
\end{assumption}
The value function is a target that all candidate abstractions share and that suffices for control. Let $\mathcal F_h:=\{(s,a)\mapsto g(\phi(s),a):\phi\in\Phi,\ g\in[0,H-h+1]^{\mathcal X\times\mathcal A}\}$ be the value functions that some candidate abstraction can express. AIF couples representation and action through a single objective, yet scores its abstraction by how well it predicts observations. We call an agent \emph{enactive} if the target that scores its abstraction is itself the value of its actions. Our enactive agent is an instance of the GOLF algorithm \citep{jin2021bellman} on $\mathcal F=\mathcal F_1\times\dots\times\mathcal F_H$. Before each episode, it keeps the functions whose squared Bellman residual exceeds that of the best predictor of the same target by at most a radius $\omega$. It then acts greedily for the most optimistic of them (Appendix~\ref{app:enactive}). Its residual is a self-prediction, since $f_h$ predicts the reward plus $f_{h+1}$ at the next state. Unlike a latent self-prediction, however, it compares $f_h$ with the best predictor of the same target. The candidate abstractions thus compete on a common target. Under a goal $R_k$ revealed before episode $k$, the agent \emph{relabels} its data, thereby recomputing its version space from all past transitions with the rewards of $R_k$.

\begin{theorem}[Value-based selection of the abstraction]
\label{thm:enactive}
Let Assumption~\ref{assum:block} hold with latent rewards and a fixed start state, let $C_{\mathrm e}:=\ln|\Phi|+XA\ln(HK/\delta)$, and let the enactive agent use the confidence radius $\omega=cH^2\ln(\mathcal NHK/\delta)$ on its squared Bellman residuals (Eq.~\eqref{eq:golf-radius} in Appendix~\ref{app:enactive}), where $\mathcal N$ is the covering number of $\mathcal F$ at scale $H/K$ in the supremum norm and $c$ is an absolute constant. With probability at least $1-\delta$, the following hold.
\begin{itemize}
\item[(a)] The enactive agent has $\mathfrak R_K=\mathcal O\big(H^2\sqrt{XAT\,C_{\mathrm e}\,(1+\ln K)}\big)$.
\item[(b)] Let a latent goal $R_k(s,a)=r_k(\phi^\dagger(s),a)$, with $r_k:\mathcal X\times\mathcal A\to[0,1]$, be revealed before each episode, possibly chosen by an adaptive adversary. The relabelling enactive agent has $\mathfrak R^{\mathrm{goal}}_K=\mathcal O\big(H^2\sqrt{XAT\,C_{\mathrm e}\,(1+\ln K)}\big)$, with $\mathfrak R^{\mathrm{goal}}_K$ as in Section~\ref{sec:goals} on the ground MDP.
\end{itemize}
\end{theorem}
Our contribution to part~(a) is to build the class $\mathcal F$ from the candidate abstractions and to show that, with latent rewards, it is closed under Bellman backups. The coverability analysis of \citet{xie2023role} therefore applies, and the observations enter only through $\ln|\Phi|$.

\emph{Other targets.} AIF predicts the next ground state by posterior averaging over the candidate abstractions, a consistent score (Proposition~\ref{prop:selection}), but its information gain about the emission directs its exploration (Appendix~\ref{app:aif-selection}). Given the decoder, the latent kernel and the rewards of a block MDP with $S$ states and latent rewards, AIF prefers a worse action that yields many distinct observations, with regret linear in $K$ whenever $S\ge2K+2$ (Proposition~\ref{prop:lure}). Removing the emission model, as joint-embedding predictive architectures and self-predictive RL \citep{schwarzer2021data,ni2024bridging} do, removes this information gain but changes the target. The decoder-free counterpart of AIF scores each candidate abstraction by the Bayesian likelihood $\Lambda^{\mathrm{lat}}_\phi$ of its own labels, and a \emph{latent-predictive} agent acts through a maximiser of this score, exploring by any rule. A coarse candidate abstraction that merges two states whose best actions differ can score higher than the true abstraction under every policy, and the labels that the agent generates by acting through it cannot refute it (Appendix~\ref{app:aif-selection}). Every latent-predictive agent therefore incurs regret $\Omega(HK)$, whereas GOLF on the same candidate set has regret $\widetilde{\mathcal O}(H^2\sqrt T)$ (Proposition~\ref{prop:self-confirming} in Appendix~\ref{app:aif-selection}).

\section{Conclusion}
\label{sec:conclusion}

For a fixed reward function, raw information gain of vanilla AIF incurs linear regret, and calibrated information gain does not improve on the bound attainable with visit counts. Under adaptive goals, calibrated information gain makes the agent optimistic for every goal at once, as an $\ell_1$ confidence set does, whereas the value-specific widths of Bayes-UCBVI lock out with regret $\Omega(K/\ln K)$. The price of adaptive goals is a factor $S^{1/4}$ for every agent and $\sqrt S$ for goal-memoryless agents such as AIF, up to logarithmic factors. For learned abstractions, predicting observations or the abstraction's own labels misleads exploration and selection, whereas the value function is a target that all candidate abstractions share. Calibrated information gain is therefore advantageous over value-specific confidence widths when the goals are chosen from the agent's own data, and not when the reward function is fixed. Open questions are whether calibrated AIF with Bayes-UCBVI's pseudo-state prior is optimistic for every goal, whether the price of adaptive goals holds with a fixed start state and for every horizon, whether it holds on every MDP rather than in expectation over a prior, and whether a single agent that uses past goals as data attains the factor $S^{1/4}$ without knowing the adversary's strategy. A further open problem concerns an agent that moves its own goals towards the attractors into which the MDP and its exploration rule drift it. Goals and policy then form a coupled dynamical system, and it is unknown whether this system converges and to which attractors. Whether physiological set points are defended targets or settling points of such coupled dynamics is itself debated in mathematical biology \citep{mrosovsky1990rheostasis,speakman2011set}. A two-time-scale analysis \citep{borkar1997stochastic} with goals adapting slower than the policy is a natural starting point.

\section{AI Disclosure}

The discovery of the untapped problem in the way it stands in the main question box in the introduction belongs completely to the authors. All the most creative moments of the investigation process have also been due to the deep and to-the-point involvement of the authors with their both broad field knowledge and their nuanced intuitions. That being said, the authors used frontier AI technologies in all steps of the work, sometimes as a companion and other times as the main workhorse. The eventual outcome benefited substantially from AI support, while being an original product of the authors.

The overall structure of the paper is engineered in detail by the authors and AI is used across several tens of times to revise the whole text end-to-end based on lengthy prompts created by the authors after their careful critical reading. These prompts contained structural improvements, new literature elements to be checked, new perspectives to be put on some of the existing content, and new concept pairs to be connected. Nearly all the abstract is the original text of the authors. The paragraph-level design of the introduction and a substantial amount of the text contained therein are native to the authors. The authors provided the AI agents a long list of papers known well to them, pointing to their precise pieces of content that are relevant for the work to be done. The follow-up literature survey text is then left to AI generation, though its post-hoc fact checking is done first by the authors and by alternative AI models afterwards.

The approaching angle to the problem, for example putting AIF and Bayes-UCBVI into the equal footing by defining an abstract MDP, belongs completely to the authors. Similarly, the authors gave several rounds of long instructions to the AI agents about what sort of theoretical claims to be sought and which proof techniques may lead to a successful solution. For example, a verbal (though coarse) description of Theorem 1 was provided to the AI agent as a very likely outcome expected from putting together certain pieces of a particular list of existing work. The counter-examples found for regret lower bounds came from AI agents, despite after long interactions where many intermediate strategic decisions were made by the authors based on their scrutiny and discussion. The idea of looking into adaptive goals emerged as a long and argumentative interaction of the first author with Claude Opus 5.5 high effort, where both sides made creative contributions to the final decision of investigating that realm. 

The authors used multiple AI agents in different roles at different stages of the work, that span Claude Opus 5.5 medium, high, and max effort levels mainly for content-related matters, GPT Sol 6.1 for auditing, and GLM 5.3 for narrowly defined text manipulation tasks. The authors made tedious and repeated checks on the whole spectrum of the final product and to the best of their efforts, the claims are sound.

This very AI disclosure text is written verbatim by the authors without any computer involvement whatsoever, including typo finding or text polishing, in order to set a clear reference for their own writing style. Potential grammatical or typographical errors should be attributed to this fact.

\vskip 0.2in
\bibliography{refs}

\begin{thebibliography}{89}
\providecommand{\natexlab}[1]{#1}
\providecommand{\url}[1]{\texttt{#1}}
\expandafter\ifx\csname urlstyle\endcsname\relax
  \providecommand{\doi}[1]{doi: #1}\else
  \providecommand{\doi}{doi: \begingroup \urlstyle{rm}\Url}\fi

\bibitem[Abel et~al.(2016)Abel, Hershkowitz, and Littman]{abel2016near}
David Abel, D.~Ellis Hershkowitz, and Michael~L. Littman.
\newblock Near optimal behavior via approximate state abstraction.
\newblock In \emph{Proceedings of the 33rd International Conference on Machine
  Learning (ICML)}, 2016.

\bibitem[Abel et~al.(2018)Abel, Arumugam, Lehnert, and Littman]{abel2018state}
David Abel, Dilip Arumugam, Lucas Lehnert, and Michael~L. Littman.
\newblock State abstractions for lifelong reinforcement learning.
\newblock In \emph{Proceedings of the 35th International Conference on Machine
  Learning (ICML)}, 2018.

\bibitem[Agarwal et~al.(2020)Agarwal, Kakade, Krishnamurthy, and
  Sun]{agarwal2020flambe}
Alekh Agarwal, Sham Kakade, Akshay Krishnamurthy, and Wen Sun.
\newblock {FLAMBE}: Structural complexity and representation learning of low
  rank {MDP}s.
\newblock In \emph{Advances in Neural Information Processing Systems}, 2020.

\bibitem[Arumugam and Van~Roy(2021)]{arumugam2021deciding}
Dilip Arumugam and Benjamin Van~Roy.
\newblock Deciding what to learn: A rate-distortion approach.
\newblock In \emph{Proceedings of the 38th International Conference on Machine
  Learning (ICML)}, 2021.

\bibitem[Arumugam and Van~Roy(2022)]{arumugam2022deciding}
Dilip Arumugam and Benjamin Van~Roy.
\newblock Deciding what to model: Value-equivalent sampling for reinforcement
  learning.
\newblock In \emph{Advances in Neural Information Processing Systems}, 2022.

\bibitem[Assran et~al.(2023)Assran, Duval, Misra, Bojanowski, Vincent, Rabbat,
  LeCun, and Ballas]{assran2023self}
Mahmoud Assran, Quentin Duval, Ishan Misra, Piotr Bojanowski, Pascal Vincent,
  Michael Rabbat, Yann LeCun, and Nicolas Ballas.
\newblock Self-supervised learning from images with a joint-embedding
  predictive architecture.
\newblock In \emph{Proceedings of the IEEE/CVF Conference on Computer Vision
  and Pattern Recognition (CVPR)}, pages 15619--15629, 2023.

\bibitem[Auer et~al.(2002)Auer, Cesa-Bianchi, and Fischer]{auer2002finite}
Peter Auer, Nicol{\`o} Cesa-Bianchi, and Paul Fischer.
\newblock Finite-time analysis of the multiarmed bandit problem.
\newblock \emph{Machine Learning}, 47\penalty0 (2--3):\penalty0 235--256, 2002.

\bibitem[Azar et~al.(2017)Azar, Osband, and Munos]{azar2017minimax}
Mohammad~Gheshlaghi Azar, Ian Osband, and R{\'e}mi Munos.
\newblock Minimax regret bounds for reinforcement learning.
\newblock In \emph{Proceedings of the 34th International Conference on Machine
  Learning (ICML)}, 2017.

\bibitem[Bacon et~al.(2017)Bacon, Harb, and Precup]{bacon2017optioncritic}
Pierre-Luc Bacon, Jean Harb, and Doina Precup.
\newblock The option-critic architecture.
\newblock In \emph{Proceedings of the 31st AAAI Conference on Artificial
  Intelligence}, 2017.

\bibitem[Borkar(1997)]{borkar1997stochastic}
Vivek~S. Borkar.
\newblock Stochastic approximation with two time scales.
\newblock \emph{Systems \& Control Letters}, 29\penalty0 (5):\penalty0
  291--294, 1997.

\bibitem[Boucheron et~al.(2013)Boucheron, Lugosi, and
  Massart]{boucheron2013concentration}
St{\'e}phane Boucheron, G{\'a}bor Lugosi, and Pascal Massart.
\newblock \emph{Concentration Inequalities: A Nonasymptotic Theory of
  Independence}.
\newblock Oxford University Press, 2013.

\bibitem[Brafman and Tennenholtz(2002)]{brafman2002rmax}
Ronen~I. Brafman and Moshe Tennenholtz.
\newblock {R-MAX} -- a general polynomial time algorithm for near-optimal
  reinforcement learning.
\newblock \emph{Journal of Machine Learning Research}, 3:\penalty0 213--231,
  2002.

\bibitem[Burda et~al.(2019)Burda, Edwards, Pathak, Storkey, Darrell, and
  Efros]{burda2019largescale}
Yuri Burda, Harri Edwards, Deepak Pathak, Amos Storkey, Trevor Darrell, and
  Alexei~A. Efros.
\newblock Large-scale study of curiosity-driven learning.
\newblock In \emph{International Conference on Learning Representations
  (ICLR)}, 2019.

\bibitem[Cesa-Bianchi and Lugosi(2006)]{cesabianchi2006prediction}
Nicol{\`o} Cesa-Bianchi and G{\'a}bor Lugosi.
\newblock \emph{Prediction, Learning, and Games}.
\newblock Cambridge University Press, 2006.

\bibitem[Cesa-Bianchi et~al.(2017)Cesa-Bianchi, Gentile, Lugosi, and
  Neu]{cesabianchi2017boltzmann}
Nicol{\`o} Cesa-Bianchi, Claudio Gentile, G{\'a}bor Lugosi, and Gergely Neu.
\newblock Boltzmann exploration done right.
\newblock In \emph{Advances in Neural Information Processing Systems}, 2017.

\bibitem[Colas et~al.(2022)Colas, Karch, Sigaud, and
  Oudeyer]{colas2022autotelic}
C{\'e}dric Colas, Tristan Karch, Olivier Sigaud, and Pierre-Yves Oudeyer.
\newblock Autotelic agents with intrinsically motivated goal-conditioned
  reinforcement learning: A short survey.
\newblock \emph{Journal of Artificial Intelligence Research}, 74:\penalty0
  1159--1199, 2022.

\bibitem[Cover and Thomas(2006)]{cover2006elements}
Thomas~M. Cover and Joy~A. Thomas.
\newblock \emph{Elements of Information Theory}.
\newblock Wiley, 2nd edition, 2006.

\bibitem[Da~Costa et~al.(2020)Da~Costa, Parr, Sajid, Veselic, Neacsu, and
  Friston]{dacosta2020synthesis}
Lancelot Da~Costa, Thomas Parr, Noor Sajid, Sebastijan Veselic, Victorita
  Neacsu, and Karl Friston.
\newblock Active inference on discrete state-spaces: A synthesis.
\newblock \emph{Journal of Mathematical Psychology}, 99:\penalty0 102447, 2020.

\bibitem[Da~Costa et~al.(2023)Da~Costa, Sajid, Parr, Friston, and
  Smith]{dacosta2023reward}
Lancelot Da~Costa, Noor Sajid, Thomas Parr, Karl Friston, and Ryan Smith.
\newblock Reward maximization through discrete active inference.
\newblock \emph{Neural Computation}, 35\penalty0 (5):\penalty0 807--852, 2023.

\bibitem[Domingues et~al.(2021)Domingues, M{\'e}nard, Kaufmann, and
  Valko]{domingues2021episodic}
Omar~Darwiche Domingues, Pierre M{\'e}nard, Emilie Kaufmann, and Michal Valko.
\newblock Episodic reinforcement learning in finite {MDP}s: Minimax lower
  bounds revisited.
\newblock In \emph{Proceedings of the 32nd International Conference on
  Algorithmic Learning Theory}, 2021.

\bibitem[Dong et~al.(2019)Dong, Van~Roy, and Zhou]{dong2019provably}
Shi Dong, Benjamin Van~Roy, and Zhengyuan Zhou.
\newblock Provably efficient reinforcement learning with aggregated states.
\newblock \emph{arXiv preprint arXiv:1912.06366}, 2019.

\bibitem[Du et~al.(2019)Du, Krishnamurthy, Jiang, Agarwal, Dud{\'i}k, and
  Langford]{du2019provably}
Simon~S. Du, Akshay Krishnamurthy, Nan Jiang, Alekh Agarwal, Miroslav
  Dud{\'i}k, and John Langford.
\newblock Provably efficient {RL} with rich observations via latent state
  decoding.
\newblock In \emph{Proceedings of the 36th International Conference on Machine
  Learning (ICML)}, 2019.

\bibitem[Dwork et~al.(2015)Dwork, Feldman, Hardt, Pitassi, Reingold, and
  Roth]{dwork2015preserving}
Cynthia Dwork, Vitaly Feldman, Moritz Hardt, Toniann Pitassi, Omer Reingold,
  and Aaron Roth.
\newblock Preserving statistical validity in adaptive data analysis.
\newblock In \emph{ACM Symposium on Theory of Computing}, 2015.

\bibitem[Efroni et~al.(2022)Efroni, Misra, Krishnamurthy, Agarwal, and
  Langford]{efroni2022provably}
Yonathan Efroni, Dipendra Misra, Akshay Krishnamurthy, Alekh Agarwal, and John
  Langford.
\newblock Provably filtering exogenous distractors using multistep inverse
  dynamics.
\newblock In \emph{International Conference on Learning Representations}, 2022.

\bibitem[Farahmand et~al.(2017)Farahmand, Barreto, and
  Nikovski]{farahmand2017value}
Amir-massoud Farahmand, Andr{\'e} Barreto, and Daniel Nikovski.
\newblock Value-aware loss function for model-based reinforcement learning.
\newblock In \emph{International Conference on Artificial Intelligence and
  Statistics}, 2017.

\bibitem[Freedman(1975)]{freedman1975tail}
David~A. Freedman.
\newblock On tail probabilities for martingales.
\newblock \emph{The Annals of Probability}, 3\penalty0 (1):\penalty0 100--118,
  1975.
\newblock \doi{10.1214/aop/1176996452}.

\bibitem[Friston et~al.(2017)Friston, FitzGerald, Rigoli, Schwartenbeck, and
  Pezzulo]{friston2017active}
Karl Friston, Thomas FitzGerald, Francesco Rigoli, Philipp Schwartenbeck, and
  Giovanni Pezzulo.
\newblock Active inference: A process theory.
\newblock \emph{Neural Computation}, 29\penalty0 (1):\penalty0 1--49, 2017.

\bibitem[Friston et~al.(2021)Friston, Da~Costa, Hafner, Hesp, and
  Parr]{friston2021sophisticated}
Karl Friston, Lancelot Da~Costa, Danijar Hafner, Casper Hesp, and Thomas Parr.
\newblock Sophisticated inference.
\newblock \emph{Neural Computation}, 33\penalty0 (3):\penalty0 713--763, 2021.

\bibitem[Fruit et~al.(2018)Fruit, Pirotta, Lazaric, and
  Ortner]{fruit2018efficient}
Ronan Fruit, Matteo Pirotta, Alessandro Lazaric, and Ronald Ortner.
\newblock Efficient bias-span-constrained exploration-exploitation in
  reinforcement learning.
\newblock In \emph{Proceedings of the 35th International Conference on Machine
  Learning (ICML)}, 2018.

\bibitem[Gelada et~al.(2019)Gelada, Kumar, Buckman, Nachum, and
  Bellemare]{gelada2019deepmdp}
Carles Gelada, Saurabh Kumar, Jacob Buckman, Ofir Nachum, and Marc~G.
  Bellemare.
\newblock {DeepMDP}: Learning continuous latent space models for representation
  learning.
\newblock In \emph{Proceedings of the 36th International Conference on Machine
  Learning (ICML)}, 2019.

\bibitem[Givan et~al.(2003)Givan, Dean, and Greig]{givan2003equivalence}
Robert Givan, Thomas Dean, and Matthew Greig.
\newblock Equivalence notions and model minimization in {M}arkov decision
  processes.
\newblock \emph{Artificial Intelligence}, 147\penalty0 (1-2):\penalty0
  163--223, 2003.

\bibitem[Grimm et~al.(2020)Grimm, Barreto, Singh, and Silver]{grimm2020value}
Christopher Grimm, Andr{\'e} Barreto, Satinder Singh, and David Silver.
\newblock The value equivalence principle for model-based reinforcement
  learning.
\newblock In \emph{Advances in Neural Information Processing Systems}, 2020.

\bibitem[Han et~al.(2015)Han, Jiao, and Weissman]{han2015minimax}
Yanjun Han, Jiantao Jiao, and Tsachy Weissman.
\newblock Minimax estimation of discrete distributions under $\ell_1$ loss.
\newblock \emph{IEEE Transactions on Information Theory}, 61\penalty0
  (11):\penalty0 6343--6354, 2015.

\bibitem[Hardt and Ullman(2014)]{hardt2014preventing}
Moritz Hardt and Jonathan Ullman.
\newblock Preventing false discovery in interactive data analysis is hard.
\newblock In \emph{IEEE Symposium on Foundations of Computer Science}, 2014.

\bibitem[Houthooft et~al.(2016)Houthooft, Chen, Duan, Schulman, De~Turck, and
  Abbeel]{houthooft2016vime}
Rein Houthooft, Xi~Chen, Yan Duan, John Schulman, Filip De~Turck, and Pieter
  Abbeel.
\newblock {VIME}: Variational information maximizing exploration.
\newblock In \emph{Advances in Neural Information Processing Systems
  (NeurIPS)}, volume~29, 2016.

\bibitem[Hutto and Myin(2013)]{hutto2013radicalizing}
Daniel~D. Hutto and Erik Myin.
\newblock \emph{Radicalizing Enactivism: Basic Minds without Content}.
\newblock MIT Press, 2013.

\bibitem[Jaksch et~al.(2010)Jaksch, Ortner, and Auer]{jaksch2010near}
Thomas Jaksch, Ronald Ortner, and Peter Auer.
\newblock Near-optimal regret bounds for reinforcement learning.
\newblock \emph{Journal of Machine Learning Research}, 11:\penalty0 1563--1600,
  2010.

\bibitem[Jiang et~al.(2015)Jiang, Kulesza, and Singh]{jiang2015abstraction}
Nan Jiang, Alex Kulesza, and Satinder Singh.
\newblock Abstraction selection in model-based reinforcement learning.
\newblock In \emph{Proceedings of the 32nd International Conference on Machine
  Learning}, 2015.

\bibitem[Jiang et~al.(2017)Jiang, Krishnamurthy, Agarwal, Langford, and
  Schapire]{jiang2017contextual}
Nan Jiang, Akshay Krishnamurthy, Alekh Agarwal, John Langford, and Robert~E.
  Schapire.
\newblock Contextual decision processes with low {B}ellman rank are
  {PAC}-learnable.
\newblock In \emph{Proceedings of the 34th International Conference on Machine
  Learning (ICML)}, 2017.

\bibitem[Jin et~al.(2018)Jin, Allen-Zhu, Bubeck, and Jordan]{jin2018qlearning}
Chi Jin, Zeyuan Allen-Zhu, Sebastien Bubeck, and Michael~I. Jordan.
\newblock Is {Q}-learning provably efficient?
\newblock In \emph{Advances in Neural Information Processing Systems}, 2018.

\bibitem[Jin et~al.(2020)Jin, Krishnamurthy, Simchowitz, and
  Yu]{jin2020rewardfree}
Chi Jin, Akshay Krishnamurthy, Max Simchowitz, and Tiancheng Yu.
\newblock Reward-free exploration for reinforcement learning.
\newblock In \emph{Proceedings of the 37th International Conference on Machine
  Learning (ICML)}, pages 4870--4879, 2020.

\bibitem[Jin et~al.(2021)Jin, Liu, and Miryoosefi]{jin2021bellman}
Chi Jin, Qinghua Liu, and Sobhan Miryoosefi.
\newblock Bellman eluder dimension: New rich classes of {RL} problems, and
  sample-efficient algorithms.
\newblock In \emph{Advances in Neural Information Processing Systems}, 2021.

\bibitem[Kakade and Langford(2002)]{kakade2002approximately}
Sham Kakade and John Langford.
\newblock Approximately optimal approximate reinforcement learning.
\newblock In \emph{Proceedings of the Nineteenth International Conference on
  Machine Learning (ICML)}, 2002.

\bibitem[Kaufmann et~al.(2012)Kaufmann, Capp{\'e}, and
  Garivier]{kaufmann2012bayesian}
Emilie Kaufmann, Olivier Capp{\'e}, and Aur{\'e}lien Garivier.
\newblock On {B}ayesian upper confidence bounds for bandit problems.
\newblock In \emph{Proceedings of the 15th International Conference on
  Artificial Intelligence and Statistics (AISTATS)}, 2012.

\bibitem[Kolter and Ng(2009)]{kolter2009near}
J.~Zico Kolter and Andrew~Y. Ng.
\newblock Near-{B}ayesian exploration in polynomial time.
\newblock In \emph{International Conference on Machine Learning}, 2009.

\bibitem[LeCun(2022)]{lecun2022path}
Yann LeCun.
\newblock A path towards autonomous machine intelligence.
\newblock OpenReview preprint, version 0.9.2, 2022.

\bibitem[Levine(2018)]{levine2018reinforcement}
Sergey Levine.
\newblock Reinforcement learning and control as probabilistic inference:
  Tutorial and review.
\newblock \emph{arXiv preprint arXiv:1805.00909}, 2018.

\bibitem[Levy et~al.(2023)Levy, Cohen, Cassel, and Mansour]{levy2023efficient}
Orin Levy, Alon Cohen, Asaf Cassel, and Yishay Mansour.
\newblock Efficient rate optimal regret for adversarial contextual {MDP}s using
  online function approximation.
\newblock In \emph{Proceedings of the 40th International Conference on Machine
  Learning (ICML)}, 2023.

\bibitem[Li et~al.(2024)Li, Yan, Chen, and Fan]{li2024minimax}
Gen Li, Yuling Yan, Yuxin Chen, and Jianqing Fan.
\newblock Minimax-optimal reward-agnostic exploration in reinforcement
  learning.
\newblock In \emph{Proceedings of the 37th Conference on Learning Theory
  (COLT)}, pages 3431--3436, 2024.

\bibitem[Li et~al.(2006)Li, Walsh, and Littman]{li2006towards}
Lihong Li, Thomas~J. Walsh, and Michael~L. Littman.
\newblock Towards a unified theory of state abstraction for {MDP}s.
\newblock In \emph{Ninth International Symposium on Artificial Intelligence and
  Mathematics (ISAIM)}, 2006.

\bibitem[Li et~al.(2026{\natexlab{a}})Li, Tsuchiya, and
  Yamanishi]{li2026policy}
Mingyi Li, Taira Tsuchiya, and Kenji Yamanishi.
\newblock Policy optimization achieves data-dependent regret bounds in {MDP}s
  with unknown transitions.
\newblock \emph{arXiv preprint arXiv:2606.31769}, 2026{\natexlab{a}}.

\bibitem[Li et~al.(2026{\natexlab{b}})Li, Parashar, Zhou, and
  Fan]{li2026curiosity}
Yingke Li, Anjali Parashar, Enlu Zhou, and Chuchu Fan.
\newblock Curiosity is knowledge: Self-consistent learning and no-regret
  optimization with active inference.
\newblock \emph{arXiv preprint arXiv:2602.06029}, 2026{\natexlab{b}}.

\bibitem[Lu and Van~Roy(2019)]{lu2019information}
Xiuyuan Lu and Benjamin Van~Roy.
\newblock Information-theoretic confidence bounds for reinforcement learning.
\newblock In \emph{Advances in Neural Information Processing Systems}, 2019.

\bibitem[Ma and Lee(2026)]{ma2026eubrl}
Jianfei Ma and Wee~Sun Lee.
\newblock {EUBRL}: Epistemic uncertainty directed {B}ayesian reinforcement
  learning.
\newblock In \emph{International Conference on Learning Representations
  (ICLR)}, 2026.

\bibitem[Maillard et~al.(2011)Maillard, Ryabko, and
  Munos]{maillard2011selecting}
Odalric-Ambrym Maillard, Daniil Ryabko, and R{\'e}mi Munos.
\newblock Selecting the state-representation in reinforcement learning.
\newblock In \emph{Advances in Neural Information Processing Systems
  (NeurIPS)}, 2011.

\bibitem[Maillard et~al.(2013)Maillard, Nguyen, Ortner, and
  Ryabko]{maillard2013optimal}
Odalric-Ambrym Maillard, Phuong Nguyen, Ronald Ortner, and Daniil Ryabko.
\newblock Optimal regret bounds for selecting the state representation in
  reinforcement learning.
\newblock In \emph{Proceedings of the 30th International Conference on Machine
  Learning (ICML)}, 2013.

\bibitem[McDiarmid(1989)]{mcdiarmid1989method}
Colin McDiarmid.
\newblock On the method of bounded differences.
\newblock In \emph{Surveys in Combinatorics, 1989}, volume 141 of \emph{London
  Mathematical Society Lecture Note Series}, pages 148--188. Cambridge
  University Press, 1989.

\bibitem[Millidge et~al.(2020)Millidge, Tschantz, Seth, and
  Buckley]{millidge2021relationship}
Beren Millidge, Alexander Tschantz, Anil~K. Seth, and Christopher~L. Buckley.
\newblock On the relationship between active inference and control as
  inference.
\newblock In \emph{Active Inference: First International Workshop (IWAI 2020)},
  2020.

\bibitem[Misra et~al.(2020)Misra, Henaff, Krishnamurthy, and
  Langford]{misra2020kinematic}
Dipendra Misra, Mikael Henaff, Akshay Krishnamurthy, and John Langford.
\newblock Kinematic state abstraction and provably efficient rich-observation
  reinforcement learning.
\newblock In \emph{Proceedings of the 37th International Conference on Machine
  Learning (ICML)}, 2020.

\bibitem[Modi and Tewari(2020)]{modi2020noregret}
Aditya Modi and Ambuj Tewari.
\newblock No-regret exploration in contextual reinforcement learning.
\newblock In \emph{Proceedings of the 36th Conference on Uncertainty in
  Artificial Intelligence (UAI)}, 2020.

\bibitem[Mrosovsky(1990)]{mrosovsky1990rheostasis}
Nicholas Mrosovsky.
\newblock \emph{Rheostasis: The Physiology of Change}.
\newblock Oxford University Press, 1990.

\bibitem[Ni et~al.(2024)Ni, Eysenbach, Seyedsalehi, Ma, Gehring, Mahajan, and
  Bacon]{ni2024bridging}
Tianwei Ni, Benjamin Eysenbach, Erfan Seyedsalehi, Michel Ma, Clement Gehring,
  Aditya Mahajan, and Pierre-Luc Bacon.
\newblock Bridging state and history representations: Understanding
  self-predictive {RL}.
\newblock In \emph{International Conference on Learning Representations
  (ICLR)}, 2024.

\bibitem[Nuijten et~al.(2026)Nuijten, Lukashchuk, van~de Laar, and
  de~Vries]{nuijten2026whattype}
Wouter W.~L. Nuijten, Mykola Lukashchuk, Thijs van~de Laar, and Bert de~Vries.
\newblock What type of inference is active inference?
\newblock In \emph{Proceedings of the 42nd Conference on Uncertainty in
  Artificial Intelligence (UAI)}, volume 337 of \emph{Proceedings of Machine
  Learning Research}, pages 5003--5039, 2026.

\bibitem[O'Donoghue(2021)]{odonoghue2021variational}
Brendan O'Donoghue.
\newblock Variational {B}ayesian reinforcement learning with regret bounds.
\newblock In \emph{Advances in Neural Information Processing Systems}, 2021.

\bibitem[O'Donoghue et~al.(2020)O'Donoghue, Osband, and
  Ionescu]{odonoghue2020making}
Brendan O'Donoghue, Ian Osband, and Catalin Ionescu.
\newblock Making sense of reinforcement learning and probabilistic inference.
\newblock In \emph{International Conference on Learning Representations}, 2020.

\bibitem[Ortner et~al.(2014)Ortner, Maillard, and Ryabko]{ortner2014selecting}
Ronald Ortner, Odalric-Ambrym Maillard, and Daniil Ryabko.
\newblock Selecting near-optimal approximate state representations in
  reinforcement learning.
\newblock In \emph{Proceedings of the 25th International Conference on
  Algorithmic Learning Theory}, 2014.

\bibitem[Osband and Van~Roy(2016)]{osband2016lower}
Ian Osband and Benjamin Van~Roy.
\newblock On lower bounds for regret in reinforcement learning.
\newblock Technical Report 1608.02732, arXiv, 2016.

\bibitem[Paninski(2008)]{paninski2008coincidence}
Liam Paninski.
\newblock A coincidence-based test for uniformity given very sparsely sampled
  discrete data.
\newblock \emph{IEEE Transactions on Information Theory}, 54\penalty0
  (10):\penalty0 4750--4755, 2008.

\bibitem[Pezzulo et~al.(2015)Pezzulo, Rigoli, and Friston]{pezzulo2015active}
Giovanni Pezzulo, Francesco Rigoli, and Karl Friston.
\newblock Active inference, homeostatic regulation and adaptive behavioural
  control.
\newblock \emph{Progress in Neurobiology}, 134:\penalty0 17--35, 2015.

\bibitem[Ridel and Peled-Cohen(2026)]{ridel2026improved}
Oran Ridel and Alon Peled-Cohen.
\newblock Improved bounds for reward-agnostic and reward-free exploration.
\newblock In \emph{Proceedings of the 43rd International Conference on Machine
  Learning (ICML)}, 2026.

\bibitem[Rosenberg and Mansour(2019)]{rosenberg2019online}
Aviv Rosenberg and Yishay Mansour.
\newblock Online convex optimization in adversarial {M}arkov decision
  processes.
\newblock In \emph{Proceedings of the 36th International Conference on Machine
  Learning (ICML)}, 2019.

\bibitem[Russo and Van~Roy(2018)]{russo2018learning}
Daniel Russo and Benjamin Van~Roy.
\newblock Learning to optimize via information-directed sampling.
\newblock \emph{Operations Research}, 66\penalty0 (1):\penalty0 230--252, 2018.

\bibitem[Schwartenbeck et~al.(2019)Schwartenbeck, Passecker, Hauser,
  FitzGerald, Kronbichler, and Friston]{schwartenbeck2019computational}
Philipp Schwartenbeck, Johannes Passecker, Tobias~U. Hauser, Thomas H.~B.
  FitzGerald, Martin Kronbichler, and Karl~J. Friston.
\newblock Computational mechanisms of curiosity and goal-directed exploration.
\newblock \emph{eLife}, 8:\penalty0 e41703, 2019.

\bibitem[Schwarzer et~al.(2021)Schwarzer, Anand, Goel, Hjelm, Courville, and
  Bachman]{schwarzer2021data}
Max Schwarzer, Ankesh Anand, Rishab Goel, R.~Devon Hjelm, Aaron Courville, and
  Philip Bachman.
\newblock Data-efficient reinforcement learning with self-predictive
  representations.
\newblock In \emph{International Conference on Learning Representations
  (ICLR)}, 2021.

\bibitem[Sekar et~al.(2020)Sekar, Rybkin, Daniilidis, Abbeel, Hafner, and
  Pathak]{sekar2020planning}
Ramanan Sekar, Oleh Rybkin, Kostas Daniilidis, Pieter Abbeel, Danijar Hafner,
  and Deepak Pathak.
\newblock Planning to explore via self-supervised world models.
\newblock In \emph{Proceedings of the 37th International Conference on Machine
  Learning (ICML)}, 2020.

\bibitem[Speakman et~al.(2011)Speakman, Levitsky, Allison, Bray, de~Castro,
  Clegg, Clapham, Dulloo, Gruer, Haw, Hebebrand, Hetherington, Higgs, Jebb,
  Loos, Luckman, Luke, Mohammed-Ali, O'Rahilly, Pereira, Perusse, Robinson,
  Rolls, Symonds, and Westerterp-Plantenga]{speakman2011set}
John~R. Speakman, David~A. Levitsky, David~B. Allison, Molly~S. Bray, John~M.
  de~Castro, Deborah~J. Clegg, John~C. Clapham, Abdul~G. Dulloo, Laurence
  Gruer, Sally Haw, Johannes Hebebrand, Marion~M. Hetherington, Susanne Higgs,
  Susan~A. Jebb, Ruth J.~F. Loos, Simon Luckman, Amy Luke, Vidya Mohammed-Ali,
  Stephen O'Rahilly, Mark Pereira, Louis Perusse, Tom~N. Robinson, Barbara
  Rolls, Michael~E. Symonds, and Margriet~S. Westerterp-Plantenga.
\newblock Set points, settling points and some alternative models: Theoretical
  options to understand how genes and environments combine to regulate body
  adiposity.
\newblock \emph{Disease Models \& Mechanisms}, 4\penalty0 (6):\penalty0
  733--745, 2011.

\bibitem[Srinivas et~al.(2010)Srinivas, Krause, Kakade, and
  Seeger]{srinivas2010gaussian}
Niranjan Srinivas, Andreas Krause, Sham~M. Kakade, and Matthias Seeger.
\newblock {G}aussian process optimization in the bandit setting: No regret and
  experimental design.
\newblock In \emph{International Conference on Machine Learning}, 2010.

\bibitem[Starre et~al.(2023)Starre, Loog, Congeduti, and
  Oliehoek]{starre2023analysis}
Rolf A.~N. Starre, Marco Loog, Elena Congeduti, and Frans~A. Oliehoek.
\newblock An analysis of model-based reinforcement learning from abstracted
  observations.
\newblock \emph{Transactions on Machine Learning Research}, 2023.

\bibitem[Steinke and Ullman(2015)]{steinke2015interactive}
Thomas Steinke and Jonathan Ullman.
\newblock Interactive fingerprinting codes and the hardness of preventing false
  discovery.
\newblock In \emph{Conference on Learning Theory}, 2015.

\bibitem[Sterling(2012)]{sterling2012allostasis}
Peter Sterling.
\newblock Allostasis: A model of predictive regulation.
\newblock \emph{Physiology \& Behavior}, 106\penalty0 (1):\penalty0 5--15,
  2012.

\bibitem[Sutton et~al.(1999)Sutton, Precup, and Singh]{sutton1999between}
Richard~S. Sutton, Doina Precup, and Satinder Singh.
\newblock Between {MDP}s and semi-{MDP}s: A framework for temporal abstraction
  in reinforcement learning.
\newblock \emph{Artificial Intelligence}, 112\penalty0 (1-2):\penalty0
  181--211, 1999.

\bibitem[Szita and L{\H o}rincz(2008)]{szita2008many}
Istv{\'a}n Szita and Andr{\'a}s L{\H o}rincz.
\newblock The many faces of optimism: A unifying approach.
\newblock In \emph{Proceedings of the 25th International Conference on Machine
  Learning (ICML)}, 2008.

\bibitem[Tarbouriech et~al.(2023)Tarbouriech, Lattimore, and
  O'Donoghue]{tarbouriech2023probabilistic}
Jean Tarbouriech, Tor Lattimore, and Brendan O'Donoghue.
\newblock Probabilistic inference in reinforcement learning done right.
\newblock In \emph{Advances in Neural Information Processing Systems}, 2023.

\bibitem[Tiapkin et~al.(2022)Tiapkin, Belomestny, Moulines, Naumov, Samsonov,
  Tang, Valko, and Menard]{tiapkin2022dirichlet}
Daniil Tiapkin, Denis Belomestny, Eric Moulines, Alexey Naumov, Sergey
  Samsonov, Yunhao Tang, Michal Valko, and Pierre Menard.
\newblock From {D}irichlet to {R}ubin: Optimistic exploration in {RL} without
  bonuses.
\newblock In \emph{Proceedings of the 39th International Conference on Machine
  Learning (ICML)}, 2022.

\bibitem[Tiapkin et~al.(2025)Tiapkin, Chzhen, and Stoltz]{tiapkin2025narrowing}
Daniil Tiapkin, Evgenii Chzhen, and Gilles Stoltz.
\newblock Narrowing the gap between adversarial and stochastic {MDP}s via
  policy optimization.
\newblock In \emph{Proceedings of the 28th International Conference on
  Artificial Intelligence and Statistics (AISTATS)}, volume 258 of
  \emph{Proceedings of Machine Learning Research}, pages 3331--3339, 2025.

\bibitem[Toussaint(2009)]{toussaint2009robot}
Marc Toussaint.
\newblock Robot trajectory optimization using approximate inference.
\newblock In \emph{Proceedings of the 26th International Conference on Machine
  Learning (ICML)}, 2009.

\bibitem[Varela et~al.(1991)Varela, Thompson, and Rosch]{varela1991embodied}
Francisco~J. Varela, Evan Thompson, and Eleanor Rosch.
\newblock \emph{The Embodied Mind: Cognitive Science and Human Experience}.
\newblock MIT Press, Cambridge, MA, 1991.

\bibitem[Wei(2024)]{wei2024value}
Ran Wei.
\newblock Value of information and reward specification in active inference and
  {POMDP}s.
\newblock In \emph{NeurIPS 2024 Workshop on NeuroAI}, 2024.

\bibitem[Xie et~al.(2023)Xie, Foster, Bai, Jiang, and Kakade]{xie2023role}
Tengyang Xie, Dylan~J. Foster, Yu~Bai, Nan Jiang, and Sham~M. Kakade.
\newblock The role of coverage in online reinforcement learning.
\newblock In \emph{International Conference on Learning Representations}, 2023.

\end{thebibliography}

\clearpage
\appendix
\crefalias{section}{appendix}
\crefalias{subsection}{subappendix}

\section{Notation}
\label{app:notation-table}

Table~\ref{tab:notation} collects the notation that the main text and several appendices share. A few letters denote local quantities inside a single proof, which defines them there.

\begin{longtable}{@{}>{\raggedright\arraybackslash}p{3.4cm}>{\raggedright\arraybackslash}p{11.2cm}@{}}
\caption{Notation.}\label{tab:notation}\\
\toprule
Symbol & Meaning \\
\midrule
\endfirsthead
\toprule
Symbol & Meaning \\
\midrule
\endhead
\bottomrule
\endfoot
\multicolumn{2}{@{}l}{\emph{MDP and regret}}\\
$\mathcal S,\mathcal A$; $S,A$ & ground state space and action space; their cardinalities \\
$P^*,R^*,H$ & true transition kernel, deterministic reward $R^*:\mathcal S\times\mathcal A\to[0,1]$ fixed across episodes, horizon \\
$K$, $T=HK$ & number of episodes, number of steps \\
$k$, $h$, $\tau$ & episode index, step index within an episode, step index within a rollout \\
$s^k_h$, $a^k_h$; $s_{k,h}$, $a_{k,h}$ & state and action at step $h$ of episode $k$ (both notations are used) \\
$\pi^k$ & policy of episode $k$ \\
$V^\pi_h,Q^\pi_h$; $V^*_h,Q^*_h$ & value functions of $\pi$; optimal value functions \\
$\mathfrak R_K$ & regret over $K$ episodes \\
$\mathcal T_h$, $d^\pi_h$ & Bellman operator of step $h$; occupancy of policy $\pi$ at step $h$ \\
\midrule
\multicolumn{2}{@{}l}{\emph{Abstraction}}\\
$\phi:\mathcal S\to\mathcal X$; $X=|\mathcal X|$ & abstraction and abstract state space; its cardinality \\
$\phi^{-1}(x)$ & block of $x$, the ground states that $\phi$ merges into $x$ \\
$\phi_\#$ & pushforward through $\phi$; $\phi_\#P^*(\cdot\mid s,a)$ is the law of the next abstract state \\
$W(s\mid x)$ & weighting of block $\phi^{-1}(x)$ that defines the abstract MDP \\
$\overline{\mathcal M}_\phi,\overline P_\phi,\overline R_\phi$ & abstract MDP of Eq.~\eqref{eq:abstract-mdp}, its kernel and its reward \\
$\overline V^\pi_{\phi,h},\overline V^*_{\phi,h}$ & value functions of $\overline{\mathcal M}_\phi$ \\
$\eta_R,\eta_P$; $\epsilon_{\text{rep}}$ & within-block spreads of reward and transitions; representation scale (Definition~\ref{def:spread}) \\
exact & $\epsilon_{\text{rep}}(\phi)=0$ \\
$\widehat R^k_\phi$, $N_k(x,a)$ & empirical reward estimate (Eq.~\eqref{eq:rhat}); number of visits to $(x,a)$ before episode $k$ \\
$\Gamma_\phi$ & branching number (Theorem~\ref{thm:calibrated}), the abstract counterpart of the maximal support size of \citet{fruit2018efficient} \\
\midrule
\multicolumn{2}{@{}l}{\emph{Posteriors and widths}}\\
$\rho^k_h(x,a)$ & Dirichlet posterior of the row $(x,a,h)$ before episode $k$ \\
$\widehat P^k_h(\cdot\mid x,a)$ & its mean \\
$n$, $\bar n=\bar n^k_h(x,a)$ & visit count of a row; visit count plus $n_0$ \\
$n_0$, $r_0$, $x_{\mathrm E}$ ($s_{\mathrm E}$) & pseudo-counts, pseudo-state reward, pseudo-state of Bayes-UCBVI's prior (Definition~\ref{def:dirichlet-ground}) \\
$\kappa^k_h$, $\ell=\ell^k_h$ & quantile order of Bayes-UCBVI; $\ell=\ln(1/(1-\kappa))$ \\
$\Delta^k_h$ & quantile gap, Bayes-UCBVI's width (Lemma~\ref{lem:tiapkin-quantile}) \\
$\widetilde P_n$ & empirical law of the first $n$ successors of a row \\
$\mathrm{span}(f)$ & $\max f-\min f$, over all categories, the pseudo-state included \\
$\mathrm{Novelty}(x,a)$ & information gain about the row, Eq.~\eqref{eq:novelty-def} \\
$\Gamma_T$ & novelty budget (Theorem~\ref{thm:calibrated}) \\
$\Psi$, $\Psi^+$, $\Psi^{\mathrm u}$ & epistemic term of Eq.~\eqref{eq:psi-slot}; calibrated width (Eq.~\eqref{eq:calibrated}); calibrated width under a uniform prior (Section~\ref{sec:goal-uniform}) \\
$\lambda$, $q$ & weight and exponent of an information-driven agent \\
$\delta$ & failure probability \\
\midrule
\multicolumn{2}{@{}l}{\emph{Active Inference and CaI}}\\
$\mathbf A_\phi$, $B_\phi$ & emission matrix and transition matrix of AIF's generative model \\
$\mathcal G(\pi)$ & expected free energy of policy $\pi$, Eq.~\eqref{eq:efe-canonical} \\
$\mathcal Q$ & the agent's approximate posterior and the predictive distributions that it induces \\
$\widetilde P$ & preference distribution over outcomes \\
$\beta$ & precision \\
\midrule
\multicolumn{2}{@{}l}{\emph{Goals}}\\
$R_k$, $x^k_1$ & goal (episode-specific reward function) and start state of episode $k$ \\
$\mathfrak R^{\mathrm{goal}}_K$ & regret against the goals \\
records; $U_{x,j}$ & successors observed at a step, state and action; the variables of the last-step rule of goal-memoryless agents (Definition~\ref{def:goal-memoryless}) \\
$M$, $Y=2m$ & numbers of start states and of successors in the Gaussian family \\
$\theta^{(i)}\sim N(0,I_m)$ & independent standard Gaussian vectors in $\mathds R^m$; $\tilde\theta$ is the clipped vector \\
$\varepsilon$ & scale of the perturbations in the Gaussian and hypercube families \\
$g$, $v$ & payoff of a goal at the successors and at the other absorbing state \\
$\mathcal U$ & uniform law \\
\midrule
\multicolumn{2}{@{}l}{\emph{Learned abstractions}}\\
$\Phi$ & candidate set of abstractions \\
$\phi^\dagger$, $B^*$, $\mathbf A^*$ & true decoder, latent kernel and emission of a block MDP (Assumption~\ref{assum:block}) \\
$\mathcal F_h$, $C_{\mathrm e}$ & value class of the enactive agent; complexity term of Theorem~\ref{thm:enactive} \\
$\Lambda_\phi$, $\Lambda^{\mathrm{lat}}_\phi$ & marginal likelihood of candidate $\phi$ for the observations; for its own labels \\
\midrule
\multicolumn{2}{@{}l}{\emph{Information and probability}}\\
$\mathds H$, $I$, $D_{\mathrm{KL}}$ & entropy, mutual information, Kullback--Leibler divergence, in nats \\
$\mathds E$, $\mathds V$, $\mathds P$ & expectation, variance, probability \\
$\mathrm{Dir}$, $\mathrm{Cat}$, $\mathrm{Ber}$, $\mathrm{Bin}$ & Dirichlet, categorical, Bernoulli and binomial laws \\
$f\equiv c$ & $f(y)=c$ for every argument $y$ \\
$\widetilde{\mathcal O}$ & $\mathcal O$ up to factors polylogarithmic in $HSXAT/\delta$ \\
\end{longtable}

\section{Conditions of Table 1}
\label{app:notation}

The bounds of Table~\ref{tab:regret} show only the dependence on the quantities named in its caption, and the cited results state the remaining constants. $\widetilde{\mathcal O}$ hides logarithmic factors but no polynomial lower-order terms, which the full bounds contain. Upper bounds hold with probability $1-\delta$, except the upper bound of row~10, which bounds the expected regret under a prior. Lower bounds hold in expectation, rows~9 and~10 also over a prior, except row~6, which holds deterministically, and rows~12 and~16, which hold with probability $1-e^{-(X-1)/32}$ and $1-\delta$. In rows~1, 6, 9 and~14, the hard MDP depends on $K$. These rows are therefore minimax statements.

\emph{Fixed reward function.} Row~1 takes $A=2$, $H\ge4$ and $K\ge X$. The leading terms of rows~2 and~3 hold for $T\ge H^4X^3A$ and $T\ge H^3X^3A$ (Theorem~\ref{thm:ucbvi-fixed}). On the ground states, UCBVI-BF has $\widetilde{\mathcal O}(\sqrt{HSAT})$ and Bayes-UCBVI $\widetilde{\mathcal O}(H\sqrt{SAT})$, both minimax for their classes (Theorems~\ref{thm:ucbvi-ground} and~\ref{thm:bayesucbvi-ground}). Row~4 takes $\beta$ and $r_0$ independent of $K$, and it holds for every $n_0$. Rows~5 and~7 take the precision, the prior, and the weight $\lambda$ of an information bonus $\lambda\,\mathrm{Novelty}^q$ independent of $K$. Row~5 takes the known reward and a maximum backup. With Bayes-UCBVI's prior, $n_0=\Theta(\ln K)$, it becomes $\Omega(K/\ln K)$, on a hard MDP that depends on $K$. Row~6 needs $0<\lambda\le r_0/2$. Row~7 needs $\beta\Delta<\ln J$ and is proved for AIF that is given $P^*$ and $R^*$. We do not prove lock-out for greedy AIF with the state term and novelty combined, a gap that does not affect AIF as we specify it, since its softmax already incurs linear regret by Proposition~\ref{prop:aif-cai-regret}. The leading term of row~8 holds for $T\ge H^3X^3A$ (Theorem~\ref{thm:calibrated}(iii)).

\emph{Adaptive goals.} Rows~9--13 assume an exact abstraction and a goal revealed before each episode, chosen from the history in rows~9--12. Rows~9 and~10 take the Gaussian family of Theorem~\ref{thm:price} with $M=Y$ start states and horizon $2$, for which $X=2Y+1$, in expectation over its prior, and need $K\ge CMY^2\ln^3(YK)$. Row~9 also needs $Y\ge600\ln(3K)$. Both let the adversary choose the start states and use the transition law. Row~9 concerns the goal-memoryless agents of Definition~\ref{def:goal-memoryless}, against an adversary that may depend on the agent. The lower bound of row~10 concerns every agent, and its upper bound an agent that uses past goals as data and knows the adversary's strategy. Proposition~\ref{prop:goal-reading} gives an upper bound of the same order for a single agent and all adversaries, on a hypercube prior.
 Row~12 needs $X\ge X_0=\mathcal O(\ln(n_0/\delta))$, and its loss per episode is at least $\sqrt{\ell_1/(4n_1)}=\Omega(1/n_0)$ for $n_0^2\ge X$, since $\ell_1\ge1$ and $n_1=\mathcal O(X+n_0^2)$. With Bayes-UCBVI's $n_0=\Theta(\ln K)$, the regret is therefore $\Omega(K/\ln K)$ at a fixed $\delta$ and $X$. Row~13 takes goals and start states fixed before the first episode, possibly using the transition law, and its leading term holds for $T\ge H^4X^3A$.

\emph{Learned abstractions.} Rows~14--17 assume a realizable block MDP with latent rewards. Row~14 takes $A=2$, $H\ge3$ and $K\ge S$. Row~15 needs $S\ge12$ and $\beta<0.09/\Delta$, is proved for AIF that is given the decoder, and is linear in $K$ for $S\ge2K+2$. Row~16 needs $H\ge4$, concerns agents that score candidates by the likelihood of their own labels over a common label set, and holds after $k_0=\mathcal O(H\ln(HK/\delta))$ episodes. Row~17 needs a fixed start state. Under goals, these are latent and the agent relabels its data (Theorem~\ref{thm:enactive}(b)).

Throughout, counts, posteriors and widths carry the episode index $k$ and the step index $h$ of the algorithm that maintains them, and $\tau$ counts steps within a rollout.

\section{Further Related Work}
\label{app:related}

\begin{table}[tb]
\centering
\caption{The closest analyses, with their setting, the epistemic term they analyse, how the agent obtains its representation, and the type of guarantee. The setting names the model class, and the objective where it is average reward or adversarial. GP stands for Gaussian process.}
\label{tab:related}
\footnotesize
\setlength{\tabcolsep}{2pt}
\begin{tabular}{@{}l>{\raggedright\arraybackslash}p{2.1cm}>{\raggedright\arraybackslash}p{3.1cm}l>{\raggedright\arraybackslash}p{3.3cm}@{}}
\toprule
Work & Setting & Epistemic term & Representation & Guarantee \\
\midrule
\citet{kolter2009near} & tabular MDP & bonus $\propto N^{-1}$ & ground states & near Bayes-optimal \\
\citet{lu2019information} & tabular MDP & $\sqrt{\text{information gain}}$ & ground states & Bayesian regret \\
\citet{tiapkin2022dirichlet} & tabular MDP & posterior quantile & ground states & regret \\
\citet{maillard2013optimal} & tabular, average reward & optimism & selected from $\Phi$ & regret \\
\citet{ortner2014selecting} & tabular, average reward & optimism & selected from $\Phi$ & regret \\
\citet{xie2023role} & coverable MDP & optimism over $\mathcal F$ & implicit in $\mathcal F$ & regret \\
\citet{wei2024value} & known POMDP & expected free energy & given & gap to Bayes-optimal \\
\citet{li2026curiosity} & GP bandit & expected free energy & no state & regret, tuned weight \\
\citet{ni2024bridging} & MDP, POMDP & none & latent self-prediction & stationary points, no regret \\
\citet{tiapkin2025narrowing} & tabular, adversarial & policy optimisation & ground states & regret, oblivious rewards \\
\midrule
This paper & tabular MDP & novelty, state entropy, $\sqrt{\text{novelty}}$ & given or learned & regret, upper and lower \\
\bottomrule
\end{tabular}
\end{table}

Table~\ref{tab:related} positions our contribution among the closest analyses, and the footnotes of Table~\ref{tab:regret} name the results we build on. Information-gain bonuses \citep{houthooft2016vime} and ensemble-disagreement bonuses \citep{sekar2020planning} drive exploration in deep RL without regret guarantees. Both measure information or posterior variance. Hence, Theorem~\ref{thm:exponent} applies to their tabular counterparts.

\emph{Adaptive goals and adversarial MDPs.} With unknown transitions and full-information rewards, $\ell_1$ confidence sets attain $\widetilde{\mathcal O}(HS\sqrt{AK})$ against adaptive reward sequences \citep{rosenberg2019online}, as \citet[\S3.2]{tiapkin2025narrowing} note, and policy optimisation attains $\mathrm{poly}(H)\sqrt{SAT}$ against oblivious ones \citep{tiapkin2025narrowing}. These works compare with the best fixed policy and reveal each reward after its episode. We instead reveal each goal before its episode and compare with each goal's optimum. Neither setting therefore contains the other. Our protocol is that of adversarial contextual MDPs \citep{modi2020noregret,levy2023efficient}, with the goal as the context and transitions that do not depend on it. These works learn context-dependent transitions and rewards, whereas our agents are given the goal, and our lower bounds rest on an adversary that chooses it from the agent's samples. Reward-free exploration needs a number of episodes quadratic in $S$ for arbitrary rewards \citep{jin2020rewardfree,ridel2026improved}, but only linear in $S$ for polynomially many rewards fixed in advance \citep{li2024minimax}. Theorem~\ref{thm:price}(b) and Proposition~\ref{prop:oblivious} are the online counterparts of these two rates. For policy optimisation against adversarial rewards, closing the remaining gaps is open even in the worst case \citep{li2026policy}.

\emph{Adaptive data analysis.} The goals of Theorem~\ref{thm:price}(b) are aligned with the agent's sampling error, which makes its estimates biased exactly along the values it is asked to compute. This is the false-discovery problem of adaptive data analysis \citep{dwork2015preserving,hardt2014preventing,steinke2015interactive}, with one difference. There, the analyst sees only the answers, and noise added to them restores validity. Our adversary sees the samples. Hence, no noise added to the agent's answers helps.

\emph{Self-predictive representations.} Joint-embedding predictive architectures \citep{lecun2022path,assran2023self} and self-predictive RL \citep{gelada2019deepmdp,schwarzer2021data} learn a representation by predicting its own next value rather than the next observation. \citet{ni2024bridging} relate latent self-prediction, together with reward prediction, to model-irrelevance abstractions through a condition on the full state or history. Proposition~\ref{prop:self-confirming} shows that selecting among candidates by the likelihood of their own labels can favour, under every policy, a candidate that merges two states whose best actions differ, and that the labels of such a candidate cannot refute it when the agent acts through it. GOLF's Bellman residual is also a self-prediction, but it compares each function with the best predictor of the same target. Selecting an abstraction by its consequences for value relates to the value-equivalence principle, which learns a model only as far as it matters for value \citep{farahmand2017value,grimm2020value,arumugam2022deciding}. Representation selection with regret guarantees incurs a factor $\sqrt{|\Phi|}$ through \emph{hard selection}, an $\arg\max$ or an elimination test \citep{maillard2011selecting,maillard2013optimal,ortner2014selecting}, whereas Theorem~\ref{thm:enactive} incurs $\sqrt{\ln|\Phi|}$ in block MDPs with latent rewards.

State abstractions, unlike the temporal abstractions of hierarchical RL \citep{sutton1999between,bacon2017optioncritic}, are computed from a known model \citep{givan2003equivalence,li2006towards}, learned from tasks \citep{abel2018state}, or selected from data \citep{jiang2015abstraction}. The regret of a fixed aggregation is bounded from above by \citet{dong2019provably}. Data observed through an abstraction call for martingale concentration \citep{starre2023analysis}. Work on rich-observation RL learns the decoder of a block MDP while exploring, which is the realizable case of our setting \citep{du2019provably,misra2020kinematic,agarwal2020flambe,efroni2022provably,xie2023role}. The case $\eta_R=\eta_P=0$ of Definition~\ref{def:spread} is model irrelevance \citep[Def.~3]{li2006towards}, and a map with small within-block spreads is an approximate model-similarity abstraction \citep[Def.~13]{abel2016near}. The $\varepsilon$-approximation of \citet[Def.~1]{ortner2014selecting} compares the full next-state distributions of merged states, and every $\varepsilon$-approximation has $\eta_R\le\varepsilon$ and $\eta_P\le2\varepsilon$, since $\phi_\#$ contracts $\ell_1$ distances. Among several abstractions, the one with the smallest worst-case bound of Section~\ref{sec:fixed-abstraction} balances $\epsilon_{\text{rep}}(\phi)K$ against $\sqrt{XT}$, as in the approximation--estimation trade-off of \citet{jiang2015abstraction} and the rate--distortion view of \citet{arumugam2021deciding}. \citet{millidge2021relationship} relate AIF and CaI through the placement of value in the generative model, which Appendix~\ref{app:models} uses. We use UCBVI and Bayes-UCBVI as their authors define them \citep{azar2017minimax,tiapkin2022dirichlet}. For AIF and CaI \citep{friston2017active,millidge2021relationship} we analyse explicitly specified variants (Appendix~\ref{app:algs}), which instantiate novelty for the transition model, fix the backup and the action rule, and give AIF and CaI Bayes-UCBVI's prior.

Theorem~\ref{thm:exponent}(a) puts into regret form the result of \citet{kolter2009near} that bonuses decaying faster than $N^{-1/2}$ are not PAC-MDP. The mechanism applies to AIF because novelty has the order $\bar n^{-1}$ of their Bayesian exploration bonus, which is near Bayes-optimal (their Thm.~1). The calibrated width is a frequentist counterpart of the widths of \citet{lu2019information} and rests on the Dirichlet quantile bounds of \citet{tiapkin2022dirichlet}. \citet{maillard2013optimal} suggest that a dependence on $\sqrt{\log|\Phi|}$, instead of their $\sqrt{|\Phi|}$, seems plausible when information can be propagated to all candidates. In block MDPs with latent rewards, Theorem~\ref{thm:enactive} attains $\sqrt{\ln|\Phi|}$, inside the square root with $T$. Lemma~\ref{lem:enactive-complete} gives a function class closed under Bellman backups, built from the candidate decoders, for block MDPs with latent rewards. Existing guarantees for AIF either assume that the agent knows the transition kernel and the observation model of the POMDP it acts in \citep{wei2024value}, or concern a Gaussian-process bandit whose curiosity weight exceeds a problem-dependent threshold \citep{li2026curiosity}. Theorem~\ref{thm:exponent} shows how fast such a weight must grow in an MDP.

\emph{Beyond the tabular setting.} The mismatch of orders between information gains and valid widths is not specific to Dirichlet posteriors. For a Gaussian process with noise variance $\sigma^2$, an observation at $x$ carries at most $\sigma_t^2(x)/(2\sigma^2)$ nats, quadratic in the posterior standard deviation $\sigma_t(x)$, whereas the GP-UCB width is linear in $\sigma_t(x)$ \citep{srinivas2010gaussian}, the order of the square root of the information gain and the form of our calibrated width.

Several results connect to earlier work beyond Table~\ref{tab:related}. Proposition~\ref{prop:aif-cai-regret} puts into regret form, for a fixed precision, the observation of \citet{odonoghue2020making} that CaI ignores epistemic uncertainty and can fail to explore. Proposition~\ref{prop:noisy-tv} formalises the maximum-entropy reading of AIF's epistemic term \citep{friston2017active} and is an instance of the noisy-TV phenomenon \citep{burda2019largescale}. It is the finite-precision counterpart of the Bellman optimality of recursive AIF as $\beta\to\infty$ \citep{dacosta2023reward}. The form $H\sqrt{T\,\Gamma_T}$ of Theorem~\ref{thm:calibrated} is the tabular analogue of the $\sqrt{T\beta_T\gamma_T}$ of GP-UCB \citep{srinivas2010gaussian} and of \citet[Prop.~10]{russo2018learning}. Proposition~\ref{prop:lure} is the MDP form of the decision-irrelevant information of \citet[Example~5]{russo2018learning} and the AIF counterpart of the exogenous-distractor problem \citep{efroni2022provably}.

\section{Agents and Their Generative Models}
\label{app:agents}

This appendix restates the value-based baselines and the results that we cite from their authors, derives the planning objectives of AIF and CaI from their generative models, and specifies every analysed agent.

\subsection{Baseline Algorithms and Cited Bounds}
\label{app:baselines}
\label{sec:ucbvi-algo}

We restate the two algorithms of Section~\ref{sec:bayesucbvi} and the results that we cite from their authors.

\emph{UCBVI.} At the start of episode $k$, UCBVI \citep{azar2017minimax} holds the counts $N_k(s,a)$ and $N_k(s,a,s')$ of the transitions observed in episodes $1,\dots,k-1$ and the empirical kernel $\widehat P_k(s'\mid s,a)=N_k(s,a,s')/N_k(s,a)$, which is arbitrary if $N_k(s,a)=0$. It computes optimistic values by backward induction from $\widehat V_{k,H+1}\equiv 0$,
\begin{equation}
\widehat Q_{k,h}(s,a) = \min\Big\{\widehat Q_{k-1,h}(s,a),\, H,\; R^*(s,a) + \big(\widehat P_k \widehat V_{k,h+1}\big)(s,a) + b_{k,h}(s,a)\Big\}, \label{eq:ucbvi-recursion}
\end{equation}
with $\widehat V_{k,h}(s)=\max_a \widehat Q_{k,h}(s,a)$ and $\widehat Q_{k,h}(s,a)=H$ if $N_k(s,a)=0$, and it acts greedily. UCBVI-CH uses the Chernoff--Hoeffding bonus of their Algorithm~3,
\begin{equation}
b_{k,h}(s,a)=7HL\sqrt{1/N_k(s,a)},\qquad L=\ln(5SAT/\delta),
\label{eq:azar-bonus-ch}
\end{equation}
which depends on the data only through the count. UCBVI-BF, which attains their optimal rate, uses a Bernstein--Freedman bonus that combines an empirical-variance term, a range correction, and a term that certifies optimism under the empirical variance (their Algorithm~4).

\begin{theorem}[UCBVI regret, \citealp{azar2017minimax}, Thms.~1--2]
\label{thm:ucbvi-ground}
Let $L=\ln(5HSAT/\delta)$. With probability at least $1-\delta$, UCBVI-CH has $\mathfrak R_K\le20HL\sqrt{SAT}+250H^2S^2AL^2$, which is $\widetilde{\mathcal O}(H\sqrt{SAT})$ for $T\ge H^2S^3A$, and UCBVI-BF satisfies
\begin{equation}
\mathfrak R_K \le 30HL\sqrt{SAK} + 2500H^2S^2AL^2 + 4H^{3/2}\sqrt{KL},
\end{equation}
which is $\widetilde{\mathcal O}(\sqrt{HSAT})$ for $T\ge H^3S^3A$ and $SA\ge H$.
\end{theorem}
Up to poly-logarithmic factors, the bound of UCBVI-BF matches the minimax lower bound $\Omega(\sqrt{HSAT})$ \citep{jaksch2010near,osband2016lower}, which \citet[App.~C]{domingues2021episodic} prove for episodic MDPs. \citet{azar2017minimax} state the CH bound as $20H^{3/2}L\sqrt{SAK}$, which equals $20HL\sqrt{SAT}$ for $K=T/H$. Its second term is of lower order, up to logarithmic factors, once $T\ge H^2S^3A$. Their bonuses use $\ln(5SAT/\delta)$ in place of $L$, and their theorems state the larger $L$ above, which covers both.

\emph{Bayes-UCBVI.} \citet{tiapkin2022dirichlet} replace UCBVI's hand-built bonus by an optimistic rule read off a Dirichlet posterior over the transition kernel, with $R^*$ known. Their algorithm is stated for time-inhomogeneous MDPs, which include our time-homogeneous class, and it keeps a separate belief at every step. We therefore index its posteriors and quantile orders by $h$.

\begin{definition}[Dirichlet posterior with a pseudo-state, \citealp{tiapkin2022dirichlet}]
\label{def:dirichlet-ground}
Extend $\mathcal S$ to $\mathcal S'=\mathcal S\cup\{s_{\mathrm E}\}$ with an absorbing pseudo-state $s_{\mathrm E}$ of reward $r_0=2$, so that $V^*_h(s_{\mathrm E})=r_0(H-h+1)$ \citep[cf.][]{brafman2002rmax,szita2008many}. With the pseudo-count $n_0=\lceil c_{n_0}+\log_{17/16}K\rceil$, where $c_{n_0}$ is an absolute constant (their Appendix~B.2), place on every row the prior $\rho^0_h(s,a)=\mathrm{Dir}\big(n_0\cdot\mathds{1}\{s'=s_{\mathrm E}\}\big)_{s'\in\mathcal S'}$, and update it with the counts $N^h_k(s,a,s')$ of the transitions from $(s,a)$ to $s'$ at step $h$ in episodes $1,\dots,k-1$. The posterior is $\rho^k_h(s,a)=\mathrm{Dir}\big(\bar n^k_h(s'\mid s,a)\big)_{s'}$ with $\bar n^k_h(s'\mid s,a)=N^h_k(s,a,s')+n_0\mathds{1}\{s'=s_{\mathrm E}\}$. \end{definition}
The prior is improper, and we follow \citet{tiapkin2022dirichlet} in using it only as an instrumental device that defines $\rho^k_h$ by this update, never through a marginal likelihood.

\begin{definition}[Optimistic Bayesian value iteration, \citealp{tiapkin2022dirichlet}, Eqs.~1--2]
\label{def:bayesq-ground}
With the raw visit count $n^k_h(s,a)$, the total mass $\bar n^k_h(s,a)=n^k_h(s,a)+n_0$, and the quantile order
\begin{equation}
\kappa^k_h(s,a) = 1 - \frac{C_\kappa\,\delta}{SAH\,[2n^k_h(s,a)+1]^3\,[\bar n^k_h(s,a)]^{3/2}}, \qquad C_\kappa=\tfrac{1}{5(e\pi)^3},
\end{equation}
define, for $h=H,\dots,1$, with $\overline V^k_{H+1}\equiv0$ and $\overline V^k_h(s_{\mathrm E})\equiv V^*_h(s_{\mathrm E})$,
\begin{equation}
\overline Q^k_h(s,a) = R^*(s,a) + \mathds{Q}_{p\sim\rho^k_h(s,a)}\big(p^\top \overline V^k_{h+1},\, \kappa^k_h(s,a)\big), \qquad \overline V^k_h(s) = \max_a \overline Q^k_h(s,a),
\end{equation}
where $\mathds{Q}_{p\sim\rho}(f,\kappa)$ is the order-$\kappa$ quantile of the pushforward of $\rho$ through $p\mapsto p^\top f$. The agent plays $\pi^k_h(s)=\arg\max_a\overline Q^k_h(s,a)$.
\end{definition}
\begin{lemma}[Quantile-gap bound, \citealp{tiapkin2022dirichlet}]
\label{lem:tiapkin-quantile}
Let $\ell=\ell^k_h(s,a):=\ln\big(1/(1-\kappa^k_h(s,a))\big)$ and $\bar n:=\bar n^k_h(s,a)$. On every sample path, the quantile gap satisfies
\begin{align*}
\Delta^k_h(s,a) &:= \mathds{Q}_{p\sim\rho^k_h(s,a)}\big(p^\top\overline V^k_{h+1},\kappa^k_h(s,a)\big) - \widehat P^k_h(\cdot\mid s,a)^\top \overline V^k_{h+1} \\
&\;\le\; 2\sqrt{\mathds{V}_{\widehat P^k_h(\cdot\mid s,a)}[\overline V^k_{h+1}]\,\ell/\bar n} \;+\; 2\sqrt2\, r_0 H\, \ell/\bar n.
\end{align*}
\end{lemma}
Here $\widehat P^k_h(\cdot\mid s,a)=\bar n^k_h(\cdot\mid s,a)/\bar n^k_h(s,a)$ is the posterior mean, which differs from the empirical kernel by $\mathcal O(n_0/\bar n^k_h(s,a))$. The lemma is the upper inequality of their Corollary~B.7, stated as their Lemma~3.4. It follows from a Bernstein-type inequality for weighted Dirichlet sums (their Lemma~C.8), which needs no high-probability event and holds for every $f$ with $0\le f\le f(s_{\mathrm E})$ and integer Dirichlet parameters, the form in which Lemma~\ref{lem:transport} uses it. Validity (Section~\ref{sec:setting}) is one-sided, since a width must compensate an underestimation of the next-step value but not an overestimation. At an unvisited row, for instance, the posterior is the point mass at the pseudo-state, and the quantile gap is zero. The posterior mean then assigns the next-step value $r_0(H-h)$, which exceeds every value of the MDP, and the agent stays optimistic without a width.

\begin{theorem}[Bayes-UCBVI regret, \citealp{tiapkin2022dirichlet}, Theorem 3.1]
\label{thm:bayesucbvi-ground}
With the $n_0,r_0$ of Definition~\ref{def:dirichlet-ground}, with probability at least $1-\delta$,
\begin{equation}
\mathfrak R_K = \mathcal O\big(H\sqrt{SAT}\cdot L + H^3S^2A L^2\big), \qquad L=\mathcal O(\ln(HSAT/\delta)).
\end{equation}
\end{theorem}
For $T\ge H^4S^3A$, this bound matches the $\Omega(H\sqrt{SAT})$ lower bound for time-inhomogeneous kernels \citep{jin2018qlearning,domingues2021episodic} up to poly-logarithmic factors. \citet{tiapkin2022dirichlet} count episodes by $T$ and state the bound as $\mathcal O(\sqrt{H^3SAK}\,L+H^3S^2AL^2)$, which becomes the form above with $K=T/H$. They note that it matches the lower bound for $K\ge H^3S^3A$. Its factor $\sqrt H$ over UCBVI is the price of their larger, time-inhomogeneous class (their \S1, footnote~1), not of their exploration rule.

\subsection{Generative Models}
\label{app:models}
\label{app:models-common}

This subsection derives the planning objectives of AIF and CaI (Appendix~\ref{app:algs}) from their generative models. We specify the steps that arise from a modelling choice. Within an episode, the models generate latent abstract states $x_{1:H+1}$, observations $s_{1:H+1}$ and actions $a_{1:H}$, condition on the start state $s_1$, and emit one observation for each transition.

\emph{\texttt{model(G)}, with the abstraction given.} We assume $\phi(\mathcal S)=\mathcal X$, extend the latent and observation spaces by the pseudo-states, $\mathcal X'=\mathcal X\cup\{x_{\mathrm E}\}$ and $\mathcal S'=\mathcal S\cup\{s_{\mathrm E}\}$ with $\phi(s_{\mathrm E}):=x_{\mathrm E}$, and take as parameters the rows $B=(B_{h,x,a})$, distributions on $\mathcal X'$, with the absorbing row $B_{h,x_{\mathrm E},a}:=\delta_{x_{\mathrm E}}$. The forward model is
\begin{equation}
P\big(s_{2:H+1},x_{1:H+1}\mid s_1,a_{1:H},B\big) \;=\; \mathds{1}\{x_1=\phi(s_1)\}\prod_{h=1}^{H}\mathrm{Cat}\big(x_{h+1};B_{h,x_h,a_h}\big)\;\mathbf A^0_\phi\big(s_{h+1}\mid x_{h+1}\big),
\label{eq:model-forward}
\end{equation}
where the emission $\mathbf A^0_\phi(\cdot\mid x)$ is a fixed distribution on the block $\phi^{-1}(x)$ and $\mathbf A^0_\phi(\cdot\mid x_{\mathrm E}):=\delta_{s_{\mathrm E}}$. It is the abstract MDP of Eq.~\eqref{eq:abstract-mdp}, extended by the pseudo-state and written as a latent-variable model. Since the blocks partition $\mathcal S'$, the observation determines $x_\tau=\phi(s_\tau)$, and the emission factor does not depend on $B$. It therefore drops out of the posterior over $B$ and out of every planning quantity below, which makes the choice of $\mathbf A^0_\phi$ immaterial. On each row we place the instrumental prior of Definition~\ref{def:dirichlet-ground}, with the conjugate update $\rho^k_h(x,a)=\mathrm{Dir}(\bar n^k_h(\cdot\mid x,a))$ as the posterior. Under the agent's model, the reward outcome of a step is its estimate $r_h=\widehat R^k_\phi(x_h,a_h)$, with $\widehat R^k_\phi(x_{\mathrm E},\cdot):=r_0$, whereas the data are generated by $R^*$ and $P^*$.

\emph{\texttt{model(L)}, with the abstraction learned.} Each candidate $\phi\in\Phi$ is surjective onto $\mathcal X$. Given $\phi$, the parameters are the transition rows $B_{\phi,h,x,a}$, distributions on $\mathcal X$, and the emission rows $\mathbf A_{\phi,h,x}$, distributions on $\phi^{-1}(x)$, one for the state reached at each step, with independent uniform Dirichlet priors and no pseudo-state. The forward model is Eq.~\eqref{eq:model-forward} with $B_{\phi,h}$ and $\mathbf A_{\phi,h}$ in place of $B_h$ and $\mathbf A^0_\phi$. A candidate's posterior predictive of the next observation, $P_\phi(s'\mid s,a)=\sum_{x'}\widehat B_\phi(x'\mid\phi(s),a)\,\widehat{\mathbf A}_\phi(s'\mid x')$ with posterior means $\widehat B_\phi$ and $\widehat{\mathbf A}_\phi$, sums to one because every block is nonempty, whereas a map with an empty block would leak the mass that its rows put on the empty latent state. The marginal likelihood $\Lambda_\phi$ of a candidate is the product of its sequential predictive probabilities, and its posterior is $\mathcal Q(\phi\mid\mathcal H)\propto P(\phi)\Lambda_\phi$. \texttt{model(L)} needs proper priors, since under the instrumental prior every real successor seen for the first time would have predictive probability zero.

\emph{\texttt{model(L)} without its emission.} Dropping the emission factor from \texttt{model(L)} leaves, for each candidate, a Bayesian model of its own labels, with rows $B_{\phi,h,x,a}$, distributions on a common label set $\mathcal X$, with independent uniform Dirichlet priors. Its marginal likelihood $\Lambda^{\mathrm{lat}}_\phi$ of the label sequences $\phi(s^j_{1:H+1})$ given the actions is the score of the latent-predictive agent of Appendix~\ref{app:aif-selection}. It needs no surjectivity, since it predicts labels rather than ground states, and its rows do not grow with $S$. Since different candidates assign likelihoods to different sequences, $\Lambda^{\mathrm{lat}}_\phi$ is not a likelihood of the observed data, and its ratios across candidates are not likelihood ratios.

\subsubsection{Active Inference}
\label{app:models-aif}

At the current step $t$ of an episode, after the observations $s_{1:t}$, AIF performs variational inference over the structured family
\begin{multline}
\mathcal Q\big(x_{1:H+1},\theta,\pi\mid s_{1:t}\big) = \mathcal Q(\pi)\,\mathcal Q(\phi)\,\mathcal Q(B_\phi\mid\phi)\,\mathcal Q(\mathbf A_\phi\mid\phi)\\
\times\prod_{\tau\le t}\delta_{\phi(s_\tau)}(x_\tau)\prod_{\tau>t}B_{\phi,\tau-1}(x_\tau\mid x_{\tau-1},a_{\tau-1}),
\label{eq:aif-meanfield}
\end{multline}
with $\theta=(\phi,\mathbf A_\phi,B_\phi)$, in which each observed latent state is tied to the abstraction that decodes it and each future latent state is drawn from the row of $B_\phi$ for the previous state and action. The observations enter only through the point masses $\delta_{\phi(s_\tau)}$ and through the posterior factors over $\phi$, $B_\phi$ and $\mathbf A_\phi$, and we suppress them in the notation below. A factor $\mathcal Q(x_\tau)$ that does not depend on $\phi$ could not represent the posterior, since two candidates with positive weight can decode the same observation differently. In \texttt{model(L)}, the minimiser of the variational free energy of the observed data is the exact posterior, by the Gibbs variational principle. In \texttt{model(G)}, where the improper prior leaves the free energy undefined, the posterior is the conjugate update $\rho^k_h$ by definition. Hence AIF and Bayes-UCBVI hold the same posterior.

\emph{Predictions.} The factors of Eq.~\eqref{eq:aif-meanfield} define a predictive joint in which a row is drawn from the posterior and the next state from the drawn row, and every information term below is a mutual information under this joint. In particular, novelty is $I_{\mathcal Q}(B;x_{\tau+1}\mid x_\tau,a_\tau)$, Eq.~\eqref{eq:novelty-def}, and the information about the emission used in Proposition~\ref{prop:lure} is $I_{\mathcal Q}(\mathbf A;s_{\tau+1}\mid x_{\tau+1})$. With $\phi$ given, the rows used along a path are indexed by distinct steps and are therefore independent under the posterior, and the marginal law of a path is the product of the posterior means $\widehat P^k_h$. With $\phi$ uncertain, the rows are independent only given $\phi$, and the predictive law of an episode is the mixture $\sum_\phi\mathcal Q(\phi)\prod_\tau\widehat B_{\phi,\tau}$ over a shared decoder.

\emph{Expected free energy.} The outcome of step $\tau$ is $o_\tau=(s_\tau,r_\tau)$ and its latent variables are $z_\tau=(x_\tau,\theta)$. AIF adds a policy prior $P(\pi)$ and the preference distribution $\widetilde P(o_\tau)=\widetilde P(r_\tau)\,\mathcal U(s_\tau)$ with $\ln\widetilde P(r)=\beta r-\ln Z_\beta$ on $[0,r_0]$, where $\mathcal U$ is uniform on $\mathcal S'$. The preference thus concerns only the reward \citep[cf.][]{millidge2021relationship}. Value enters through the biased joint $\widetilde P(o_\tau,z_\tau\mid\pi)=\widetilde P(o_\tau)\,\mathcal Q(z_\tau\mid o_\tau,\pi)$ \citep[Eq.~2.5]{friston2017active}, and AIF scores a plan by
\begin{align}
\mathcal G(\pi) \;&=\; \sum_{\tau>t}\mathds{E}_{\mathcal Q(o_\tau,z_\tau\mid\pi)}\Big[\ln\mathcal Q\big(z_\tau\mid\pi\big)-\ln\widetilde P\big(o_\tau,z_\tau\mid\pi\big)\Big]\nonumber\\
&=\;\sum_{\tau>t}\Big(-\mathds{E}_{\mathcal Q}\big[\ln\widetilde P(o_\tau)\big]-I_{\mathcal Q}\big(z_\tau;o_\tau\mid\pi\big)\Big).
\label{eq:efe-raw}
\end{align}
The first term is $-\beta\,\mathds{E}_{\mathcal Q}[r_\tau]$ up to constants that do not depend on $\pi$. The plan $\pi$ fixes the action $a_\tau$. In \texttt{model(G)}, the abstraction $\phi$ is fixed, and the outcome $r_\tau=\widehat R^k_\phi(\phi(s_\tau),a_\tau)$ is thus a function of $s_\tau$ given $\pi$. In \texttt{model(L)}, we take the reward to be a known function $R^*(s_\tau,a_\tau)$ of the ground state and the action, the same for every candidate abstraction, as in Proposition~\ref{prop:lure}, where AIF knows the rewards. In both models, $r_\tau$ is therefore a function of $s_\tau$ given $\pi$. Hence $I_{\mathcal Q}(z_\tau;o_\tau\mid\pi)=I_{\mathcal Q}(z_\tau;s_\tau\mid\pi)$, and the chain rule gives
\begin{equation}
\mathcal G(\pi) = \sum_{\tau>t}\Big(\underbrace{-\beta\,\mathds{E}_{\mathcal Q}\big[r_\tau\big]}_{\text{extrinsic value}} \;-\; \underbrace{I_{\mathcal Q}\big(x_\tau;s_\tau\mid\pi\big)}_{\text{intrinsic value}} \;-\; \underbrace{I_{\mathcal Q}\big(\theta;s_\tau\mid x_\tau,\pi\big)}_{\text{parameter information gain}}\Big)+\text{const}.
\label{eq:efe-three-term}
\end{equation}
In \texttt{model(G)}, the parameter term vanishes, since $s_\tau$ is independent of $B$ given $x_\tau$, and under exact perception the intrinsic value is the entropy $\mathds{H}[\mathcal Q(x_\tau\mid\pi)]$ of the predicted abstract state. One step ahead from $(x,a)$, this entropy is $\mathds{E}_B\mathds{H}[B]+\mathrm{Novelty}(x,a)$ by Eq.~\eqref{eq:novelty-def}. It thus contains the information about $B$ once, and its aleatoric part does not shrink with data (Proposition~\ref{prop:noisy-tv}).

\emph{The analysed variant.} The AIF agent we analyse makes three modelling choices in addition to Eq.~\eqref{eq:efe-three-term}. First, it adds $\mathrm{Novelty}(x_\tau,a_\tau)$ as a separate term, as the information gain about the emission parameters does in \citet[Eq.~16]{dacosta2020synthesis} and \citet{schwartenbeck2019computational}, which add it to the Dirichlet learning of \citet{friston2017active}. Hence, novelty is counted twice. Second, it evaluates these terms recursively, one step ahead from each $(x,a)$ \citep{friston2021sophisticated}, with maximum backups and the clamp at $r_0(H-h+1)$. Third, it samples its action at each step from a softmax over the resulting values (Algorithm~\ref{alg:aif}). Sophisticated inference selects actions step by step, as our variant does, whereas the discrete-state AIF of \citet{friston2017active} samples a whole plan from its policy posterior $\mathcal Q(\pi)\propto P(\pi)e^{-\mathcal G(\pi)}$. Proposition~\ref{prop:noisy-tv} holds for both rules under a uniform plan prior, and Proposition~\ref{prop:lure} holds for the greedy rule and the per-step softmax. Counting novelty once changes no result of this paper. Theorem~\ref{thm:exponent} holds for every weight on novelty, and the instance of Proposition~\ref{prop:noisy-tv} has no novelty. In \texttt{model(L)}, the parameter term contains the novelty of the emission rows, which drives Proposition~\ref{prop:lure}.

\subsubsection{Control as Inference}
\label{app:models-cai}

CaI uses \texttt{model(G)} and places value in the likelihood of a binary optimality variable,
\begin{equation}
\ln P\big(O_\tau=1\mid x_\tau,a_\tau\big) \;=\; \beta\big(\widehat R^k_\phi(x_\tau,a_\tau)-r_0\big) \;\le\;0 ,
\label{eq:optimality}
\end{equation}
where subtracting the constant $r_0$ keeps the likelihood at most one, the pseudo-state included, and shifts every backed-up value of a step by the same amount. CaI uses the structured variational family of \citet{levine2018reinforcement}, in which only the policy factors $\mathcal Q(a_h\mid x_h)$ are free and the dynamics factors are the model's predictive dynamics, here exactly $\widehat P^k_h$. Minimising $D_{\mathrm{KL}}\big(\mathcal Q\,\|\,P(\cdot\mid O_{1:H}=1)\big)$ over the policy factors then amounts to minimising
\begin{equation}
\mathcal G_{\mathrm{CaI}}(\pi) \;=\; -\beta\sum_\tau\mathds{E}_{\mathcal Q}\big[\widehat R^k_\phi(x_\tau,a_\tau)\big] \;+\; \sum_\tau\mathds{E}_{\mathcal Q}\,D_{\mathrm{KL}}\big(\mathcal Q(a_\tau\mid x_\tau)\,\|\,P(a_\tau\mid x_\tau)\big)
\label{eq:cai-objective}
\end{equation}
up to a constant, which is Eq.~1 of \citet{millidge2021relationship} without the state-divergence term, which vanishes under the structured family, and the observation-ambiguity term of POMDPs. With a uniform action prior, backward induction gives the soft backup $V_h(x)=\beta^{-1}\ln\big(A^{-1}\sum_a\exp(\beta Q_h(x,a))\big)$ and the softmax policy of Algorithm~\ref{alg:aif}. The structured constraint matters. Under a mean-field factorisation $\mathcal Q(B)\,\mathcal Q(x_{1:H+1})$ without it, the optimal dynamics factor involves the geometric mean $\exp\mathds{E}_\rho[\ln B]$ instead of $\widehat P^k_h$ and tilts towards successors of high value. In \texttt{model(L)}, the predictive law of an episode is a mixture of products, not a product of mixtures, and we derive no CaI planner for it.

\subsection{Algorithms}
\label{app:algs}

Table~\ref{tab:agents} specifies every analysed agent. UCBVI appears as equations in Appendix~\ref{app:baselines}, calibrated AIF with a uniform prior in Section~\ref{sec:goal-uniform} and Appendix~\ref{app:goal-uniform}, and the agents that learn $\phi$ in Section~\ref{sec:learning-phi} and Appendix~\ref{app:enactive}. Algorithm~\ref{alg:aif} gives AIF, calibrated AIF and CaI on the same abstract space and the same posterior $\rho^k_h$, so that the three can be compared line by line. Every agent plans once per episode, by backward induction with the beliefs held at the start of the episode, and acts by a per-step rule on the resulting values, which makes its policy closed-loop.

\begin{table}[tbp]
\centering
\caption{The analysed agents. All Bayesian agents on $\phi$ use the instrumental pseudo-state prior with one Dirichlet row per step, and estimate rewards by $\widehat R^k_\phi$, which is $1$ at unvisited pairs (Eq.~\eqref{eq:rhat}). The column \emph{Value} states how $V_h(x)$ is computed from $Q_h(x,\cdot)$ in Eq.~\eqref{eq:psi-slot}, by a maximum or a soft maximum, the column \emph{Epistemic term} gives $\Psi$, and ``clamp'' is the cap at $r_0(H-h+1)$.}
\label{tab:agents}
\footnotesize
\setlength{\tabcolsep}{3pt}
\rowcolors{2}{gray!12}{white}
\begin{tabular}{@{}>{\raggedright\arraybackslash}p{2.6cm}>{\raggedright\arraybackslash}p{2.2cm}>{\raggedright\arraybackslash}p{1.0cm}>{\raggedright\arraybackslash}p{3.0cm}>{\raggedright\arraybackslash}p{2.5cm}>{\raggedright\arraybackslash}p{2.0cm}@{}}
\toprule
Agent (result) & Reward, belief & Value & Epistemic term & Action rule & Abstraction \\
\midrule
UCBVI-CH on $\phi$ (Thm.~\ref{thm:ucbvi-fixed}(a)) & $\widehat R^k_\phi$, empirical kernel pooled over steps & max & Hoeffding bonus Eq.~\eqref{eq:azar-bonus-ch}, clamps at $H$ and at the previous values & greedy & given \\
Bayes-UCBVI on $\phi$ (Thm.~\ref{thm:ucbvi-fixed}(b), Prop.~\ref{prop:lockout}) & $\widehat R^k_\phi$ or the goal $R_k$, with $\rho^k_h$ & max & $\kappa^k_h$-quantile of $p^\top V$ & greedy & given \\
information-driven (Thm.~\ref{thm:exponent}) & $R^*$, $\rho^k_h$ & max & $\lambda\,\mathrm{Novelty}^q$, optional clamp & greedy & given, exact \\
AIF (Prop.~\ref{prop:aif-cai-regret}, Alg.~\ref{alg:aif}) & $\widehat R^k_\phi$, $\rho^k_h$ & max & $\beta^{-1}(\mathds{H}[\widehat P^k_h]+\mathrm{Novelty})$, novelty counted twice, clamp & softmax at $\beta$ per step, or posterior over plans (App.~\ref{app:noisy-tv}) & given \\
calibrated AIF (Thm.~\ref{thm:calibrated}, Alg.~\ref{alg:aif}) & $\widehat R^k_\phi$, $\rho^k_h$ & max & $\Psi^+_{k,h}$ of Eq.~\eqref{eq:calibrated}, clamp & greedy & given \\
calibrated AIF, uniform prior (Thm.~\ref{thm:price}) & goal $R_k$, $\mathrm{Dir}(\mathds 1+\text{counts})$ on $\mathcal X$ & max & $\Psi^{\mathrm u}_{k,h}$, clamp at $H-h+1$ & greedy & given, exact \\
CaI (Prop.~\ref{prop:aif-cai-regret}, Alg.~\ref{alg:aif}) & $\widehat R^k_\phi$, $\rho^k_h$ & soft max & none & softmax at $\beta$ per step & given (\texttt{model(G)}) \\
posterior averaging (Prop.~\ref{prop:selection}) & --, exact mixture over $\Phi$ & -- & -- (a predictor) & -- & learned \\
AIF learning its emission (Prop.~\ref{prop:lure}) & $R^*$, known latent kernel, Dirichlet emission rows & max & $\beta^{-1}$ times the novelty about the emission, clamp & greedy or softmax at $\beta$ & decoder known \\
latent-predictive (Prop.~\ref{prop:self-confirming}) & any, through $\phi_k$ & any & any & any rule that acts through $\phi_k$ & maximiser $\phi_k$ of $\Lambda^{\mathrm{lat}}_\phi$ \\
enactive agent, GOLF (Thm.~\ref{thm:enactive}) & squared Bellman residuals, relabelled with $R_k$ in (b) & max & none; optimism over the version space of $\mathcal F$ & greedy & one candidate per step, through $\mathcal F$ \\
\bottomrule
\end{tabular}
\end{table}

Algorithm~\ref{alg:aif} turns AIF into calibrated AIF by three changes. First, it drops the state term $\beta^{-1}\mathds{H}[\widehat P^k_h(\cdot\mid x,a)]$. Second, it replaces $\beta^{-1}\mathrm{Novelty}$ by $\Psi^+_{k,h}$ of Eq.~\eqref{eq:calibrated}, whose logarithmic factor $\ell^k_h$ comes from Bayes-UCBVI's quantile schedule (Definition~\ref{def:bayesq-ground}). Third, it acts greedily instead of sampling from the softmax. The precision $\beta$ thus disappears. The posterior $\rho^k_h$, the pseudo-state prior, the clamp, and the recursive expected free energy stay the same. CaI differs from AIF by having no epistemic term and a soft maximum in the backup.

\begin{algorithm}[tbp]
\caption{AIF, calibrated AIF and CaI on an abstraction (Section~\ref{sec:aif}, Proposition~\ref{prop:aif-cai-regret} and Theorem~\ref{thm:calibrated}). Comments mark this paper's own constructions and the three changes from AIF to calibrated AIF.}
\label{alg:aif}
\label{alg:cai}
\begin{algorithmic}[1]
\State \textbf{Input:} $\mathcal X,\mathcal A,H,K,\delta,\phi$, $n_0$, $r_0=2$, precision $\beta$, variant $\in\{\text{AIF},\text{calibrated},\text{CaI}\}$
\State $\bar n_h(x'\mid x,a)\gets n_0\mathds{1}\{x'=x_{\mathrm E}\}$ for all $h,x,a,x'$, and $\mathrm{cnt}(x,a)\gets0$, $\mathrm{sum}(x,a)\gets0$ for all $x,a$
\For{episode $k=1,2,\dots,K$}
    \State $\widehat R^k_\phi(x,a)\gets\mathrm{sum}(x,a)/\mathrm{cnt}(x,a)$ if $\mathrm{cnt}(x,a)\ge1$, else $1$ \Comment{Eq.~\eqref{eq:rhat}}
    \State $\rho^k_h(x,a)\gets\mathrm{Dir}(\bar n_h(\cdot\mid x,a))$ and $\widehat P^k_h(\cdot\mid x,a)\gets\bar n_h(\cdot\mid x,a)/\bar n_h(x,a)$ for all $h,x,a$
    \State $V_{H+1}\gets0$ and $V_h(x_{\mathrm E})\gets r_0(H-h+1)$ for all $h$
    \For{$h=H,H-1,\dots,1$} \Comment{recursive expected free energy}
        \For{each $(x,a)\in\mathcal X\times\mathcal A$}
            \State $\mathrm{Nov}\gets\mathrm{Novelty}(x,a)$ under $\rho^k_h(x,a)$, Eq.~\eqref{eq:novelty-def} \Comment{our instantiation}
            \If{variant $=$ AIF}
                \State $\Psi\gets\beta^{-1}\big(\mathds{H}[\widehat P^k_h(\cdot\mid x,a)]+\mathrm{Nov}\big)$
            \ElsIf{variant $=$ calibrated} \Comment{first change, no state term}
                \State $\ell\gets\ln\big(1/(1-\kappa^k_h(x,a))\big)$
                \State $\Psi\gets2\,\mathrm{span}(V_{h+1})\sqrt{\ell\,\mathrm{Nov}}+2\sqrt2\,r_0H\ell/\bar n_h(x,a)$ \Comment{second change}
            \Else
                \State $\Psi\gets0$ \Comment{CaI}
            \EndIf
            \State $Q_h(x,a)\gets\widehat R^k_\phi(x,a)+\widehat P^k_h(\cdot\mid x,a)^\top V_{h+1}+\Psi$, capped at $r_0(H-h+1)$ unless variant $=$ CaI
        \EndFor
        \State $V_h(x)\gets\max_aQ_h(x,a)$, or $\beta^{-1}\ln\big(A^{-1}\sum_a\exp(\beta Q_h(x,a))\big)$ if variant $=$ CaI, for all $x\in\mathcal X$
    \EndFor
    \For{$h=1,\dots,H$}
        \State Observe $s^k_h$ and set $x^k_h\gets\phi(s^k_h)$
        \State $a^k_h\gets\arg\max_aQ_h(x^k_h,a)$ if variant $=$ calibrated \Comment{third change}
        \State otherwise $a^k_h\sim\mathcal Q(\cdot\mid x^k_h)\propto\exp\big(\beta Q_h(x^k_h,\cdot)\big)$ \Comment{softmax at precision $\beta$}
        \State Execute $a^k_h$, observe $r^k_h$ and $s^k_{h+1}$, and set $x'\gets\phi(s^k_{h+1})$
        \State $\mathrm{cnt}(x^k_h,a^k_h)\mathrel{+}=1$, $\mathrm{sum}(x^k_h,a^k_h)\mathrel{+}=r^k_h$, $\bar n_h(x'\mid x^k_h,a^k_h)\mathrel{+}=1$
    \EndFor
\EndFor
\end{algorithmic}
\end{algorithm}

The softmax of Algorithm~\ref{alg:aif} puts mass at least $e^{-\beta r_0H}/A$ on every action, which gives the linear lower bound of Proposition~\ref{prop:aif-cai-regret} at every fixed $\beta$. Making this floor vanish requires $\beta\to\infty$, which removes the epistemic terms from AIF's action rule. Boltzmann exploration with a monotone schedule is known to be suboptimal already in multi-armed bandits \citep{cesabianchi2017boltzmann}.

\section{Proofs for Section~\ref{sec:setting}}
\label{app:proofs-setting}

\subsection{Representation Error}
\label{app:lem1}

The following lemma is standard. It is the episodic form of the simulation argument behind approximate model-similarity abstractions \citep[Lem.~2]{abel2016near}, with the within-block spreads of Definition~\ref{def:spread}. The second gap of Section~\ref{sec:abstraction} is not zero, because the projected process $\{\phi(s_\tau)\}$ is in general not Markov. Transition errors compound over the remaining steps, hence the factor $H^2$ in $\epsilon_{\text{rep}}(\phi)$. We include its short proof because the cited form is discounted and concerns optimal values only, whereas our proofs use the episodic constants and part~(c), which bounds the value of $\phi$-measurable policies although the projected process is not Markov.

\begin{lemma}[Control of representation error]
\label{lem:bias-metric}
Let $\epsilon_{\text{bias}}(\phi):=\max_{h,s,a}|Q^*_h(s,a)-\overline Q^*_{\phi,h}(\phi(s),a)|$ be the optimal-value bias of $\phi$. For all $x,a$, all $s\in\phi^{-1}(x)$, and all Markov $\phi$-measurable $\pi$,
\begin{itemize}
\item[(a)] $\big|R^*(s,a)-\overline R_\phi(x,a)\big|\le \eta_R(\phi)$ and $\big\|\phi_\# P^*(\cdot\mid s,a)-\overline P_\phi(\cdot\mid x,a)\big\|_1 \le \eta_P(\phi)$.
\item[(b)] $\epsilon_{\text{bias}}(\phi)\le\epsilon_{\text{rep}}(\phi)$.
\item[(c)] $\max_{h,s}\big|V^\pi_h(s)-\overline V^\pi_{\phi,h}(\phi(s))\big|\le\epsilon_{\text{rep}}(\phi)$.
\end{itemize}
\end{lemma}

\begin{proof}
Fix $\phi$, write $\eta_R:=\eta_R(\phi)$, $\eta_P:=\eta_P(\phi)$, and fix $x$, $a$ and $s\in\phi^{-1}(x)$.

\emph{Part (a), the one-step bounds.} Summing Eq.~\eqref{eq:abstract-mdp} over the block gives $\overline P_\phi(\cdot\mid x,a)=\sum_{s''}W(s''\mid x)\,\phi_\#P^*(\cdot\mid s'',a)$, and $\overline R_\phi(x,a)=\sum_{s''}W(s''\mid x)R^*(s'',a)$. Both are convex combinations, over $s''\in\phi^{-1}(x)$, of points of sets whose diameters Definition~\ref{def:spread} calls $\eta_R$ and $\eta_P$. A point $f(s)$ of such a set is within its diameter $\eta$ of every convex combination,
\begin{equation}
\Big|f(s)-\sum_{s''}w(s'')f(s'')\Big| \;=\; \Big|\sum_{s''}w(s'')\big(f(s)-f(s'')\big)\Big| \;\le\; \eta,
\label{eq:step1-episodic}
\end{equation}
applied coordinatewise and summed for the $\ell_1$ case. Hence
\begin{equation}
\big|R^*(s,a)-\overline R_\phi(x,a)\big| \le \eta_R, \qquad \big\|\phi_\#P^*(\cdot\mid s,a) - \overline P_\phi(\cdot\mid x,a)\big\|_1 \le \eta_P.
\label{eq:step1b}
\end{equation}

\emph{Part (b), the recurrence for optimal values.} Let $q_h:=\max_{s,a}|Q^*_h(s,a)-\overline Q^*_{\phi,h}(\phi(s),a)|$ and $e_h:=\max_s|V^*_h(s)-\overline V^*_{\phi,h}(\phi(s))|$, with $q_{H+1}=e_{H+1}=0$, so that $\epsilon_{\text{bias}}(\phi)=\max_hq_h$ and $e_h\le q_h$, since $|\max_af(a)-\max_ag(a)|\le\max_a|f(a)-g(a)|$. Write $g:=\overline V^*_{\phi,h+1}$, which takes values in $[0,H-h]$. Since
\begin{equation*}
\sum_{s'}P^*(s'\mid s,a)g(\phi(s'))=\phi_\#P^*(\cdot\mid s,a)^\top g,
\end{equation*}
\begin{align}
\big|Q^*_h(s,a)-\overline Q^*_{\phi,h}(x,a)\big| &\le \big|R^*(s,a)-\overline R_\phi(x,a)\big| + \Big|\sum_{s'}P^*(s'\mid s,a)\big(V^*_{h+1}(s')-g(\phi(s'))\big)\Big| \nonumber\\
&\quad + \big|\big(\phi_\#P^*(\cdot\mid s,a)-\overline P_\phi(\cdot\mid x,a)\big)^\top g\big| \;\le\; \eta_R+e_{h+1}+(H-h)\eta_P,
\label{eq:step3-episodic}
\end{align}
by Eq.~\eqref{eq:step1b} and Hölder's inequality. Hence $q_h\le\eta_R+(H-h)\eta_P+q_{h+1}$, and summing the recurrence from $h$ to $H$,
\begin{equation*}
q_h\;\le\;(H-h+1)\,\eta_R+\tbinom{H-h+1}{2}\,\eta_P\;\le\;H\eta_R+\tfrac12H^2\eta_P\;=\;\epsilon_{\text{rep}}(\phi).
\end{equation*}

\emph{Part (c), the recurrence for a fixed policy.} Let $\pi$ be Markov and $\phi$-measurable, and $e^\pi_h:=\max_s|V^\pi_h(s)-\overline V^\pi_{\phi,h}(\phi(s))|$. Both value functions use the same action $a=\pi_h(\phi(s))$. Hence, Eq.~\eqref{eq:step3-episodic} with $V^\pi$, $\overline V^\pi_\phi$ in place of $V^*$, $\overline V^*_\phi$ gives $e^\pi_h\le\eta_R+(H-h)\eta_P+e^\pi_{h+1}$ without the step $e_h\le q_h$, and the same sum gives $e^\pi_h\le\epsilon_{\text{rep}}(\phi)$.
\end{proof}

\subsection{Proof of Theorem~\ref{thm:ucbvi-fixed}}
\label{app:thm-fixed}

Under abstract optimism (Section~\ref{sec:abstraction}), we run an optimistic algorithm unchanged on $\mathcal X$, with $x_\tau=\phi(s_\tau)$ as the observed state. It estimates the unknown $\overline R_\phi(x,a)$ by the empirical mean $\widehat R^k_\phi(x,a)$ of the rewards observed at the $N_k(x,a)$ visits to $(x,a)$ before episode $k$ \citep{jaksch2010near}, and by $\widehat R^k_\phi(x,a):=1$ if $N_k(x,a)=0$. Given $(x_\tau,a_\tau)$, the next abstract state is drawn from $\phi_\#P^*(\cdot\mid s_\tau,a_\tau)$, a law that depends on $s_\tau$. The samples of a row are thus not identically distributed. Optimism nevertheless survives up to $\epsilon_{\text{rep}}(\phi)$ per episode, without widening the width, because every sample law lies within $\eta_P(\phi)$ of $\overline P_\phi(\cdot\mid x_\tau,a_\tau)$, as we now show.

Fix $\phi$, abbreviate $\eta_R:=\eta_R(\phi)$ and $\eta_P:=\eta_P(\phi)$, and let $\eta:=\eta_R+\frac12H\eta_P$, so that $H\eta=\epsilon_{\text{rep}}(\phi)$ and
\begin{equation}
H\eta_R + H^2\eta_P \;\le\; 2\,\epsilon_{\text{rep}}(\phi). \label{eq:eta-to-rep}
\end{equation}
We may assume $\epsilon_{\text{rep}}(\phi)\le H$, since otherwise every bound below exceeds the trivial $HK$. Let $P_{k,h}:=\phi_\#P^*(\cdot\mid s_{k,h},a_{k,h})$ be the law of $x_{k,h+1}$ given everything observed before it, ground states included. The policies $\pi^k$ of the agents below are Markov and $\phi$-measurable within an episode.

\textbf{Three-way decomposition.} For each $k$, at $s^k_1$ and $\phi(s^k_1)$,
\begin{equation}
V^*_1 - V^{\pi^k}_1 \;=\; \underbrace{\big(V^*_1 - \overline V^*_{\phi,1}\big)}_{\text{(I) optimal-value bias}} \;+\; \underbrace{\big(\overline V^*_{\phi,1} - \overline V^{\pi^k}_{\phi,1}\big)}_{\text{(II) abstract-MDP regret}} \;+\; \underbrace{\big(\overline V^{\pi^k}_{\phi,1} - V^{\pi^k}_1\big)}_{\text{(III) policy-evaluation gap}}. \label{eq:fixed-split}
\end{equation}
Term~(III) is not zero, since the projected process $\{\phi(s_\tau)\}$ is not the Markov chain $\overline{\mathcal M}_\phi$, but Lemma~\ref{lem:bias-metric}(b)--(c) bounds (I) and (III) by $\epsilon_{\text{rep}}(\phi)$ each. Term~(II) is the regret of $\pi^k$ in $\overline{\mathcal M}_\phi$, whose reward is unknown and whose kernel $\overline P_\phi$ is not the law $\phi_\#P^*(\cdot\mid s_\tau,a_\tau)$ of the observed abstract transitions, which depends on the ground state $s_\tau$.

\textbf{The reward estimate.} The agents use
\begin{equation}
\widehat R^k_\phi(x,a) = \begin{cases}\displaystyle\frac{1}{N_k(x,a)}\sum_{\tau:\,\phi(s_\tau)=x,\ a_\tau=a} R^*(s_\tau,a_\tau) & \text{if } N_k(x,a)\ge1,\\[1ex] 1 & \text{if } N_k(x,a)=0.\end{cases}
\label{eq:rhat}
\end{equation}
For $N_k(x,a)\ge1$, let $\widehat W_k(\cdot\mid x,a)$ be the empirical distribution of the ground states realised at the visits to $(x,a)$ before episode $k$. Since $R^*$ is deterministic,
\begin{equation}
\widehat R^k_\phi(x,a) = \sum_{s\in\phi^{-1}(x)}\widehat W_k(s\mid x,a)R^*(s,a), \qquad \overline R_\phi(x,a) = \sum_{s\in\phi^{-1}(x)}W(s\mid x)R^*(s,a),
\label{eq:rhat-convex}
\end{equation}
two convex combinations of $\{R^*(s,a):s\in\phi^{-1}(x)\}$, a set of diameter $\eta_R$. With $p:=\widehat W_k$ and $q:=W$,
\begin{equation}
\big|\widehat R^k_\phi(x,a)-\overline R_\phi(x,a)\big| = \Big|\sum_{s,s'}p(s)q(s')\big(R^*(s,a)-R^*(s',a)\big)\Big| \le \eta_R,
\label{eq:rhat-bias-proof}
\end{equation}
deterministically. At an unvisited pair, $\widehat R^k_\phi(x,a)=1\ge\overline R_\phi(x,a)$. Hence, for all $k$ and $(x,a)$,
\begin{equation}
\overline R_\phi(x,a)-\eta_R\;\le\;\widehat R^k_\phi(x,a)\;\le\;\overline R_\phi(x,a)+\eta_R+\mathds{1}\{N_k(x,a)=0\}.
\label{eq:rhat-two-sided}
\end{equation}

\textbf{The samples of a row.} A \emph{row} is a pair $(x,a)$ pooled over steps, for UCBVI, or a triple $(x,a,h)$, for Bayes-UCBVI. At the $j$-th visit of a row, the next abstract state $x'_j$ has conditional law $P_j$, one of the laws $P_{k,h}$, given the history. Let $\widetilde P_n$ be the empirical law of $x'_1,\dots,x'_n$ and $\bar P_n:=n^{-1}\sum_{j\le n}P_j$. The laws $P_j$, $P_{k,h}$ at a visit of the row, $\bar P_n$ and $\overline P_\phi(\cdot\mid x,a)$ all lie in the convex hull of $\{\phi_\#P^*(\cdot\mid s,a):s\in\phi^{-1}(x)\}$, a set of $\ell_1$-diameter at most $\eta_P$ by Lemma~\ref{lem:bias-metric}(a). As in Eq.~\eqref{eq:step1-episodic}, any two of them are within $\eta_P$ in $\ell_1$,
\begin{equation}
\|P_{k,h}-\bar P_n\|_1\le\eta_P,\qquad\|\bar P_n-\overline P_\phi(\cdot\mid x,a)\|_1\le\eta_P,\qquad\|P_{k,h}-\overline P_\phi(\cdot\mid x,a)\|_1\le\eta_P .
\label{eq:three-facts}
\end{equation}
For probability vectors $P,P'$ and every $f$, $|(P-P')f|\le\frac12\mathrm{span}(f)\|P-P'\|_1$, since $(P-P')f=(P-P')(f-c)$ for the midpoint $c$ of the range of $f$. The estimate of a row thus concentrates around $\bar P_n$, which lies within $\eta_P$ of $\overline P_\phi$,
\begin{equation}
\big\|\widetilde P_n-\overline P_\phi(\cdot\mid x,a)\big\|_1\;\le\;\big\|\widetilde P_n-\bar P_n\big\|_1+\eta_P .
\label{eq:misspec-split}
\end{equation}

\begin{lemma}[Concentration on a row]
\label{lem:row-conc}
Fix a row, a count $n\ge1$ and $\delta'\in(0,1/e]$, and consider the event that the row has at least $n$ visits.
\begin{itemize}
\item[(i)] For every $g:\mathcal X\to[0,b]$ fixed in advance, with probability at least $1-2\delta'$, on this event, $|(\widetilde P_n-\bar P_n)g|\le b\sqrt{\ln(1/\delta')/(2n)}$.
\item[(ii)] For every $y\in\mathcal X$, on this event and with probability at least $1-2\delta'$, $|\widetilde P_n(y)-\bar P_n(y)| \le 3\sqrt{\bar P_n(y)\,\iota'/n}+2\iota'/n$, with $\iota':=\ln\big((2+\log_2n)/\delta'\big)$.
\end{itemize}
\end{lemma}
\begin{proof}
Index the steps of play by $t$, let $\mathcal F_t$ be the history before $x_{t+1}$ is drawn and $P_t$ the law of $x_{t+1}$ given $\mathcal F_t$. The indicator $\chi_t$ that step $t$ is one of the first $n$ visits of the row is $\mathcal F_t$-measurable, and $\sum_t\chi_t\le n$, with equality on the event. For $g$ fixed in advance, $Z_t:=\chi_t\big(g(x_{t+1})-P_tg\big)$ satisfies $\mathds{E}[Z_t\mid\mathcal F_t]=0$, and $\sum_tZ_t=n(\widetilde P_n-\bar P_n)g$ on the event. (i)~$Z_t$ lies in an interval of length $b\chi_t$ with $\mathcal F_t$-measurable endpoints. By Hoeffding's lemma, $\exp\big(\lambda\sum_{t'\le t}Z_{t'}-\lambda^2b^2\sum_{t'\le t}\chi_{t'}/8\big)$ is a supermartingale for every $\lambda>0$. Markov's inequality at $\lambda=4\varepsilon/(b^2n)$ bounds the probability that $\sum_tZ_t\ge\varepsilon$ on the event by $e^{-2\varepsilon^2/(b^2n)}$, which is $\delta'$ at $\varepsilon=b\sqrt{n\ln(1/\delta')/2}$. The lower tail is symmetric. (ii)~For $g=\mathds 1\{y\}$, $|Z_t|\le\chi_t$ and $\mathds{E}[Z_t^2\mid\mathcal F_t]\le\chi_tP_t(y)$. For $\lambda\in(0,1]$, the inequality $e^u\le1+u+u^2$ for $u\le1$ gives $\mathds{E}[e^{\lambda Z_t}\mid\mathcal F_t]\le\exp(\lambda^2\chi_tP_t(y))$, making $\exp\big(\lambda\sum_{t'\le t}Z_{t'}-\lambda^2\sum_{t'\le t}\chi_{t'}P_{t'}(y)\big)$ a supermartingale. With $V:=n\bar P_n(y)$, for each fixed $\lambda$, $\sum_tZ_t\le\lambda V+\ln(1/\delta'')/\lambda$ on the event, with probability at least $1-\delta''$ \citep[cf.][]{freedman1975tail}. Take $\lambda\in\{2^{-i}:0\le i\le I\}$ with $I:=\lceil\frac12\log_2n\rceil$ and $\delta''=\delta'/(I+1)$, so that $\ln(1/\delta'')\le\iota'$ and $\iota'\ge1$. If $V\le\iota'$, then $\lambda=1$ gives $\sum_tZ_t\le2\iota'$. Otherwise $\lambda^*:=\sqrt{\iota'/V}$ lies in $[2^{-I},1)$, since $V\le n\le4^I\iota'$. Hence, a grid point lies in $(\lambda^*/2,\lambda^*]$, and $\sum_tZ_t\le\lambda^*V+2\iota'/\lambda^*=3\sqrt{V\iota'}$. The lower tail is symmetric, and dividing by $n$ gives (ii).
\end{proof}

Let $\iota:=\ln(16X^2AHT^2/\delta)$. Apply Lemma~\ref{lem:row-conc}(i) to the at most $XAH$ rows, the counts $n\le T$ and the functions $g=\overline V^*_{\phi,h+1}$, $h\le H$, which take values in $[0,H]$, each at level $\delta'=\delta/(8XAH^2T)$, and part~(ii) to the rows, the counts and the $y\in\mathcal X$, each at level $\delta'=\delta/(8X^2AHT)$. Since $2+\log_2T\le2T$, both logarithms are at most $\iota$, and a union bound gives an event $\mathcal E_{\mathrm{conc}}$ of probability at least $1-\delta/2$ on which, for every row, every count $n\ge1$ that the row reaches, every $h$ and every $y$,
\begin{equation}
\big|(\widetilde P_n-\bar P_n)\overline V^*_{\phi,h+1}\big|\le H\sqrt{\iota/(2n)},\qquad\big|\widetilde P_n(y)-\bar P_n(y)\big|\le3\sqrt{\bar P_n(y)\,\iota/n}+2\iota/n .
\label{eq:row-events}
\end{equation}
The first bound concerns fixed functions only. The second holds for all coordinates at once, which is what makes it applicable to value functions computed from the data. On $\mathcal E_{\mathrm{conc}}$, for every row with $n\ge1$ visits and every $\Delta:\mathcal X\to[0,B]$,
\begin{equation}
(\widetilde P_n-\bar P_n)\Delta\;\le\;\tfrac1H\,\bar P_n\Delta+5XHB\iota/n .
\label{eq:second-order}
\end{equation}
Indeed $(\widetilde P_n-\bar P_n)\Delta\le\sum_y|\widetilde P_n(y)-\bar P_n(y)|\Delta(y)$, and the AM--GM inequality gives
\begin{equation*}
3\sqrt{\bar P_n(y)\iota/n}\,\Delta(y)\;=\;3\sqrt{(\bar P_n(y)\Delta(y)/H)(H\iota\Delta(y)/n)}\;\le\;\bar P_n(y)\Delta(y)/H+\tfrac94H\iota\Delta(y)/n,
\end{equation*}
while $\Delta(y)\le B$ and $\frac94H+2\le5H$.

\subsubsection*{Part (a), UCBVI-CH}
UCBVI-CH runs Eq.~\eqref{eq:ucbvi-recursion} on $\mathcal X$ with $\widehat R^k_\phi$ in place of $R^*$, the empirical kernel $\widehat P_k=\widetilde P_n$ of the pooled row $(x,a)$, $n=N_k(x,a)$, and the bonus $b_k(x,a)=7HL_A/\sqrt n$ of Eq.~\eqref{eq:azar-bonus-ch}, $L_A=\ln(5XAT/\delta)$. Since $16X^2AHT^2/\delta\le(5XAT/\delta)^4$, $\iota\le4L_A\le98L_A^2$. Hence, on $\mathcal E_{\mathrm{conc}}$ the bonus dominates the first bound of Eq.~\eqref{eq:row-events}, $b_k(x,a)\ge H\sqrt{\iota/(2n)}$.

\emph{Approximate optimism.} On $\mathcal E_{\mathrm{conc}}$, $\widehat Q_{k,h}\ge\overline Q^*_{\phi,h}-(H-h+1)\eta$ for all $k$ and $h$. We prove the bound by an outer induction over $k$, in increasing order, and, for each $k$, an inner induction over $h$, in decreasing order. The outer hypothesis is the bound at $(k-1,h)$ for every $h$, and the inner hypothesis is the bound at $(k,h+1)$. The case $h=H+1$ is trivial, and at an unvisited pair $\widehat Q_{k,h}=H\ge\overline Q^*_{\phi,h}$. Otherwise the inner hypothesis gives $\widehat V_{k,h+1}\ge\overline V^*_{\phi,h+1}-(H-h)\eta$, and since $\widehat P_k$ is a probability vector, Eqs.~\eqref{eq:rhat-two-sided}, \eqref{eq:row-events} and~\eqref{eq:three-facts} give
\begin{align*}
\widehat R^k_\phi+\widehat P_k\widehat V_{k,h+1}+b_k\;&\ge\;\overline R_\phi-\eta_R+\widehat P_k\overline V^*_{\phi,h+1}-(H-h)\eta+b_k\\
&\ge\;\overline R_\phi-\eta_R+\bar P_n\overline V^*_{\phi,h+1}-(H-h)\eta\\
&\ge\;\overline R_\phi+\overline P_\phi\overline V^*_{\phi,h+1}-\eta_R-\tfrac12H\eta_P-(H-h)\eta=\overline Q^*_{\phi,h}-(H-h+1)\eta .
\end{align*}
The clamps preserve this bound, since $H\ge\overline Q^*_{\phi,h}$ and $\widehat Q_{k-1,h}\ge\overline Q^*_{\phi,h}-(H-h+1)\eta$ by the outer hypothesis, with $\widehat Q_{0,h}:=H$. In particular $\widehat V_{k,1}\ge\overline V^*_{\phi,1}-\epsilon_{\text{rep}}(\phi)$.

\emph{The recursion.} Let $\delta_{k,h}:=\widehat V_{k,h}(x_{k,h})-\overline V^{\pi^k}_{\phi,h}(x_{k,h})$, $(x,a):=(x_{k,h},a_{k,h})$, $n:=N_k(x,a)$, and $\xi_{k,h}:=P_{k,h}\big(\widehat V_{k,h+1}-\overline V^{\pi^k}_{\phi,h+1}\big)-\big(\widehat V_{k,h+1}-\overline V^{\pi^k}_{\phi,h+1}\big)(x_{k,h+1})$, a martingale difference with $|\xi_{k,h}|\le2H$, since $\widehat V_{k,h+1}-\overline V^{\pi^k}_{\phi,h+1}$ is fixed before episode $k$ and takes values in $[-H,H]$. By approximate optimism and $\overline V^*_{\phi}\ge\overline V^{\pi^k}_\phi$, $\delta_{k,h+1}+\xi_{k,h}=P_{k,h}(\widehat V_{k,h+1}-\overline V^{\pi^k}_{\phi,h+1})\ge-\epsilon_{\text{rep}}(\phi)$. If $n=0$, then $\delta_{k,h}\le H\le(1+1/H)(\delta_{k,h+1}+\xi_{k,h})+H+2\epsilon_{\text{rep}}(\phi)$. If $n\ge1$, greediness and the fact that the clamps only lower $\widehat Q_{k,h}$ give, at $(x,a)$,
\begin{equation*}
\delta_{k,h}\;\le\;\big(\widehat R^k_\phi-\overline R_\phi\big)+b_k+\big(\widehat P_k-\overline P_\phi\big)\widehat V_{k,h+1}+\overline P_\phi\big(\widehat V_{k,h+1}-\overline V^{\pi^k}_{\phi,h+1}\big).
\end{equation*}
The reward term is at most $\eta_R$. The span of $\widehat V_{k,h+1}-\overline V^{\pi^k}_{\phi,h+1}$ is at most $2H$. By Eq.~\eqref{eq:three-facts}, the last term is therefore at most $\delta_{k,h+1}+\xi_{k,h}+H\eta_P$. For the middle term, let $g:=\overline V^*_{\phi,h+1}$ and $\Delta:=\widehat V_{k,h+1}-g+(H-h)\eta$, which takes values in $[0,2H]$. Since both kernels are probability vectors, $(\widehat P_k-\overline P_\phi)\widehat V_{k,h+1}=(\widehat P_k-\overline P_\phi)g+(\widehat P_k-\overline P_\phi)\Delta$. By Eqs.~\eqref{eq:row-events} and~\eqref{eq:three-facts}, $(\widehat P_k-\overline P_\phi)g\le b_k+\frac12H\eta_P$. By Eqs.~\eqref{eq:second-order} with $B=2H$ and~\eqref{eq:three-facts},
\begin{equation*}
(\widehat P_k-\overline P_\phi)\Delta\le\frac1H\bar P_n\Delta+10XH^2\iota/n+H\eta_P,
\end{equation*}
and
\begin{equation*}
\bar P_n\Delta\;\le\;P_{k,h}\Delta+H\eta_P\;=\;P_{k,h}\big(\widehat V_{k,h+1}-g\big)+(H-h)\eta+H\eta_P\;\le\;\delta_{k,h+1}+\xi_{k,h}+H\eta+H\eta_P,
\end{equation*}
using $g\ge\overline V^{\pi^k}_{\phi,h+1}$. Collecting, with $\eta=\eta_R+\frac12H\eta_P$,
\begin{equation}
\delta_{k,h}\;\le\;\big(1+\tfrac1H\big)\big(\delta_{k,h+1}+\xi_{k,h}\big)+2b_k+10XH^2\iota/n+2\eta_R+4H\eta_P .
\label{eq:ch-recursion}
\end{equation}
Let $u_{k,h}$ be the sum of the last four terms if $n\ge1$, and $3H\ge H+2\epsilon_{\text{rep}}(\phi)$ if $n=0$. Unrolling from $h=1$ with $\delta_{k,H+1}=0$ gives $\delta_{k,1}\le e\sum_hu_{k,h}+\sum_h(1+1/H)^h\xi_{k,h}$.

\emph{Count sums.} For a pair $(x,a)$, let $n_k\le H$ be its visits in episode $k$ and $N_k=\sum_{j<k}n_j$. Episodes with $1\le N_k<H$ contribute fewer than $2H$ visits in total. If $N_k\ge H$, then $N_{k+1}\le2N_k$, which yields $n_k/\sqrt{N_k}\le2\sqrt2(\sqrt{N_{k+1}}-\sqrt{N_k})$ and $n_k/N_k\le2\ln(N_{k+1}/N_k)$, by $(b-a)/\sqrt b\le2(\sqrt b-\sqrt a)$ and $(b-a)/b\le\ln(b/a)$. Summing over episodes and pairs, with Cauchy--Schwarz over the pairs,
\begin{equation*}
\sum_{k,h:\,n\ge1}\frac1{\sqrt{N_k(x_{k,h},a_{k,h})}}\le2HXA+2\sqrt{2XAT},\qquad\sum_{k,h:\,n\ge1}\frac1{N_k(x_{k,h},a_{k,h})}\le2XA(H+\ln T),
\end{equation*}
and the steps at unvisited pairs number at most $XAH$. The biases sum to $K(2H\eta_R+4H^2\eta_P)\le8\epsilon_{\text{rep}}(\phi)K$ by Eq.~\eqref{eq:eta-to-rep}. By the Azuma--Hoeffding inequality, with probability at least $1-\delta/4$, $\sum_{k,h}(1+1/H)^h\xi_{k,h}\le2eH\sqrt{2T\ln(4/\delta)}$.

\emph{Collecting.} By approximate optimism $(\mathrm{II})_k\le\delta_{k,1}+\epsilon_{\text{rep}}(\phi)$, and (I) and (III) cost $\epsilon_{\text{rep}}(\phi)$ each. With probability at least $1-\delta$,
\begin{align*}
\mathfrak R_K\;\le\;&(3+8e)\,\epsilon_{\text{rep}}(\phi)K+14eHL_A\big(2HXA+2\sqrt2\sqrt{XAT}\big)\\
&+20eX^2AH^2\iota(H+\ln T)+3eXAH^2+2eH\sqrt{2T\ln(4/\delta)},
\end{align*}
which is part~(a), since the factor $L$ of Theorem~\ref{thm:ucbvi-fixed} satisfies $L\ge\iota\ge L_A$.

\subsubsection*{Part (b), Bayes-UCBVI}
We prove part~(b) through two lemmas that also serve calibrated AIF (Appendix~\ref{app:calibrated}). Both concern agents that use Bayes-UCBVI's rows $\rho^k_h$ on $\mathcal X$, with posterior mean $\widehat P^k_h=(n\widetilde P_n+n_0\delta_{x_{\mathrm E}})/\bar n$ for a row with $n$ visits, and that compute $V_{k,H+1}\equiv0$, $V_{k,h}(x_{\mathrm E})=r_0(H-h+1)$, and
\begin{equation}
Q_{k,h}(x,a)=\min\big\{r_0(H-h+1),\ \widehat R^k_\phi(x,a)+\widehat P^k_h(\cdot\mid x,a)^\top V_{k,h+1}+w_{k,h}(x,a)\big\},
\label{eq:dirichlet-agent}
\end{equation}
$V_{k,h}(x)=\max_aQ_{k,h}(x,a)$,
with a width $w_{k,h}\ge0$, and act greedily. Bayes-UCBVI is such an agent with $w_{k,h}=\Delta^k_h$, the quantile gap at $V_{k,h+1}$. Its clamp is never active, since $\widehat R^k_\phi\le1$ and the quantile is at most $r_0(H-h)$.

\begin{lemma}[Approximate optimism of Dirichlet-optimistic agents]
\label{lem:bucbvi-opt}
There is an event, denoted $\mathcal E_{\mathrm{opt}}$, of probability at least $1-\delta/5$ on which, for all $k,h,x,a$, with $\overline V^*_{\phi,h+1}$ extended by $r_0(H-h)$ at $x_{\mathrm E}$,
\begin{equation*}
\mathds{Q}_{p\sim\rho^k_h(x,a)}\big(p^\top\overline V^*_{\phi,h+1},\kappa^k_h(x,a)\big)\;\ge\;\big(\overline P_\phi\overline V^*_{\phi,h+1}\big)(x,a)-\tfrac12H\eta_P .
\end{equation*}
On $\mathcal E_{\mathrm{opt}}$, every agent of the form Eq.~\eqref{eq:dirichlet-agent} whose width is at least Bayes-UCBVI's quantile gap at its own next-step values, $w_{k,h}\ge\mathds{Q}_{p\sim\rho^k_h}(p^\top V_{k,h+1},\kappa^k_h)-\widehat P^k_hV_{k,h+1}$, satisfies $V_{k,h}\ge\overline V^*_{\phi,h}-(H-h+1)\eta$ on $\mathcal X$ for all $k$ and $h$.
\end{lemma}
\begin{proof}
\emph{The quantile.} Fix $(x,a,h)$ and let $f:=\overline V^*_{\phi,h+1}$, $b:=f(x_{\mathrm E})=r_0(H-h)$ and $u:=\overline P_\phi f-\frac12(H-h)\eta_P$. The claim is trivial for $h=H$, where $f\equiv0$. We therefore let $h<H$, whence $u<b$. At the $j$-th visit of the row, the conditional mean of $f$ is $P_jf\ge\overline P_\phi f-\frac12(H-h)\eta_P=u$, by Eq.~\eqref{eq:three-facts}, since $f$ takes values in $[0,H-h]$ on $\mathcal X$. The optimism lemma of \citet{tiapkin2022dirichlet} (their Lemmas~B.4--B.5 and Theorem~C.4) shows, for i.i.d.\ successors with mean $\overline P_\phi f$, that the $\kappa^k_h$-quantile exceeds $\overline P_\phi f$ for all $k$, $h$, $x$, $a$ on an event of probability at least $1-\delta/5$. Its only probabilistic step is the method of mixtures applied to the nonnegative martingales
\begin{equation*}
\prod_{j\le n}\big(1-\lambda(f(x'_j)-\overline P_\phi f)/(b-\overline P_\phi f)\big),
\end{equation*}
 $\lambda\in[0,1]$. With $u$ in place of $\overline P_\phi f$, each factor is nonnegative and has conditional mean $1-\lambda(P_jf-u)/(b-u)\le1$, making the products nonnegative supermartingales along the visits of the row, and Ville's inequality gives the same event with $u$ in place of $\overline P_\phi f$. The remaining steps of their proof, the anti-concentration of the Dirichlet distribution and the case $\bar n<2n_0$, are deterministic and use only $u\le H-h$. They give $\mathds{Q}(p^\top f,\kappa^k_h)\ge u\ge\overline P_\phi f-\tfrac12H\eta_P$.

\emph{The induction.} Let $d:=(H-h)\eta$, and assume $V_{k,h+1}\ge f-d$ on $\mathcal X$, with $V_{k,h+1}(x_{\mathrm E})=f(x_{\mathrm E})$. The quantile is monotone in the function and $p^\top(f-d\mathds 1_{\mathcal X})\ge p^\top f-d$, which yields
\begin{equation*}
\widehat P^k_hV_{k,h+1}+w_{k,h}\ge\mathds{Q}(p^\top V_{k,h+1},\kappa^k_h)\ge\mathds{Q}(p^\top f,\kappa^k_h)-d\ge\overline P_\phi f-\frac12H\eta_P-d.
\end{equation*}
With $\widehat R^k_\phi\ge\overline R_\phi-\eta_R$ (Eq.~\eqref{eq:rhat-two-sided}) and $\overline Q^*_{\phi,h}\le H-h+1\le r_0(H-h+1)$, this gives $Q_{k,h}\ge\overline Q^*_{\phi,h}-\eta-d$, hence $V_{k,h}\ge\overline V^*_{\phi,h}-(H-h+1)\eta$.
\end{proof}

\begin{lemma}[Direct regret bound for Dirichlet-optimistic agents]
\label{lem:direct}
Let $r_0=2$, and let an agent of the form Eq.~\eqref{eq:dirichlet-agent} satisfy $V_{k,h}\ge\overline V^*_{\phi,h}-(H-h+1)\eta$ on $\mathcal X$ for all $k$ and $h$, on an event $\mathcal E'$. There is an event $\mathcal E_{\mathrm{mart}}$ of probability at least $1-\delta/4$ such that, on $\mathcal E'\cap\mathcal E_{\mathrm{conc}}\cap\mathcal E_{\mathrm{mart}}$,
\begin{align*}
\mathfrak R_K\;\le\;&(3+13e)\,\epsilon_{\text{rep}}(\phi)K+e\sum_{k,h}w_{k,h}(x_{k,h},a_{k,h})+2eH\sqrt{\iota XAHT}\\
&+e\big(4n_0H+20XH^2\iota\big)XAH(1+\ln K)+4eXAH^2+3eH\sqrt{2T\ln(4/\delta)} .
\end{align*}
\end{lemma}
\begin{proof}
Let $\delta_{k,h}:=V_{k,h}(x_{k,h})-\overline V^{\pi^k}_{\phi,h}(x_{k,h})$, $(x,a):=(x_{k,h},a_{k,h})$, let $n$ be the visit count of the row $(x,a,h)$ before episode $k$, $\bar n=n+n_0$, and let $\xi_{k,h}$ be as in part~(a) with $V_{k,h+1}$ in place of $\widehat V_{k,h+1}$, now with $|\xi_{k,h}|\le(1+r_0)H$, since $V_{k,h+1}-\overline V^{\pi^k}_{\phi,h+1}$ takes values in $[-H,r_0H]$ on $\mathcal X$. As in part~(a), $\delta_{k,h+1}+\xi_{k,h}\ge-\epsilon_{\text{rep}}(\phi)$. If $n=0$, then
\begin{equation*}
\delta_{k,h}\le r_0H\le(1+1/H)(\delta_{k,h+1}+\xi_{k,h})+r_0H+2\epsilon_{\text{rep}}(\phi),
\end{equation*}
which happens at most once for each of the $XAH$ rows, since a row is visited at most once per episode. If $n\ge1$, greediness and the clamp give, at $(x,a)$,
\begin{equation*}
\delta_{k,h}\le\big(\widehat R^k_\phi-\overline R_\phi\big)+w_{k,h}+\big(\widehat P^k_h-\overline P_\phi\big)V_{k,h+1}+\overline P_\phi\big(V_{k,h+1}-\overline V^{\pi^k}_{\phi,h+1}\big),
\end{equation*}
where $\overline P_\phi$ puts no mass on $x_{\mathrm E}$. The reward term is at most $\eta_R$, and Eq.~\eqref{eq:three-facts} bounds the last term by $\delta_{k,h+1}+\xi_{k,h}+\frac12(1+r_0)H\eta_P$. For the middle term let $g:=\overline V^*_{\phi,h+1}$ on $\mathcal X$, $g(x_{\mathrm E}):=r_0(H-h)=V_{k,h+1}(x_{\mathrm E})$, and $\Delta:=V_{k,h+1}-g+(H-h)\eta$ on $\mathcal X$, which takes values in $[0,2r_0H]$. Then
\begin{equation*}
\big(\widehat P^k_h-\overline P_\phi\big)V_{k,h+1}=\big(\widehat P^k_h-\overline P_\phi\big)g+\tfrac n{\bar n}\widetilde P_n\Delta-\overline P_\phi\Delta+\tfrac{n_0}{\bar n}(H-h)\eta .
\end{equation*}
The first part equals
\begin{equation*}
\frac n{\bar n}(\widetilde P_n-\overline P_\phi)g+\frac{n_0}{\bar n}(g(x_{\mathrm E})-\overline P_\phi g),
\end{equation*}
 which is at most $H\sqrt{\iota/(2\bar n)}+\frac12H\eta_P+n_0r_0H/\bar n$ by Eqs.~\eqref{eq:row-events} and~\eqref{eq:three-facts} and $(n/\bar n)\sqrt{1/n}\le\sqrt{1/\bar n}$. Since $\Delta\ge0$, the middle two terms are at most
\begin{equation*}
\frac n{\bar n}(\widetilde P_n-\bar P_n)\Delta+(\bar P_n-\overline P_\phi)\Delta\le\frac1H\bar P_n\Delta+10r_0XH^2\iota/\bar n+r_0H\eta_P,
\end{equation*}
 by Eqs.~\eqref{eq:second-order} with $B=2r_0H$ and~\eqref{eq:three-facts}. As in part~(a), $\bar P_n\Delta\le\delta_{k,h+1}+\xi_{k,h}+(H-h)\eta+r_0H\eta_P$, and $(n_0/\bar n)(H-h)\eta\le n_0H/\bar n$. Collecting, with $r_0=2$ and $\eta_P\le H\eta_P$,
\begin{equation*}
\delta_{k,h}\le\big(1+\tfrac1H\big)\big(\delta_{k,h+1}+\xi_{k,h}\big)+w_{k,h}+H\sqrt{\iota/(2\bar n)}+\big(4n_0H+20XH^2\iota\big)/\bar n+2\eta_R+6.5H\eta_P .
\end{equation*}
A row visited $N\le K$ times contributes $\sum_{n=1}^{N-1}(n+n_0)^{-1/2}\le2\sqrt N$ and $\sum_{n=1}^{N-1}(n+n_0)^{-1}\le1+\ln K$, and $\sum_{\text{rows}}2\sqrt N\le2\sqrt{XAHT}$ by Cauchy--Schwarz. The biases sum to $K(2H\eta_R+6.5H^2\eta_P)\le13\epsilon_{\text{rep}}(\phi)K$. By the Azuma--Hoeffding inequality, the event $\mathcal E_{\mathrm{mart}}$ on which
\begin{equation*}
\sum_{k,h}(1+1/H)^h\xi_{k,h}\le e(1+r_0)H\sqrt{2T\ln(4/\delta)}
\end{equation*}
has probability at least $1-\delta/4$. Unrolling and summing as in part~(a) on $\mathcal E'\cap\mathcal E_{\mathrm{conc}}\cap\mathcal E_{\mathrm{mart}}$, with $(\mathrm{II})_k\le\delta_{k,1}+\epsilon_{\text{rep}}(\phi)$, gives the claim.
\end{proof}

\begin{proof}[Proof of part~(b)]
Lemma~\ref{lem:bucbvi-opt} with $w_{k,h}=\Delta^k_h$ gives approximate optimism on $\mathcal E_{\mathrm{opt}}$. By Lemma~\ref{lem:tiapkin-quantile} and $\mathds{V}_{\widehat P^k_h}[V_{k,h+1}]\le(r_0H)^2/4$, $\Delta^k_h\le r_0H\sqrt{\ell/\bar n}+2\sqrt2\,r_0H\ell/\bar n$, which yields
\begin{equation*}
\sum_{k,h}\Delta^k_h\le2r_0H\sqrt{\bar\ell XAHT}+2\sqrt2\,r_0H\bar\ell XAH(1+\ln K),
\end{equation*}
 with $\bar\ell:=\max_{k,h,x,a}\ell^k_h(x,a)=\mathcal O(\ln(XAHK/\delta))$. Lemma~\ref{lem:direct} applies on $\mathcal E_{\mathrm{opt}}\cap\mathcal E_{\mathrm{conc}}$, and the three failure probabilities sum to at most $\frac{19}{20}\delta$. This gives part~(b) with the factor $L=\max\{\iota,\bar\ell,n_0,1+\ln K\}$ of Theorem~\ref{thm:ucbvi-fixed}. Parts~(a) and~(b) sum the widths through the range of the value function. For Bayes-UCBVI this costs a factor $\sqrt H$ against the variance-aware analysis behind part~(c), which sums them through the law of total variance. We do not know whether that analysis survives misspecification, since its empirical-variance and total-variance steps would need concentration for data-dependent value functions under non-identical sampling.

\subsubsection*{Part (c), exact abstractions}
If $\eta_R=\eta_P=0$, then $\phi_\#P^*(\cdot\mid s,a)=\overline P_\phi(\cdot\mid x,a)$ and $R^*(s,a)=\overline R_\phi(x,a)$ for every $s\in\phi^{-1}(x)$. Given the history, $x_{k,h+1}$ therefore has law $\overline P_\phi(\cdot\mid x_{k,h},a_{k,h})$, and the reward is $\overline R_\phi(x_{k,h},a_{k,h})$. The abstract data are therefore generated by the episodic MDP $\overline{\mathcal M}_\phi$ on $\mathcal X$, which is time-homogeneous. An agent that observes only these data and is given $\overline R_\phi$ is UCBVI or Bayes-UCBVI run on $\overline{\mathcal M}_\phi$. Hence, Theorems~\ref{thm:ucbvi-ground} and~\ref{thm:bayesucbvi-ground} bound its regret in $\overline{\mathcal M}_\phi$, with $S$ replaced by $X$. Since Lemma~\ref{lem:bias-metric}(b)--(c) with $\epsilon_{\text{rep}}(\phi)=0$ make terms~(I) and~(III) of Eq.~\eqref{eq:fixed-split} vanish, its regret in $\overline{\mathcal M}_\phi$ equals its regret in $\mathcal M$.
\end{proof}

\subsection{The Cost of Acting Through an Abstraction}
\label{app:lb-abstraction}

\begin{proposition}[Acting through an abstraction has a linear cost]
\label{prop:lb-abstraction}
For every $H\ge4$ and every $\eta>0$ with $(H-2)\eta\le2$, some MDP with $S=4$ and $A=2$ has an abstraction $\phi$ with $X=3$, $\eta_R(\phi)=0$ and $\eta_P(\phi)=\eta$ on which every agent that sees the current state only through $\phi$ has regret $\mathfrak R_K\ge\epsilon_{\text{rep}}(\phi)K/32$ deterministically, while UCBVI on its ground states has $\mathfrak R_K=\widetilde{\mathcal O}(H\sqrt T)$.
\end{proposition}
The instance below merges two states whose best actions differ. An agent that cannot tell them apart acts correctly half of the time, and each wrong choice forfeits probability $\eta/2$ of entering a rewarding absorbing state, $\Theta(\eta H^2)=\Theta(\epsilon_{\text{rep}}(\phi))$ per episode. \citet[Thm.~2]{ortner2014selecting} show the average-reward analogue for the aggregated model's optimal gain, $\Theta(\epsilon D)$ in the diameter $D$.

\begin{proof}[Proof of Proposition~\ref{prop:lb-abstraction}]
\emph{The instance.} Let $\mathcal S=\{s_0,u_A,u_B,g\}$, $\mathcal A=\{a,b\}$, and $\delta_0:=\eta/4$, so that $(H-2)\delta_0\le\frac12$ by the range of $\eta$. Every episode starts at $s_0$. The reward is $R^*(g,\cdot)=1$ and $R^*=0$ elsewhere. From $s_0$ either action leads to $u_A$ or $u_B$ with probability $\frac12$ each. From $u_A$, action $a$ leads to $g$ with probability $2\delta_0$ and otherwise to $u_A$ or $u_B$ with probability $\frac12(1-2\delta_0)$ each, while action $b$ leads to $u_A$ or $u_B$ with probability $\frac12$ each. The state $u_B$ is the mirror image, with the roles of $a$ and $b$ exchanged. The state $g$ is absorbing. The abstraction maps $s_0\mapsto x_0$, $\{u_A,u_B\}\mapsto x_1$ and $g\mapsto x_g$, and thus $X=3$.

\emph{Spreads.} All states of the block $\{u_A,u_B\}$ have reward $0$ under both actions. Hence, $\eta_R(\phi)=0$. Under action $a$, $\phi_\#P^*(\cdot\mid u_A,a)$ puts mass $1-2\delta_0$ on $x_1$ and $2\delta_0$ on $x_g$, whereas $\phi_\#P^*(\cdot\mid u_B,a)$ puts all its mass on $x_1$. The two pushforwards are at $\ell_1$ distance $4\delta_0=\eta$, under $b$ as well. The other blocks are singletons. Hence $\eta_P(\phi)=\eta$ and $\epsilon_{\text{rep}}(\phi)=\frac12H^2\eta=2H^2\delta_0$.

\emph{Values.} Call $\{u_A,u_B\}$ the pool. From the pool, the optimal ground policy, which plays $a$ at $u_A$ and $b$ at $u_B$, enters $g$ with probability $r^*=2\delta_0$ at every step. Consider any agent whose action at a pool step depends on the current state only through $\phi$, possibly at random and on the whole history. Every transition into the pool lands on $u_A$ and $u_B$ with equal probability, whatever the previous state and action, and everything the agent observes (abstract states, actions, rewards) has the same law under both. The current ground state is therefore uniform on $\{u_A,u_B\}$ given the agent's information, its action is the correct one with probability exactly $\frac12$, and it enters $g$ with probability $r=\frac12\cdot2\delta_0=\delta_0$ at every pool step. From the pool with $m$ transitions left, the value under an entry probability $r$ is $v(m;r)=\sum_{j=1}^m\big(1-(1-r)^j\big)$, since $g$ is absorbing and the agent is in $g$ at the $j$-th subsequent step with probability $1-(1-r)^j$. The episode spends step $1$ at $s_0$ and step $2$ in the pool. Hence, $V^*_1(s_0)=v(H-2;2\delta_0)$ and $V^\pi_1(s_0)=v(H-2;\delta_0)$ for every such agent's policy $\pi$, and
\begin{equation}
G\;:=\;V^*_1(s_0)-V^\pi_1(s_0)\;=\;\sum_{j=1}^{H-2}\Big[(1-\delta_0)^j-(1-2\delta_0)^j\Big].
\end{equation}

\emph{Lower bound.} Since $1-2\delta_0\le(1-\delta_0)^2$, each summand is at least $(1-\delta_0)^j\big(1-(1-\delta_0)^j\big)$. For $j\le H-2$ we have $j\delta_0\le\frac12$, hence $(1-\delta_0)^j\ge1-j\delta_0\ge\frac12$ and
\begin{equation*}
1-(1-\delta_0)^j\ge1-e^{-j\delta_0}\ge j\delta_0-\frac12(j\delta_0)^2\ge\frac34j\delta_0.
\end{equation*}
Each summand is thus at least $\frac38j\delta_0$, and
\begin{align}
G\;&\ge\;\frac38\,\delta_0\sum_{j=1}^{H-2}j\;=\;\frac{3}{16}\,\delta_0\,(H-2)(H-1) \nonumber\\
&=\;\frac{3}{32}\,\epsilon_{\text{rep}}(\phi)\,\frac{(H-2)(H-1)}{H^2}\;\ge\;\frac{9}{256}\,\epsilon_{\text{rep}}(\phi)\;\ge\;\frac{\epsilon_{\text{rep}}(\phi)}{32},
\end{align}
using that $(H-2)(H-1)/H^2=1-3/H+2/H^2$ is increasing for $H\ge2$ and equals $\frac38$ at $H=4$. Every episode starts at $s_0$ and every admissible policy has the same value. Hence, $\mathfrak R_K=KG\ge\epsilon_{\text{rep}}(\phi)K/32$ for every agent in the class, deterministically.

\emph{Ground optimism.} Since the ground MDP has $S=4$ and $A=2$, Theorem~\ref{thm:ucbvi-ground} gives
\begin{equation*}
\mathfrak R_K\le30HL\sqrt{8K}+2500\cdot32\,H^2L^2+4H^{3/2}\sqrt{KL}=\widetilde{\mathcal O}(H\sqrt T).
\end{equation*}
\end{proof}

\subsection{The Price of an Unstructured Class and Proof of Corollary~\ref{cor:minimax}}
\label{app:unstructured}

\begin{corollary}[Minimax price of a fixed abstraction]
\label{cor:minimax}
For $X\ge4$ and $0<\bar\varepsilon\le H$, let $\mathcal C(X,\bar\varepsilon)$ contain the episodic MDPs with $A=2$ and an abstraction $\phi$ onto $X$ blocks with $\epsilon_{\text{rep}}(\phi)\le\bar\varepsilon$. For $H\ge4$ and $K\ge X$, every agent that observes the state only through $\phi$ has
\begin{equation*}
\sup_{\mathcal C(X,\bar\varepsilon)}\mathds{E} \mathfrak R_K\;\ge\;\bar\varepsilon K/64+\sqrt{HXT}/600.
\end{equation*}
UCBVI-CH on $\phi$ attains $\mathcal O(\bar\varepsilon K)+\widetilde{\mathcal O}(H\sqrt{XT})$ with probability $1-\delta$ for $T\ge H^4X^3A$.
\end{corollary}
The corollary combines two lower bounds. The representation term $\bar\varepsilon K$ comes from Proposition~\ref{prop:lb-abstraction}, whose abstraction merges two states whose best actions differ. An agent that acts through it cannot tell them apart and errs at one of them in every episode. The exploration term $\sqrt{HXT}$ comes from Proposition~\ref{prop:unstructured} below, applied with $S=X$ and the identity as the abstraction. Its MDPs differ only in which action is better at each state, and every agent must test both actions at every state to find out. The proposition is stated for an unknown decoder because it also gives the lower bound of row~14 of Table~\ref{tab:regret} for learned abstractions. The upper bound matches the representation term up to a constant and the exploration term up to $\sqrt H$ and logarithmic factors.

\begin{proposition}[The price of an unstructured abstraction]
\label{prop:unstructured}
For every $H\ge3$, $S\ge4$, and $K\ge S$ there is a family $\mathcal C_{\mathrm u}$ of MDPs with $S$ states and $A=2$ actions, each a realizable block MDP with at most five latent states, latent rewards, and separated latent states, whose decoders range over all assignments of $S-3$ ground states to two latent states, such that every agent has $\max_{\mathcal M\in\mathcal C_{\mathrm u}}\mathds{E}_{\mathcal M}\mathfrak R_K\ge\sqrt{HST}/300$. Furthermore, an agent that knows the decoder in advance has regret $\widetilde{\mathcal O}(H\sqrt T)$, and UCBVI on the ground states $\widetilde{\mathcal O}(\sqrt{HST})$ for $T\ge H^3S^3$ and $S\ge H/2$.
\end{proposition}
Each of the $S-3$ ground states $u_i$ carries an unknown label $\sigma_i$ that decides which action is better there \citep[cf.][]{jaksch2010near}. Under a uniform prior over labellings, the labels are independent. Data observed at $u_i$ therefore carry no information about $\sigma_j$ for $j\ne i$, and every agent must test both actions at every $u_i$. As a block MDP, the family has the class of all labellings as its candidate class, $\ln|\Phi|=(S-3)\ln2$, and Theorem~\ref{thm:enactive} matches the lower bound in $S$ and $T$.

\begin{proof}[Proof of Proposition~\ref{prop:unstructured}]
\emph{The family.} Let $S':=S-3\ge1$, $\mathcal S=\{s_0,u_1,\dots,u_{S'},g,z\}$, and $\mathcal A=\{1,2\}$. Every episode starts at $s_0$. For a labelling $\sigma\in\{-1,+1\}^{S'}$, both actions lead from $s_0$ to a uniformly random $u_i$. At $u_i$, action~1 leads to $g$ with probability $\frac12+\sigma_i\varepsilon$ and to $z$ otherwise, and action~2 leads to $g$ with probability $\frac12-\sigma_i\varepsilon$ and to $z$ otherwise. The states $g$ and $z$ are absorbing, $R^*(g,\cdot)=1$, and $R^*=0$ elsewhere, with $\varepsilon\in(0,\frac14]$ fixed below. The latent states are $s_0$, $g$, $z$, and those of $u^+$ and $u^-$ whose blocks are nonempty, so that $X\le5$. The decoder maps $u_i$ to $u^{\sigma_i}$ and every other state to itself, the emission of $u^\pm$ is uniform over its block, and the latent rows of $u^+$ and $u^-$ differ by $4\varepsilon$ in $\ell_1$ under either action. Every member of the family is therefore a realizable block MDP with $X\le5$, latent rewards, and separated latent states. As $\sigma$ varies, the decoder ranges over all labellings of $u_1,\dots,u_{S'}$.

\emph{Regret.} From step~$3$ on the agent is in $g$ or $z$, where its actions do not matter. An episode's value thus depends only on the action taken at the state $u_i$ visited at step~$2$. The action matching $\sigma_i$ enters $g$ with probability $\frac12+\varepsilon$, the other with probability $\frac12-\varepsilon$, and each of the remaining $H-2$ steps in $g$ yields reward~$1$. Hence
\begin{equation*}
V^*_1(s_0)-V^{\pi^k}_1(s_0)\;=\;\frac{2\varepsilon(H-2)}{S'}\sum_{i\le S'}\pi^k\big(\text{wrong action}\mid u_i\big),
\end{equation*}
where $\pi^k(\text{wrong action}\mid u_i)$ is the probability that the agent's policy for episode $k$ plays, at a step-2 visit of $u_i$, the action that does not match $\sigma_i$.

\emph{A Bayesian lower bound.} Draw $\sigma$ uniformly. The number $N_i(k)$ of visits to $u_i$ before episode $k$ is Binomial$(k-1,1/S')$, independent of $\sigma$ and of the agent, and the labels are independent across $i$. Recode each outcome observed at $u_i$ as $\omega=\mathds{1}\{g\}$ after action~1 and $\omega=\mathds{1}\{z\}$ after action~2. Whatever actions the agent takes, the recoded outcomes are i.i.d.\ Bernoulli$(\frac12+\sigma_i\varepsilon)$ given $\sigma_i$, and the likelihood of the whole history depends on $\sigma_i$ only through them. Given the history before episode $k$, the agent's policy is fixed up to its own randomisation. It therefore plays the wrong action at $u_i$ with conditional probability at least $\min(q_i,1-q_i)$, where $q_i$ is the posterior probability of $\sigma_i=+1$. Given $N_i(k)=j$, the expectation of $\min(q_i,1-q_i)$ is the Bayes error of testing $\sigma_i$ from $j$ recoded outcomes,
\begin{equation*}
\frac12\big(1-\mathrm{TV}(\mathrm{Ber}(\frac12+\varepsilon)^{\otimes j},\mathrm{Ber}(\frac12-\varepsilon)^{\otimes j})\big).
\end{equation*}
By Pinsker's inequality and Eq.~\eqref{eq:kl-chi2},
\begin{equation}
\mathrm{TV}\;\le\;\sqrt{\tfrac j2\,D_{\mathrm{KL}}\big(\mathrm{Ber}(\tfrac12+\varepsilon)\,\big\|\,\mathrm{Ber}(\tfrac12-\varepsilon)\big)}\;\le\;\sqrt{\frac j2\cdot\frac{4\varepsilon^2}{\frac14-\varepsilon^2}}\;\le\;\varepsilon\sqrt{\frac{32j}3}.
\label{eq:tv-bernoulli}
\end{equation}
Let $J:=\lceil2K/S'\rceil$ and $\varepsilon:=\sqrt{3/(128J)}$, which is at most $\frac14$. For $j\le J$ the total variation is then at most $\frac12$ and the Bayes error at least $\frac14$. By Markov's inequality, $\mathds{P}(N_i(k)\le J)\ge\mathds{P}(\mathrm{Bin}(K,1/S')\le2K/S')\ge\frac12$ for every $k\le K$. Summing over $i$ and $k$,
\begin{equation*}
\mathds{E} \mathfrak R_K\;\ge\;\frac{2\varepsilon(H-2)}{S'}\sum_{i\le S'}\sum_{k\le K}\frac14\cdot\frac12\;=\;\frac{\varepsilon(H-2)K}4.
\end{equation*}
Since $K\ge S'$, we have $J\le3K/S'$ and $\varepsilon\ge\sqrt{S'/(128K)}$, which yields $\mathds{E}\mathfrak R_K\ge(H-2)\sqrt{S'K}/(32\sqrt2)$. With $K=T/H$, $(H-2)/\sqrt H\ge\sqrt H/3$ for $H\ge3$, and $S'\ge S/4$ for $S\ge4$, this is at least $\sqrt{HST}/(192\sqrt2)\ge\sqrt{HST}/300$, and some $\sigma$ attains at least the average.

\emph{Upper bounds.} An agent that knows $\sigma$ runs UCBVI-CH on the at most five latent states, an exact abstraction with $X\le5$ and $A=2$. By Theorem~\ref{thm:ucbvi-fixed}(a) with $\epsilon_{\text{rep}}=0$, its regret is $\mathcal O(HL\sqrt T+H^3L^2+H^2L\ln T)=\widetilde{\mathcal O}(H\sqrt T)$ for $T\ge H^4$. UCBVI-BF on the ground states has the regret of Theorem~\ref{thm:ucbvi-ground} with $A=2$, $\widetilde{\mathcal O}(\sqrt{HST})$ for $T\ge H^3S^3$ and $2S\ge H$.
\end{proof}

\begin{proof}[Proof of Corollary~\ref{cor:minimax}]
The class $\mathcal C(X,\bar\varepsilon)$ contains two hard instances. The first is the MDP of Proposition~\ref{prop:lb-abstraction} with $\eta:=2\bar\varepsilon/H^2$, so that $\epsilon_{\text{rep}}(\phi)=\frac12H^2\eta=\bar\varepsilon$ and $(H-2)\eta\le2$ because $\bar\varepsilon\le H$, padded with $X-3$ unreachable singleton blocks. On it every agent that acts through $\phi$ has $\mathfrak R_K\ge\bar\varepsilon K/32$. The second is the family above with $S=X$ and $\phi$ the identity, so that $\epsilon_{\text{rep}}(\phi)=0$. Its proof never uses that the agent cannot see the states, since the labels $\sigma$ remain unknown. Every agent therefore has $\mathds{E}\mathfrak R_K\ge\sqrt{HXT}/300$ on some member. Hence
\begin{equation*}
\sup_{\mathcal C(X,\bar\varepsilon)}\mathds{E}\mathfrak R_K\ge\max\{\bar\varepsilon K/32,\sqrt{HXT}/300\}\ge\bar\varepsilon K/64+\sqrt{HXT}/600.
\end{equation*}
The upper bound is Theorem~\ref{thm:ucbvi-fixed}(a) with $A=2$ and $\epsilon_{\text{rep}}(\phi)\le\bar\varepsilon$, whose lower-order terms are dominated by $HL\sqrt{XAT}$ for $T\ge H^4X^3A$.
\end{proof}

\emph{The crossover horizon.} On every MDP with $SA\ge H$, UCBVI-BF on the ground states has regret at most $30HL\sqrt{SAK}+2500H^2S^2AL^2+4H^{3/2}\sqrt{KL}$ (Theorem~\ref{thm:ucbvi-ground}, with its logarithmic factor $L$). Each of these terms is at most $\bar\varepsilon K/192$ as soon as $T\ge T^*(\phi):=c\,H^3SAL^2\big(\bar\varepsilon^{-2}+S\bar\varepsilon^{-1}\big)$, for an absolute constant $c$, where the third term is handled by $SA\ge H$. Beyond $T^*(\phi)$ the ground bound is therefore below $\bar\varepsilon K/64$, the worst-case regret over $\mathcal C(X,\bar\varepsilon)$ of every agent acting through $\phi$, and acting through $\phi$ is worse in the worst case. For $T\ge H^3S^3AL^2$ the ground bound is at most $c'HL\sqrt{SAK}$. An abstraction can thus be competitive in the worst case only if $\bar\varepsilon\le64c'H^{3/2}L\sqrt{SA/T}$. Below $T^*(\phi)$, acting through the abstraction can be beneficial. The bound of Theorem~\ref{thm:ucbvi-fixed}(a) is smaller than that of UCBVI-CH on the ground states when $\epsilon_{\text{rep}}(\phi)K\lesssim H\sqrt{AT}(\sqrt S-\sqrt X)$. This comparison orders our guarantees rather than the regrets themselves.

\section{Proofs for Section~\ref{sec:fixed-goals}}
\label{app:proofs-fixed}

\subsection{Exact Identity and Bounds for Novelty}
\label{app:novelty-approx}

Let $\rho=\mathrm{Dir}(\alpha)$ over $m$ categories with $\alpha_i>0$, $\bar\alpha=\sum_i\alpha_i$, $\bar p=\alpha/\bar\alpha$, and $B\sim\rho$. For any mixture, the mutual information between the mixing variable and the outcome is the expected divergence of the outcome's conditional law from its mean law. Hence, $\mathrm{Novelty}=\mathds{E}_{B\sim\rho}[D_{\mathrm{KL}}(B\|\bar p)]$. With the digamma function $\psi$, \citet[Lemma~15]{lu2019information} give the closed form
\begin{equation}
\mathrm{Novelty} \;=\; \sum_i \bar p_i\big[\psi(\alpha_i+1)-\ln\alpha_i\big] \;-\; \big[\psi(\bar\alpha+1)-\ln\bar\alpha\big]. \label{eq:novelty-exact}
\end{equation}

\emph{Upper bound.} For probability vectors $p,q$ with $q_i>0$, the inequality $\ln u\le u-1$ at $u=p_i/q_i$ gives the $\chi^2$ bound on the divergence,
\begin{equation}
D_{\mathrm{KL}}(p\|q) \;\le\; \sum_i p_i\Big(\frac{p_i}{q_i}-1\Big) \;=\; \sum_i\frac{(p_i-q_i)^2}{q_i}.
\label{eq:kl-chi2}
\end{equation}
We apply it at $p=B$ and $q=\bar p$ and take expectations. With the Dirichlet variance $\mathrm{Var}(B_i)=\bar p_i(1-\bar p_i)/(\bar\alpha+1)$, we obtain $\mathrm{Novelty}\le\sum_i\mathrm{Var}(B_i)/\bar p_i=(m-1)/(\bar\alpha+1)$. For a row $\rho^k_h(x,a)$, $\bar\alpha=\bar n^k_h(x,a)$, which yields
\begin{equation}
\mathrm{Novelty}(x,a) \;\le\; \frac{m-1}{\bar n^k_h(x,a)+1} \qquad \text{at every count}. \label{eq:novelty-exact-ub}
\end{equation}

\emph{Lower bound.} If every category has $\alpha_i\ge1$, as for posteriors with integer parameters, then
\begin{equation}
\mathrm{Novelty}\;\ge\;\frac{2m-3}{6\bar\alpha}\;\ge\;\frac{m-1}{6\bar\alpha}\qquad(m\ge2).
\label{eq:novelty-lb}
\end{equation}
Indeed, the bounds $\ln(z+\frac12)\le\psi(z+1)\le\ln z+1/(2z)$ for $z>0$ \citep[Lemma~16]{lu2019information} give $\psi(\alpha_i+1)-\ln\alpha_i\ge\ln(1+1/(2\alpha_i))\ge1/(2\alpha_i+1)$ and $\psi(\bar\alpha+1)-\ln\bar\alpha\le1/(2\bar\alpha)$. Eq.~\eqref{eq:novelty-exact} then gives
\begin{equation*}
\mathrm{Novelty}\ge\frac1{\bar\alpha}\sum_i\frac{\alpha_i}{2\alpha_i+1}-\frac1{2\bar\alpha}\ge\frac m{3\bar\alpha}-\frac1{2\bar\alpha},
\end{equation*}
since $\alpha/(2\alpha+1)\ge\frac13$ for $\alpha\ge1$. Up to a factor~$6$, novelty is therefore the number of categories with mass, minus one, per unit of posterior mass.

\subsection{Proof of Lemma~\ref{lem:novelty-order}}
\label{app:lem-novelty-order}

\begin{proof}
Fix $(x,a,h)$ with visit count $n$ and write $\rho^k_h(x,a)=\mathrm{Dir}(\alpha)$, $\bar\alpha=\bar n^k_h(x,a)=\bar n=n+n_0$, $\bar p=\alpha/\bar\alpha$, and $m\le X+1$ for the number of categories with $\alpha_i>0$.

\emph{(i).} This follows from Eq.~\eqref{eq:novelty-exact-ub} (Appendix~\ref{app:novelty-approx}) and $m-1\le X$. The bound involves only $\bar\alpha$. In particular, novelty does not depend on $\overline V^k_{h+1}$, which is what the rest of the proof turns on.

\emph{(ii).} The quantile-gap bound of Lemma~\ref{lem:tiapkin-quantile} is an upper bound of order $H\sqrt\ell\,\bar n^{-1/2}$, which we do not use in the other direction. Part~(ii) replaces the matching lower bound by a statement about the lower tail of the deviation itself, which is what validity constrains. Since the abstraction is exact, in the standard model of the data as a table of independent samples drawn in advance for each row, the successors $x_1,x_2,\dots$ recorded at the row's successive visits are i.i.d.\ draws from $\overline P_\phi(\cdot\mid x,a)$, whatever algorithm chooses when to visit the row. Fix $n$ and set
\begin{equation}
Z_j := \overline V^*_{\phi,h+1}(x_j) - \big(\overline P_\phi\overline V^*_{\phi,h+1}\big)(x,a), \qquad j=1,\dots,n.
\end{equation}
Because $\overline V^*_{\phi,h+1}$ is a \emph{fixed} function that does not depend on the data, the $Z_j$ are i.i.d., mean zero, $|Z_j|\le H$, and $\mathds{E} Z_j^2=\sigma_V^2$. The error of the empirical kernel of the first $n$ samples is $e_n=n^{-1}\sum_{j\le n}Z_j$.

Expanding the fourth moment of a sum of independent centred variables, in which only the terms $\mathds{E} Z_j^4$ and the $3n(n-1)$ terms $\mathds{E} Z_i^2Z_j^2$ survive,
\begin{equation}
\mathds{E} e_n^4 \;=\; \frac{1}{n^4}\Big(n\,\mathds{E} Z_1^4 + 3n(n-1)\sigma_V^4\Big) \;\le\; \frac{H^2\sigma_V^2}{n^3} + \frac{3\sigma_V^4}{n^2}\;\le\;\frac{4\sigma_V^4}{n^2},
\label{eq:fourth-moment}
\end{equation}
using $\mathds{E} Z_1^4\le H^2\mathds{E} Z_1^2=H^2\sigma_V^2$ and, in the last step, $\sigma_V^2\ge\nu H^2$ and $n\nu\ge1$, so that $H^2/(n\sigma_V^2)\le1$. Hölder's inequality gives $\mathds{E} e_n^2\le(\mathds{E}|e_n|)^{2/3}(\mathds{E} e_n^4)^{1/3}$, hence $\mathds{E}|e_n|\ge(\mathds{E} e_n^2)^{3/2}/(\mathds{E} e_n^4)^{1/2}\ge\sigma_V/(2\sqrt n)$. Since $\mathds{E} e_n=0$, the negative part $e_n^-:=\max\{-e_n,0\}$ has $\mathds{E} e_n^-=\tfrac12\mathds{E}|e_n|\ge\sigma_V/(4\sqrt n)$, and $\mathds{E}(e_n^-)^2\le\mathds{E} e_n^2=\sigma_V^2/n$. The Paley--Zygmund inequality $\mathds{P}(e_n^-\ge\tfrac12\mathds{E} e_n^-)\ge\tfrac14(\mathds{E} e_n^-)^2/\mathds{E}(e_n^-)^2$ therefore gives
\begin{equation}
\mathds{P}\Big(e_n\le-\frac{\sigma_V}{8\sqrt n}\Big)\;\ge\;\frac1{64},
\end{equation}
and $\sigma_V\ge H\sqrt\nu$ gives the claim.

Consequently, let $b_n$ be any quantity, possibly computed from the same samples, with $\mathds{P}(b_n\ge-e_n)\ge1-\delta'$. Then
\begin{equation}
\mathds{P}\Big(b_n\ge \frac{H\sqrt\nu}{8\sqrt n}\Big)\;\ge\;\mathds{P}\Big(-e_n\ge \frac{H\sqrt\nu}{8\sqrt n}\Big)-\mathds{P}\big(b_n<-e_n\big)\;\ge\;\tfrac1{64}-\delta' ,
\end{equation}
which is positive when $\delta'<\frac1{64}$. A quantity that is below $H\sqrt\nu/(8\sqrt n)$ almost surely is therefore not valid at count $n$. Comparing with (i), and using $n\le\bar n$,
\begin{equation*}
\mathrm{Novelty}(x,a)\le X/(\bar n+1)<H\sqrt\nu/(8\sqrt{\bar n})\le H\sqrt\nu/(8\sqrt n)
\end{equation*}
as soon as $\bar n\ge64X^2/(\nu H^2)$, and in that regime novelty is of the order of the square of any valid width.

\emph{The posterior mean.} After the first $n$ samples, Bayes-UCBVI's posterior mean is $\widehat P^k_h=(n\widetilde P_n+n_0\delta_{x_{\mathrm E}})/\bar n$. With $\overline V^*_{\phi,h+1}$ extended by its value $r_0(H-h)$ at the pseudo-state, the shortfall of the posterior mean is
\begin{equation*}
\big(\overline P_\phi-\widehat P^k_h\big)\overline V^*_{\phi,h+1}\;=\;-\tfrac n{\bar n}e_n-\tfrac{n_0}{\bar n}\big(r_0(H-h)-\overline P_\phi\overline V^*_{\phi,h+1}\big)\;\ge\;-\tfrac n{\bar n}e_n-\frac{n_0r_0H}{\bar n}.
\end{equation*}
The pseudo-state shift is of novelty's order, $\mathcal O(\ln T/\bar n)$ for $n_0=\mathcal O(\ln T)$. Hence, on the event of (ii) the shortfall is still of order $\bar n^{-1/2}$ once $\bar n$ is large, and a valid width for the posterior mean must have this order as well.

\emph{Fixed versus data-dependent counts.} The statement concerns a count $n$ fixed in advance. At a count chosen by the data, the probability can be smaller. An agent that stops sampling the row as soon as its empirical error vanishes visits it a random number of times at which $e_n=0$. Regret consequences therefore go through Theorem~\ref{thm:exponent}, which handles the agent's adaptivity directly.

\emph{The variance condition.} If $\overline V^k_{h+1}$ is constant on the support of $\widehat P^k_h(\cdot\mid x,a)$ then $\mathds{V}_{\widehat P^k_h}[\overline V^k_{h+1}]=0$, Lemma~\ref{lem:tiapkin-quantile}'s leading term vanishes and $\Delta^k_h(x,a)=\mathcal O(H\ell\bar n^{-1})$, of strictly smaller order than $\bar n^{-1/2}$, while $\mathrm{Novelty}(x,a)$ is unchanged by (i). At such a row both quantities are of order $\bar n^{-1}$. Part~(ii) therefore needs the variance condition.
\end{proof}

\subsection{Proof of Theorem~\ref{thm:exponent}}
\label{app:exponent}

We prove Theorem~\ref{thm:exponent} in the following explicit form.

\begin{proposition}[Explicit form of Theorem~\ref{thm:exponent}]
\label{prop:exponent-explicit}
Let an information-driven agent with $n_0,r_0\ge1$, $q$, and $\lambda$ act on an exact abstraction.
\begin{itemize}
\item[(a)] Let $q\in(\frac12,1]$ and $\lambda\ge0$, and let $N^\star$ be the smallest integer $N$ with $4\lambda N^{-q}+n_0r_0/N\le1/(4\sqrt{2N})$. Then $N^\star \le \max\big\{(32\sqrt2\lambda)^{2/(2q-1)},\,128n_0^2r_0^2\big\}+1$.
On the MDP of the proof below, with probability at least $\frac1{24}$, the agent plays the optimal action at most $N^\star$ times in total, and $\mathds{E} \mathfrak R_K\ge(K-N^\star)/(192\sqrt{2N^\star})$ for every $K$.
\item[(b)] Let $q\in(0,\frac12)$, $0<\lambda\le r_0/2$, and $K\ge16(\max\{n_0,4^{1/q}\}+1)$. On the MDP of the proof below, whose gap depends on $K$, the regret satisfies $\mathfrak R_K \ge \min\big\{K,\,(1-(6/7)^q)\,\lambda K^{1-q}\big\}/32$.
\end{itemize}
\end{proposition}

The MDP of part~(a) depends on $n_0$, $r_0$, $q$, and $\lambda$ through $N^\star$, and on nothing else. If $n_0$ grows with $K$, as Bayes-UCBVI's $n_0=\lceil c_{n_0}+\log_{17/16}K\rceil$ does, then $N^\star=\mathcal O(\lambda^{2/(2q-1)}+\ln^2K)$. Part~(a) then gives $\mathds{E}\mathfrak R_K=\Omega(K/\ln K)$ at constant $\lambda$ and $r_0$, on an MDP that depends on $K$.

\emph{Conventions.} Throughout, the agent is the information-driven agent of Section~\ref{sec:exponent} on the exact abstraction. Its rows are therefore the ground rows and it uses $R=R^*$. Its pseudo-state prior itself acts as a bonus of exponent one, since it inflates the posterior mean of the next-step value by at most $n_0r_0H/\bar n$, which is why $n_0r_0$ enters $N^\star$. We write $\mathrm{Nov}_h(x,a)$ for the novelty of the row $\rho^k_h(x,a)$, whose counts include the $n_0$ pseudo-counts on $x_{\mathrm E}$. A row whose only category with positive mass is $x_{\mathrm E}$ has novelty zero, and a row with $m$ such categories has novelty at most $(m-1)/(\bar n+1)$ by Eq.~\eqref{eq:novelty-exact-ub}.

\begin{proof}[Proof of part~(a)]
\emph{The instance.} Let $\mathcal S=\{s_0,g,b,c\}$, $\mathcal A=\{1,2\}$ and $H=2$, and let every episode start at $s_0$. Action~1 at $s_0$ (\emph{risky}) leads to $g$ or $b$ with probability $\frac12$ each, action~2 (\emph{safe}) leads to $c$, and $g,b,c$ are absorbing under both actions. Rewards are $R^*(s_0,\cdot)=0$, $R^*(g,\cdot)=1$, $R^*(b,\cdot)=0$ and $R^*(c,\cdot)=\frac12-\Delta$, with $\Delta:=1/(8\sqrt{2N^\star})$. The risky action is optimal with $V^*_1(s_0)=\frac12$, and every episode that plays safe costs $\Delta$.

\emph{The safe index never falls below $\frac12-\Delta$.} Since the epistemic term is nonnegative and $V_{k,3}\equiv0$, $V_{k,2}(c)\ge R^*(c,\cdot)=\frac12-\Delta$, and the pseudo-state has value $V_{k,2}(x_{\mathrm E})=r_0\ge1$. The posterior mean of the safe row averages these two values. Hence, $Q_{k,1}(s_0,2)\ge\frac12-\Delta$, with or without a clamp at $2r_0$.

\emph{The risky index after $N$ risky episodes.} Let $N_g+N_b=N$ count the risky episodes that reached $g$ and $b$, and let $\widehat\mu_N=N_g/N$. At step $2=H$ the next-step value vanishes, at the pseudo-state as well. For $y\in\{g,b\}$, the value $V_{k,2}(y)$ is therefore at most $R^*(y,\cdot)+\lambda\max_a\mathrm{Nov}_2(y,a)^q$. The row $(y,a,2)$ records the self-loop at $y$. While unvisited, its novelty is zero, and once visited $n\ge1$ times, it has two categories and novelty in $(0,1/(n+n_0+1)]$. For $\lambda>0$, the agent maximises $\lambda\,\mathrm{Nov}_2(y,\cdot)^q$ at $y$ and therefore keeps playing the action it tried first there. (With a clamp and $r_0=1$, the two indices at $g$ may tie, but then $V_{k,2}(g)=1=R^*(g,\cdot)$.) Hence
\begin{equation*}
V_{k,2}(y)\;\le\;R^*(y,\cdot)+\lambda\,(N_y+n_0+1)^{-q}\;\le\;R^*(y,\cdot)+\lambda N_y^{-q}\qquad(N_y\ge1).
\end{equation*}
For $\lambda=0$, the bound holds trivially. The risky row has at most three categories. Its own epistemic term is thus at most $\lambda(2/(N+n_0+1))^q\le2\lambda N^{-q}$. Its posterior mean puts mass $n_0/(N+n_0)$ on the pseudo-state, whose value is $r_0$. Therefore
\begin{align}
Q_{k,1}(s_0,1)\;&\le\;\frac{N_g\big(1+\lambda N_g^{-q}\big)+N_b\lambda N_b^{-q}+n_0r_0}{N+n_0}+2\lambda N^{-q} \nonumber\\
&\le\;\widehat\mu_N+\lambda\,\frac{N_g^{1-q}+N_b^{1-q}}{N}+\frac{n_0r_0}{N}+2\lambda N^{-q}\;\le\;\widehat\mu_N+E_N,
\label{eq:risky-index}
\end{align}
with $E_N:=4\lambda N^{-q}+n_0r_0/N$, using $N_y^{1-q}\le N^{1-q}$, valid for $q\le1$, where a term with $N_y=0$ is absent. The risky row and the rows at $g$ and $b$ receive data only in risky episodes. After the $N$-th risky episode, the bound therefore holds at every later episode until the agent plays risky again.

\emph{Lock-out.} Draw $\omega_1,\omega_2,\dots$ i.i.d.\ Bernoulli$(\frac12)$ in advance, and let the $j$-th risky episode reach $g$ if and only if $\omega_j=1$. This does not change the law of the interaction. Let $\widehat\mu_N:=N^{-1}\sum_{j\le N}\omega_j$, which equals $N_g/N$ after $N$ risky episodes, and let $Z:=(\widehat\mu_{N^\star}-\frac12)^2$. The centred outcomes $\omega_j-\frac12=\pm\frac12$ give $\mathds{E}Z=1/(4N^\star)$ and $\mathds{E}Z^2\le3/(16(N^\star)^2)$, as in the fourth-moment computation of Eq.~\eqref{eq:fourth-moment} in Appendix~\ref{app:lem-novelty-order}. The Paley--Zygmund inequality therefore gives $\mathds{P}\big(Z\ge\mathds{E}Z/2\big)\ge\tfrac14(\mathds{E}Z)^2/\mathds{E}Z^2\ge\tfrac1{12}$. Since $\widehat\mu_{N^\star}-\frac12$ is symmetric about zero, the event $G:=\{\widehat\mu_{N^\star}\le\frac12-1/(2\sqrt{2N^\star})\}$ has probability at least $\frac1{24}$. On $G$, by the definition of $N^\star$,
\begin{equation*}
\widehat\mu_{N^\star}+E_{N^\star}\;\le\;\frac12-\frac1{2\sqrt{2N^\star}}+\frac1{4\sqrt{2N^\star}}\;=\;\frac12-2\Delta\;<\;\frac12-\Delta.
\end{equation*}
If the agent ever plays risky for the $N^\star$-th time, Eq.~\eqref{eq:risky-index} and the bound on the safe index therefore show that it never plays risky again. On $G$, the agent plays risky at most $N^\star$ times and safe at least $K-N^\star$ times. Hence,
\begin{equation*}
\mathds{E}\mathfrak R_K\;\ge\;\mathds{P}(G)(K-N^\star)\Delta\;\ge\;(K-N^\star)/(192\sqrt{2N^\star}).
\end{equation*}

\emph{The bound on $N^\star$.} The inequality $4\lambda N^{-q}\le1/(8\sqrt{2N})$ holds once $N^{q-1/2}\ge32\sqrt2\lambda$, and $n_0r_0/N\le1/(8\sqrt{2N})$ holds once $\sqrt N\ge8\sqrt2\,n_0r_0$. Together they give $E_N\le1/(4\sqrt{2N})$. Every integer $N\ge\max\{(32\sqrt2\lambda)^{2/(2q-1)},128n_0^2r_0^2\}$ therefore satisfies the defining inequality, and $N^\star$ obeys the stated bound.
\end{proof}

\begin{proof}[Proof of part~(b)]
Fix $K$. Let $\mathcal S=\{s_0,y\}$ and $H=2$, let both actions lead from $s_0$ to the absorbing state $y$, and let $R^*(s_0,1)=\frac12$, $R^*(s_0,2)=\frac12-\Delta$ and $R^*(y,\cdot)=0$, with
\begin{equation*}
\Delta:=\min\{\tfrac12,\Delta_0\},\qquad \Delta_0:=\tfrac12\lambda(3K/4)^{-q}(1-(6/7)^q),
\end{equation*}
so that $R^*(s_0,2)\in[0,1]$. Every episode that plays action~2 costs $\Delta$. Both actions lead to $y$. They therefore share the value $v:=V_{k,2}(y)=\lambda\max_a\mathrm{Nov}_2(y,a)^q\le\lambda2^{-q}\le r_0/2$. With $n_a$ the count of the row $(s_0,a,1)$,
\begin{equation*}
Q_{k,1}(s_0,a)\;=\;R^*(s_0,a)+v+\varphi(n_a),\qquad \varphi(n):=\frac{n_0}{n+n_0}(r_0-v)+\lambda\,\mathrm{Nov}(n)^q,
\end{equation*}
where $\mathrm{Nov}(n)$ is the novelty of the two-category row with counts $(n,n_0)$. It is zero for $n=0$, and for $n\ge1$ it lies between $1/(6(n+n_0))$ and $1/(n+n_0+1)$, by Eqs.~\eqref{eq:novelty-lb} and~\eqref{eq:novelty-exact-ub}. A clamp at $2r_0$ never binds, since $Q_{k,1}<\frac12+r_0+\lambda\le2r_0$. The agent plays action~2 only if $\varphi(n_2)-\varphi(n_1)\ge\Delta$, and action~1 only if $\varphi(n_2)-\varphi(n_1)\le\Delta$.

Let $K_1:=\max\{n_0,4^{1/q}\}+1$. Action~2 is played within the first $\lceil K_1\rceil$ episodes. Indeed, while $n_2=0$ and $n_1\ge K_1-1$, the bounds $r_0-v\ge r_0/2$, $n_1\ge n_0$, $\lambda\le r_0/2$ and $n_1^{-q}\le\frac14$ give
\begin{equation*}
\varphi(0)-\varphi(n_1)\;\ge\;\frac{r_0}2\cdot\frac{n_1}{n_1+n_0}-\lambda n_1^{-q}\;\ge\;\frac{r_0}4-\frac{r_0}8\;>\;\Delta,
\end{equation*}
since $\Delta\le\lambda(1-(6/7)^q)/2<r_0/8$. Now let $K\ge16K_1$, and suppose that $n_2<K/16$ after $K$ episodes. Then action~1 was played more than $15K/16$ times. At its last play, which comes after the first play of action~2, the counts satisfy $n_1\ge7K/8$ and $1\le n_2<K/16$. Since the first part of $\varphi$ is nonincreasing and $n_0\le K/16$, we attain
\begin{align*}
\varphi(n_2)-\varphi(n_1)\;&\ge\;\lambda\Big[\big(6(n_2+n_0)\big)^{-q}-(n_1+n_0+1)^{-q}\Big]\\
&\ge\;\lambda\Big[\big(\tfrac{3K}4\big)^{-q}-\big(\tfrac{7K}8\big)^{-q}\Big]\;=\;2\Delta_0\;\ge\;2\Delta,
\end{align*}
and action~2 had the strictly larger index, a contradiction. Hence $n_2\ge K/16$, and $\mathfrak R_K\ge K\Delta/16\ge\min\big\{K,\,(1-(6/7)^q)\lambda K^{1-q}\big\}/32$. \end{proof}

\emph{Exponent one half at a constant weight.} The same mechanism locks out a bonus of exponent $\frac12$, with a smaller but still $K$-independent probability. Take the instance of part~(a) with $R^*(c,\cdot)=\frac14$, $n_0=1$, $r_0=2$, $q=\frac12$, and $\lambda=1$. If the first $64$ risky episodes all reach $b$, which has probability $2^{-64}$, the risky row has counts $(64,1)$ on $b$ and the pseudo-state, and the row at $b$ has the same counts. The closed form of Eq.~\eqref{eq:novelty-exact} then gives a risky index of $0.191<\frac14$. The index stays there, since only risky episodes change it, whereas the safe index never falls below $\frac14$. Hence $\mathds{E}\mathfrak R_K\ge2^{-64}(K-64)/4$. A weight that grows with a confidence factor, as in calibrated AIF, makes the probability of such a run negligible.

\emph{A posterior-variance bonus.} Replace $\mathrm{Novelty}$ by $\mathds{V}_{B\sim\rho}[B^\top V_{k,h+1}]$ in the bonus $\lambda\,\mathrm{Novelty}^q$ of Section~\ref{sec:exponent}. At step $2=H$ the next-step value vanishes, at the pseudo-state as well. This bonus is thus zero on every step-2 row and $V_{k,2}(y)=R^*(y,\cdot)$ for $y\in\{g,b\}$. On the risky row, $V_{k,2}$ takes values in $[0,r_0]$, and the Dirichlet covariance gives $\mathds{V}_{B\sim\rho}[B^\top V_{k,2}]=\mathds{V}_{\bar p}[V_{k,2}]/(N+n_0+1)\le r_0^2/(4N)$. The safe index is still at least $\frac12-\Delta$, since the bonus is nonnegative. The proof of part~(a) therefore goes through with $E_N:=\lambda(r_0^2/4)^qN^{-q}+n_0r_0/N$ in place of $4\lambda N^{-q}+n_0r_0/N$, which replaces $32\sqrt2\lambda$ by $8\sqrt2\lambda(r_0^2/4)^q$ in the bound on $N^\star$.

\emph{Extension to every $H\ge2$.} Prepend to either instance a deterministic chain of $H-2$ zero-reward states, through which both actions lead to the next chain state and finally to $s_0$, so that the instance occupies steps $H-1$ and $H$. Every chain row has a single real successor. The agent's choices on the chain therefore do not affect the trajectory. At steps $H-1$ and $H$ the pseudo-state values $r_0(H-h+1)$ and the clamps equal those at steps $1$ and $2$ of the instance with $H=2$, and the chain collects no reward. The arguments above use only the rows at steps $H-1$ and $H$. They therefore apply unchanged, with the same regret per episode.

\subsection{Proof of Proposition~\ref{prop:aif-cai-regret}}
\label{app:softmax}

\begin{proof}
AIF and CaI sample $a^k_h\sim\mathcal Q(\cdot\mid x^k_h)\propto\exp\big(\beta Q_{k,h}(x^k_h,\cdot)\big)$ (Algorithm~\ref{alg:aif}). Their planning values lie in $[0,r_0(H-h+1)]$. AIF caps the values $Q_{k,h}$ that enter its softmax at $r_0(H-h+1)$, and its rewards and epistemic terms are nonnegative. CaI's values satisfy $0\le Q_{k,h}\le1+\max V_{k,h+1}\le r_0(H-h+1)$ by backward induction. Indeed, the soft maximum $\beta^{-1}\ln\big(A^{-1}\sum_a e^{\beta Q_{k,h}(x,a)}\big)$ lies between the minimum and the maximum of its arguments, and $r_0\ge1$. Every action therefore has probability at least $1/\big(1+(A-1)e^{\beta r_0H}\big)\ge e^{-\beta r_0H}/A$. The floor depends on the prior only through $r_0$. It therefore holds for every $n_0$, including Bayes-UCBVI's $n_0=\Theta(\ln K)$. Let $s_1$ be a state with an action $a$ of gap $\Delta:=V^*_1(s_1)-Q^*_1(s_1,a)>0$. By the performance-difference lemma \citep{kakade2002approximately}, for every possibly randomised policy $\pi$, with $V^\pi_h(s)=\sum_a\pi_h(a\mid s)Q^\pi_h(s,a)$,
\begin{equation*}
V^*_1(s_1)-V^\pi_1(s_1)\;=\;\mathds{E}^\pi\Big[\sum_{h=1}^H\big(V^*_h(s_h)-Q^*_h(s_h,a_h)\big)\Big]\;\ge\;\pi_1(a\mid s_1)\,\Delta,
\end{equation*}
since every summand is nonnegative. The policy of episode $k$ is fixed by the planning values computed before the episode and plays every action at $s^k_1$ with probability at least $e^{-\beta r_0H}/A$. Each episode whose start state $s^k_1$ has an action with gap at least $\Delta$ therefore costs at least $\Delta e^{-\beta r_0H}/A$, and $\mathds{E}\mathfrak R_K\ge K\Delta e^{-\beta r_0H}/A$ when every start state has one. The argument uses the reward only through $0\le\widehat R^k_\phi\le1$. It therefore also holds when a goal $R_k:\mathcal X\times\mathcal A\to[0,1]$ replaces $\widehat R^k_\phi$ in the backup, with the gaps taken under $R_k$, whatever the adversary does.
\end{proof}

\subsection{AIF's State Term}
\label{app:noisy-tv}

\begin{proposition}[AIF's state term is a bias]
\label{prop:noisy-tv}
Under exact perception, AIF's state term is $\mathds{H}[\mathcal Q(x_\tau\mid\pi)]\le\ln(X+1)$. For an exact abstraction, one step ahead, it converges to the entropy of the true next-state law as counts grow and does not vanish unless that law is deterministic. For every $H\ge2$, $\Delta\in(0,1]$, $J\ge2$ and $\beta$ with $\beta\Delta<\ln J$, there is an MDP with $S=J+2$ and $A=2$ on which the following holds for AIF that is given $P^*$ and $R^*$. AIF prefers a branch that loses $(H-1)\Delta$ per episode. Its policy posterior over plans under a uniform plan prior and its per-step softmax put mass above $\frac12$ on that branch, and its greedy rule plays it. Hence $\mathds{E}\mathfrak R_K\ge K(H-1)\Delta/2$.
\end{proposition}

\begin{proof}[Proof of Proposition~\ref{prop:noisy-tv}]
\emph{The closed form.} Under exact perception the emission $\mathbf A_\phi$ has disjoint supports, which makes $\mathcal Q(x_\tau\mid s_\tau,\pi)$ a point mass and $\mathds{H}[\mathcal Q(x_\tau\mid s_\tau,\pi)]=0$. Hence
\begin{equation*}
I_{\mathcal Q}(x_\tau;s_\tau\mid\pi)\;=\;\mathds{H}[\mathcal Q(x_\tau\mid\pi)]-\mathds{E}_{\mathcal Q(s_\tau\mid\pi)}\mathds{H}[\mathcal Q(x_\tau\mid s_\tau,\pi)]\;=\;\mathds{H}[\mathcal Q(x_\tau\mid\pi)]\;\le\;\ln(X+1),
\end{equation*}
the $X$ abstract states and the pseudo-state being the only categories. At a single step from $(x,a)$ the predicted law is the posterior mean $\widehat P^k_h(\cdot\mid x,a)$, which, for an exact abstraction, converges almost surely to $\overline P_\phi(\cdot\mid x,a)$ as the row's count grows. Its entropy, $\mathds{E}_B\mathds{H}[B]+\mathrm{Novelty}$ by Eq.~\eqref{eq:novelty-def}, converges to that of the true transition law. This limit is a property of the MDP that does not depend on the counts, and it is positive unless that law is deterministic.

\emph{The instance.} Let $\mathcal S=\{s_0,c,n_1,\dots,n_J\}$, with $S=J+2$, and $\mathcal A=\{a_c,a_n\}$, with $\phi$ the identity (an exact abstraction). From $s_0$, action $a_c$ leads to $c$ and action $a_n$ to a uniformly random state of $\{n_1,\dots,n_J\}$. The state $c$ is absorbing, and every $n_i$ leads to a uniformly random state of $\{n_1,\dots,n_J\}$ under both actions. Rewards are $R^*(s_0,\cdot)=0$, $R^*(c,\cdot)=1$ and $R^*(n_i,\cdot)=1-\Delta$. Every episode starts at $s_0$. Every plan that plays $a_c$ first is optimal and collects $V^*_1(s_0)=H-1$. Any plan that plays $a_n$ first collects $(H-1)(1-\Delta)$ and loses $(H-1)\Delta$.

\emph{AIF's policy posterior.} Give AIF the true kernel and rewards, so that its $B$-posterior is a point mass and novelty vanishes. Actions after the first affect neither the predicted states nor the predicted rewards. The expected free energy of Eq.~\eqref{eq:efe-canonical}, summed over the predicted rewards $r_2,\dots,r_H$ and the predicted states $x_2,\dots,x_{H+1}$, therefore depends on a plan only through its first action. Up to a common constant,
\begin{align}
\mathcal G(\text{plan starting with }a_c) &= -(H-1)\beta, \nonumber\\
\mathcal G(\text{plan starting with }a_n) &= -(H-1)\beta(1-\Delta)-H\ln J, \nonumber
\end{align}
since the predicted state is $c$ with certainty on the first branch and uniform on $\{n_1,\dots,n_J\}$ on the second. The second is smaller whenever $(H-1)\beta\Delta<H\ln J$, in particular whenever $\beta\Delta<\ln J$. Since the two branches contain equally many plans, under a uniform prior over plans the policy posterior $\mathcal Q(\pi)\propto P(\pi)e^{-\mathcal G(\pi)}$ places mass $\big(1+e^{-(H\ln J-(H-1)\beta\Delta)}\big)^{-1}>\frac12$ on the second branch.

\emph{AIF's per-step rules.} With the recursion of Eq.~\eqref{eq:psi-slot}, $\Psi_{k,h}(x,a)=\beta^{-1}\mathds{H}[P^*(\cdot\mid x,a)]$, and the clamp at $r_0(H-h+1)$, backward induction over $h$ gives, since $r_0\ge1$, $V_{k,h}(c)=H-h+1$ and $V_{k,h}(n_i)\ge\min\{r_0(H-h+1),(H-h+1)(1-\Delta+\beta^{-1}\ln J)\}$. Hence, whenever $\beta\Delta<\ln J$,
\begin{equation*}
Q_{k,1}(s_0,a_n)\;\ge\;\min\{r_0H,\,V_{k,2}(n_1)+\beta^{-1}\ln J\}\;>\;H-1\;=\;Q_{k,1}(s_0,a_c),
\end{equation*}
since the second argument of the minimum exceeds $H-1$ by at least $\min\{\beta^{-1}\ln J,\;H\beta^{-1}\ln J-(H-1)\Delta\}>0$. Therefore, the per-step softmax puts mass above $\frac12$ on $a_n$, and the greedy rule plays it. After the first step the actions do not matter. Under each rule, each episode therefore loses $(H-1)\Delta$ with probability above $\frac12$, and $\mathds{E}\mathfrak R_K\ge K(H-1)\Delta/2$. We do not analyse the learning agent here. With its softmax, it still incurs linear regret by Proposition~\ref{prop:aif-cai-regret}. CaI, which has no state term, prefers $a_c$ at every $\beta>0$.
\end{proof}

The lower bounds for novelty (Theorem~\ref{thm:exponent}(a)) and for the state term (Proposition~\ref{prop:noisy-tv}) use different instances, since on the instance of Theorem~\ref{thm:exponent}(a) the state term favours the risky action, which is optimal there.

\subsection{The Novelty Budget}
\label{app:prop-novelty-budget}

\begin{proposition}[Novelty budget is poly-logarithmic]
\label{prop:novelty-budget}
Let the rows $\rho^k_h(x,a)$ carry Bayes-UCBVI's prior (Definition~\ref{def:dirichlet-ground}) on $\mathcal X$, as in Theorem~\ref{thm:calibrated}. On any trajectory, whatever policy generated it, every row $(x,a,h)$ satisfies $\sum_{k\le K:\,(x_{k,h},a_{k,h})=(x,a)}\mathrm{Novelty}^k_h(x,a)\le\Gamma_\phi(1+\ln K)\le X(1+\ln K)$, with the branching number $\Gamma_\phi$ of Theorem~\ref{thm:calibrated}. Hence the novelty budget of Theorem~\ref{thm:calibrated}(iii) satisfies
\begin{equation*}
\Gamma_T\;=\;\sum_{k\le K,h\le H}\mathrm{Novelty}^k_h(x_{k,h},a_{k,h})\;\le\;\Gamma_\phi XAH(1+\ln K)\;=\;\mathcal O\big(X^2AH\ln T\big).
\end{equation*}
\end{proposition}

\begin{proof}
Fix a posterior row $(x,a,h)$, and let $m$ be the number of its categories with positive mass, as in Eq.~\eqref{eq:novelty-exact-ub}. These are the pseudo-state and the successors observed so far. Each observed successor lies in the support of some $\phi_\#P^*(\cdot\mid s,a)$ with $s\in\phi^{-1}(x)$. Hence, $m-1\le\Gamma_\phi$. The row is visited at most once per episode, at step $h$, hence at most $K$ times. If its $j$-th visit falls in episode $k$, its $j-1$ earlier visits fall in episodes $1,\dots,k-1$. The posterior $\rho^k_h(x,a)$ therefore has total mass $\bar n^k_h(x,a)=j-1+n_0$, and by Eq.~\eqref{eq:novelty-exact-ub} the novelty accrued at that visit is at most $(m-1)/(j+n_0)\le\Gamma_\phi/j$. Hence
\begin{equation*}
\sum_{k\le K:\,(x_{k,h},a_{k,h})=(x,a)}\mathrm{Novelty}^k_h(x,a)\;\le\;\Gamma_\phi\sum_{j=1}^{K} \frac1{j}\;\le\;\Gamma_\phi\Big(1+\int_1^K\frac{dz}{z}\Big)\;=\;\Gamma_\phi(1+\ln K).
\end{equation*}
Summing over the $XAH$ rows gives the claim.
\end{proof}

The budget thus grows only logarithmically in $T$. The term $H\sqrt{L\,T\,\Gamma_T}$ of Theorem~\ref{thm:calibrated}(iii) is therefore at most $H^{3/2}\sqrt{\Gamma_\phi XAT\,L(1+\ln K)}=\widetilde{\mathcal O}\big(H^{3/2}\sqrt{\Gamma_\phi XAT}\big)$, a factor $\sqrt{\Gamma_\phi(1+\ln K)}$ above the leading term of Theorem~\ref{thm:ucbvi-fixed}(b). With rows pooled over $h$, as in the time-invariant $B$ of \citet{friston2017active}, whose parameters are updated at the end of each trial, a row can be visited up to $H$ times under one posterior, and the argument gives $\Gamma_\phi XAH(1+\ln T)$. If the pooled posterior is instead updated after every transition, the same argument gives $\Gamma_\phi XA(1+\ln T)$.

\subsection{Proofs of Lemma~\ref{lem:transport} and Theorem~\ref{thm:calibrated}}
\label{app:calibrated}

\begin{lemma}[A transportation inequality for novelty]
\label{lem:transport}
For every Dirichlet posterior $\rho$ with mean $\bar p$ and every $f$,
\begin{equation*}
\mathds{V}_{B\sim\rho}\big[B^\top f\big]\;\le\;\tfrac12\,\mathrm{span}(f)^2\,\mathrm{Novelty}.
\end{equation*}
Consequently, let $\rho^k_h(x,a)$ be a row, whose Dirichlet parameters are integers, and let $0\le f\le f(x_{\mathrm E})\le r_0H$. Then Bayes-UCBVI's quantile gap at $f$, $\mathds{Q}_{p\sim\rho^k_h(x,a)}\big(p^\top f,\kappa^k_h(x,a)\big)-\widehat P^k_h(\cdot\mid x,a)^\top f$ (Lemma~\ref{lem:tiapkin-quantile} with $f$ in place of $\overline V^k_{h+1}$), is at most
\begin{equation*}
2\,\mathrm{span}(f)\sqrt{\ell^k_h(x,a)\,\mathrm{Novelty}(x,a)}+2\sqrt2\,r_0H\ell^k_h(x,a)/\bar n^k_h(x,a).
\end{equation*}
\end{lemma}
\begin{proof}[Proof of Lemma~\ref{lem:transport}]
Let $B\sim\rho$, and let $c$ be the midpoint of the range of $f$ on the support of $\bar p$, where $\mathrm{span}(f):=\max f-\min f$ is taken over all categories, the pseudo-state included. For every row $\rho^k_h(x,a)$ this support contains $x_{\mathrm E}$. Hence, in Eq.~\eqref{eq:calibrated}, the span of $V^+_{k,h+1}$ includes $V^+_{k,h+1}(x_{\mathrm E})$. Since $B-\bar p$ sums to zero, $B^\top f-\bar p^\top f=(B-\bar p)^\top(f-c)$, and $|f-c|\le\frac12\mathrm{span}(f)$ on the support, which yields
\begin{equation*}
\big|B^\top f-\bar p^\top f\big|\;\le\;\tfrac12\,\mathrm{span}(f)\,\|B-\bar p\|_1\;\le\;\mathrm{span}(f)\sqrt{D_{\mathrm{KL}}(B\|\bar p)/2},
\end{equation*}
the second step being Pinsker's inequality $\frac12\|B-\bar p\|_1\le\sqrt{D_{\mathrm{KL}}(B\|\bar p)/2}$ \citep[Thm.~4.19]{boucheron2013concentration}. Squaring and taking the expectation over $\rho$, whose mean is $\bar p$,
\begin{equation*}
\mathds{V}_{B\sim\rho}\big[B^\top f\big]\;\le\;\tfrac12\,\mathrm{span}(f)^2\;\mathds{E}_{B\sim\rho}\,D_{\mathrm{KL}}(B\|\bar p)\;=\;\tfrac12\,\mathrm{span}(f)^2\,\mathrm{Novelty},
\end{equation*}
by the identity $\mathrm{Novelty}=\mathds{E}_{B\sim\rho}[D_{\mathrm{KL}}(B\|\bar p)]$ (Appendix~\ref{app:novelty-approx}). For a Dirichlet posterior of total mass $\bar n$, the covariance of $B$ is $(\mathrm{diag}(\bar p)-\bar p\bar p^\top)/(\bar n+1)$, which yields $\mathds{V}_{B\sim\rho}[B^\top f]=\mathds{V}_{\bar p}[f]/(\bar n+1)$ and, for $\bar n\ge1$,
\begin{equation}
\frac{\mathds{V}_{\bar p}[f]}{\bar n}\;=\;\frac{\bar n+1}{\bar n}\;\mathds{V}_{B\sim\rho}\big[B^\top f\big]\;\le\;\mathrm{span}(f)^2\,\mathrm{Novelty}.
\label{eq:var-nov}
\end{equation}
Lemma~C.8 of \citet{tiapkin2022dirichlet} bounds, for a Dirichlet posterior with integer parameters and every $f$ with $0\le f\le f(x_{\mathrm E})=b$, the order-$\kappa$ quantile of $B^\top f$ by
\begin{equation*}
\bar p^\top f+2\sqrt{\mathds{V}_{\bar p}[f]\,\ell/\bar n}+2\sqrt2\,b\,\ell/\bar n,\qquad \ell=\ln(1/(1-\kappa))
\end{equation*}
(Appendix~\ref{app:baselines}). Every row has $\bar n\ge n_0\ge1$, unvisited rows included. Substituting Eq.~\eqref{eq:var-nov} and $b\le r_0H$ gives the second claim.
\end{proof}

\begin{proof}[Proof of Theorem~\ref{thm:calibrated}(i)--(ii)]
Calibrated AIF clamps its values at $r_0(H-h+1)$, the value of the pseudo-state. Hence, $0\le V^+_{k,h+1}\le V^+_{k,h+1}(x_{\mathrm E})=r_0(H-h)$, and its rows have integer parameters. Lemma~\ref{lem:transport} therefore applies at $f=V^+_{k,h+1}$, and $\Psi^+_{k,h}$ exceeds Bayes-UCBVI's quantile gap at calibrated AIF's own next-step values. Calibrated AIF is thus an agent of the form Eq.~\eqref{eq:dirichlet-agent} to which Lemma~\ref{lem:bucbvi-opt} applies. On $\mathcal E_{\mathrm{opt}}$ it gives $V^+_{k,h}\ge\overline V^*_{\phi,h}-(H-h+1)\eta$, which is optimism for an exact abstraction and $V^+_{k,1}\ge\overline V^*_{\phi,1}-\epsilon_{\text{rep}}(\phi)$ in general. For (ii), a row's categories with positive mass are the pseudo-state and at most $\Gamma_\phi$ observed successors (proof of Proposition~\ref{prop:novelty-budget}). Hence, Eq.~\eqref{eq:novelty-exact-ub} gives $\mathrm{Novelty}\le\Gamma_\phi/(\bar n+1)$, and with $\mathrm{span}(V^+_{k,h+1})\le r_0H$,
\begin{equation*}
\Psi^+_{k,h}\;\le\;2r_0H\sqrt{\Gamma_\phi\ell/(\bar n+1)}+2\sqrt2\,r_0H\ell/\bar n\;=\;\mathcal O\big(H\sqrt{\Gamma_\phi\ell/\bar n}+H\ell/\bar n\big),
\end{equation*}
which is $\mathcal O(H\sqrt{\Gamma_\phi\ell/\bar n})$ for $\bar n\ge\ell$.
\end{proof}

\begin{proof}[Proof of Theorem~\ref{thm:calibrated}(iii)]
Lemma~\ref{lem:direct} applies with $w_{k,h}=\Psi^+_{k,h}$ on $\mathcal E_{\mathrm{opt}}\cap\mathcal E_{\mathrm{conc}}$. By $\mathrm{span}(V^+_{k,h+1})\le r_0H$ and Cauchy--Schwarz,
\begin{equation*}
\sum_{k,h}2\,\mathrm{span}(V^+_{k,h+1})\sqrt{\ell^k_h\,\mathrm{Novelty}^k_h}\;\le\;2r_0H\sqrt{\bar\ell}\sqrt{T\sum_{k,h}\mathrm{Novelty}^k_h}\;=\;2r_0H\sqrt{\bar\ell\,T\,\Gamma_T},
\end{equation*}
with $\ell^k_h$ and novelty evaluated at $(x_{k,h},a_{k,h})$ and $\bar\ell:=\max_{k,h,x,a}\ell^k_h(x,a)$ as in the proof of Theorem~\ref{thm:ucbvi-fixed}(b). Moreover,
\begin{equation*}
\sum_{k,h}2\sqrt2\,r_0H\ell^k_h/\bar n^k_h\;\le\;2\sqrt2\,r_0H\bar\ell XAH(1+\ln K).
\end{equation*} Proposition~\ref{prop:novelty-budget} gives $\Gamma_T\le\Gamma_\phi XAH(1+\ln K)$. The other terms of Lemma~\ref{lem:direct} are those of Theorem~\ref{thm:ucbvi-fixed}(b), and the failure probabilities sum to at most $\frac{19}{20}\delta$. With the logarithmic factor $L=\max\{\iota,\bar\ell,n_0,1+\ln K\}$ of Theorem~\ref{thm:ucbvi-fixed}, this is the stated bound. For $T\ge H^3X^3A$ we have $H^3X^2A\le H^{3/2}\sqrt{XAT}$, and $\Gamma_T\le\Gamma_\phi XAH(1+\ln K)$ gives the $\widetilde{\mathcal O}$ form.
\end{proof}

\emph{Why calibrated AIF's bound exceeds Bayes-UCBVI's.} Our bound for calibrated AIF exceeds that of Theorem~\ref{thm:ucbvi-fixed}(b) for Bayes-UCBVI on $\phi$ by $\sqrt{\Gamma_\phi}$, and the variance-aware bound of Theorem~\ref{thm:ucbvi-fixed}(c) for exact abstractions by $\sqrt{H\Gamma_\phi}$. The factor $\sqrt H$ arises because the range of the next-step value $\overline V$ (here $\overline V^*_{\phi,h+1}$ or $V^+_{k,h+1}$) stands in for its standard deviation. The variance-aware analysis sums its widths through the law of total variance, $\sum_{k,h}\mathds{V}_{\overline P_\phi}[\overline V^*_{\phi,h+1}]\lesssim H^2K=HT$, against $\sum_{k,h}H^2=H^2T$ for the range. A width proportional to the standard deviation of the value can exploit this. A width proportional to its range, or a width bounded through its range as in our direct analysis of Bayes-UCBVI (Appendix~\ref{app:thm-fixed}), cannot. The factor $\sqrt{\Gamma_\phi}$ arises because novelty is information about the whole row $B$, whereas only the projection $B^\top\overline V$ enters the value, whose deviation is of order $\sqrt{\mathds{V}_{\bar p}[\overline V]/\bar n}$. Replacing $\mathrm{span}(V^+_{k,h+1})^2\,\mathrm{Novelty}$ by $\mathds{V}_{\bar p}[V^+_{k,h+1}]/\bar n$ in Eq.~\eqref{eq:calibrated}, that is, undoing Eq.~\eqref{eq:var-nov}, removes this factor and returns the leading term of Lemma~\ref{lem:tiapkin-quantile}. Information does not improve on counts here. By Eqs.~\eqref{eq:novelty-exact-ub} and~\eqref{eq:novelty-lb}, $(m-1)/(6\bar n)\le\mathrm{Novelty}\le(m-1)/(\bar n+1)$ for a row with $m$ categories of positive mass, the pseudo-state and $m-1$ observed successors. Hence the first term of $\Psi^+_{k,h}$ is within constant factors of the count bonus $\mathrm{span}(V^+_{k,h+1})\sqrt{(m-1)\ell/\bar n}$, which counts the observed successors. That novelty concerns the whole row $B$ rather than one projection also makes it uniform over value functions, which is what Section~\ref{sec:goal-uniform} exploits.

\section{Proofs for Section~\ref{sec:goal-uniform}}
\label{app:goal-uniform}

We prove the results of Section~\ref{sec:goal-uniform} in the order of their dependencies. Appendix~\ref{app:adaptive-upper} shows that calibrated AIF with a uniform prior is optimistic for every goal at once and bounds its regret, the first claim of Theorem~\ref{thm:price}(c). Appendix~\ref{app:adaptive-lockout} proves that Bayes-UCBVI's value-specific quantiles lock out under adaptive goals (Proposition~\ref{prop:lockout}), and that every index that is optimistic for all goals at once overestimates some goal (Lemma~\ref{lem:goal-uniform-width}). Its $\ell_1$ argument underlies the lower bound for goal-memoryless agents, Theorem~\ref{thm:price}(b), in Appendix~\ref{app:adaptive-b}, which also shows which agents are goal-memoryless. Appendix~\ref{app:concentration} collects the concentration tools of the lower bound for every agent, Theorem~\ref{thm:price}(a), which Appendix~\ref{app:all-agents} proves. Appendix~\ref{app:adaptive-reading} gives agents that use past goals as data and attain this rate, Proposition~\ref{prop:goal-reading} and the second claim of Theorem~\ref{thm:price}(c). Appendix~\ref{app:adaptive-oblivious} treats goals fixed in advance, Theorem~\ref{thm:price}(d) and Proposition~\ref{prop:oblivious}.

\emph{The Gaussian family.} In the Gaussian family of Section~\ref{sec:goal-uniform}, at start state $i$, action $a$ reaches $y_{2j-1}$ with probability $(1+\varepsilon\tilde\theta^{(i)}_j)/(2m)$ and $y_{2j}$ with probability $(1-\varepsilon\tilde\theta^{(i)}_j)/(2m)$. Here $\theta^{(1)},\dots,\theta^{(M)}$ are independent standard Gaussian vectors in $\mathds R^m$, $\varepsilon^2=mM/(64K)$, and $\tilde\theta$ clips $\theta$ so that $\varepsilon|\tilde\theta_j|\le\frac12$.

\emph{Proof overview.} In both lower bounds, the goal asks the agent to compare $P_ig$ with a value $v$. A goal-memoryless agent decides from its records and the current goal, and its records are functions of the samples at the current start state and of randomness independent of $P_i$ (Lemma~\ref{lem:reconstruction}). Its decisions for all goals together therefore amount to an estimate of $P_i$ in $\ell_1$. This estimate errs by order $\sqrt{Y/n}$ after $n$ samples \citep{paninski2008coincidence,han2015minimax}. An adversary that knows $P_i$ picks the goal on which the estimate errs, which gives part~(b) (Lemma~\ref{lem:gauss-l1}). Goals aligned with the agent's errors, however, reveal $P_i$, and an agent that uses past goals as data learns from them. Against such an agent, the adversary of part~(a) proceeds in about $\sqrt Y$ phases at each start state (Appendix~\ref{app:all-agents}). In each phase, it draws a goal direction uniformly from the directions along which the unknown $\theta^{(i)}$ has a projection of order $Y^{1/4}$, restricted to the complement of all earlier directions and of the agent's data. The reflection that flips the sign of this projection leaves the prior, the earlier goals, the probability of the adversary's choice and the linear part of the data likelihood exactly unchanged. The sign therefore stays unknown until about $K/\sqrt Y$ samples have resolved it, and every episode until then costs every agent a constant fraction of the projection's value.

Throughout, $\phi$ is exact. We work on the abstract MDP with kernel $P:=\overline P_\phi$, and on a row $(x,a,h)$ the first $n$ successors are i.i.d.\ draws from $P(\cdot\mid x,a)$, whatever the agent does, as in Appendix~\ref{app:lem-novelty-order}.

\subsection{Optimism for Every Goal and the First Claim of Theorem~\ref{thm:price}(c)}
\label{app:adaptive-upper}

\emph{Novelty bounds every value function at once.} For a Dirichlet posterior $\rho$ with mean $\bar p$, Pinsker's inequality gives
\begin{equation*}
\sup_{f:\,\mathrm{span}(f)\le1}(B-\bar p)^\top f=\frac12\|B-\bar p\|_1\le\sqrt{D_{\mathrm{KL}}(B\|\bar p)/2},
\end{equation*}
and Jensen's inequality gives
\begin{equation}
\mathds{E}_{B\sim\rho}\sup_{f:\,\mathrm{span}(f)\le1}(B-\bar p)^\top f\;\le\;\sqrt{\mathrm{Novelty}/2}.
\label{eq:goal-uniform-posterior}
\end{equation}
The quantile gap of Lemma~\ref{lem:tiapkin-quantile}, in contrast, concerns one function $p^\top f$ at a time.

\begin{proof}[Proof of the first claim of Theorem~\ref{thm:price}(c)]
\emph{The agent.} Calibrated AIF with a uniform prior keeps, for each row, the posterior $\mathrm{Dir}(\mathds 1+N^k_h(\cdot\mid x,a))$ on $\mathcal X$, whose mean is $\widehat p=(n\widetilde P_n+X\,\mathcal U)/\bar n$ with $\bar n=n+X$ and $\mathcal U$ the uniform law on $\mathcal X$. It computes $V_{k,H+1}\equiv0$ and
\begin{equation*}
Q_{k,h}(x,a)=\min\{H-h+1,\ R_k(x,a)+\widehat p^\top V_{k,h+1}+\Psi^{\mathrm u}_{k,h}(x,a)\}
\end{equation*}
 with $\Psi^{\mathrm u}_{k,h}=2\,\mathrm{span}(V_{k,h+1})\sqrt{\ell'\,\mathrm{Novelty}}$, $\ell'=9+3\iota_{\mathrm u}$ and $\iota_{\mathrm u}=\ln(2XAHK/\delta)$, and acts greedily. Let $X\ge2$. For $X=1$ there is nothing to learn.

\emph{A confidence set uniform over value functions.} For a row and a count $n\ge1$, and a sign vector $v\in\{-1,1\}^{\mathcal X}$, $\langle v,\widetilde P_n-P\rangle$ is an average of $n$ independent terms that lie in intervals of length $2$. Hence, $\mathds{P}(\langle v,\widetilde P_n-P\rangle\ge\varepsilon)\le e^{-n\varepsilon^2/2}$. Since $\|\mu\|_1=\max_v\langle v,\mu\rangle$, a union bound over the $2^X$ sign vectors, the $XAH$ rows and the counts $n\le K$ gives an event $\mathcal E_{\ell_1}$ of probability at least $1-\delta/2$ on which $\|\widetilde P_n-P(\cdot\mid x,a)\|_1\le\sqrt{2(X\ln2+\iota_{\mathrm u})/n}$ for all of them. On $\mathcal E_{\ell_1}$, for every row, every count and every $f:\mathcal X\to\mathds R$,
\begin{align*}
\big|(\widehat p-P)f\big|&\;\le\;\tfrac n{\bar n}\big|(\widetilde P_n-P)f\big|+\tfrac X{\bar n}\big|(\mathcal U-P)f\big|\;\le\;\mathrm{span}(f)\Big(\tfrac{\sqrt n}{\bar n}\sqrt{\tfrac{X\ln2+\iota_{\mathrm u}}2}+\tfrac X{\bar n}\Big)\\
&\;\le\;\mathrm{span}(f)\Big(\sqrt{\tfrac{X\ln2+\iota_{\mathrm u}}{2\bar n}}+\sqrt{\tfrac X{\bar n}}\Big)\;\le\;\mathrm{span}(f)\sqrt{\frac{3X+\iota_{\mathrm u}}{\bar n}},
\end{align*}
using $|(P-P')f|\le\frac12\mathrm{span}(f)\|P-P'\|_1$, $X\le\bar n$, and $\sqrt a+\sqrt b\le\sqrt{2(a+b)}$. The first term is absent for $n=0$. Since every category carries mass at least one, Eq.~\eqref{eq:novelty-lb} with $m=X$ gives $\mathrm{Novelty}\ge(2X-3)/(6\bar n)\ge X/(12\bar n)$. Hence
\begin{equation}
\Psi^{\mathrm u}(f):=2\,\mathrm{span}(f)\sqrt{\ell'\,\mathrm{Novelty}}\;\ge\;\mathrm{span}(f)\sqrt{\frac{\ell'X}{3\bar n}}\;\ge\;\big|(\widehat p-P)f\big|\qquad\text{for every }f,
\label{eq:goal-uniform-valid}
\end{equation}
since $\ell'X/3\ge3X+\iota_{\mathrm u}$.

\emph{Optimism for every goal.} Fix $k$ and the goal $R_k$. If $V_{k,h+1}\ge V^*_{h+1}(\cdot\,;R_k)$, then Eq.~\eqref{eq:goal-uniform-valid} at $f=V_{k,h+1}$, which is allowed because it holds for every function simultaneously, gives
\begin{equation*}
R_k+\widehat p^\top V_{k,h+1}+\Psi^{\mathrm u}_{k,h}\ge R_k+PV_{k,h+1}\ge R_k+PV^*_{h+1}(\cdot\,;R_k)=Q^*_h(\cdot\,;R_k),
\end{equation*}
and the clamp $H-h+1$ preserves this. By induction, $V_{k,h}\ge V^*_h(\cdot\,;R_k)$ on $\mathcal E_{\ell_1}$, for every $k$ and every goal sequence, because the event does not involve the goals.

\emph{Regret.} Let $\delta_{k,h}:=V_{k,h}(x_{k,h})-V^{\pi^k}_h(x_{k,h};R_k)$ and $\xi_{k,h}:=P(V_{k,h+1}-V^{\pi^k}_{h+1})(x_{k,h},a_{k,h})-(V_{k,h+1}-V^{\pi^k}_{h+1})(x_{k,h+1})$, a martingale difference with $|\xi_{k,h}|\le2H$, since $R_k$ and $\pi^k$ are fixed before episode $k$. Since the agent is greedy and the clamp only lowers $Q_{k,h}$, Eq.~\eqref{eq:goal-uniform-valid} gives
\begin{align*}
\delta_{k,h}\;&\le\;(\widehat p-P)V_{k,h+1}+\Psi^{\mathrm u}_{k,h}+P(V_{k,h+1}-V^{\pi^k}_{h+1})\\
&\le\;2\Psi^{\mathrm u}_{k,h}+\delta_{k,h+1}+\xi_{k,h}.
\end{align*}
By optimism, $\mathfrak R^{\mathrm{goal}}_K\le\sum_k\delta_{k,1}$, which is at most $\sum_{k,h}(2\Psi^{\mathrm u}_{k,h}+\xi_{k,h})$. With $\mathrm{Novelty}\le(X-1)/(\bar n+1)$ (Eq.~\eqref{eq:novelty-exact-ub} with $m=X$) and $\mathrm{span}(V_{k,h+1})\le H$, a row visited $N$ times contributes $\sum_{j<N}2H\sqrt{\ell'X/(j+X+1)}\le4H\sqrt{\ell'XN}$, and Cauchy--Schwarz over the $XAH$ rows gives $\sum_{k,h}\Psi^{\mathrm u}_{k,h}\le4H^{3/2}X\sqrt{A\ell'T}$. With probability at least $1-\delta/2$, the Azuma--Hoeffding inequality gives $\sum_{k,h}\xi_{k,h}\le2H\sqrt{2T\ln(2/\delta)}$. This event depends on the adversary. The regret bound therefore holds for every adversary separately, whereas optimism holds on $\mathcal E_{\ell_1}$ for all goal sequences at once. For $H=2$, the step-$2$ values are exact, $\xi_{k,h}=0$, and the regret bound also holds on $\mathcal E_{\ell_1}$ for all goal sequences at once.
\end{proof}

\subsection{Lock-Out of Value-Specific Quantiles and the Price of Simultaneous Optimism}
\label{app:adaptive-lockout}

Value-specific widths are not even optimistic for all goals at once.

\begin{proposition}[Value-specific quantiles lock out]
\label{prop:lockout}
Let $\phi$ be exact. There is $X_0=\mathcal O(\ln(n_0/\delta))$ such that for every $X\ge X_0$ there are an MDP with $X+2$ states and $A=H=2$, and a goal sequence computed from the agent's history alone, on which the following holds with probability at least $1-e^{-(X-1)/32}$. Bayes-UCBVI, with each goal in its backup, plays a suboptimal action in every episode $k>K_0$, where $K_0$ depends on $K$ only through $n_0$. In each such episode it loses at least $\sqrt{\ell_1/(4n_1)}$, where the count $n_1=\mathcal O(X+n_0^2)$ is defined in the proof and $\ell_1\ge1$ is the logarithmic factor $\ell$ of Lemma~\ref{lem:tiapkin-quantile} of the row of the suboptimal action at count $n_1$.
\end{proposition}

\begin{proof}[Proof of Proposition~\ref{prop:lockout}]
\emph{The instance.} Let $\mathcal S=\{s_0,z,y_1,\dots,y_X\}$, $\mathcal A=\{a,b\}$, $H=2$, and let every episode start at $s_0$. Action $a$ leads from $s_0$ to a uniformly random $y_i$, action $b$ leads to $z$, and $z$ and the $y_i$ are absorbing. A goal sets $R(s_0,\cdot)=0$, $R(y_i,\cdot)=g(i)\in[0,1]$ and $R(z,\cdot)=v\in[0,1]$. Action $a$ is thus worth $X^{-1}\sum_ig(i)$ and action $b$ is worth $v$. Bayes-UCBVI runs with the goal $R_k$ in place of $R^*$. Its step-$2$ values are the goal's rewards, $\overline V^k_2(y_i)=g(i)$ and $\overline V^k_2(z)=v$, since $\overline V^k_3\equiv0$ also at the pseudo-state, and $\overline V^k_2(x_{\mathrm E})=r_0=2$. Its step-$1$ indices are the $\kappa$-quantiles of $p^\top\overline V^k_2$ under the posteriors of the rows $(s_0,a,1)$ and $(s_0,b,1)$. Write $n$ and $n_b$ for their counts, $\ell(n)$ for the logarithmic factor $\ell$ of Lemma~\ref{lem:tiapkin-quantile} of the first row at count $n$, and $\widetilde P_n$ for its empirical law.

\emph{The adversary.} \emph{Phase~1.} The goal is $g\equiv1$, $v=0$. The index of $a$ is at least $1$, since $p^\top\overline V^k_2\ge1$ on the support of the posterior, and by Lemma~\ref{lem:tiapkin-quantile} the index of $b$ is at most
\begin{equation*}
2n_0/(n_b+n_0)+2\sqrt{\ell_b/(n_b+n_0)}+8\sqrt2\,\ell_b/(n_b+n_0),
\end{equation*}
with $\ell_b$ the logarithmic factor $\ell$ of Lemma~\ref{lem:tiapkin-quantile} of the second row, which is below $1$ for all $n_b\ge N_b$, for a number $N_b$ that depends only on $n_0$, $\delta$ and $X$. Hence Bayes-UCBVI plays $b$ in at most $N_b$ episodes of phase~1 and $a$ in all others. Phase~1 ends when the first row has $n_1$ samples, after at most $K_0:=n_1+N_b$ episodes, where $n_1$ is the smallest integer $n\ge2X$ with $n\ge(2n_0+12\ell(n))^2/\ell(n)$, so that $n_1=\mathcal O(X+n_0^2+\ln(1/\delta))$, which is $\mathcal O(X+n_0^2)$ for $X\ge X_0$ below. \emph{Phase~2.} Let $\ell_1:=\ell(n_1)$, let $U:=\{i:\widetilde P_{n_1}(y_i)<1/X\}$ be the under-represented successors, and fix from now on the goal $g:=\mathds 1_U$ and $v:=\min\{1,\overline Q_a+\frac12\sqrt{\ell_1/n_1}\}$, where $\overline Q_a$ is Bayes-UCBVI's index of $a$ under $g$, which does not depend on $v$. The quantities $n_1$, $\ell_1$, $U$ and $\overline Q_a$, and the time at which phase~2 starts, are functions of the agent's history and of the fixed parameters $X$, $n_0$, $K$ and $\delta$. The goal sequence therefore uses only the agent's history.

\emph{The deficit.} The $n_1$ samples are i.i.d.\ from the uniform law $\mathcal U$ on $\{y_1,\dots,y_X\}$, which is the meaning of $\mathcal U$ in the rest of this proof. Each count $N_i$ is $\mathrm{Bin}(n_1,1/X)$ with variance $\sigma^2=(n_1/X)(1-1/X)\ge1$ for $n_1\ge2X$ and $X\ge2$, and central fourth moment at most $\sigma^2+3\sigma^4$. Hölder's inequality, as in Appendix~\ref{app:lem-novelty-order}, gives $\mathds{E}|N_i-n_1/X|\ge\sigma^3/\sqrt{\sigma^2+3\sigma^4}\ge\sigma/2$, which yields $\mathds{E}\|\widetilde P_{n_1}-\mathcal U\|_1\ge\tfrac12\sqrt{(X-1)/n_1}$. Changing one sample changes $\|\widetilde P_{n_1}-\mathcal U\|_1$ by at most $2/n_1$. Hence, the bounded-difference inequality \citep{mcdiarmid1989method} gives, with probability at least $1-e^{-(X-1)/32}$, $\|\widetilde P_{n_1}-\mathcal U\|_1\ge\tfrac14\sqrt{(X-1)/n_1}$. On this event the true value of $a$ under the phase-2 goal exceeds its empirical value by
\begin{equation*}
\tfrac{|U|}X-\widetilde P_{n_1}(U)\;=\;\tfrac12\|\widetilde P_{n_1}-\mathcal U\|_1\;\ge\;\tfrac18\sqrt{(X-1)/n_1}.
\end{equation*}
By Lemma~\ref{lem:tiapkin-quantile} with $\mathds{V}\le1$, $r_0=2$ and $H=2$, the index is
\begin{equation*}
\overline Q_a\le\widetilde P_{n_1}(U)+2n_0/n_1+2\sqrt{\ell_1/n_1}+8\sqrt2\,\ell_1/n_1.
\end{equation*}
If $X-1\ge1024\,\ell_1$, the deficit is at least $4\sqrt{\ell_1/n_1}$, and the choice of $n_1$ gives $(2n_0+12\ell_1)/n_1\le\sqrt{\ell_1/n_1}$. Hence,
\begin{equation*}
\tfrac{|U|}X-\overline Q_a\;\ge\;4\sqrt{\ell_1/n_1}-2\sqrt{\ell_1/n_1}-\sqrt{\ell_1/n_1}\;=\;\sqrt{\ell_1/n_1}.
\end{equation*}
Since $\ell_1=\mathcal O(\ln((X+n_0)/\delta))$, the condition $X-1\ge1024\,\ell_1$ holds for all $X\ge X_0=\mathcal O(\ln(n_0/\delta))$.

\emph{Lock-out.} Under the phase-2 goal, $\overline Q_a+\frac12\sqrt{\ell_1/n_1}\le|U|/X-\frac12\sqrt{\ell_1/n_1}<1$, so that the minimum in the definition of $v$ is not active and $v$ lies strictly between $\overline Q_a$ and the value $|U|/X$ of $a$. The index of $b$ is at least $v$, since $p^\top\overline V^k_2\ge v$ on the support $\{z,x_{\mathrm E}\}$ of its posterior. Bayes-UCBVI therefore plays $b$. The first row then receives no data. Hence, $\overline Q_a$ does not change, and Bayes-UCBVI plays $b$ in every later episode. Each such episode loses $|U|/X-v=|U|/X-\overline Q_a-\frac12\sqrt{\ell_1/n_1}\ge\sqrt{\ell_1/(4n_1)}$, which gives the claim.
\end{proof}

\emph{Scope of Proposition~\ref{prop:lockout}.} Its construction uses that Bayes-UCBVI adds no width at the last step, where its next-step values vanish. UCBVI adds its bonus at every step, including the last, where the successors' own counts, of order $n_1/X$, keep it large on this instance. We therefore do not claim Proposition~\ref{prop:lockout} for UCBVI.

No index avoids the price of simultaneous optimism at a fixed count. Every index that is optimistic for all goals at once overestimates some goal by order $(H-1)\sqrt{X/n}$ after $n$ samples, the order of calibrated AIF's width under a uniform prior.

\begin{lemma}[Indices that are optimistic for all goals overestimate]
\label{lem:goal-uniform-width}
Let $X,H\ge2$ and $n\ge X$, and let $U_n(g)$ be an index of the action $a$ of the MDP of Proposition~\ref{prop:lockout}, with horizon $H$, under the goal that pays $g\in\{0,1\}^X$ at the successors, computed from the first $n$ successors of $a$ and independent randomness. Then some successor law $P$ of $a$ has the following property. If $U_n(g)\ge(H-1)Pg$ for all $g$ at once with probability at least $1-\delta>\frac{15}{16}$, then $\max_g\big(U_n(g)-(H-1)Pg\big)\ge(H-1)\sqrt{X/n}/256$ with probability at least $\frac1{16}-\delta$.
\end{lemma}
\begin{proof}
Let $u(g):=U_n(g)/(H-1)$, so that the hypothesis reads $u(g)\ge Pg$ for all $g\in\{0,1\}^X$ on an event $\mathcal E$ of probability at least $1-\delta$, and let $W:=\max_g\big(u(g)-Pg\big)$. The construction is the $\ell_1$ lower bound for discrete distributions \citep{paninski2008coincidence,han2015minimax}, which we reproduce because the index enters it through a confidence set.

\emph{From a goal-uniform index to an estimate of $P$.} The set $C:=\{Q\in\Delta(\{y_1,\dots,y_X\}):Qg\le u(g)\text{ for all }g\}$ is computed from the data and contains $P$ on $\mathcal E$. Let $\widehat Q$ be a point of $C$ if $C$ is nonempty, and the uniform law on $\{y_1,\dots,y_X\}$ otherwise. On $\mathcal E$, $(\widehat Q-P)g\le u(g)-Pg\le W$ for every $g$. The maximum of $(\widehat Q-P)g=\widehat Q(J)-P(J)$ over the sets $J=\{i:g(i)=1\}$ is $\frac12\|\widehat Q-P\|_1$, which yields $\|\widehat Q-P\|_1\le2W$ on $\mathcal E$. An index with a small overestimate would therefore estimate $P$ in $\ell_1$ faster than any estimator can.

\emph{The family.} Let $m:=\lfloor X/2\rfloor\ge X/3$ and $\varepsilon:=\sqrt{3m/(64n)}$, which is at most $\frac12$ since $n\ge X\ge 2m$. For $\sigma\in\{-1,1\}^m$, let $P_\sigma(y_{2i-1})=(1+\sigma_i\varepsilon)/(2m)$ and $P_\sigma(y_{2i})=(1-\sigma_i\varepsilon)/(2m)$ for $i\le m$, and $P_\sigma(y_X)=0$ if $X$ is odd.

\emph{Testing the signs.} Draw $\sigma$ uniformly, and let $\widehat\sigma_i:=1$ if $\widehat Q(y_{2i-1})>\widehat Q(y_{2i})$ and $\widehat\sigma_i:=-1$ otherwise. If $\widehat\sigma_i\ne\sigma_i$, the differences $\widehat Q(y_{2i-1})-\widehat Q(y_{2i})$ and $P_\sigma(y_{2i-1})-P_\sigma(y_{2i})=\sigma_i\varepsilon/m$ do not have the same strict sign. Hence, the triangle inequality gives $|\widehat Q(y_{2i-1})-P_\sigma(y_{2i-1})|+|\widehat Q(y_{2i})-P_\sigma(y_{2i})|\ge\varepsilon/m$. The number $N_i$ of samples in the pair $\{y_{2i-1},y_{2i}\}$ is $\mathrm{Bin}(n,1/m)$ and does not depend on $\sigma$. Given the pairs of all samples, the samples in pair~$i$ equal $y_{2i-1}$ with probability $\frac12+\sigma_i\varepsilon/2$, independently of each other and of the other pairs. Given $N_i=j$, $\sigma_i$ and the samples of pair~$i$ are therefore independent of all other data and of the agent's randomness, and $\widehat\sigma_i$ errs with probability at least the Bayes error
\begin{equation*}
\frac12\big(1-\mathrm{TV}(\mathrm{Ber}(\frac12+\frac\varepsilon2)^{\otimes j},\mathrm{Ber}(\frac12-\frac\varepsilon2)^{\otimes j})\big).
\end{equation*}
 The bound of Eq.~\eqref{eq:tv-bernoulli} in Appendix~\ref{app:unstructured}, with $\varepsilon/2\le\frac14$ in place of $\varepsilon$, bounds this total variation by $\varepsilon\sqrt{8j/3}$, which is at most $\frac12$ for $j\le2n/m$. By Markov's inequality $\mathds{P}(N_i\le2n/m)\ge\frac12$, which yields $\mathds{P}(\widehat\sigma_i\ne\sigma_i)\ge\frac18$. The number $N_{\mathrm{err}}\le m$ of wrong signs thus has $\mathds{E}N_{\mathrm{err}}\ge m/8$, hence $\mathds{P}(N_{\mathrm{err}}\ge m/16)\ge(\mathds{E}N_{\mathrm{err}}-m/16)/m\ge\tfrac1{16}$, and on this event $\|\widehat Q-P_\sigma\|_1\ge\varepsilon/16$.

\emph{Conclusion.} Some $\sigma$ does at least as well as the average, which yields $\mathds{P}_\sigma(\|\widehat Q-P_\sigma\|_1\ge\varepsilon/16)\ge\tfrac1{16}$ for $P=P_\sigma$. On the intersection of this event with $\mathcal E$, which has probability at least $\frac1{16}-\delta$, we obtain $W\ge\varepsilon/32=\sqrt{3m/n}/256\ge\sqrt{X/n}/256$. Multiplying by $H-1$ gives the lemma. On the MDP of Proposition~\ref{prop:lockout}, all randomness of the history lies in the successors of $a$. Calibrated AIF's index of $a$ at count $n$ is a function of the first $n$ successors of $a$ and of the adversary's independent randomness, and it is optimistic for all goals on the event of Theorem~\ref{thm:price}(c). The lemma therefore applies to it at every fixed count. Bayes-UCBVI's index is not optimistic for all goals at once. Under the uniform successor law, Proposition~\ref{prop:lockout} exhibits a goal chosen after the data that it underestimates.
\end{proof}

\subsection{Proof of Theorem~\ref{thm:price}(b)}
\label{app:adaptive-b}

\begin{lemma}[$\ell_1$ error under a Gaussian prior]
\label{lem:gauss-l1}
Let $m\ge2$ and $K\ge m$, let $\theta\sim N(0,I_m)$, and let $P_\theta$ be the law on $y_1,\dots,y_{2m}$ with $P_\theta(y_{2j-1})=(1+\varepsilon\tilde\theta_j)/(2m)$ and $P_\theta(y_{2j})=(1-\varepsilon\tilde\theta_j)/(2m)$, where $\varepsilon:=\sqrt{m/(64K)}$, $B:=1/(2\varepsilon)$ and $\tilde\theta_j:=\max\{-B,\min\{B,\theta_j\}\}$. Let $\omega_1,\omega_2,\dots$ be i.i.d.\ draws from $P_\theta$ given $\theta$, and for $k\le K$ and $n<K$ let $F_{k,n}$ map $\omega_{1:n}$ and randomness independent of $\theta$ and of the draws to a law on $\{y_1,\dots,y_{2m}\}$. If $m\ge128\ln(8K^2)$, then with probability at least $\frac34$, $\|F_{k,n}(\omega_{1:n})-P_\theta\|_1\ge\varepsilon/64$ for all $k\le K$ and $n<K$.
\end{lemma}
\begin{proof}
Fix $k$, $n$ and $f:=F_{k,n}(\omega_{1:n})$, and let $t_j:=m\big(f(y_{2j-1})-f(y_{2j})\big)/\varepsilon$. By the triangle inequality, pair $j$ contributes at least
\begin{equation*}
\big|(f(y_{2j-1})-f(y_{2j}))-(P_\theta(y_{2j-1})-P_\theta(y_{2j}))\big|=\frac\varepsilon m|t_j-\tilde\theta_j|
\end{equation*}
to $\|f-P_\theta\|_1$, which yields
\begin{equation*}
\|f-P_\theta\|_1\;\ge\;\frac{\varepsilon}{2m}\,N,\qquad N:=\big|\{j:|t_j-\tilde\theta_j|\ge\tfrac12\}\big|.
\end{equation*}
\emph{Posterior structure.} The likelihood of $\omega_{1:n}$ is $\prod_{s\le n}(1+\varepsilon y_s\tilde\theta_{j_s})/(2m)$, in the notation of the tape of Appendix~\ref{app:concentration}, and the prior is a product. Given $\omega_{1:n}$, the coordinates $\theta_j$ are therefore independent, and the posterior $\pi_j$ of $\theta_j$ depends only on the signs of the $c_j$ samples in pair $j$. Let $e_j:=\inf_{t\in\mathds Q}\pi_j(|\tilde\theta_j-t|\ge\tfrac12)$. The map $t\mapsto\pi_j(|\tilde\theta_j-t|<\frac12)$ is lower semicontinuous. Hence, the infimum over $\mathds Q$ equals that over $\mathds R$. Given $\omega_{1:n}$ and the randomness of $F_{k,n}$, the events $\{|t_j-\tilde\theta_j|\ge\frac12\}$ are thus independent across $j$, with probabilities at least $e_j$.

\emph{A two-point bound for one pair.} We show $\mathds{E}[e_j\mid c_j=c]\ge\frac14$ for $c\le2K/m$. Let $P^c_\vartheta$ be the law of $c$ signs that are i.i.d.\ with mean $\varepsilon\tilde\vartheta$, and let $t(\cdot)$ be any estimator from these signs, with $A_\vartheta:=\{|t-\tilde\vartheta|<\frac12\}$. Its risk is
\begin{equation*}
r:=\int\varphi(\vartheta)P^c_\vartheta(A_\vartheta^c)\,d\vartheta=\int\varphi(\vartheta+1)P^c_{\vartheta+1}(A_{\vartheta+1}^c)\,d\vartheta,
\end{equation*}
with $\varphi$ the standard normal density. Since $B=\sqrt{16K/m}\ge4$, neither $\vartheta$ nor $\vartheta+1$ is clipped for $\vartheta\in[-3,2]$. Therefore, $A_\vartheta$ and $A_{\vartheta+1}$ are disjoint there and $A^c_{\vartheta+1}\supseteq A_\vartheta$. Hence
\begin{align*}
2r\;&\ge\;\int_{-3}^2\min\{\varphi(\vartheta),\varphi(\vartheta+1)\}\big(P^c_\vartheta(A_\vartheta^c)+P^c_{\vartheta+1}(A_\vartheta)\big)\,d\vartheta\\
&\ge\;\int_{-3}^2\min\{\varphi(\vartheta),\varphi(\vartheta+1)\}\big(1-\mathrm{TV}(P^c_\vartheta,P^c_{\vartheta+1})\big)\,d\vartheta.
\end{align*}
A sign has success probability in $[\frac14,\frac34]$ under both parameters, which differ by $\varepsilon/2$ in this probability. Its KL divergence is therefore at most $(\varepsilon/2)^2/(3/16)=\frac43\varepsilon^2$. For $c\le2K/m$, the total is at most $\frac43\cdot(2K/m)\cdot m/(64K)=\frac1{24}$, and Pinsker's inequality gives $\mathrm{TV}\le(1/48)^{1/2}\le0.145$. The integral of the minimum equals $2(\Phi_{\mathrm N}(3)-\Phi_{\mathrm N}(\frac12))\ge0.614$, with $\Phi_{\mathrm N}$ the standard normal distribution function. Hence $r\ge\frac12\cdot0.614\cdot0.855\ge0.26$ for every estimator. For every $\eta'>0$, the estimator that picks the first rational, in a fixed enumeration, whose posterior risk is within $\eta'$ of $e_j$ is measurable, which yields $\mathds{E}[e_j\mid c_j=c]\ge0.26-\eta'$, which gives the claim.

\emph{Many pairs.} Since $\sum_jc_j=n<K$, fewer than $m/2$ pairs have $c_j>2K/m$. Given the counts, the $e_j\in[0,1]$ are independent, because the pairs have independent parameters and their signs are independent given the counts, and their means sum to at least $m/8$. Hoeffding's inequality gives $\sum_je_j\ge m/16$ except with probability $e^{-m/128}$. On this event, $N$ dominates a sum of independent Bernoulli variables with total mean at least $m/16$, and the Chernoff bound gives $N\ge m/32$ except with probability $e^{-m/128}$. Then $\|f-P_\theta\|_1\ge\varepsilon/64$. A union bound over the at most $K^2$ pairs $(k,n)$ gives failure probability at most $2K^2e^{-m/128}\le\frac14$ for $m\ge128\ln(8K^2)$.
\end{proof}

The next lemma shows that a goal-memoryless agent on the MDPs of Theorem~\ref{thm:price} acts on the samples of the current start state and on randomness that carries no information about them. Write $s_1,\dots,s_M$ for the start states, $y_1,\dots,y_Y$ for the successors and $z$ for the state reached by action $b$, let $K':=K/M$, and let the $i$-th block of episodes be $\mathcal B_i:=\{(i-1)K'+1,\dots,iK'\}$.

\begin{lemma}[Records of a goal-memoryless agent]
\label{lem:reconstruction}
Let a goal-memoryless agent act on an MDP of Theorem~\ref{thm:price} whose laws $P_1,\dots,P_M$ of action $a$ at the start states are independent, and let the adversary start every episode of $\mathcal B_i$ at $s_i$ and, during $\mathcal B_i$, use the transition law only through $P_i$. Let $\mathcal F_{i-1}$ collect $P_1,\dots,P_{i-1}$, everything observed before $\mathcal B_i$, the variables $U_{x,j}$ of Definition~\ref{def:goal-memoryless} and the randomness of the adversary. For $k\in\mathcal B_i$, let $D_k$ be the list of successors of $a$ at $s_i$ observed in $\mathcal B_i$ before episode $k$.
\begin{itemize}
\item[(a)] The records of all triples before episode $k$ are a fixed function of $(\mathcal F_{i-1},D_k,k)$.
\item[(b)] $P_i$ and the successors of $a$ at $s_i$ are independent of $\mathcal F_{i-1}$. Given $P_i$, these successors are i.i.d.\ draws from $P_i$, whatever the agent and the adversary do.
\item[(c)] Given the history before episode $k$ and its goal $(g,v)$, the agent plays $a$ in episode $k$ with probability $p_k(\mathcal F_{i-1},D_k,g,v)$, for a fixed function $p_k$.
\end{itemize}
\end{lemma}
\begin{proof}
\emph{(a), the start states.} Among the start states, the agent has visited only $s_1,\dots,s_{i-1}$ before $\mathcal B_i$, and only $s_i$ during $\mathcal B_i$. The records of the triples at $s_{i'}$, $i'\ne i$, are therefore those at the start of $\mathcal B_i$, which are functions of $\mathcal F_{i-1}$. At $s_i$, the records of $(1,s_i,a)$ are $D_k$, and those of $(1,s_i,b)$ repeat $z$ exactly $k-k_i-|D_k|$ times, with $k_i:=(i-1)K'+1$.

\emph{(a), the last step.} The successors and $z$ are absorbing. For $x\in\{y_1,\dots,y_Y,z\}$, the records of $(2,x,a')$ therefore repeat $x$ exactly $N_x(a')$ times, where $N_x(a')$ counts the earlier last-step plays of $a'$ at $x$. Since every goal pays at $x$ the same reward under all actions, condition~(ii) of Definition~\ref{def:goal-memoryless} applies at every last-step visit to $x$. At the $j$-th visit, the agent plays $\nu_x(N_x,U_{x,j})$ for a fixed rule $\nu_x$, with $N_x$ the counts before the visit. By induction over $j$, the counts after the $j$-th visit are a fixed function $c_x(j;U_{x,1},\dots,U_{x,j})$, since all counts vanish before the first visit. The number $J_x$ of last-step visits to $x$ before episode $k$ is the number of times $x$ was reached at step~$1$. Hence $J_{y_l}$ is the number of occurrences of $y_l$ in the records at step~$1$ before $\mathcal B_i$ and in $D_k$, and $J_z$ is the number of plays of $b$ at step~$1$ before $\mathcal B_i$ plus $k-k_i-|D_k|$. The records at the last step before episode $k$ are therefore the functions $c_x(J_x;U_{x,1},\dots,U_{x,J_x})$ of $(\mathcal F_{i-1},D_k,k)$.

\emph{(b).} Draw, for each start state, a table of successors of $a$ in advance, i.i.d.\ from $P_{i'}$ given $P_{i'}$, and let the $n$-th play of $a$ at $s_{i'}$ reveal the $n$-th entry, as in Appendix~\ref{app:lem-novelty-order}. This does not change the law of the interaction. All other transitions are deterministic, and the randomness of the agent and of the adversary is independent of the transition law. During $\mathcal B_{i'}$, the adversary uses only $P_{i'}$, and the agent observes only the table of $s_{i'}$. By induction over the blocks, $\mathcal F_{i-1}$ is a function of $P_1,\dots,P_{i-1}$, of the tables of $s_1,\dots,s_{i-1}$ and of randomness independent of the transition law. Since the laws are independent across start states, $\mathcal F_{i-1}$ is independent of $P_i$ and of the table of $s_i$, whose first entries are the successors observed at $s_i$. In particular, the goals of the earlier blocks carry no information about $P_i$.

\emph{(c).} By condition~(i) of Definition~\ref{def:goal-memoryless}, the policy of episode $k$ is a function of the records, of the goal and of randomness drawn afresh for the episode. Given the history and the goal, the probability of playing $a$ at $s_i$ is therefore a fixed function of the records and the goal, and (a) gives the claim. The records are part of the history, and an adversary that knows the agent can therefore compute $p_k$.
\end{proof}

Condition~(ii) of Definition~\ref{def:goal-memoryless} is what makes part~(a) hold. Suppose that the agent only had to choose among actions with equal rewards independently of the goal, and let it play, at the last step, its first action if it has played $b$ at step~$1$ an even number of times and its second action otherwise. With $M=1$, the histories whose first three episodes play $b,a,a$ and $a,b,a$ at step~$1$ can share the same list $D_k$ and lead to different records at the last step. These records then encode the order of the agent's earlier choices, which depend on the goals and thereby on $P_1$. The definition concerns the last step only, which suffices for horizon~$2$. At intermediate steps of a longer horizon, an agent could record past goals through its choices among equivalent actions, and a factor $H$ in Theorem~\ref{thm:price}(b) would need a finer definition of the class.

\begin{lemma}[Members of the goal-memoryless class]
\label{lem:membership}
Run with the goal $R_k$ in place of the reward, the following agents are goal-memoryless: Bayes-UCBVI and calibrated AIF with a uniform prior, both breaking ties at the last step by a fixed order of the actions; CaI; AIF, if at the last step it samples from the softmax of its index before the cap at $r_0$; and calibrated AIF with Bayes-UCBVI's prior, if at the last step it maximises its index before the cap, breaking ties by a fixed order of the actions. UCBVI is not goal-memoryless.
\end{lemma}
\begin{proof}
\emph{Condition~(i).} Before episode $k$, each of the five agents computes its values by backward induction from the counts of its step-indexed rows and from the goal $R_k$. Its quantile orders $\kappa^k_h(x,a)$ and logarithmic factors depend only on the counts of the row and on the fixed $X$, $A$, $H$, $K$ and $\delta$ (Definition~\ref{def:bayesq-ground}). The greedy agents are deterministic given these values, and AIF and CaI sample their actions from softmax laws with randomness drawn afresh.

\emph{Condition~(ii).} Let all actions at $x$ have the reward $c$ under $R_k$. Since $V_{k,H+1}\equiv0$, also at the pseudo-state, the index of $(x,a)$ at the last step is $c$ plus a term that depends only on the records of $(H,x,a)$ and on fixed parameters. For Bayes-UCBVI, the quantile of $p^\top V_{k,H+1}$ vanishes, and its index is $c$. For calibrated AIF with a uniform prior, $\Psi^{\mathrm u}_{k,H}=0$, since $\mathrm{span}(V_{k,H+1})=0$, and its index is $\min\{1,c\}$. For CaI, $\Psi=0$ and there is no cap. For AIF, the term is $\beta^{-1}\big(\mathds H[\widehat P^k_H(\cdot\mid x,a)]+\mathrm{Novelty}^k_H(x,a)\big)$. For calibrated AIF with Bayes-UCBVI's prior, $\mathrm{span}(V_{k,H+1})=0$, and the term is $2\sqrt2\,r_0H\ell^k_H(x,a)/\bar n^k_H(x,a)$, a function of the counts of $(H,x,a)$. The greedy agents maximise these indices with a fixed order for ties, and the softmax laws of AIF and CaI depend only on the differences of the indices, from which $c$ cancels. Each agent thus draws its action from a fixed law given the records of $(H,x,\cdot)$, with fresh randomness. Drawing it instead by the quantile transform of $U_{x,j}$, with $j$ the number of earlier last-step visits to $x$, which the records determine, does not change the law of the interaction.

\emph{The cap and UCBVI.} With the cap, AIF and calibrated AIF with Bayes-UCBVI's prior act on $\min\{r_0,c+\Psi_{k,H}(x,a)\}$, which depends on $c$ once the cap binds and thereby violates condition~(ii). The convention leaves the other results unchanged. For $H\ge2$, the floor of Proposition~\ref{prop:aif-cai-regret} concerns the first step. A maximiser of the index before the cap also maximises the capped index. Calibrated AIF therefore stays greedy for its capped values, which is all that Theorem~\ref{thm:calibrated} uses. UCBVI clamps its values at those of the previous episode, which depend on the previous goals, and violates condition~(i).
\end{proof}

\begin{proof}[Proof of Theorem~\ref{thm:price}(b)]
\emph{The instance and the reduction to one start state.} In this proof, $m=Y/2$ is the number of pairs of successors. The number of start states is $M$, and $K'=K/M$. The MDP is the Gaussian family of Appendix~\ref{app:goal-uniform}, with start states $s_1,\dots,s_M$, successors $y_1,\dots,y_Y$ and the state $z$. At $s_i$, action $a$ leads to $y_j$ with probability $P_i(y_j)$ and action $b$ leads to $z$, and the $y_j$ and $z$ are absorbing. The horizon is $H=2$. Hence $H-1=1$ in the formulas below. A goal pays $g\in\{0,1\}^{Y}$ at the successors, $v\in[0,1]$ at $z$ and $0$ at the start states. From $s_i$, action $a$ is therefore worth $P_ig$ and action $b$ is worth $v$. The adversary starts every episode of the block $\mathcal B_i$ at $s_i$ and, during $\mathcal B_i$, uses the transition law only through $P_i$. By Lemma~\ref{lem:reconstruction}, given $\mathcal F_{i-1}$, the successors of $a$ at $s_i$ are i.i.d.\ draws from $P_i$, the parameter $\theta^{(i)}$ of $P_i=P_{\theta^{(i)}}$ is still $N(0,I_m)$, and the agent plays $a$ in episode $k\in\mathcal B_i$ with probability $p_k(\mathcal F_{i-1},D_k,g_k,v_k)$. With $\varepsilon=\sqrt{m/(64K')}$, the single-start-state argument below, applied to the $K'$ episodes of $\mathcal B_i$ conditionally on $\mathcal F_{i-1}$, bounds the expected regret of $\mathcal B_i$, averaged over $\theta^{(i)}$, from below by $(H-1)\sqrt{YK'}/16000$. It needs $K'\ge m$ and $m\ge128\ln(8K'^2)$, which the hypotheses give. Indeed, $K'\ge Cm^2\ln^3(mK)\ge m$, and $m\ge300\ln(3K)$ implies $m/128\ge2.34\ln(3K)\ge\ln(8K^2)$. Summing over the $M$ blocks, the expected regret under the prior is at least $M(H-1)\sqrt{YK/M}/16000=(H-1)\sqrt{MYK}/16000$. For $M=1$ the start state is fixed and the argument below is the whole proof. In the rest of the proof, $K$ stands for $K'$, the start state is $s_0:=s_i$, $\theta:=\theta^{(i)}$, $P:=P_\theta$ and $D:=D_k$, and we suppress the fixed $\mathcal F_{i-1}$.

\emph{One start state.} The goal of episode $k$ pays $g_k\in\{0,1\}^{Y}$ at the successors $y_j$ and $v_k\in[0,1]$ at $z$, and the adversary chooses $(g_k,v_k)$ from $P$ and the history. All states after the first step are absorbing and their rewards do not depend on the action. The regret of episode $k$ therefore depends only on the action at $s_0$, where the agent plays $a$ with probability $p_k(D,g_k,v_k)$ by Lemma~\ref{lem:reconstruction}(c). The functions $p_k$ do not depend on $P$, and the list $D$ consists of the first entries of an i.i.d.\ sequence drawn from $P$.

\emph{A threshold for every goal.} Fix $k$ and $D$, write $p(g,v):=p_k(D,g,v)$, let $\tau(g):=\sup\{v\in[0,1]:p(g,v)\ge\frac12\}$ with $\sup\emptyset:=0$, and let $\Delta_k:=\max_g|\tau(g)-Pg|$, attained at some $g$. If $\tau(g)>Pg$, the adversary picks some $v\ge(\tau(g)+Pg)/2$ with $p(g,v)\ge\frac12$, so that $b$ is optimal while the agent plays $a$ with probability at least $\frac12$. If $\tau(g)<Pg$, it picks $v=(\tau(g)+Pg)/2$, for which $p(g,v)<\frac12$, so that $a$ is optimal while the agent plays $b$ with probability above $\frac12$. In both cases the expected regret of episode $k$, given the history, is at least $(H-1)\Delta_k/4$. The adversary thus depends on the agent through $p_k$, which is why Theorem~\ref{thm:price}(b) states that every goal-memoryless agent has the bound against some adversary. This is the standard form of a minimax lower bound, $\inf_{\text{agent}}\sup_{\text{adversary}}\mathds{E}\mathfrak R^{\mathrm{goal}}_K\ge\sqrt{MYK}/16000$ over goal-memoryless agents, and it is the form that the upper bound of part~(c), which holds for one agent against every adversary, matches.

\emph{From thresholds to an estimate of $P$.} Let $\widehat Q_k$ minimise $\max_g|Qg-\tau(g)|$ over the laws $Q$ on the successors. Since $P$ is a candidate, $\max_g|\widehat Q_kg-\tau(g)|\le\Delta_k$, which yields $(\widehat Q_k-P)g\le2\Delta_k$ for every $g$, and $\frac12\|\widehat Q_k-P\|_1\le2\Delta_k$ as in the proof of Lemma~\ref{lem:goal-uniform-width}. Hence
\begin{equation*}
\mathds{E}\mathfrak R^{\mathrm{goal}}_K\ge\frac{H-1}{16}\sum_k\mathds{E}\|\widehat Q_k-P\|_1.
\end{equation*}
Let $\omega_1,\omega_2,\dots$ be the successors of $a$ in the order in which the agent observes them, which are i.i.d.\ draws from $P$ whatever the agent and the adversary do. By Lemma~\ref{lem:reconstruction}(c), $\tau$ and hence $\widehat Q_k$ are functions of $D=\omega_{1:n}$ at $n=n_k$. Thus $\widehat Q_k=F_{k,n}(\omega_{1:n})$ for functions $F_{k,n}$ of the samples and of randomness independent of $\theta$ and of the samples, namely $\mathcal F_{i-1}$. The count $n_k\le k-1$ is random and may depend on $P$ through the goals. We therefore bound the error at all pairs $(k,n)$ at once, with Lemma~\ref{lem:gauss-l1}. These steps use only that the successors are i.i.d.\ and that indicator goals detect $\ell_1$ errors. The prior enters only through the lemma.

\emph{Conclusion.} On the event $G$ of Lemma~\ref{lem:gauss-l1}, $\|\widehat Q_k-P\|_1\ge\varepsilon/64$ for every $k$, whatever the count $n_k$, which yields
\begin{equation*}
\mathds{E}\mathfrak R^{\mathrm{goal}}_K\ge\frac{H-1}{16}\cdot\frac34\cdot\frac{K\varepsilon}{64}=\frac{3(H-1)K\varepsilon}{4096},
\end{equation*}
in expectation over the prior. Since $K\varepsilon=\sqrt{mK}/8$ and $m=Y/2$, this is at least $(H-1)\sqrt{YK}/16000$. An adversary that knows $P$ reveals information about $P$ through its goals, and an agent that uses past goals as data could learn from them, which is why part~(b) restricts the agents.
\end{proof}

\emph{The hypercube family.} Proposition~\ref{prop:goal-reading} uses the same argument, in the notation of the proof of Theorem~\ref{thm:price}(b), on the family $P_\sigma$ of the proof of Lemma~\ref{lem:goal-uniform-width}, with $\sigma$ uniform on $\{-1,1\}^m$, $m=\lfloor Y/2\rfloor$ and $\varepsilon:=\sqrt{3m/(64K)}\le\frac12$, drawn independently at each start state. Lemma~\ref{lem:reconstruction} holds for every prior under which the laws at the start states are independent, and hence on this family. The reduction to one start state is therefore unchanged, and in the rest of this paragraph $K$ stands for $K'=K/M$. Only the error bound changes. Let $\widehat\sigma_i:=1$ if $f(y_{2i-1})>f(y_{2i})$ and $\widehat\sigma_i:=-1$ otherwise. As in that proof, each wrong sign adds at least $\varepsilon/m$ to $\|f-P_\sigma\|_1$. Given $\omega_{1:n}$, the signs $\sigma_i$ are independent, and the posterior probability $q_i$ of $\sigma_i=1$ depends only on the samples in pair~$i$. Given $\omega_{1:n}$, the number of wrong signs is therefore a sum of independent Bernoulli variables whose means are at least $e_i:=\min(q_i,1-q_i)$. For $c_i\le2K/m$, the bound of Eq.~\eqref{eq:tv-bernoulli} in Appendix~\ref{app:unstructured} with $\varepsilon/2$ in place of $\varepsilon$ bounds the total variation between the laws of the pair's samples under $\sigma_i=\pm1$ by $\varepsilon\sqrt{8c_i/3}\le\frac12$, which yields $\mathds{E}[e_i\mid c_1,\dots,c_m]\ge\frac14$. Hoeffding's inequality and the Chernoff bound, as above, give $\|f-P_\sigma\|_1\ge\varepsilon/32$ except with probability $2e^{-m/128}$. For $Y\ge600\ln(3K)$, $m/128\ge\ln(8K^2)$. The bound therefore holds for all $k\le K$ and $n<K$ with probability at least $\frac34$, and the conclusion above gives
\begin{equation*}
\mathds{E}\mathfrak R^{\mathrm{goal}}_K\ge\frac{3(H-1)K\varepsilon}{2048}\ge(H-1)\sqrt{YK}/6000
\end{equation*}
under the uniform prior, for $K\ge Y$. Summed over the $M$ blocks, this is $(H-1)\sqrt{MYK}/6000$ for $K'\ge Y$, the lower bound of Proposition~\ref{prop:goal-reading}. Some $\sigma$ attains at least this average. The bound therefore also holds on a fixed MDP.

\subsection{Concentration Tools}
\label{app:concentration}

This subsection collects, with proofs, the concentration inequalities of Appendix~\ref{app:all-agents} and the tape notation that Appendix~\ref{app:adaptive-b} also uses. A real random variable $\chi$ is \emph{sub-Gaussian with proxy} $\sigma^2$ if $\mathds{E}e^{\lambda(\chi-\mathds{E}\chi)}\le e^{\lambda^2\sigma^2/2}$ for all $\lambda\in\mathds R$. We use five standard facts \citep{boucheron2013concentration}. A variable with values in an interval of length $\ell$ is sub-Gaussian with proxy $\ell^2/4$ (Hoeffding's lemma). The proxies of independent variables add, and so do those of martingale differences whose conditional proxies are fixed in advance. A sub-Gaussian variable with proxy $\sigma^2$ satisfies $\mathds{P}(|\chi-\mathds{E}\chi|\ge t)\le2e^{-t^2/(2\sigma^2)}$. Bernstein's inequality states that independent $\chi_1,\dots,\chi_n$ with $|\chi_i-\mathds{E}\chi_i|\le b$ and $\sum_i\mathds{V}\chi_i\le v$ satisfy
\begin{equation*}
\mathds{P}\big(|\sum_i(\chi_i-\mathds{E}\chi_i)|\ge\sqrt{2vx}+bx/3\big)\le2e^{-x}
\end{equation*}
for $x>0$. Finally, a $\varrho$-net of the unit sphere of a $d$-dimensional space, that is, a subset within distance $\varrho$ of every point, exists with at most $(1+2/\varrho)^d$ points.

\emph{The tape.} Fix $\theta\in\mathds R^m$ and $\varepsilon>0$, and let the clipped coordinates of the Gaussian family be $\tilde\theta_j:=\max\{-B,\min\{B,\theta_j\}\}$ and $B:=1/(2\varepsilon)$, which gives $\varepsilon|\tilde\theta_j|\le\frac12$. The \emph{tape} is an i.i.d.\ sequence $((j_s,y_s))_{s\ge1}$ in which $j_s$ is uniform on $\{1,\dots,m\}$ and $y_s\in\{-1,1\}$ has $\mathds{P}(y_s=1\mid j_s)=(1+\varepsilon\tilde\theta_{j_s})/2$. We call $j_s$ the \emph{label} of entry $s$. Label $j$ belongs to the pair of successors $y_{2j-1},y_{2j}$. The $s$-th sample of action $a$ at the start state with parameter $\theta$ is the successor $y_{2j_s-1}$ if the sign $y_s$ equals $1$ and the successor $y_{2j_s}$ otherwise. The subscript of a sign is a tape position, and that of a successor is a state index. The agent and the adversary only decide when the next entry is read. Every statement about fixed entries of the tape therefore holds whatever they do. For an interval $I$ of positions, let $n_j(I):=|\{s\in I:j_s=j\}|$, and let $n_j(n):=n_j(\{1,\dots,n\})$ and $d_j(n):=\sum_{s\le n:\,j_s=j}y_s$.

\begin{lemma}[Gaussian norm and maximum]
\label{lem:gauss-norm}
Let $\theta\sim N(0,I_m)$ with $m\ge2$, and let $x>0$. Then $\mathds{P}\big(\|\theta\|^2\ge m+2\sqrt{mx}+2x\big)\le e^{-x}$, $\mathds{P}\big(\|\theta\|^2\le m-2\sqrt{mx}\big)\le e^{-x}$, and $\mathds{P}\big(\max_j|\theta_j|>\sqrt{6\ln m}\big)\le m^{-2}$. In particular, $0.95m\le\|\theta\|^2\le1.05m$ except with probability $2e^{-m/1700}$.
\end{lemma}
\begin{proof}
The first two bounds are the tail bounds for the $\chi^2$ distribution with $m$ degrees of freedom \citep{boucheron2013concentration}. For a standard normal $g$ and $t\ge1$, $\mathds{P}(|g|>t)\le\sqrt{2/\pi}\,e^{-t^2/2}/t\le e^{-t^2/2}$, which is $m^{-3}$ at $t=\sqrt{6\ln m}$, and a union bound over the $m$ coordinates gives the third bound. For $x=m/1700$, $2\sqrt{mx}+2x\le0.0486m+0.0012m<0.05m$.
\end{proof}

\begin{lemma}[Counts and signs on the tape]
\label{lem:tape}
Fix $\theta$ and let $x>0$.
\begin{itemize}
\item[(i)] Let $I$ be an interval of positions and $a\in\mathds R^m$, both fixed given the tape entries before $I$ and randomness independent of the tape. Let $\bar a:=\sum_ja_j/m$ and $b:=\max_j|a_j-\bar a|$. With probability at least $1-2e^{-x}$, $\Big|\sum_{s\in I}a_{j_s}-|I|\bar a\Big| \le \sqrt{2|I|\,\|a\|^2x/m}+bx/3$.
In particular, $\big|n_j(I)-|I|/m\big|\le\sqrt{2|I|x/m}+x/3$.
\item[(ii)] For each label $j$, with probability at least $1-4Ke^{-x}$, $\big|d_j(n)-n_j(n)\varepsilon\tilde\theta_j\big|\le\sqrt{2n_j(n)x}$ for all $n\le2K$.
\end{itemize}
\end{lemma}
\begin{proof}
\emph{(i)}~Given the entries before $I$ and the independent randomness, $a$ is fixed and the labels $j_s$, $s\in I$, are i.i.d.\ uniform. The variables $a_{j_s}$ have mean $\bar a$, variance at most $\sum_ja_j^2/m$, and deviation at most $b$ from their mean. Hence, Bernstein's inequality gives the bound. For $a$ the $j$-th standard basis vector, we have $b\le1$. \emph{(ii)}~Let $y^{(j)}_i$ be the sign of the $i$-th entry with label $j$. Given $\theta$, the labels and signs are independent across entries, which makes $y^{(j)}_1,y^{(j)}_2,\dots$ i.i.d.\ with mean $\varepsilon\tilde\theta_j$ and values in an interval of length $2$, and $d_j(n)=\sum_{i\le n_j(n)}y^{(j)}_i$. Hoeffding's lemma, the additivity of proxies and the sub-Gaussian tail give $\mathds{P}\big(|\sum_{i\le r}(y^{(j)}_i-\varepsilon\tilde\theta_j)|\ge\sqrt{2rx}\big)\le2e^{-x}$ for each $r$, and a union bound over $r\le2K$ gives the claim. The sign sequence of a pair is indexed by its own visit number. The claim therefore holds whatever rule decides how many entries are read.
\end{proof}

\begin{lemma}[Overshoot on a spherical cap]
\label{lem:cap}
Let $V$ be a subspace of $\mathds R^m$ of dimension $d\ge4$, let $w\in V$ with $\|w\|=r$, let $0<z_0<r$, and let $u$ be uniform on the unit sphere of $V$ and $s:=\langle u,w\rangle$. Then $s$ has density proportional to $(1-s^2/r^2)^{(d-3)/2}$ on $[-r,r]$. Conditionally on $|s|\ge z_0$, the overshoot $|s|-z_0$ is stochastically dominated by the exponential law with mean $\mu:=r^2/((d-3)z_0)$. Hence, $\mathds{E}[s^2\mid|s|\ge z_0]\le z_0^2+2z_0\mu+2\mu^2$ and $\mathds{P}(|s|>z_0+t\mu\mid|s|\ge z_0)\le e^{-t}$ for $t\ge0$. Conditionally on $s$, the component of $u$ orthogonal to $w$ is uniform on the sphere of radius $\sqrt{1-s^2/r^2}$ in $V\cap w^\perp$.
\end{lemma}
\begin{proof}
The density of the coordinate of a uniform point on the unit sphere of $\mathds R^d$ along a fixed unit vector is proportional to $(1-t^2)^{(d-3)/2}$, and given this coordinate the remaining component is uniform on the sphere of the orthogonal complement with the matching radius. Scaling by $r$ gives the first and last claims. By symmetry it suffices to consider $s\ge z_0$. The density $g(s)=(1-s^2/r^2)^{(d-3)/2}$ has $-(\ln g)'(s)=(d-3)s/(r^2-s^2)\ge(d-3)z_0/r^2=1/\mu$ for $s\in[z_0,r)$, which yields $g(s+t)\le e^{-t/\mu}g(s)$ for $s\ge z_0$ and $t\ge0$. Hence $\int_{z_0+t}^rg=\int_{z_0}^{r-t}g(v+t)\,dv\le e^{-t/\mu}\int_{z_0}^rg$, which is the stochastic domination. The overshoot $\zeta:=s-z_0$ thus has $\mathds E\zeta\le\mu$ and $\mathds E\zeta^2\le2\mu^2$, which yields $\mathds{E}[(z_0+\zeta)^2]\le z_0^2+2z_0\mu+2\mu^2$.
\end{proof}

\begin{lemma}[Projections of a uniform point on a sphere]
\label{lem:sphere}
Let $V$ be a subspace of $\mathds R^m$ of dimension $d'\ge4$, let $P_V$ be the orthogonal projection onto $V$, and let $u$ be uniform on the unit sphere of $V$. For every fixed $x\in\mathds R^m$ and $t\ge0$,
\begin{equation*}
\mathds{P}\big(|\langle u,x\rangle|>t\|P_Vx\|\big)\le2e^{-(d'-3)t^2/2}.
\end{equation*}
\end{lemma}
\begin{proof}
Since $\langle u,x\rangle=\langle u,P_Vx\rangle$, the event is empty if $P_Vx=0$, and otherwise we may take $x\in V$ with $\|x\|=1$. By Lemma~\ref{lem:cap} with $r=1$, $s=\langle u,x\rangle$ has density proportional to $(1-s^2)^k$ with $k:=(d'-3)/2\ge\frac12$. Using $1-s^2\le e^{-s^2}$ and $(t+v)^2\ge t^2+v^2$,
\begin{equation*}
\int_t^1(1-s^2)^k\,ds\;\le\;e^{-kt^2}\int_0^\infty e^{-kv^2}\,dv\;=\;\tfrac12\sqrt{\pi/k}\,e^{-kt^2}.
\end{equation*}
The normaliser is $\int_{-1}^1(1-s^2)^k\,ds=\sqrt\pi\,\Gamma(k+1)/\Gamma(k+\tfrac32)\ge\sqrt{\pi/(k+1)}$, since the log-convexity of $\Gamma$ gives $\Gamma(k+\tfrac32)\le\Gamma(k+1)^{1/2}\Gamma(k+2)^{1/2}=\Gamma(k+1)\sqrt{k+1}$. Hence $\mathds{P}(s\ge t)\le\tfrac12\sqrt{(k+1)/k}\,e^{-kt^2}\le e^{-kt^2}$ for $k\ge\frac13$, and symmetry gives the claim.
\end{proof}

\emph{Increments.} For Lemmas~\ref{lem:increments} and~\ref{lem:cap-data}, fix positions $0=N_0<N_1<\dots<N_F$ and let $I_k:=\{N_{k-1}+1,\dots,N_k\}$ be the \emph{increments}. Define $\xi_k,\mu_k\in\mathds R^m$ by
\begin{equation*}
\xi_{k,j}:=\varepsilon\textstyle\sum_{s\in I_k:\,j_s=j}\big(y_s-\varepsilon\tilde\theta_j\big),\qquad\mu_{k,j}:=\varepsilon^2n_j(I_k)\tilde\theta_j,
\end{equation*}
so that $\mu_k+\xi_k$ is the vector with coordinates $\varepsilon\sum_{s\in I_k:\,j_s=j}y_s$. Let
\begin{equation*}
e_k^2:=\mathds{E}\big[\|\xi_k\|^2\mid j_1,j_2,\dots\big]=\varepsilon^2\sum_jn_j(I_k)(1-\varepsilon^2\tilde\theta_j^2),
\end{equation*}
which lies in $[\frac34,1]\cdot\varepsilon^2|I_k|$, let $\Gamma:=\mathrm{diag}(e_1,\dots,e_F)$, and let $\xi(\gamma):=\sum_k\gamma_k\xi_k$ for $\gamma\in\mathds R^F$.

\begin{lemma}[Increment noise]
\label{lem:increments}
Fix $\theta$ and the labels $j_1,j_2,\dots$, and assume $n_j(I_k)\le2|I_k|/m$ for all $k$ and $j$. Let $\mathcal S$ be a subspace of dimension at most $F+1$ that is fixed given $\theta$ and the labels. Let $\hat c_1,\dots,\hat c_F$ be random vectors such that each $\hat c_f$ is a function of the first $N_f$ entries of the tape and of randomness independent of the tape, and such that the nonzero ones are orthonormal. Let $\delta\in(0,1)$, let $\eta:=18\sqrt{(F\ln(12m)+\ln(8/\delta))/m}$, and assume $\eta\le\frac14$. With probability at least $1-\delta$ over the signs, the following bounds hold for all $\gamma,\beta\in\mathds R^F$.
\begin{itemize}
\item[(B1)] $\|\xi(\gamma)\|\ge(1-\eta)\|\Gamma\gamma\|$.
\item[(B2)] $\|P_{\mathcal S}\,\xi(\gamma)\|\le\eta\|\Gamma\gamma\|$. In particular, $|\langle\xi(\gamma),\theta\rangle|\le\eta\|\theta\|\|\Gamma\gamma\|$ if $\theta\in\mathcal S$.
\item[(B3)] $\big|\sum_k\gamma_k\big\langle\sum_{f<k}\beta_f\hat c_f,\xi_k\big\rangle\big|\le\eta\|\beta\|\|\Gamma\gamma\|$.
\end{itemize}
\end{lemma}
\begin{proof}
\emph{A variance proxy.} For fixed $a\in\mathds R^m$ and $\gamma\in\mathds R^F$, $\langle a,\xi(\gamma)\rangle=\sum_k\gamma_k\sum_{s\in I_k}\varepsilon a_{j_s}(y_s-\varepsilon\tilde\theta_{j_s})$ is a sum of independent centred terms, the term of entry $s$ taking values in an interval of length $2\varepsilon|\gamma_k||a_{j_s}|$. Its proxy is therefore
\begin{equation*}
\varepsilon^2\sum_k\gamma_k^2\sum_jn_j(I_k)a_j^2\le\frac2m\|a\|^2\sum_k\gamma_k^2\varepsilon^2|I_k|\le \sigma_\star^2\|a\|^2\|\Gamma\gamma\|^2,
\end{equation*}
with $\sigma_\star^2:=8/(3m)$.

\emph{Nets for bilinear forms.} Let $\mathfrak b(\cdot,\cdot)$ be bilinear on two spaces, and let $\mathcal N_1,\mathcal N_2$ be $\frac14$-nets of their unit spheres. For unit $x,y$ with nearest net points $x_0,y_0$,
\begin{equation*}
\mathfrak b(x,y)=\mathfrak b(x_0,y_0)+\mathfrak b(x-x_0,y)+\mathfrak b(x_0,y-y_0),
\end{equation*}
which yields $\sup|\mathfrak b|\le\max_{\mathcal N_1\times\mathcal N_2}|\mathfrak b|+\tfrac12\sup|\mathfrak b|$ and $\sup|\mathfrak b|\le2\max_{\mathcal N_1\times\mathcal N_2}|\mathfrak b|$, the suprema over unit vectors.

\emph{(B2).} Apply this to $\mathfrak b(a,\gamma'):=\langle a,\xi(\Gamma^{-1}\gamma')\rangle$ on the unit spheres of $\mathcal S$ and $\mathds R^F$, with nets of sizes at most $9^{F+1}$ and $9^F$. By the proxy, each value at a pair of net points exceeds $t$ in absolute value with probability at most $2e^{-t^2/(2\sigma_\star^2)}$. With $t^2:=2\sigma_\star^2\big((2F+1)\ln9+\ln(8/\delta)\big)$ and a union bound, $\sup|\mathfrak b|\le2t$ except with probability $\delta/4$, and $2t=\sqrt{(64/(3m))((2F+1)\ln9+\ln(8/\delta))}\le\eta$.

\emph{(B3).} For fixed unit $\beta$ and $\gamma'$, let $u_k:=\sum_{f<k}\beta_f\hat c_f$, for which $\|u_k\|\le1$, and $S:=\sum_k(\gamma'_k/e_k)\langle u_k,\xi_k\rangle$. Given the first $N_{k-1}$ entries and the independent randomness, $u_k$ is fixed, because $\hat c_f$ with $f<k$ depends only on the first $N_f\le N_{k-1}$ entries, and the signs of $I_k$ are still independent with the same laws. By the proxy, the $k$-th term is conditionally sub-Gaussian with proxy $\sigma_\star^2(\gamma'_k)^2$, making $S$ sub-Gaussian with proxy $\sigma_\star^2$. The map $(\beta,\gamma')\mapsto S$ is bilinear, and the argument of (B2) with nets of size $9^F$ bounds it by $\eta$ except with probability $\delta/4$.

\emph{(B1).} Fix a unit $\gamma'$ and let $\gamma:=\Gamma^{-1}\gamma'$ and $\chi_j:=\xi(\gamma)_j$. Since the $\chi_j$ use disjoint entries of the tape, they are independent and centred, with $\sum_j\mathds{E}\chi_j^2=\|\Gamma\gamma\|^2=1$. By the proxy with $a$ the $j$-th standard basis vector, $\chi_j$ is sub-Gaussian with proxy $\sigma_j^2:=\varepsilon^2\sum_k\gamma_k^2n_j(I_k)$, where $\sigma_j^2\le \sigma_\star^2$ and $\sum_j\sigma_j^2=\varepsilon^2\sum_k\gamma_k^2|I_k|\le\tfrac43$. Integrating the tail bound gives $\mathds{E}\chi_j^4\le16\sigma_j^4$. For independent nonnegative $W_j$,
\begin{equation*}
\mathds{P}\big(\sum_j(W_j-\mathds{E}W_j)\le-u\big)\le\exp\big(-u^2/(2\sum_j\mathds{E}W_j^2)\big)
\end{equation*}
 \citep{boucheron2013concentration}. With $W_j=\chi_j^2$ and
\begin{equation*}
\sum_j\mathds{E}\chi_j^4\le16\sigma_\star^2\cdot\frac43=512/(9m),
\end{equation*}
 this gives $\mathds{P}(\|\xi(\gamma)\|^2\le1-u)\le e^{-9mu^2/1024}$. For the upper tail, let $g$ be a standard normal variable independent of $\chi_j$ and $\lambda:=1/(4\sigma_\star^2)$. Then $\mathds{E}e^{\lambda \chi_j^2}=\mathds{E}e^{\sqrt{2\lambda}g\chi_j}\le\mathds{E}e^{\lambda\sigma_j^2g^2}=(1-2\lambda\sigma_j^2)^{-1/2}\le e^{\sigma_j^2/(2\sigma_\star^2)}$, since $2\lambda\sigma_j^2\le\frac12$. Markov's inequality gives $\mathds{P}(\|\xi(\gamma)\|^2\ge4)\le e^{2/(3\sigma_\star^2)-1/\sigma_\star^2}=e^{-m/8}$. Let $u:=\frac{32}3\sqrt{(F\ln(9m)+\ln(4/\delta))/m}$, which lies in $[1/m,\frac14]$, since $u\ge\frac{32}3\sqrt{\ln9/m}\ge1/m$ and $\frac53u\le\eta\le\frac14$, and let $\mathcal N'$ be a $u/4$-net of the unit sphere of $\mathds R^F$, of size at most $(9/u)^F\le(9m)^F$. A union bound shows that $1-u\le\|\xi(\Gamma^{-1}\gamma')\|^2\le4$ for all $\gamma'\in\mathcal N'$, except with probability $(9m)^F(e^{-9mu^2/1024}+e^{-m/8})\le\delta/2$, where $m/8\ge F\ln(9m)+\ln(4/\delta)$ because $\eta\le\frac14$. On this event the linear map $\gamma'\mapsto\xi(\Gamma^{-1}\gamma')$ has norm at most $2/(1-u/4)\le\frac83$, and every unit $\gamma'$ has
\begin{equation*}
\|\xi(\Gamma^{-1}\gamma')\|\ge\sqrt{1-u}-u/4\cdot\frac83\ge1-\frac53u\ge1-\eta.
\end{equation*}
\end{proof}

\begin{lemma}[Caps against data means]
\label{lem:cap-data}
Fix $\theta$ with $\|\theta\|_\infty\le\sqrt{6\ln m}\le B$, let $|I_k|\ge b$ for all $k$, and let $\hat c_1,\dots,\hat c_F$ be as in Lemma~\ref{lem:increments}, with $Z_f:=\langle\hat c_f,\theta\rangle$. Let $\delta\in(0,1)$ and $x:=\ln(2F^2/\delta)$. With probability at least $1-\delta$ over the tape, the $F\times F$ matrix with entries $\mathds 1\{k>f\}\langle\hat c_f,\mu_k\rangle/e_k$ has Frobenius norm at most
\begin{equation*}
\Big(\tfrac43\,\varepsilon^2N_F\textstyle\sum_fZ_f^2\big/m^2\Big)^{1/2}+4F\varepsilon\sqrt{\ln m}\,\big(\sqrt{x/m}+x/\sqrt b\big).
\end{equation*}
\end{lemma}
\begin{proof}
Since no coordinate of $\theta$ is clipped, $\langle\hat c_f,\mu_k\rangle=\varepsilon^2\sum_{s\in I_k}a_{j_s}$ with $a_j:=\hat c_{f,j}\theta_j$, $\bar a=Z_f/m$, $\|a\|^2\le6\ln m$ and $\max_j|a_j-\bar a|\le2\sqrt{6\ln m}$. For $k>f$, $a$ is fixed given the first $N_f\le N_{k-1}$ entries and the independent randomness. Lemma~\ref{lem:tape}(i) and a union bound over the at most $F^2$ pairs $(f,k)$ give, except with probability $\delta$,
\begin{equation*}
\big|\langle\hat c_f,\mu_k\rangle-\varepsilon^2|I_k|Z_f/m\big|\;\le\;\varepsilon^2\Big(\sqrt{12|I_k|x\ln m/m}+\tfrac23\sqrt{6\ln m}\,x\Big).
\end{equation*}
Split each entry into its main part $\mathds 1\{k>f\}\varepsilon^2|I_k|Z_f/(me_k)$ and a deviation. Divided by $e_k\ge(\sqrt3/2)\varepsilon\sqrt{|I_k|}$, the right-hand side bounds each deviation by $4\varepsilon\sqrt{\ln m}(\sqrt{x/m}+x/\sqrt b)$. The deviations therefore have Frobenius norm at most $F$ times this. The main parts have entries at most $2\varepsilon\sqrt{|I_k|}\,|Z_f|/(\sqrt3\,m)$, and
\begin{equation*}
\sum_{f,k}\frac43\varepsilon^2|I_k|Z_f^2/m^2=\frac43\varepsilon^2N_F\sum_fZ_f^2/m^2.
\end{equation*}
The triangle inequality for the Frobenius norm gives the claim.
\end{proof}

\subsection{Proof of Theorem~\ref{thm:price}(a)}
\label{app:all-agents}

\begin{proof}[Proof of Theorem~\ref{thm:price}(a)]
\emph{Reduction to one start state.} The MDP has start states $s_1,\dots,s_M$, successors $y_1,\dots,y_Y$ with $Y=2m$, and a state $z$. At $s_i$, action $a$ leads to $y_{2j-1}$ with probability $(1+\varepsilon\tilde\theta^{(i)}_j)/(2m)$ and to $y_{2j}$ with probability $(1-\varepsilon\tilde\theta^{(i)}_j)/(2m)$, and action $b$ leads to $z$. All successors and $z$ are absorbing. The parameters $\theta^{(1)},\dots,\theta^{(M)}$ are independent draws from $N(0,I_m)$, with $\tilde\theta_j:=\max\{-B,\min\{B,\theta_j\}\}$, $\varepsilon:=\sqrt{m/(64K')}$, $K':=K/M$, and $B:=1/(2\varepsilon)$. This is the Gaussian family of Section~\ref{sec:goal-uniform}.

A goal pays $g\in[0,1]^Y$ at the successors and $v$ at $z$. Action $a$ at $s_i$ is thus worth
\begin{equation*}
P_ig\;=\;\alpha_g+\frac{\varepsilon}{2m}\langle c_g,\tilde\theta^{(i)}\rangle,\qquad \alpha_g:=\frac1{2m}\sum_{j\le m}\big(g(y_{2j-1})+g(y_{2j})\big),
\end{equation*}
with $c_{g,j}:=g(y_{2j-1})-g(y_{2j})\in[-1,1]$. The offset $\alpha_g$ is the payoff of the goal under the uniform law on the successors, which the agent can compute. As in the proof of Theorem~\ref{thm:price}(b), the adversary starts the episodes of block $i$ at $s_i$. Since samples at $s_i$ depend on $\theta^{(i)}$ alone, $\theta^{(i)}$ is still $N(0,I_m)$ given everything observed before block $i$. The single-start-state bound below therefore applies to each block with $K'$ in place of $K$, and summing $c\,m^{1/4}\sqrt{K'/\ln m}$ over the $M$ blocks gives part~(a).

From now on $M=1$, and we write $K$ for $K'$ and $\theta$ for $\theta^{(1)}$. The samples of action $a$ are the entries of the tape of Appendix~\ref{app:concentration}, the free samples below included. The agent is arbitrary. It may know the adversary's strategy and use the goals as data. We also give it side information, namely the directions $\hat c_f$ below and free samples of action $a$ at the start of each phase. The side information is generated from the true transition law and is therefore correct. An agent that receives it can therefore simulate every agent that does not, by ignoring it, and the smallest regret over agents with side information is at most the smallest regret over agents without it. An adversary that cannot hand out samples can simulate the free samples privately from the transition law and compute $r^{(f)}$ and $\hat c_f$ from them and the history. The resulting game has the law of the augmented one with an agent that ignores its side information. Therefore, the lower bound also holds for adversaries of Theorem~\ref{thm:price}.

\emph{The adversary.} Let $\rho:=\sqrt{m/(36\ln m)}$, $z_0:=(m/16)^{1/4}$, $L:=\lfloor K/(4z_0^2)\rfloor$, $F:=\lfloor K/L\rfloor$, $b:=\lceil m\ln^2(mK)\rceil$ and $D_0:=3\varepsilon\rho z_0/(8m)$. The episodes form phases $f=1,\dots,F$ of $L$ episodes each.
\begin{itemize}
\item At the start of phase $f$, the agent receives free samples until it has observed exactly $N_f:=(f-1)L+fb$ samples of action $a$. Since it plays $a$ at most $L$ times in a phase, this takes at least $b$ free samples. The samples observed before phase $f$ are thus the first $N_f$ entries of the tape, a number that does not depend on the agent.
\item Let $r^{(f)}_j:=\varepsilon(\#y_{2j-1}-\#y_{2j})$ count these $N_f$ samples.
\item Let $W_f:=\mathrm{span}\{r^{(1)},\dots,r^{(f)},\hat c_1,\dots,\hat c_{f-1}\}$ and $w_f:=P_{W_f^\perp}\theta$, and let $\sigma_f$ be the uniform law on the unit sphere $\mathbb S_f$ of $W_f^\perp$.
\item If $\|w_f\|\le z_0$, the phase is \emph{neutral}. Then $\hat c_f:=0$, and every goal of the phase pays $g(y)=0$ at every successor $y$ and $v=0$. Otherwise the adversary draws $\hat c_f\sim\sigma_f(\cdot\mid\mathrm{Cap}_f(\theta))$, where $\mathrm{Cap}_f(\theta):=\{u\in\mathbb S_f:|\langle u,\theta\rangle|\ge z_0\}$. \item If the phase is not neutral, the goal direction is fixed throughout the phase, and every goal of the phase pays $g(y_{2j-1})=\frac12(1+c_{f,j})$ and $g(y_{2j})=\frac12(1-c_{f,j})$, with $c_f:=\mathrm{clip}_{[-1,1]}(\rho\hat c_f)$ coordinatewise. Before every episode of the phase, the adversary draws a fresh value $v$ uniformly from $[\frac12-D_0,\frac12+D_0]$, independently of everything else.
\item The at most $L$ episodes after phase $F$ have $g(y)=0$ at every successor and $v=0$. They cost nothing.
\end{itemize}
Write $Z_f:=\langle\hat c_f,\theta\rangle$, $\tau_f:=\mathrm{sign}\,Z_f$ and $R_f:=I-2\hat c_f\hat c_f^\top$. Let $\mathcal H_f$ be the agent's information at the start of phase $f$, including $\hat c_1,\dots,\hat c_f$, and $\pi_f$ the posterior density of $\theta$ given $\mathcal H_f$. The space $W_f$ and the direction $\hat c_f$ are functions of the first $N_f$ entries of the tape and of the adversary's randomness, and do not depend on the agent's actions. The nonzero directions are orthonormal, since $\hat c_f\perp\hat c_{f'}$ for $f'<f$.

\emph{Exact invariances.} Given $\theta$ and the history, the agent's actions, the number of free samples and the values $v$ have probabilities that do not depend on $\theta$. Each $W_{f'}$ is a function of the history, and $\hat c_{f'}$ has density $\mathds 1_{\mathrm{Cap}_{f'}(\theta)}/h_{f'}(\theta)$ with respect to $\sigma_{f'}$, where $h_{f'}(\theta):=\sigma_{f'}(\mathrm{Cap}_{f'}(\theta))$. Hence
\begin{equation*}
\pi_f(\theta)\;\propto\;\varphi(\theta)\,\Lambda_f(\theta)\textstyle\prod_{f'\le f}\mathds 1\{\|w_{f'}\|>z_0,\ |\langle\hat c_{f'},\theta\rangle|\ge z_0\}/h_{f'}(\theta),
\end{equation*}
where $\varphi$ is the Gaussian density and $\Lambda_f$ the likelihood of the first $N_f$ samples. Neutral phases contribute the indicators $\mathds 1\{\|w_{f'}\|\le z_0\}$ instead. Since $W_{f'}\subseteq W_f$, the direction $\hat c_f$ is orthogonal to every $r^{(f')}$ with $f'\le f$, to every $\hat c_{f'}$ with $f'<f$, and to $W_{f'}$. Consequently $R_f$ preserves $\varphi$, every $\langle\hat c_{f'},\theta\rangle$ with $f'<f$, every $\langle r^{(f')},\theta\rangle$, and every $\|P_{W_{f'}^\perp}\theta\|$. By rotation invariance of $\sigma_{f'}$, the normaliser $h_{f'}(\theta)$ depends only on $\|P_{W_{f'}^\perp}\theta\|$. Finally, $R_f$ maps $\{\tau_f=+1\}$ onto $\{\tau_f=-1\}$. Therefore
\begin{equation}
\pi_f(\theta)/\pi_f(R_f\theta)=\Lambda_f(\theta)/\Lambda_f(R_f\theta).
\label{eq:aa-ratio}
\end{equation}

\emph{The likelihood ratio.} Let sample $s$ fall in pair $j_s$ with sign $y_s=\pm1$. Let $n_j$ be the counts of the first $N_f$ samples, $\bar n:=N_f/m$ and $d_j:=\sum_{s\le N_f:\,j_s=j}y_s$. Assume that no coordinate of $\theta$ or of $\vartheta:=R_f\theta=\theta-2Z_f\hat c_f$ is clipped. Expand $\ln(1+x)$ to third order, with a remainder whose derivative is at most $2|x|^3$ for $|x|\le\frac12$.
\begin{itemize}
\item The linear term is $\varepsilon\langle d,\theta-\vartheta\rangle=2Z_f\langle\hat c_f,r^{(f)}\rangle=0$. \item The quadratic term is
\begin{equation*}
-\frac{\varepsilon^2}2\sum_jn_j(\theta_j^2-\vartheta_j^2)=-2\varepsilon^2Z_f\langle\hat c_f,\Delta(\theta-Z_f\hat c_f)\rangle,
\end{equation*}
with $\Delta:=\mathrm{diag}(n_j-\bar n)$, because $\langle\hat c_f,\theta-Z_f\hat c_f\rangle=0$.
\end{itemize}
Hence
\begin{multline}
\big|\ln\tfrac{\Lambda_f(\theta)}{\Lambda_f(\vartheta)}\big|\le2\varepsilon^2|Z_f|\big(|\langle\hat c_f,\Delta\theta\rangle|+|Z_f|\max_j|n_j-\bar n|\big)\\+\tfrac{\varepsilon^3}3\big|\textstyle\sum_jd_j(\theta_j^3-\vartheta_j^3)\big|+2\sum_{s\le N_f}\max(|x_s|,|x'_s|)^3|x_s-x'_s|,
\label{eq:aa-llr}
\end{multline}
where $x_s:=\varepsilon y_s\theta_{j_s}$ and $x'_s:=\varepsilon y_s\vartheta_{j_s}$. Since $FL\le K$ and $Fb\le K$ for $K\ge Cm^{3/2}\ln^2(mK)$, every phase start has $N_f\le N_F\le2K$, which yields $\varepsilon^2\bar n\le\varepsilon^2\cdot2K/m=\frac1{32}$.

\emph{Typical event.} Let $x_0:=\ln(mK)$, $I_k:=\{N_{k-1}+1,\dots,N_k\}$ with $N_0:=0$, and let $\mathcal T$ be the event on which the following hold for every phase $f$.
\begin{enumerate}
\item[(T1)] $0.95m\le\|\theta\|^2\le1.05m$ and $\|\theta\|_\infty\le\sqrt{6\ln m}$.
\item[(T2)] For every $j$, $|n_j(N_f)-N_f/m|\le3\sqrt{(N_f/m)x_0}+x_0$, $n_j(I_f)\le2|I_f|/m$, $\sum_jn_j(I_f)\theta_j^2\le1.1|I_f|$, and $|d_j(n)-n_j(n)\varepsilon\tilde\theta_j|\le3\sqrt{n_j(n)x_0}$ for all $n\le2K$.
\item[(T3)] If phase $f$ is not neutral, $|Z_f|\le z_0+8\ln m/z_0$.
\item[(T4)] If phase $f$ is not neutral, write $\hat c_f=(Z_f/\|w_f\|)\,\hat w_f+\sqrt{1-Z_f^2/\|w_f\|^2}\,u_f$, where $\hat w_f:=w_f/\|w_f\|$ and $u_f$ is a unit vector in $V_f:=W_f^\perp\cap w_f^\perp$. Then $\|u_f\|_\infty\le\sqrt{12\ln m/m}$ and $|\langle u_f,\Delta\theta\rangle|\le\|\Delta\theta\|\sqrt{12\ln m/m}$, with $\Delta$ the matrix of the likelihood ratio at phase $f$.
\item[(T5)] The bounds (B1)--(B3) of Lemma~\ref{lem:increments} hold for the increments $I_1,\dots,I_F$, the subspace $\mathcal S:=\mathrm{span}\{\mu_1,\dots,\mu_F,\theta\}$, the directions $\hat c_1,\dots,\hat c_F$ and $\delta=1/m$, and the bound of Lemma~\ref{lem:cap-data} holds with $\delta=1/m$.
\end{enumerate}
Then $\mathds{P}(\mathcal T^c)\le C/m$, as the following sum of failure probabilities shows. Here $F\le4z_0^2+1=\sqrt m+1$, because $L\ge K/(4z_0^2)-1$ and $K\ge2m$, and $\dim W_f\le2F$.
\begin{itemize}
\item (T1) fails with probability at most $2e^{-m/1700}+m^{-2}\le2/m$ for $m\ge C$ (Lemma~\ref{lem:gauss-norm}).
\item (T2) fails with probability at most $1/m$. Apply Lemma~\ref{lem:tape} with $x:=\ln(8m^2K)\le3x_0$, part~(i) to the $2F$ intervals $\{1,\dots,N_f\}$ and $I_f$, which are fixed in advance, with the $m+1$ vectors $e_j$ and $(\theta_j^2)_j$, and part~(ii) to the $m$ pairs. The total failure probability is at most $(4F(m+1)+4mK)e^{-x}\le1/m$. Part~(i) gives $\sqrt{6(N_f/m)x_0}+x_0$ for the first bound. For the second, $|I_f|/m\ge b/m\ge x_0^2$ gives $\sqrt{6(|I_f|/m)x_0}+x_0\le|I_f|/m$ once $x_0\ge9$. For the third, the mean is $|I_f|\,\|\theta\|^2/m\le1.05|I_f|$ on (T1), and the deviation $\sqrt{2|I_f|\cdot6.3\ln m\cdot3x_0}+6x_0\ln m$ is at most $0.05|I_f|$ for $m\ge C$, because $|I_f|\ge m x_0^2$. Part~(ii) gives $\sqrt{6n_jx_0}$ for the last bound.
\item (T3) fails with probability at most $(\sqrt m+1)m^{-6}$. Given $\theta$ and the history, $\hat c_f$ is uniform on the cap $\mathrm{Cap}_f(\theta)$ of the construction of the adversary, and $Z_f=\langle\hat c_f,w_f\rangle$. Lemma~\ref{lem:cap} applies with $d=\dim W_f^\perp\ge m-2F$ and $r=\|w_f\|\le\|\theta\|$. On (T1), $\mu\le1.05m/((m-2F-3)z_0)\le1.2/z_0$ for $m\ge C$. The overshoot $|Z_f|-z_0$ beyond the threshold of the cap therefore exceeds $8\ln m/z_0$ with probability at most $m^{-20/3}$.
\item (T4) fails with probability at most $m^{-2}$. By Lemma~\ref{lem:cap}, given $\theta$, the history and $Z_f$, the vector $u_f$ is uniform on the unit sphere of $V_f$, of dimension $d'\ge m-2F-1$. Lemma~\ref{lem:sphere} with $t=\sqrt{12\ln m/m}$ and $d'-3\ge0.9m$ bounds each of the $m+1$ fixed directions $e_j$ and $\Delta\theta$ by $2m^{-5.4}$, and a union bound over the $F$ phases gives the claim.
\item (T5) fails with probability at most $2/m$. On (T2), the counts satisfy the hypothesis of Lemma~\ref{lem:increments}, and $\mathcal S$ is fixed given $\theta$ and the labels. On (T1), $\theta$ satisfies that of Lemma~\ref{lem:cap-data}. The directions satisfy both, as noted above. Finally, $\eta=18\sqrt{(F\ln(12m)+\ln(8m))/m}\le Cm^{-1/4}\sqrt{\ln m}\le\tfrac14$ for $m\ge C$.
\end{itemize}

\emph{Energy captured by $W_f$.} We show that, on $\mathcal T$, $\|w_f\|^2\ge m/2$ for every $f$. In particular, $\|w_f\|>z_0$ for every $f$, and the adversary declares no phase neutral. Let $\mathcal C:=\mathrm{span}\{\hat c_{f'}\}_{f'<f}$ and $\mathcal D:=\mathrm{span}\{r^{(f')}\}_{f'\le f}$. Then $W_f=\mathcal C\oplus P_{\mathcal C^\perp}\mathcal D$, which yields $\|P_{W_f}\theta\|^2=\|P_{\mathcal C}\theta\|^2+\|P_{P_{\mathcal C^\perp}\mathcal D}\theta\|^2$.

\emph{Caps.} Since the nonzero caps are orthonormal, (T3) gives
\begin{align*}
\|P_{\mathcal C}\theta\|^2&=\sum_{f'<f}Z_{f'}^2\le F(z_0+8\ln m/z_0)^2\\
&\le(\sqrt m+1)(\sqrt m/4+16\ln m+256\ln^2m/\sqrt m)=\frac m4+o(m).
\end{align*}

\emph{Data.} The space $\mathcal D$ is spanned by the increments $\Delta r^{(k)}:=r^{(k)}-r^{(k-1)}$, $k\le f$, with $r^{(0)}:=0$. Since the increment $\Delta r^{(k)}$ counts the samples of $I_k$, $\Delta r^{(k)}=\mu_k+\xi_k$ in the notation of Lemma~\ref{lem:increments}, with $\mu_{k,j}=\varepsilon^2n_j(I_k)\theta_j$ on (T1), since no coordinate is clipped. Every increment has $|I_k|\ge b$. Fix $\gamma\in\mathds R^F$ with $\gamma_k=0$ for $k>f$ and $\|\Gamma\gamma\|=1$, and let $x(\gamma):=\sum_k\gamma_k\Delta r^{(k)}=\mu(\gamma)+\xi(\gamma)$, where $\mu(\gamma):=\sum_k\gamma_k\mu_k\in\mathcal S$.

\emph{Interaction with the caps.} For $k\le f'$, the cap $\hat c_{f'}$ is orthogonal to the increment $\Delta r^{(k)}$, since it is orthogonal to $r^{(k-1)}$ and to $r^{(k)}$ by the construction of $W_{f'}$. Hence the coordinates of $P_{\mathcal C}x(\gamma)$ in the basis of nonzero caps are $\sum_{k>f'}\gamma_k\langle\hat c_{f'},\mu_k+\xi_k\rangle$. By Lemma~\ref{lem:cap-data}, the part with $\mu_k$ has norm at most the Frobenius norm of that lemma, which, with $\varepsilon^2N_F\le m/32$ and $\sum_{f'}Z_{f'}^2\le m/4+o(m)$, is at most $\big(\tfrac43\cdot\tfrac1{32}\cdot\tfrac14\big)^{1/2}+o(1)\le0.11$ for $m\ge C$ and $K\ge Cm^2\ln^3(mK)$. The second term of the lemma is at most $Cm\ln^{3/2}(mK)/\sqrt K$. By (B3) with $\beta$ the unit vector along these coordinates, the part with $\xi_k$ has norm at most $\eta$. Hence $\|P_{\mathcal C}x(\gamma)\|\le0.11+\eta$.

\emph{Lower bound on the norm.} Since $\mu(\gamma)\in\mathcal S$, (B1) and (B2) give
\begin{equation*}
\|x(\gamma)\|\ge\|P_{\mathcal S^\perp}x(\gamma)\|=\|P_{\mathcal S^\perp}\xi(\gamma)\|\ge1-2\eta.
\end{equation*}
Hence $\|P_{\mathcal C^\perp}x(\gamma)\|\ge0.89-3\eta$.

\emph{Upper bound on the projection.} The numerator satisfies
\begin{equation*}
\langle P_{\mathcal C^\perp}x(\gamma),\theta\rangle=\langle\mu(\gamma),\theta\rangle+\langle\xi(\gamma),\theta\rangle-\langle P_{\mathcal C}x(\gamma),P_{\mathcal C}\theta\rangle.
\end{equation*}
\begin{itemize}
\item By (T2), $\langle\mu_k,\theta\rangle=\varepsilon^2\sum_jn_j(I_k)\theta_j^2\le1.1\varepsilon^2|I_k|$. By Cauchy--Schwarz and $e_k^2\ge\frac34\varepsilon^2|I_k|$,
\begin{equation*}
|\langle\mu(\gamma),\theta\rangle|\le1.1\varepsilon^2(\sum_k|I_k|)^{1/2}(\sum_k\gamma_k^2|I_k|)^{1/2}\le1.1\cdot\frac2{\sqrt3}\,\varepsilon\sqrt{N_F}\le0.23\sqrt m,
\end{equation*}
using $\varepsilon^2N_F\le m/32$.
\item By (B2), $|\langle\xi(\gamma),\theta\rangle|\le\eta\|\theta\|\le1.03\eta\sqrt m$. \item The last term is at most
\begin{equation*}
(0.11+\eta)\big(\frac m4+o(m)\big)^{1/2}\le(0.056+\eta)\sqrt m
\end{equation*}
for $m\ge C$.
\end{itemize}
Every unit vector of $P_{\mathcal C^\perp}\mathcal D$ has the form $P_{\mathcal C^\perp}x(\gamma)/\|P_{\mathcal C^\perp}x(\gamma)\|$ for such a $\gamma$, up to sign. Its correlation with $\theta$ is therefore at most $(0.29+3\eta)\sqrt m/(0.89-3\eta)\le0.34\sqrt m$ for $m\ge C$. Hence $\|P_{P_{\mathcal C^\perp}\mathcal D}\theta\|^2\le0.12m$, $\|P_{W_f}\theta\|^2\le0.37m+o(m)$, and $\|w_f\|^2\ge0.58m-o(m)\ge m/2$ by (T1).

\emph{Consequences on $\mathcal T$.} No phase is neutral, $|Z_f|\le2z_0$ and $|Z_f|/\|w_f\|\le3z_0/\sqrt m$.
\begin{itemize}
\item \emph{No clipping.} By (T4), $|\hat c_{f,j}|\le4z_0|w_{f,j}|/m+\sqrt{12\ln m/m}$, which yields
\begin{equation*}
|\vartheta_j-\theta_j|=2|Z_f||\hat c_{f,j}|\le4|w_{f,j}|/\sqrt m+7\sqrt{\ln m}\,m^{-1/4}\le5,
\end{equation*}
using $|w_{f,j}|\le\|\theta\|$. Hence every coordinate of $\theta$ and $\vartheta$ has square at most $(\sqrt{6\ln m}+5)^2\le12\ln m<B^2$ for $m\ge C$, and none is clipped.
\item \emph{Quadratic term.} By (T4),
\begin{equation*}
|\langle\hat c_f,\Delta\theta\rangle|\le\big(3z_0/\sqrt m+\sqrt{12\ln m/m}\big)\|\Delta\theta\|\le Cz_0\max_j|n_j-\bar n|,
\end{equation*}
 using $\|\Delta\theta\|\le\max_j|n_j-\bar n|\,\|\theta\|$. With (T2), $\varepsilon^2\bar n\le\frac1{32}$ and $z_0^2=\sqrt m/4$, the quadratic term of Eq.~\eqref{eq:aa-llr} is at most $C\varepsilon^2z_0^2\sqrt{\bar nx_0}\le C\varepsilon\sqrt m\,\sqrt{x_0}$. \item \emph{Cubic term.} By (T2), $|d_j|\le n_j\varepsilon\sqrt{6\ln m}+3\sqrt{n_jx_0}\le C\sqrt{\bar nx_0}$, since $n_j\le2\bar n$ for $\bar n\ge x_0^2$ and $x_0\ge11$, and $\varepsilon\sqrt{\bar n}\le1$. With $|\theta_j^3-\vartheta_j^3|\le36\ln m\cdot2|Z_f||\hat c_{f,j}|$ and $\|\hat c_f\|_1\le\sqrt m$, the cubic term is at most $C\varepsilon^3z_0\sqrt m\sqrt{\bar n}\,x_0^{3/2}\le C\varepsilon^2z_0\sqrt m\,x_0^{3/2}$. The remainder is at most $C\varepsilon^4\ln^{3/2}m\cdot z_0\bar n\sqrt m$, which is of the same order.
\item \emph{Total.} Since $\varepsilon=\sqrt{m/(64K)}$, the log-ratio in Eq.~\eqref{eq:aa-ratio} is at most
\begin{equation*}
\delta_K:=C'm\ln^{3/2}(mK)/\sqrt K\;\le\;\tfrac1{100}\qquad\text{for }K\ge Cm^2\ln^3(mK).
\end{equation*}

\item \emph{Clipping of the goal.} Let $Q$ be the set of clipped coordinates of $\rho\hat c_f$. Since $\rho\sqrt{12\ln m/m}=1/\sqrt3<0.6$, a coordinate is clipped only if $\rho|Z_f||\hat w_{f,j}|/\|w_f\|>0.4$. Since $\hat w_f$ is a unit vector,
\begin{equation*}
|Q|\le(3\rho z_0/\sqrt m)^2/0.16\le57\rho^2z_0^2/m=57\sqrt m/(144\ln m).
\end{equation*}
 Then
\begin{equation*}
\sum_{j\in Q}\hat c_{f,j}^2\le2(Z_f/\|w_f\|)^2+2|Q|\cdot12\ln m/m\le15/\sqrt m
\end{equation*}
 and $\sum_{j\in Q}\vartheta_j'^2\le12|Q|\ln m\le4.8\sqrt m$ for $\vartheta'\in\{\theta,\vartheta\}$. Since $|c_{f,j}-\rho\hat c_{f,j}|\le\rho|\hat c_{f,j}|$, Cauchy--Schwarz gives $|\langle c_f-\rho\hat c_f,\vartheta'\rangle|\le\rho\sqrt{72}\le\rho z_0/4$ for $m\ge C$. Since $\langle\hat c_f,\theta\rangle=Z_f=-\langle\hat c_f,\vartheta\rangle$ and $|Z_f|\ge z_0$, action $a$ is worth $\frac12+\tau D$ with $D\ge3\varepsilon\rho z_0/(8m)=D_0$ under both $\theta$ and $\vartheta$, with opposite signs $\tau$.
\end{itemize}

\emph{Near-symmetric posterior.} Whether phase $f$ is neutral is revealed by its goals and is therefore determined by $\mathcal H_f$. On a neutral phase, $\hat c_f=0$, the sign $\tau_f$ is not defined, and we bound the loss of its episodes from below by $0$. Since neutral phases occur only off $\mathcal T$, the histories with a neutral phase have probability at most $C/m$. In the rest of the argument for phase $f$, we fix a history $\mathcal H_f$ on which phase $f$ is not neutral. Let $G_f$ be the set of $\theta$ for which the consequences above hold for the pair $(\theta,R_f\theta)$. It is determined by $\mathcal H_f$, and $R_fG_f=G_f$. On $\mathcal T$, $\theta\in G_f$. By the tower property, therefore, $\mathds{E}[\pi_f(G_f^c)]=\mathds{P}(\theta\notin G_f)\le C/m$. On $G_f$, Eq.~\eqref{eq:aa-ratio} gives $\pi_f(R_f\theta)\le e^{\delta_K}\pi_f(\theta)$. Write $q_f:=\mathds{P}(\tau_f=+1\mid\mathcal H_f)$ and $a_f:=\pi_f(G_f^c)$. The bounds below that involve $\min(q_f,1-q_f)$ hold trivially if $q_f\in\{0,1\}$. We therefore assume $q_f\in(0,1)$. Integrating over $\{\tau_f=-1\}$, whose image under $R_f$ is $\{\tau_f=+1\}$, gives $q_f\le e^{\delta_K}(1-q_f)+a_f$, and symmetrically $1-q_f\le e^{\delta_K}q_f+a_f$. Hence $q_f\le(e^{\delta_K}+a_f)/(1+e^{\delta_K})$, which yields $2q_f-1\le(e^{\delta_K}-1)/2+a_f\le\delta_K+a_f$, and symmetrically for $1-q_f$. Hence $|2q_f-1|\le\delta_K+a_f$ and $\min(q_f,1-q_f)\ge\frac12-(\delta_K+a_f)/2$.

\emph{One phase.} Fix an episode of phase $f$, let $q$ be the posterior probability of $\tau_f=+1$ when the agent acts, and let $a_t$ be the posterior mass of $G_f^c$ at that time. The value of this episode is $v=\frac12+u$, with $u$ uniform on $[-D_0,D_0]$. It is drawn afresh before the episode, independently of $\theta$ and of everything observed before, and is thus independent of $q$ and $a_t$. On $G_f$, playing $b$ loses at least $D_0-u$ if $\tau_f=+1$, and playing $a$ loses at least $D_0+u$ if $\tau_f=-1$. Given the agent's information, the expected loss is therefore at least $\min\{q_+(D_0-u),q_-(D_0+u)\}$, where $q_\pm$ are the posterior masses of $\{\tau_f=\pm1\}\cap G_f$. Averaging over $u$ gives at least
\begin{equation*}
D_0q_+q_-/(q_++q_-)\ge D_0q_+q_-\ge D_0\big(q(1-q)-a_t\big)\ge D_0\big(\frac12\min(q,1-q)-a_t\big),
\end{equation*}
and $\mathds{E}a_t\le C/m$ by the tower property.

Let $\mu_\pm$ be $\pi_f$ restricted to $\{\tau_f=\pm1\}$ and normalised, and let $P^\pm$ be the corresponding laws of the observations of the phase before the episode. Then
\[
\mathds{E}[\min(q,1-q)\mid\mathcal H_f]\ge\min(q_f,1-q_f)\,\big(1-\mathrm{TV}(P^+,P^-)\big).
\]
Assume $q_f\ge\frac12$. Otherwise exchange the signs. Let $R_\#P^+$ be the law of the observations when $\theta$ is drawn from $R_\#\mu_+$, the image of $\mu_+$ under $R:=R_f$. Then
\begin{equation*}
\mathrm{TV}(P^+,P^-)\le\mathrm{TV}(P^+,R_\#P^+)+\mathrm{TV}(R_\#\mu_+,\mu_-).
\end{equation*}
\begin{itemize}
\item \emph{The coupling.} Couple $\theta\sim\mu_+$ with $R\theta$. The pair of a sample is uniform, $\varepsilon B=\frac12$, and clipping is $1$-Lipschitz. Hence, one sample has KL divergence at most
\begin{equation*}
\frac{(\varepsilon/2)^2}{3/16}\cdot\frac1m\|\theta-R\theta\|^2=\frac{16}3\varepsilon^2Z(\theta)^2/m,
\end{equation*}
 with $Z(\theta):=\langle\hat c_f,\theta\rangle$. Since the agent observes at most $L$ samples in the phase, the chain rule and convexity of KL give
\begin{equation*}
\mathrm{KL}(P^+\|R_\#P^+)\le\frac{16}3L\varepsilon^2\mathds{E}_{\mu_+}[Z^2]/m\le\frac{32}3L\varepsilon^2\mathds{E}_{\pi_f}[Z^2]/m,
\end{equation*}
 since $\mu_+\le2\pi_f$. By the tower property, $\mathds{E}[\mathds{E}_{\pi_f}Z^2]=\mathds{E}Z_f^2$. Lemma~\ref{lem:cap} with $\mu\le1.2/z_0$ gives $\mathds{E}[Z_f^2\mid\text{past}]\le z_0^2+3$ when $\|\theta\|^2\le1.05m$, and $Z_f^2\le\|\theta\|^2$ always, which yields $\mathds{E}Z_f^2\le z_0^2+3+(\mathds{E}\|\theta\|^4\cdot e^{-m/1700})^{1/2}\le z_0^2+4$ for $m\ge C$. With $L\le K/(4z_0^2)$ and $\varepsilon^2=m/(64K)$, $\mathds{E}\,\mathrm{KL}\le(z_0^2+4)/(24z_0^2)\le\frac1{12}$ for $z_0^2\ge4$. Pinsker's and Jensen's inequalities give $\mathds{E}\,\mathrm{TV}(P^+,R_\#P^+)\le(1/24)^{1/2}\le0.21$.
\item \emph{The symmetry defect.} Let $A:=\{\tau_f=-1\}$. The densities of $R_\#\mu_+$ and $\mu_-$ on $A$ are $\pi_f(R\theta)/q_f$ and $\pi_f(\theta)/(1-q_f)$, and $\int_A\pi_f(R\theta)\,d\theta=q_f$. On $A\cap G_f$, $|\pi_f(R\theta)-\pi_f(\theta)|\le2\delta_K\pi_f(\theta)$, and $A\setminus G_f$ carries mass at most $a_f$ under both $\pi_f$ and $\pi_f\circ R$. With $q_f\ge\frac12$ and the identity
\begin{equation*}
\big|1/q_f-1/(1-q_f)\big|\,(1-q_f)=|2q_f-1|/q_f,
\end{equation*}
we attain
\begin{equation*}
\mathrm{TV}(R_\#\mu_+,\mu_-)\le\frac12\big((2\delta_K+2a_f)+(2\delta_K+a_f)\big)/q_f\le4\delta_K+3a_f.
\end{equation*}
\end{itemize}

\emph{Summing.} With $\min(q_f,1-q_f)\le\frac12$, the bounds above give
\begin{equation*}
\mathds{E}\min(q,1-q)\ge\frac12-\frac{\delta_K}2-\frac Cm-\frac12\big(0.21+4\delta_K+\frac Cm\big),
\end{equation*}
which is at least $0.39-3\delta_K-C/m\ge0.35$ for $m\ge C$. The term $C/m$ also absorbs the histories with a neutral phase, on which we bound $\min(q,1-q)$ from below by $0$. Each episode of a phase therefore costs at least $D_0(0.175-C/m)\ge D_0/6$ in expectation, and the phases contain $FL\ge K-L\ge K/2$ episodes. Hence
\[
\mathds{E}\mathfrak R^{\mathrm{goal}}_K\;\ge\;\frac{D_0K}{12}\;=\;\frac{m^{1/4}\sqrt K}{3072\sqrt{\ln m}},\qquad\text{since }D_0=\frac{3\varepsilon\rho z_0}{8m}=\frac{m^{1/4}}{256\sqrt{K\ln m}}.
\]
Summed over the $M$ blocks, this gives $m^{1/4}\sqrt{MK/\ln m}/3072$. Since $m=Y/2$, this is at least $Y^{1/4}\sqrt{MK/\ln Y}/3700$, which proves part~(a) with $c=1/3700$.
\end{proof}

\subsection{Agents That Use Past Goals as Data}
\label{app:adaptive-reading}

\begin{proposition}[Using past goals as data lowers the price]
\label{prop:goal-reading}
On the MDPs of Theorem~\ref{thm:price}(b), draw the law at each start state independently from the uniform prior over the family $P_\sigma$ of the proof of Lemma~\ref{lem:goal-uniform-width}, with $m=\lfloor Y/2\rfloor$ and $\varepsilon:=\sqrt{3mM/(64K)}$ in place of $\sqrt{3m/(64n)}$, and let $Y\ge4$. Some agent that uses past goals as data has $\mathds{E}\mathfrak R^{\mathrm{goal}}_K\le2Y^{1/4}\sqrt{MK\ln Y}+1$ against every adversary that chooses goals and start states from the transition law and the history. If $Y\ge600\ln(3K)$ and $K\ge MY$, every goal-memoryless agent has $\mathds{E}\mathfrak R^{\mathrm{goal}}_K\ge\sqrt{MYK}/6000$ against some such adversary.
\end{proposition}
The agent is Thompson sampling on the posterior that also conditions on the goals, and the minimax theorem turns it into one agent for all adversaries, which Theorem~\ref{thm:price}(c) does not provide. With the constants above, which we have not optimised, the separation takes effect only for very large $Y$. On this hypercube family, agents that use past goals as data thus incur at most a factor $Y^{1/4}$, and whether they incur any factor is open.

\begin{proof}[Proof of Proposition~\ref{prop:goal-reading}]
The lower bound is the argument of Theorem~\ref{thm:price}(b) on the hypercube family (Appendix~\ref{app:adaptive-b}). For the upper bound, write $Y$ for the number of successors, $m=\lfloor Y/2\rfloor\ge2$ for the number of pairs, $K'=K/M$, and $\varepsilon=\sqrt{3m/(64K')}$. Start state $i$ has the law $P_i=P_{\sigma^{(i)}}$ of the family of the proof of Lemma~\ref{lem:goal-uniform-width}, and the signs $\sigma^{(1)},\dots,\sigma^{(M)}\in\{-1,1\}^m$ are independent and uniform. A goal that pays $g\in[0,1]^Y$ at the successors gives action $a$ at start state $i$ the value
\begin{equation*}
P_ig\;=\;\alpha(g)+\tfrac{\varepsilon}{2m}\langle c(g),\sigma^{(i)}\rangle,\qquad c_j(g):=g(y_{2j-1})-g(y_{2j})\in[-1,1],
\end{equation*}
with $\alpha(g)$ known. Rewards at the start states shift $v$ by a known amount, and Eq.~\eqref{eq:ts-ratio} holds for every real $v$. Action $b$ has the value $v$, and the regret of an episode depends only on the action at the start state.

\emph{A Thompson-sampling agent for each adversary.} Fix an adversary, possibly randomised, and let $\mathcal F_k$ collect everything the agent has seen before acting in episode $k$, the current goal and start state included. The agent knows the adversary's strategy and can therefore compute the posterior of $\sigma$ given $\mathcal F_k$. It draws $\tilde\sigma$ from this posterior and plays $a$ if and only if $P_{\tilde\sigma^{(i)}}g_k\ge v_k$. Write $W:=P_ig_k$, and let $\mu_k$, $s_k^2$ be its posterior mean and variance and $p_k:=\mathds{P}(W\ge v_k\mid\mathcal F_k)$, the probability that the agent plays $a$. The posterior expected regret of episode $k$ is
\begin{equation*}
r_k=(1-p_k)\mathds{E}[(W-v_k)^+\mid\mathcal F_k]+p_k\mathds{E}[(v_k-W)^+\mid\mathcal F_k],
\end{equation*}
and we claim
\begin{equation}
r_k\;\le\;4\sqrt{p_k}\,s_k.
\label{eq:ts-ratio}
\end{equation}
The claim concerns a single episode, and we drop the index $k$ in its proof. Cantelli's inequality states that $\mathds{P}(W-\mu\ge t)\le s^2/(s^2+t^2)$ and $\mathds{P}(\mu-W\ge t)\le s^2/(s^2+t^2)$ for $t\ge0$. If $v\ge\mu$, then
\begin{equation*}
(1-p)\mathds{E}(W-v)^+\le\mathds{E}[(W-\mu)\mathds 1\{W\ge v\}]\le s\sqrt p
\end{equation*}
 by Cauchy--Schwarz. Cantelli's inequality with $t=v-\mu$ gives $p\le s^2/(s^2+(v-\mu)^2)\le s^2/(v-\mu)^2$, hence $p(v-\mu)\le\sqrt p\,s$. With $\mathds{E}(v-W)^+\le(v-\mu)+\mathds{E}(\mu-W)^+\le(v-\mu)+s$, we attain $p\,\mathds{E}(v-W)^+\le p(v-\mu)+p\,s\le2\sqrt p\,s$. If $v<\mu$ and $\mu-v\ge s$, using Cantelli's inequality once again, we attain $1-p\le s^2/(s^2+(\mu-v)^2)\le\frac12$, which yields
\begin{equation*}
(1-p)\mathds{E}(W-v)^+\le(1-p)(\mu-v)+s\le\frac{s^2(\mu-v)}{s^2+(\mu-v)^2}+s\le\frac32s.
\end{equation*}
Moreover, $p\,\mathds{E}(v-W)^+\le\mathds{E}(\mu-W)^+\le s$. The sum is at most $\frac52s\le4\sqrt p\,s$, since $p\ge\frac12$. If $v<\mu$ and $\mu-v<s$, then, by Cauchy--Schwarz,
\begin{equation*}
(1-p)\mathds{E}(W-v)^+\le\mathds{E}[(W-v)\mathds 1\{W\ge v\}]\le\sqrt p\sqrt{s^2+(\mu-v)^2}\le\sqrt{2p}\,s
\end{equation*}
and $p\,\mathds{E}(v-W)^+\le p\,s\le\sqrt p\,s$.

\emph{Variance against correlations.} Let $C$ be the posterior covariance of $\sigma^{(i)}$ given $\mathcal F_k$, fix $\eta>0$, and let $\Xi:=\sum_{j,l}C_{jl}^2/(C_{ll}+\eta)$. Since $|c_j|\le1$, the Cauchy--Schwarz inequality gives
\begin{equation*}
c^\top Cc\;\le\;\sum_{j,l}|C_{jl}|\;\le\;\sqrt{\Xi}\,\sqrt{m\textstyle\sum_l(C_{ll}+\eta)}\;\le\;\sqrt{m\Xi(\Xi+2m\eta)}\;\le\;\sqrt m\,(\Xi+m\eta),
\end{equation*}
where the third step uses $C_{ll}\le C_{ll}^2/(C_{ll}+\eta)+\eta$. Hence $s_k^2=(\varepsilon/(2m))^2c^\top Cc\le\varepsilon^2(\Xi+m\eta)/(4m^{3/2})$.

\emph{A potential.} Let $\Lambda:=\sum_{i\le M}\sum_{j\le m}\ln\big(\mathds{V}[\sigma^{(i)}_j\mid\cdot\,]+\eta\big)$ be evaluated at the current posterior. The variance of a coordinate is concave in the law, and $\ln$ is concave and increasing. Hence, $\Lambda$ is concave in the posterior and does not increase in expectation under any further conditioning, such as a goal. If the agent plays $a$ at start state $i$, it observes one successor $\omega$, which falls in pair $j$ with probability $1/m$, whatever $\sigma$, and is its first element with probability $(1+\varepsilon\sigma^{(i)}_j)/2$. A direct computation gives, for every $l$,
\begin{equation*}
\mathds{V}\big(\mathds{E}[\sigma^{(i)}_l\mid\mathcal F_k,\omega]\big)=\varepsilon^2C_{jl}^2/(1-\varepsilon^2\mu_j^2)\ge\varepsilon^2C_{jl}^2
\end{equation*}
 given that $\omega$ falls in pair $j$, with the variance taken over $\omega$ given $\mathcal F_k$ and $\mu_j$ the posterior mean of $\sigma^{(i)}_j$. By the law of total variance, Jensen's inequality and $\ln(x-d)\le\ln x-d/x$, the term $l$ of $\Lambda$ decreases in expectation by at least $\varepsilon^2C_{jl}^2/(C_{ll}+\eta)$. Averaging over $j$, $\mathds{E}[\Lambda_k-\Lambda_{k+1}\mid\mathcal F_k]\ge p_k\varepsilon^2\Xi/m$. Since $\Lambda\le Mm\ln(1+\eta)$ and $\Lambda\ge Mm\ln\eta$, the expected decreases $\mathds{E}[\Lambda_k-\Lambda_{k+1}]$ of the potential over all episodes sum to at most $Mm\ln(1+1/\eta)$.

\emph{Summing.} By Eq.~\eqref{eq:ts-ratio} and the two bounds above,
\begin{equation*}
\mathds{E}r_k^2\le16\,\mathds{E}[p_ks_k^2]\le(4/\sqrt m)\,\mathds{E}[p_k\varepsilon^2\Xi/m]+4\varepsilon^2\eta/\sqrt m.
\end{equation*}
Summing over $k$, with $\eta=1/m$ and $K\varepsilon^2=3mM/64$,
\begin{equation*}
\textstyle\sum_k\mathds{E}r_k^2\;\le\;4M\sqrt m\ln(m+1)+\frac{3M}{16\sqrt m}\;\le\;5M\sqrt m\ln(m+1),
\end{equation*}
and the Cauchy--Schwarz inequality gives
\begin{equation*}
\mathds{E}\mathfrak R^{\mathrm{goal}}_K=\sum_k\mathds{E}r_k\le\sqrt{5Mm^{1/2}K\ln(m+1)}\le2Y^{1/4}\sqrt{MK\ln Y}.
\end{equation*}

\emph{One agent for all adversaries.} The agent above depends on the adversary, and we now remove this dependence with the minimax theorem. Let every agent round each revealed reward down to the grid $\{j/K^2:0\le j\le K^2\}$ and act on the rounded goals. Since $H=2$, rounding changes the value of each action by at most $1/K^2$, and hence the regret of every agent against every adversary by at most $2/K$ in total. Consider the game in which the adversary chooses rounded goals and start states, the agent chooses actions, chance draws $\sigma$ from the prior and the successors, and the agent pays its expected regret against the rounded goals. Each player chooses among finitely many options at each of finitely many decision points. Both players thus have finitely many pure strategies, and the expected regret is bilinear in their mixed strategies. Von Neumann's minimax theorem therefore gives
\begin{equation*}
\min_{\text{agent}}\max_{\text{adversary}}\mathds{E}\mathfrak R^{\mathrm{goal}}_K\;=\;\max_{\text{adversary}}\min_{\text{agent}}\mathds{E}\mathfrak R^{\mathrm{goal}}_K
\end{equation*}
over mixed strategies. A mixed adversary draws a pure strategy at the start and is a randomised adversary, against which the Thompson-sampling agent that knows it attains the bound above. The right-hand side is therefore at most this bound, and a minimising mixed agent on the left attains it against every pure adversary, and hence against every mixture of them. Every adversary of the original interaction, composed with the rounding, is such a mixture, with its own randomness and its unrounded choices as private randomness. This agent therefore has regret at most the bound above plus $2/K\le1$ against every adversary.
\end{proof}

\begin{proof}[Proof of the second claim of Theorem~\ref{thm:price}(c)]
The proof is that of Proposition~\ref{prop:goal-reading} with $\sigma^{(i)}$ replaced by $\tilde\theta^{(i)}\in[-B,B]^m$, for an agent that knows the adversary. Neither the variance identity for a sample, nor the inequality $c^\top Cc\le\sqrt m(\Xi+m\eta)$, nor the Thompson-sampling bound depends on the prior. The potential now ranges over at most $Mm\ln(1+B^2/\eta)=Mm\ln(1+16K')$ for $\eta=1/m$, since $mB^2=16K'$. Hence $\sum_k\mathds{E}r_k^2\le4M\sqrt m\ln(1+16K')+M$, and the Cauchy--Schwarz inequality gives the bound.
\end{proof}

\subsection{Goals Fixed in Advance}
\label{app:adaptive-oblivious}

\emph{Theorem~\ref{thm:price}(d).} A member of the Gaussian family is an MDP with an exact abstraction, the identity, with $X=S$, $A=2$ and $H=2$, so that $T=2K$. Proposition~\ref{prop:oblivious} below therefore bounds the regret of UCBVI-CH against an oblivious adversary by $\mathcal O(L\sqrt{SK})+\mathcal O(S^2L\ln T)$, which is $\widetilde{\mathcal O}(\sqrt{SK})$ for $K\ge S^3$. A fixed goal is the special case in which the adversary fixes the same goal in every episode.

\emph{Scope of Theorem~\ref{thm:price}.} The scope is narrower than the rates suggest. The horizon is $2$, and the bounds are expectations over the prior. The price over a fixed goal needs the adversary to choose the start states, and the agent of part~(c) that uses past goals knows the adversary's strategy. Since the constants $c$ and $C$ are far from optimised, part~(a) is an asymptotic statement.

\begin{proposition}[Goals fixed in advance cost no extra regret]
\label{prop:oblivious}
Let $\phi$ be exact, and let an oblivious adversary fix the goals and start states before the first episode, possibly as functions of $\overline P_\phi$. With probability at least $1-\delta$, UCBVI-CH run with the goal $R_k$ in episode $k$ and without its clamp at the previous values has
\begin{equation*}
\mathfrak R^{\mathrm{goal}}_K\;\le\;\mathcal O(HL\sqrt{XAT})+\mathcal O(H^2X^2AL(H+\ln T)),\qquad L=\mathcal O(\ln(XAHT/\delta)).
\end{equation*}
\end{proposition}

\begin{proof}[Proof of Proposition~\ref{prop:oblivious}]
Write $P:=\overline P_\phi$. The proof is that of Theorem~\ref{thm:ucbvi-fixed}(a) with $\epsilon_{\text{rep}}(\phi)=0$ and the known goal $R_k$ in place of $\widehat R^k_\phi$, with two changes. First, the fixed functions of Lemma~\ref{lem:row-conc}(i) become $V^*_{h+1}(\cdot\,;R_k)$ for all $h\le H$ and $k\le K$. Since the goals and start states are fixed before the first episode, given $P$ these functions are independent of the observed transitions, and Lemma~\ref{lem:row-conc}(i) applies to each of them. The union bound now runs over $KH$ functions per row instead of $H$. This multiplies the number of events by $K$ and leaves $\iota$ of the same order. Second, without the clamp at the previous values, optimism follows from the induction over $h$ alone, because the remaining clamp at $H$ preserves $\widehat Q_{k,h}\ge Q^*_h(\cdot\,;R_k)$. The second-order bound of Eq.~\eqref{eq:second-order} holds for all coordinates at once and needs no change, and the regret recursion uses only that the clamps lower $\widehat Q_{k,h}$. Summing as in Appendix~\ref{app:thm-fixed} gives the claim. Goals that depend on the observed transitions, as in Theorem~\ref{thm:price}(b), make $V^*_{h+1}(\cdot\,;R_k)$ depend on the sampling error, and Lemma~\ref{lem:row-conc}(i) no longer applies.
\end{proof}

\section{Proofs for Section~\ref{sec:learning-phi}}
\label{app:proofs-learned}

\subsection{AIF's Selection by Observations}
\label{app:aif-selection}
\label{sec:jepa}

This subsection states Propositions~\ref{prop:lure} and~\ref{prop:self-confirming} and describes their instances. Their proofs are in Appendix~\ref{app:enactive}.

\emph{Predicting observations.} AIF predicts the next ground state by \emph{posterior averaging}, as it averages over policies by Bayesian model averaging \citep{friston2017active}. Each candidate composes a latent transition with an emission, and AIF weights the candidates' Bayesian predictions by their posterior probabilities (Appendix~\ref{app:models}). All candidates predict the same observations, which makes this score consistent. On every trajectory, posterior averaging predicts the next ground state as well as the Bayesian prediction of $\phi^\dagger$, up to $\ln|\Phi|$ nats of log-loss \citep[\S9.2]{cesabianchi2006prediction} (Proposition~\ref{prop:selection}). The price is that AIF's information gain about the emission directs its exploration.

\begin{proposition}[Curiosity about the emission]
\label{prop:lure}
Let an AIF variant learn its emission matrix under uniform Dirichlet priors, one row per step, as in standard AIF \citep{friston2017active,schwartenbeck2019computational}, so that its parameter information gain contains novelty about the emission. For every $S\ge12$, $\Delta\in(0,\frac12]$, and $\beta<0.09/\Delta$ there is a block MDP with latent rewards, $S$ states, $A=2$, and $H=2$, on which AIF, greedy or softmax, has $\mathds{E}\mathfrak R_K\ge(\Delta/2)\min\{K,\lfloor(S-2)/2\rfloor\}$, when it is given the decoder onto $X=3$ latent states, the latent kernel, and the rewards. This regret is linear in $K$ whenever $S\ge2K+2$. On the same MDP, the enactive agent (GOLF on the class $\mathcal F$ of Section~\ref{sec:learning-phi}) has regret $\widetilde{\mathcal O}(\sqrt T)$ for every $S$. For CaI and calibrated AIF, which have no novelty about the emission, the law of the sequence of actions, and hence $\mathds{E}\mathfrak R_K$, is the same for every $S$.
\end{proposition}

\emph{Predicting the abstraction's own state.} Proposition~\ref{prop:lure} suggests removing the emission model. Joint-embedding predictive architectures \citep{lecun2022path,assran2023self} and self-predictive representations in RL \citep{gelada2019deepmdp,schwarzer2021data,ni2024bridging} predict the next abstract state rather than the next observation. The score $\Lambda^{\mathrm{lat}}_\phi$ of Section~\ref{sec:learning-phi} is AIF's model of Appendix~\ref{app:models} without its emission factor. It has no emission, and hence no information gain about an emission. Removing the emission, however, also changes the target. Each candidate abstraction is now scored by how well it predicts its own labels, a target that differs from one candidate abstraction to the next. We now show that this self-referential target can confirm a coarse abstraction, one that merges two states whose best actions differ, whatever rule the agent uses to explore.

The decoder-free counterpart of AIF scores a candidate abstraction $\phi$ by the Bayesian likelihood $\Lambda^{\mathrm{lat}}_\phi$ of its own labels $\phi(s^j_{h+1})$ given $(\phi(s^j_h),a^j_h)$ in all past transitions, under uniform Dirichlet priors on rows over a common label set $\mathcal X$, one row per step. We call an agent \emph{latent-predictive} if, before each episode $k$, it selects a candidate abstraction $\phi_k$ that maximises $\Lambda^{\mathrm{lat}}_\phi$ and acts through it. Its action at each step depends on the current state only through $\phi_k$, and on the history arbitrarily. It may explore by any rule, calibrated novelty included.

The hard instance has four states. From the start state, both actions lead to $u_A$ or $u_B$ with equal probability. At $u_A$ action $a$, and at $u_B$ action $b$, reaches an absorbing rewarding state $g$ with a small probability. Every other move returns to $u_A$ or $u_B$ at random. The true abstraction $\phi^\dagger$ keeps $u_A$ and $u_B$ apart, and the coarse candidate abstraction $\psi$ merges them.

\begin{proposition}[Latent prediction confirms a coarse abstraction]
\label{prop:self-confirming}
For every $H\ge4$, $K\ge1$ and $\delta\in(0,1)$ there are a realizable block MDP with $S=X=4$, $A=2$ and latent rewards, and a candidate set $\Phi=\{\phi^\dagger,\psi\}$, with the following properties.
\begin{itemize}
\item[(i)] Whatever the policy, with probability at least $1-\delta$, before every episode $k\le K$, $\ln\Lambda^{\mathrm{lat}}_\psi-\ln\Lambda^{\mathrm{lat}}_{\phi^\dagger} \ge 0.44(k-1)-58H\ln(8HK/\delta)$.
\item[(ii)] Under every policy that acts through $\psi$, the labels and actions $(\psi(s_h),a_h)_{h\le H}$ follow an episodic MDP on the labels of $\psi$.
\item[(iii)] With probability at least $1-\delta$, every latent-predictive agent has $\mathfrak R_K\ge H(K-k_0)/128$, with $k_0:=1+132H\ln(8HK/\delta)$. This holds even if the agent also predicts the rewards, which both candidate abstractions predict exactly. GOLF on $\Phi$ has $\mathfrak R_K=\widetilde{\mathcal O}(H^2\sqrt T)$.
\end{itemize}
\end{proposition}

In this instance, the candidate $\psi$ merges $u_A$ and $u_B$, whose best actions differ, as the abstraction of Proposition~\ref{prop:lb-abstraction} does. Under $\psi$, the merged state leads to itself except for rare entries into $g$, and the labels of $\psi$ are nearly deterministic. Under $\phi^\dagger$, the label of every transition into $\{u_A,u_B\}$ records which of the two states is reached, and each is reached with probability $\frac12$. Each such label costs $\phi^\dagger$ about $\ln2$ nats of likelihood that $\psi$ does not pay, and part~(i) shows that $\psi$ gains at least $0.44$ nats per episode, whatever the agent does.

Part~(ii) makes precise how a poor abstraction steers the agent away from the data that would reveal its error, the failure mode of the introduction. An agent that acts through $\psi$ chooses its action at $u_A$ and $u_B$ from the same information. Every transition into $\{u_A,u_B\}$ lands on $u_A$ and $u_B$ with probability $\frac12$ each, whatever the previous state and action, and the two states are therefore equally likely given the agent's history. The labels that the agent generates thus follow an MDP on the labels of $\psi$ exactly, and they fit $\psi$ perfectly. No test that uses the labels of $\psi$ alone, and no exploration rule that acts through $\psi$, can reveal the error. Only the ground states reveal it, and the emission model predicts them. Since $\psi$ uses three labels, the regularisers that keep JEPA from collapsing to a constant map do not exclude it.

Let $\widehat I_\phi$ be the empirical mutual information between $\phi(s_{h+1})$ and $(\phi(s_h),a_h)$ over the observed transitions. A score that discounts the entropy of the labels in this way never strictly prefers a coarsening of $\phi^\dagger$. Indeed, if $\psi=f\circ\phi^\dagger$ for a map $f$ on labels, then $\psi(s_{h+1})=f(\phi^\dagger(s_{h+1}))$ and $(\psi(s_h),a_h)$ is a function of $(\phi^\dagger(s_h),a_h)$, and the data-processing inequality under the empirical joint law gives $\widehat I_\psi\le\widehat I_{\phi^\dagger}$ on every data set. The target of $\widehat I_\phi$, however, still depends on the candidate, and we know of no regret guarantee for selecting by it.

\subsection{Proof of Proposition~\ref{prop:selection}}
\label{app:selection}

\begin{proposition}[The selection cost of posterior averaging]
\label{prop:selection}
\begin{itemize}
\item[(a)] On every trajectory, $-\sum_{\tau\le T}\ln\mathcal Q(s_{\tau+1}\mid\mathcal H_\tau)\le-\sum_{\tau\le T}\ln P_{\phi^\dagger}(s_{\tau+1}\mid\mathcal H_\tau)+\ln|\Phi|$.
\item[(b)] The cumulative predictive divergence of an oracle that knows $\phi^\dagger$, $\mathfrak D_T:=\sum_{\tau}\mathds{E}\,D_{\mathrm{KL}}\big(P^*(\cdot\mid s_\tau,a_\tau)\,\|\,P_{\phi^\dagger}(\cdot\mid\mathcal H_\tau)\big)$,
is $\mathcal O\big((X^2A+S)H\ln T\big)$, and $\sum_{\tau\le T}\mathds{E}\,\|P^*(\cdot\mid s_\tau,a_\tau)-\mathcal Q(\cdot\mid\mathcal H_\tau)\|_1 \le \sqrt{2T\,\mathfrak D_T}+\sqrt{2T\ln|\Phi|}$.
\end{itemize}
\end{proposition}

\begin{proof}
Let $\mathcal H_\tau=(s_1,a_1,\dots,s_\tau,a_\tau)$ be the history before the prediction of $s_{\tau+1}$, and for each $\phi\in\Phi$ let $P_\phi(\cdot\mid\mathcal H_\tau)$ be the prediction of $s_{\tau+1}$ that Section~\ref{sec:aif-model}'s model assigns once $\phi$ is fixed, with its Dirichlet posterior integrated out under a proper prior. Actions are chosen by the agent and enter every candidate's likelihood only as conditioning variables. Hence, the posterior over abstractions after $\tau$ rounds is $\mathcal Q(\phi\mid\mathcal H_{\tau+1})\propto P(\phi)\prod_{\tau'\le\tau}P_\phi(s_{\tau'+1}\mid\mathcal H_{\tau'})$, and the averaged prediction $\mathcal Q(\cdot\mid\mathcal H_\tau)=\sum_\phi\mathcal Q(\phi\mid\mathcal H_\tau)P_\phi(\cdot\mid\mathcal H_\tau)$ telescopes,
\begin{equation}
\prod_{\tau=1}^T\mathcal Q(s_{\tau+1}\mid\mathcal H_\tau)\;=\;\sum_{\phi\in\Phi}P(\phi)\prod_{\tau=1}^TP_\phi(s_{\tau+1}\mid\mathcal H_\tau)\;\ge\;P(\phi^\dagger)\prod_{\tau=1}^TP_{\phi^\dagger}(s_{\tau+1}\mid\mathcal H_\tau).
\end{equation}

\emph{Part (a).} Taking $-\ln$ of the inequality gives, on every trajectory,
\begin{equation}
\textstyle\sum_{\tau=1}^{T}\ln\big(P_{\phi^\dagger}(s_{\tau+1}\mid\mathcal H_\tau)/\mathcal Q(s_{\tau+1}\mid\mathcal H_\tau)\big)\;\le\;\ln\big(1/P(\phi^\dagger)\big)\;=\;\ln|\Phi|
\label{eq:kl-cost}
\end{equation}
for the uniform prior. For the class of all surjective maps, $\ln|\Phi|\le S\ln X$. This is the mixture bound for the logarithmic loss \citep[\S9.2]{cesabianchi2006prediction}.

\emph{Part (b).} Write $P^*_\tau:=P^*(\cdot\mid s_\tau,a_\tau)$, the law of $s_{\tau+1}$ given $\mathcal H_\tau$, and $\mathcal Q_\tau$, $P^\dagger_\tau$ for the two predictions. Conditioning on $\mathcal H_\tau$,
\begin{equation}
\mathds{E}\Big[\ln\frac{P^*_\tau(s_{\tau+1})}{\mathcal Q_\tau(s_{\tau+1})}\,\Big|\,\mathcal H_\tau\Big]=D_{\mathrm{KL}}\big(P^*_\tau\|\mathcal Q_\tau\big),\qquad
\mathds{E}\Big[\ln\frac{P^*_\tau(s_{\tau+1})}{P^\dagger_\tau(s_{\tau+1})}\,\Big|\,\mathcal H_\tau\Big]=D_{\mathrm{KL}}\big(P^*_\tau\|P^\dagger_\tau\big).
\end{equation}
Subtracting, summing over $\tau$ and using part (a) inside the expectation,
\begin{equation}
\sum_{\tau=1}^T\mathds{E}\,D_{\mathrm{KL}}\big(P^*_\tau\|\mathcal Q_\tau\big)\;=\;\mathfrak D_T+\mathds{E}\sum_{\tau=1}^T\ln\frac{P^\dagger_\tau(s_{\tau+1})}{\mathcal Q_\tau(s_{\tau+1})}\;\le\;\mathfrak D_T+\ln|\Phi| .
\end{equation}
Pinsker's inequality gives $\|P^*_\tau-\mathcal Q_\tau\|_1\le\sqrt{2D_{\mathrm{KL}}(P^*_\tau\|\mathcal Q_\tau)}$ pointwise on every history. Hence
\begin{align}
\sum_{\tau=1}^{T}\mathds{E}\,\big\|P^*_\tau-\mathcal Q_\tau\big\|_1
\;&\le\; \mathds{E}\sum_{\tau=1}^{T}\sqrt{2D_{\mathrm{KL}}(P^*_\tau\|\mathcal Q_\tau)}
\;\le\; \sqrt{T\,\mathds{E}\sum_{\tau=1}^T2D_{\mathrm{KL}}(P^*_\tau\|\mathcal Q_\tau)} \nonumber\\
&\le\; \sqrt{2T\big(\mathfrak D_T+\ln|\Phi|\big)},
\label{eq:selection-chain}
\end{align}
where the second step is Cauchy--Schwarz, $\sum_{\tau\le T}u_\tau\le\sqrt{T\sum_{\tau\le T}u_\tau^2}$, applied pathwise and followed by Jensen's inequality for the concave map $z\mapsto\sqrt z$. Since $\sqrt{a+b}\le\sqrt a+\sqrt b$, this is part~(b). For any $\mathcal H_\tau$-measurable $V_\tau$ with range $H$, $|(P^*_\tau-\mathcal Q_\tau)V_\tau|\le(H/2)\|P^*_\tau-\mathcal Q_\tau\|_1$, and part~(b) therefore also bounds the cumulative error of the averaged prediction in every such value function, up to the factor $H/2$.

\emph{The oracle's divergence.} Under Assumption~\ref{assum:block}, the ground prediction of $\phi^\dagger$ factors into an abstract prediction and the emission. The sequential predictions telescope, which yields
\begin{equation*}
\mathfrak D_T\;=\;\mathds{E}\ln\frac{P^*(\mathcal D_K)}{\Lambda_{\phi^\dagger}(\mathcal D_K)}\;\le\;H(X^2A+S)\ln(K+1)
\end{equation*}
by Lemma~\ref{lem:redundancy}, which holds on every trajectory with $(X-1)\ln(K+1)$ for each of the $XAH$ transition rows and $\mathcal O(SH\ln K)$ for the emission rows in total. Hence $\mathfrak D_T=\mathcal O((X^2A+S)H\ln T)$.
\end{proof}

The oracle's divergence rests on the following bound, whose proof also enters that of Proposition~\ref{prop:self-confirming}.

\begin{lemma}[Redundancy of the true decoder]
\label{lem:redundancy}
Let $\mathcal D_K$ be the transitions of the first $K$ episodes, collected by an arbitrary policy, $\Lambda_\phi(\mathcal D_K)$ the marginal likelihood that candidate $\phi$ assigns to them under \texttt{model(L)} of Appendix~\ref{app:models}, given the start states and actions, and $P^*(\mathcal D_K)$ their probability under $P^*$. Under Assumption~\ref{assum:block}, on every trajectory,
\begin{equation*}
\ln\big(P^*(\mathcal D_K)/\Lambda_{\phi^\dagger}(\mathcal D_K)\big)\;\le\;H(X^2A+S)\ln(K+1).
\end{equation*}
\end{lemma}
\begin{proof}
Under $\phi^\dagger$ the likelihood factorises over $XAH$ transition rows, with $X$ categories each, and $XH$ emission rows, with $|\phi^{\dagger-1}(x)|$ categories each, and $P^*(\mathcal D_K)$ is the likelihood of one parameter value of this model and is thus at most the product of the rowwise maximum likelihoods. On a row with $m$ categories and counts $n_1,\dots,n_m$ summing to $n\le K$, the uniform-Dirichlet marginal likelihood is $(m-1)!\prod_in_i!/(n+m-1)!$, and the maximum likelihood is at most the reciprocal of the multinomial coefficient $n!/\prod_in_i!$, because the type class has probability at most one. The ratio of the two is therefore at most 
\begin{equation*}
\binom{n+m-1}{m-1}=\prod_{j=1}^{m-1}\frac{n+j}{j}\le(n+1)^{m-1}.
\end{equation*}
 Summing $(m-1)\ln(K+1)$ over the rows, with $\sum(m-1)=H[XA(X-1)+S-X]\le H(X^2A+S)$, gives the claim.
\end{proof}

\subsection{Proofs of Theorem~\ref{thm:enactive} and Propositions~\ref{prop:lure} and~\ref{prop:self-confirming}}
\label{app:enactive}

Throughout, $\mathcal T_h$ denotes the Bellman operator of step $h$, given by $[\mathcal T_hf](s,a)=R^*(s,a)+\mathds{E}_{s'\sim P^*(\cdot\mid s,a)}\big[\max_{a'}f(s',a')\big]$, $\mathcal T^R_h$ is the same operator with a reward $R$ in place of $R^*$, and $d^\pi_h(s,a)$ is the probability that policy $\pi$ visits $(s,a)$ at step $h$.

\begin{lemma}[Completeness of the class $\mathcal F$]
\label{lem:enactive-complete}
Under Assumption~\ref{assum:block}, every $f_{h+1}\in\mathcal F_{h+1}$ and every latent reward $R(s,a)=r(\phi^\dagger(s),a)$ with $r:\mathcal X\times\mathcal A\to[0,1]$ satisfy $\mathcal T^R_hf_{h+1}\in\mathcal F_h$. In particular, $\mathcal T_hf_{h+1}\in\mathcal F_h$ under latent rewards.
\end{lemma}
\begin{proof}
Let $f_{h+1}(s,a)=g(\phi(s),a)$ for some $\phi\in\Phi$ and some table $g$ with values in $[0,H-h]$, and let $\psi(s'):=\max_{a'}g(\phi(s'),a')$. By Assumption~\ref{assum:block},
\begin{equation*}
\mathds{E}_{s'\sim P^*(\cdot\mid s,a)}[\psi(s')]\;=\;\sum_{x'\in\mathcal X}B^*\big(x'\mid\phi^\dagger(s),a\big)\sum_{s'\in\phi^{\dagger-1}(x')}\mathbf A^*(s'\mid x')\,\psi(s')\;=:\;u\big(\phi^\dagger(s),a\big),
\end{equation*}
a function of $(\phi^\dagger(s),a)$ with values in $[0,H-h]$, even though $\psi$ depends on the labels of an arbitrary candidate $\phi$. Hence $\mathcal T^R_hf_{h+1}(s,a)=r(\phi^\dagger(s),a)+u(\phi^\dagger(s),a)$ is a table over the labels of $\phi^\dagger\in\Phi$ with values in $[0,H-h+1]$, hence an element of $\mathcal F_h$.
\end{proof}

\begin{lemma}[Latent coverability]
\label{lem:coverability}
Under Assumption~\ref{assum:block} with a fixed start state, the coverability coefficient of \citet[Def.~2]{xie2023role} (cf.\ their Prop.~10), $C_{\mathrm{cov}}:=\inf_{\mu}\sup_{\pi,h}\|d^\pi_h/\mu_h\|_\infty$, satisfies $C_{\mathrm{cov}}\le XA$ for every $S$.
\end{lemma}
\begin{proof}
For $h\ge2$, Assumption~\ref{assum:block} gives
\begin{equation*}
\mathds{P}^\pi(s_h=s)=\mathds{P}^\pi\big(\phi^\dagger(s_h)=\phi^\dagger(s)\big)\,\mathbf A^*(s\mid\phi^\dagger(s))\le \mathbf A^*(s\mid\phi^\dagger(s))
\end{equation*}
for every policy $\pi$, which yields $d^\pi_h(s,a)\le \mathbf A^*(s\mid\phi^\dagger(s))$. The distribution $\mu_h(s,a):=\mathbf A^*(s\mid\phi^\dagger(s))/(XA)$ sums to one, since the emissions of the $X$ latent states each sum to one, and it gives $d^\pi_h/\mu_h\le XA$. At $h=1$, $d^\pi_1(s,a)\le\mathds{1}\{s=s_1\}$ and $\mu_1(s,a):=\mathds{1}\{s=s_1\}/A$ gives the ratio $A\le XA$.
\end{proof}

\emph{The class of all maps.} GOLF needs no Bayesian model. Its class may therefore contain every map from $\mathcal S$ to $\mathcal X$, surjective or not, since a table $g$ need not use every latent state. For the class of all maps, $\ln|\Phi|=S\ln X$, and Theorem~\ref{thm:enactive} gives $\widetilde{\mathcal O}(H^2\sqrt{XAST\ln X})$, which matches the lower bound $\Omega(\sqrt{HST})$ of Proposition~\ref{prop:unstructured} in $S$ and $T$.

\emph{The agent and its radius.} The enactive agent runs on the problem normalised to values in $[0,1]$, that is, with rewards and value functions divided by $H$. For a reward $R$, let
\begin{equation*}
\mathcal L^{R}_h(g,f_{h+1}):=\sum_{j<k}\big(g(s^j_h,a^j_h)-R(s^j_h,a^j_h)-\max_{a'}f_{h+1}(s^j_{h+1},a')\big)^2
\end{equation*}
be the squared Bellman residual of $g$ on the transitions of the episodes before $k$. Before episode $k$, the agent keeps the version space of the functions $f\in\mathcal F$ with $\mathcal L^{R}_h(f_h,f_{h+1})-\min_{g\in\mathcal F_h}\mathcal L^{R}_h(g,f_{h+1})\le\omega$ for every $h$, with $R=R^*$ under a fixed reward and $R=R_k$ under goals, and it acts greedily for the most optimistic element $f^{(k)}$ of this set. It uses the radius
\begin{equation}
\omega\;=\;cH^2\ln(\mathcal N HK/\delta)
\label{eq:golf-radius}
\end{equation}%
for the raw squared residuals, that is, $c\ln(\mathcal N HK/\delta)$ in normalised units, with $\mathcal N$ the covering number of $\mathcal F$ at scale $H/K$ in the supremum norm (scale $1/K$ in normalised units) and $c$ an absolute constant.

\begin{proof}[Proof of Theorem~\ref{thm:enactive}(a)]
Divide rewards and value functions by $H$, which puts the problem in the normalisation of \citet{xie2023role} (total reward in $[0,1]$, a fixed start state, and $T$ there counting episodes) and divides the regret by $H$. Lemma~\ref{lem:enactive-complete} is their completeness assumption (Assumption~1), and Lemma~\ref{lem:coverability} bounds their coverability coefficient by $XA$. Their Theorem~1 states that GOLF with radius $\omega=c\ln(KH|\mathcal F|/\delta)$ has, with probability at least $1-\delta$ for a finite class, $\mathrm{Reg}=\mathcal O\Big(H\sqrt{C_{\mathrm{cov}}\,K\ln(KH|\mathcal F|/\delta)\,\ln K}\Big)$. The factor $\ln K$ stands inside the square root, as the last step of their proof shows. We replace it by $1+\ln K$, an upper bound that is also valid at $K=1$. For the infinite class $\mathcal F$, the concentration lemmas of GOLF hold with $|\mathcal F|$ replaced by the covering number $\mathcal N$ at scale $1/K$ in the normalised units \citep[Lemmas~39--40]{jin2021bellman}, which is the standard extension \citet{xie2023role} invoke. Their argument uses $\mathcal F$ only through these lemmas and through completeness. For each $\phi\in\Phi$, a grid of mesh $H/K$ on $[0,H]^{\mathcal X\times\mathcal A}$ covers the tables of $\mathcal F_h$ at scale $H/K$ in the supremum norm with at most $(K+2)^{XA}$ points. Hence, $\ln\mathcal N\le H\big(\ln|\Phi|+XA\ln(K+2)\big)$ for the product class. Covering each $\mathcal F_h$ separately would save a factor $\sqrt H$. Undoing the normalisation multiplies the regret by $H$. With $K=T/H$, we attain
\begin{align*}
\mathfrak R_K\;&=\;\mathcal O\Big(H^2\sqrt{XA\,K\,\big(H\ln|\Phi|+HXA\ln(K+2)+\ln(HK/\delta)\big)(1+\ln K)}\Big)\\
&=\;\mathcal O\Big(H^2\sqrt{XA\,T\,\big(\ln|\Phi|+XA\ln(HK/\delta)\big)(1+\ln K)}\Big).
\end{align*}
\end{proof}

\begin{proof}[Proof of Theorem~\ref{thm:enactive}(b)]
Normalise as in part~(a), and let $\mathcal G:=\{(s,a)\mapsto r(\phi^\dagger(s),a):r\in[0,1]^{\mathcal X\times\mathcal A}\}$ be the latent goals. Before episode $k$, the relabelling agent keeps the version space defined above with $R=R_k$ and acts greedily for its most optimistic element $f^{(k)}$. The proof of Theorem~1 of \citet{xie2023role} uses three facts, which we verify for the goal $R_k$ of every episode.

\emph{(1) Concentration.} For a fixed goal $R$ and fixed functions $g,g'\in\mathcal F_h$ and $f_{h+1}\in\mathcal F_{h+1}$, let $y:=R(s_h,a_h)+\max_{a'}f_{h+1}(s_{h+1},a')$ be the target of a transition. Given the history before the transition and $(s_h,a_h)$, the target has mean $[\mathcal T^R_hf_{h+1}](s_h,a_h)$. The difference of the two squared losses therefore decomposes into a predictable part and a centred martingale difference,
\begin{equation*}
(g-y)^2-(g'-y)^2\;=\;\big(g-\mathcal T^R_hf_{h+1}\big)^2-\big(g'-\mathcal T^R_hf_{h+1}\big)^2+\zeta,\qquad \mathds{E}[\zeta\mid\text{history},s_h,a_h]=0,
\end{equation*}
with all functions evaluated at $(s_h,a_h)$, since the conditional variance of the target cancels. The predictable part is a difference of squared Bellman residuals, and concentration applies to the martingale part. \citet[Lemmas~39--40]{jin2021bellman} turn it, by Freedman's inequality and a union bound over a cover of $\mathcal F$, into two events. On the first, $Q^*(\cdot\,;R)$ lies in the version space. On the second, every $f$ in the version space has $\sum_{j<k}\mathds{E}_{d^{\pi_j}_h}\big[(f_h-\mathcal T^R_hf_{h+1})^2\big]=\mathcal O(\omega)$ for every $h$. Both steps use completeness, $\mathcal T^R_h\mathcal F_{h+1}\subseteq\mathcal F_h$, which Lemma~\ref{lem:enactive-complete} gives for every $R\in\mathcal G$. We extend the union bound to a cover $\mathcal G_K$ of $\mathcal G$ at scale $1/K$ in the supremum norm, with $|\mathcal G_K|\le(K+1)^{XA}$. Replacing $R_k$ by its nearest point in $\mathcal G_K$ changes each squared term by at most $4/K$ in normalised units, hence each loss by at most $4$, which a constant in the radius absorbs. The events hold for every goal of the cover at once. Therefore, they also hold for $R_k$, although $R_k$ is chosen from the history.

\emph{(2) Optimism.} By~(1), $\max_af^{(k)}_1(s_1,a)\ge V^*_1(s_1;R_k)$. Since $\pi^k$ is greedy for $f^{(k)}$, the policy-loss decomposition \citep[Lemma~1]{jiang2017contextual} gives $V^*_1(s_1;R_k)-V^{\pi^k}_1(s_1;R_k)\le\sum_h\mathds{E}_{d^{\pi^k}_h}\big[\delta^{(k)}_h\big]$ with $\delta^{(k)}_h:=f^{(k)}_h-\mathcal T^{R_k}_hf^{(k)}_{h+1}$.

\emph{(3) Coverability.} The argument of \citet{xie2023role} that bounds $\sum_k\mathds{E}_{d^{\pi^k}_h}[\delta^{(k)}_h]$ by the in-sample bounds $\sum_{j<k}\mathds{E}_{d^{\pi_j}_h}[(\delta^{(k)}_h)^2]=\mathcal O(\omega)$ uses only $|\delta^{(k)}_h|\le1$ and the coverability coefficient. It therefore holds for every sequence of such functions, here one that changes with the goal. Lemma~\ref{lem:coverability} bounds the coefficient by $XA$ for every goal, since the coefficient depends only on the dynamics.

The bound of part~(a) therefore holds for $\mathfrak R^{\mathrm{goal}}_K$, with $\ln\mathcal N$ increased by at most $XA\ln(K+1)$, which changes $C_{\mathrm e}$ by a constant factor.
\end{proof}

\begin{proof}[Proof of Proposition~\ref{prop:lure}]
\emph{The instance.} Let $S':=S-2\ge10$, $\mathcal S=\{s_0,u_1,\dots,u_{S'},v\}$, $\mathcal A=\{r,p\}$ and $H=2$, with every episode starting at $s_0$. Action $r$ leads to a uniformly random $u_i$, and action $p$ leads to $v$. Under both actions, every $u_i$ leads to a uniformly random $u_j$, and $v$ is absorbing. Rewards are $R^*(s_0,\cdot)=0$, $R^*(u_i,\cdot)=\frac12-\Delta$ and $R^*(v,\cdot)=\frac12$. The decoder has the three blocks $\{s_0\}$, $\{u_1,\dots,u_{S'}\}$ and $\{v\}$, the latent kernel is deterministic, and the emission of the middle block is uniform. Hence, $P^*(u_j\mid u_i,a)=1\cdot(1/S')$ factorises as Assumption~\ref{assum:block} requires, and the rewards are latent. Action $p$ is optimal, and every episode that plays $r$ costs $\Delta$.

\emph{AIF's evaluation at $s_0$.} AIF knows the decoder, the latent kernel, and the rewards, and learns the emission of each block under a Dirichlet$(1,\dots,1)$ prior, one row for the state reached at each step. Every episode emits the states $s_2$ and $s_3$ reached at steps $1$ and $2$. The blocks $\{s_0\}$ and $\{v\}$ have a single observation and carry no uncertainty. Since the predicted latent state after either action is deterministic, the state term vanishes under exact perception (Appendix~\ref{app:noisy-tv}), and the latent kernel is known, which makes novelty about $B$ vanish too. The only epistemic terms left are the informations $\iota_1$ and $\iota_2$ that $s_2$ and $s_3$ carry about the emission rows of their blocks (Eq.~\eqref{eq:efe-three-term}). Both are zero after $p$. After $r$, both $s_2$ and $s_3$ fall in the middle block, whatever the action at step~$2$, and $\iota_h$ is the novelty of the Dirichlet posterior of the middle block's row for step $h$. In units of reward, AIF's planning value is therefore $\frac12$ for $p$ and at least $\min\{\frac12-\Delta+(\iota_1+\iota_2)/\beta,\,r_0\}$ for $r$, under the recursion of Eq.~\eqref{eq:psi-slot} with its clamps as well as summed along a plan, and AIF prefers $r$ whenever $\iota_1>\beta\Delta$.

\emph{Novelty of a large and sparsely sampled block.} Each episode that plays $r$ adds one observation to each of the two rows of the middle block, and no other episode changes them. After $n$ observations of such a row, its posterior is $\mathrm{Dir}(\alpha)$ with $\alpha_i=1+c_i$, where $c_i$ counts the observations of $u_i$ and $\bar\alpha=S'+n$. At least $S'-n$ of the $\alpha_i$ equal $1$. By the closed form of Eq.~\eqref{eq:novelty-exact},
\begin{equation*}
\iota\;=\;\sum_i\frac{\alpha_i}{\bar\alpha}\big[\psi(\alpha_i+1)-\ln\alpha_i\big]-\big[\psi(\bar\alpha+1)-\ln\bar\alpha\big].
\end{equation*}
The inequalities $\ln(z+\frac12)\le\psi(z+1)\le\ln z+1/(2z)$ for $z>0$ \citep[Lemma~16]{lu2019information} make every bracket in the sum positive, make the subtracted bracket at most $1/(2\bar\alpha)$, and a bracket with $\alpha_i=1$ equals $\psi(2)=1-\gamma_{\mathrm E}\ge0.42$, with $\gamma_{\mathrm E}\approx0.5772$ the Euler--Mascheroni constant, since $\psi(1)=-\gamma_{\mathrm E}$ and $\psi(z+1)=\psi(z)+1/z$. Hence
\begin{equation*}
\iota\;\ge\;0.42\,\frac{S'-n}{S'+n}-\frac1{2(S'+n)}\;\ge\;\frac{0.42}3-\frac1{20}\;\ge\;0.09\qquad\text{for }n\le S'/2,
\end{equation*}
using $S'\ge10$.

\emph{Regret.} Let $J:=\lfloor S'/2\rfloor$. In each of the first $\min\{K,J\}$ episodes fewer than $J$ observations of either row have been made, whatever the agent did before. Hence, $\iota_1\ge0.09>\beta\Delta$ and AIF's planning value of $r$ exceeds that of $p$. Acting greedily, AIF plays $r$ in each of these episodes. Sampling from its softmax, it plays $r$ with probability above $\frac12$. In both cases we attain $\mathds{E}\mathfrak R_K\ge(\Delta/2)\min\{K,J\}$.

\emph{The other agents.} Since the MDP is a realizable block MDP with latent rewards, $X=3$ and $A=2$, Theorem~\ref{thm:enactive} with $\Phi$ the singleton of the known decoder gives the enactive agent regret $\mathcal O(H^2\sqrt{T\ln(HK/\delta)(1+\ln K)})=\widetilde{\mathcal O}(\sqrt T)$, independently of $S$. CaI and calibrated AIF act on the history only through the latent states $x_\tau=\phi(s_\tau)$, the actions and the rewards. The joint law of these variables is the same for every $S$, and so are the law of the sequence of actions and the regret.

\emph{Calibrated sensory curiosity.} The argument uses $\iota_1$ only through $\iota_1\ge0.09$ for $n\le S'/2$ and $\iota_1=\iota_2=0$ after $p$. It therefore applies unchanged to a bonus that adds $g(\iota_h)$ for each row, for any nondecreasing $g$ with $g(0)=0$ and $g(0.09)>\Delta$, which prefers $r$ whenever $g(\iota_1)>\Delta$. For $g(\iota)=c\sqrt\iota$ this holds when $c>\Delta/0.3$.
\end{proof}

\begin{proof}[Proof of Proposition~\ref{prop:self-confirming}]
\emph{The instance.} Take the MDP of Proposition~\ref{prop:lb-abstraction} with $\delta_0:=1/(8(H-2))\le\frac1{16}$, that is, $\eta=4\delta_0$ and $(H-2)\eta=\frac12$. Its states are $s_0,u_A,u_B,g$, its actions $a,b$, and its rewards are $1$ at $g$ and $0$ elsewhere. Let $\mathcal X=\{1,2,3,4\}$, let $\phi^\dagger$ map $s_0,u_A,u_B,g$ to $1,2,3,4$, and let $\psi$ map $s_0$ to $1$, both $u_A$ and $u_B$ to $2$, and $g$ to $4$. The MDP is a realizable block MDP for $\phi^\dagger$, with point-mass emissions and latent rewards. Both candidates make $R^*$ measurable. Hence, a factor that predicts rewards is the same for both. The candidate $\psi$ is the abstraction of Proposition~\ref{prop:lb-abstraction} with $\epsilon_{\text{rep}}(\psi)=\frac12H^2\eta=H^2/(4(H-2))\ge H/4$.

\emph{Part (i).} Let $D_k$ be the transitions of episodes $1,\dots,k-1$, and let $n_r$ count those that fall in a row $r=(x,c,h)$ of a candidate. A row with $m=4$ categories and $n$ observations has uniform-Dirichlet marginal likelihood at least $(n+1)^{-3}$ times its maximum likelihood (proof of Lemma~\ref{lem:redundancy}). The maximum likelihood is at least the likelihood under any fixed law. For $\psi$, fix the laws $\theta$ in which rows at label~$1$ predict label~$2$, rows at label~$2$ predict label~$4$ with probability $2\delta_0$ and label~$2$ otherwise, and rows at label~$4$ predict label~$4$. Since at most $6H$ rows of $\psi$ receive data, we attain
\begin{equation*}
\ln\Lambda^{\mathrm{lat}}_\psi(D_k)\;\ge\;\sum_{\tau}\ln\theta\big(\psi(s_{\tau+1})\mid\psi(s_\tau),a_\tau\big)-18H\ln(K+1),
\end{equation*}
the sum running over the transitions of $D_k$. Since $\phi^\dagger$ is a bijection from $\mathcal S$ onto $\mathcal X$, its rows are the rows $(s,a,h)$ of the ground MDP, with successor laws $P^*(\cdot\mid s,a)$. As the marginal likelihood of $\phi^\dagger$ is at most its maximum likelihood, we get
\begin{equation*}
\ln\Lambda^{\mathrm{lat}}_{\phi^\dagger}(D_k)\;\le\;\sum_\tau\ln P^*(s_{\tau+1}\mid s_\tau,a_\tau)+\sum_rn_rD_{\mathrm{KL}}\big(\widehat p_r\,\|\,P^*_r\big),
\end{equation*}
with $\widehat p_r$ the empirical successor law of row $r$ and $P^*_r$ its true law. Model the data as a table of independent successors drawn in advance for each row, as in Appendix~\ref{app:lem-novelty-order}. The first $n$ successors of a row are then i.i.d.\ draws from $P^*_r$, whatever the policy. By the method of types \citep[Thm.~11.2.1]{cover2006elements}, $\mathds{P}\big(nD_{\mathrm{KL}}(\widehat p_n\|P^*_r)\ge t\big)\le(n+1)^4e^{-t}$ for each row and each count $n$. A union bound over the $8H$ rows of $\phi^\dagger$ and the counts $n\le K$, at $t=4\ln(n+1)+\ln(8HK/\delta)$, shows that with probability at least $1-\delta$, for every $k$,
\begin{equation*}
\sum_rn_rD_{\mathrm{KL}}(\widehat p_r\|P^*_r)\le32H\ln(K+1)+8H\ln(8HK/\delta).
\end{equation*}

It remains to compare $\theta$ with $P^*$ transition by transition. From $s_0$, the successor is $u_A$ or $u_B$ with probability $\frac12$ each, and $\theta$ gives its label probability $1$. Hence, the term $\ln\theta-\ln P^*$ is $\ln2$. From $u_A$ under $a$, or $u_B$ under $b$, the successor is $g$ with probability $2\delta_0$, where the term is $0$, or a pool state with probability $\frac12(1-2\delta_0)$, where the term is $\ln2$. From $u_A$ under $b$, or $u_B$ under $a$, the successor is a pool state with probability $\frac12$, and the term is $\ln(2(1-2\delta_0))\ge\ln2-4\delta_0\ge\ln2-\frac14$. From $g$ the term is $0$. Every term is thus nonnegative, and the first transition of each episode contributes at least $\ln2-\frac14\ge0.44$. Collecting, with $\ln(K+1)\le\ln(8HK/\delta)$,
\begin{equation*}
\ln\Lambda^{\mathrm{lat}}_\psi(D_k)-\ln\Lambda^{\mathrm{lat}}_{\phi^\dagger}(D_k)\;\ge\;0.44(k-1)-58H\ln(8HK/\delta).
\end{equation*}

\emph{Part (ii).} Every transition into the pool $\{u_A,u_B\}$ lands on $u_A$ and $u_B$ with equal probability, whatever the previous state and action. Under a policy that acts through $\psi$, the action at a pool state depends on the history but not on whether the state is $u_A$ or $u_B$. Given the history and the action, the current pool state is therefore uniform. The next state is then $g$ with probability $\frac12\cdot2\delta_0=\delta_0$, and a uniform pool state otherwise. The labels and actions thus follow the MDP on $\{1,2,4\}$ in which label~$1$ leads to label~$2$, label~$2$ leads to label~$4$ with probability $\delta_0$ under both actions, and label~$4$ is absorbing.

\emph{Part (iii).} On the event of part~(i), the bound is positive for $k>k_0$, since $0.44\cdot132>58$. Then $\psi$ is the unique maximiser of $\Lambda^{\mathrm{lat}}_\phi$, and the agent acts through $\psi$. The proof of Proposition~\ref{prop:lb-abstraction} uses only that the action at a pool state does not depend on which pool state the agent is in. Each such episode therefore has regret at least $\epsilon_{\text{rep}}(\psi)/32\ge H/128$, and $\mathfrak R_K\ge H(K-k_0)/128$. Since GOLF on $\Phi$ has $\ln|\Phi|=\ln2$, $X=4$ and $A=2$, Theorem~\ref{thm:enactive}(a) gives $\mathfrak R_K=\mathcal O\big(H^2\sqrt{T\ln(HK/\delta)(1+\ln K)}\big)$.
\end{proof}

\end{document}